\documentclass[a5paper,twoside,12pt]{book}

\usepackage{floatrow}
\usepackage{diss}
\usepackage[german,american]{babel}
\usepackage{amsmath,amsfonts}
\usepackage{amsthm}
\usepackage{bm}
\usepackage{amssymb}
\usepackage{xcolor}
\usepackage{color,colortbl}
\usepackage{graphicx}
\usepackage[utf8]{inputenc}
\usepackage[T1]{fontenc}
\usepackage{subcaption}

\usepackage[pdftitle={Thesis Zhongnan Qu}, pdfauthor={Zhongnan Qu}, colorlinks={true}, linkcolor={purple}, citecolor={blue}, urlcolor={orange}]{hyperref}

\usepackage{threeparttable}
\usepackage{multirow}
\usepackage{xspace}
\usepackage{booktabs}
\usepackage{appendix}
\usepackage[capitalise, noabbrev]{cleveref}
\usepackage{arydshln}
\usepackage{algorithmicx}
\usepackage[ruled,vlined,linesnumbered,resetcount,algochapter]{algorithm2e}
\usepackage{lipsum}
\usepackage{CJKutf8}

\usepackage{pxfonts}

\usepackage[ampersand]{easylist}
\usepackage{comment}

\allowdisplaybreaks

\begin{document}
\pagestyle{empty}
\pagenumbering{Alph}

\selectlanguage{american}
\newcommand{\currdir}{}

\newcommand{\fakeparagraph}[1]{\vspace{2mm}\noindent\textbf{\boldmath#1.}}
\newcommand{\faketitle}[1]{\vspace{2mm}\noindent\textbf{\boldmath#1}}
\newcommand{\phase}[1]{\mbox{\ensuremath{\langle#1\rangle}}\xspace}

\newcommand{\etal}{et~al.\@\xspace}
\newcommand{\eg}{e.g.,\xspace}
\newcommand{\ie}{i.e.,\xspace}
\newcommand{\etc}{etc.\@\xspace}
\newcommand{\wrt}{w.r.t.\@\xspace}
\newcommand{\cf}{cf.\@\xspace}

\newtheorem{theorem}{Theorem}
\newtheorem{lemma}{Lemma}
\newtheorem{definition}{Definition}
\newtheorem{example}{Example}

\newcommand\figref[1]{\cref{#1}}
\newcommand\tabref[1]{\cref{#1}}
\newcommand\secref[1]{\cref{#1}}
\newcommand\chref[1]{\cref{#1}}
\newcommand\lstref[1]{\cref{#1}}
\newcommand\equref[1]{Eq.~(\ref{#1})}
\newcommand\algoref[1]{Algorithm~\ref{#1}}
\newcommand{\thmref}[1]{Theorem~\ref{#1}}
\newcommand{\defref}[1]{Definition~\ref{#1}}

\newcommand{\alq}{ALQ\xspace}
\newcommand{\dress}{DRESS\xspace}
\newcommand{\pMeta}{p-Meta\xspace}
\newcommand{\dpu}{DPU\xspace}

\newcommand{\TODO}[1]{\noindent\textit{ \color{red}\textbf{TODO:}~#1} }
\newcommand{\DISCUSS}[1]{\noindent\fbox{\parbox{\linewidth}{ \textit{ \color{blue}\textbf{Discuss:}~#1} }} }

\numberwithin{figure}{chapter}
\numberwithin{table}{chapter}
\numberwithin{theorem}{chapter}
\numberwithin{lemma}{chapter}
\numberwithin{example}{chapter}
\numberwithin{definition}{chapter}

\newcommand{\dissnumstring}{28528}
\newcommand{\titlestring}{Enabling Deep Learning on Edge Devices}
\newcommand{\titlestringNOBR}{Enabling Deep Learning on Edge Devices}

\newcommand{\degreestring}{Doctor of Sciences of ETH Zurich}
\newcommand{\degreestringabr}{(Dr. sc. ETH Zurich)}
\newcommand{\authorstring}{ZHONGNAN QU}
\newcommand{\acatitlestring}{M.Sc. TU Munich}
\newcommand{\dateofbirthstring}{05.05.1992}
\newcommand{\citizenstringA}{China}
\newcommand{\citizenstringB}{Henan}
\newcommand{\examinerstring}{Prof. Dr.  Lothar Thiele}
\newcommand{\coexaminerstring}{Prof. Dr. Olga Saukh}
\newcommand{\datestring}{2022}

\begin{titlepage}
{

\begin{centerline}
{\large\noindent Diss.\ ETH No.\ \dissnumstring}
\end{centerline}
\vfill

\begin{center}
\LARGE\bfseries
\titlestringNOBR
\end{center}
\vfill

\begin{center}
\large A thesis submitted to attain the degree of
\end{center}
\vspace{0.55\fill}

\begin{center}
\large \degreestring \linebreak \degreestringabr
\end{center}
\vspace{0.55\fill}

\begin{center}
\large presented by \linebreak \authorstring \linebreak \acatitlestring
\linebreak \linebreak born on \dateofbirthstring \linebreak citizen of
\linebreak \citizenstringA
\linebreak \citizenstringB
\end{center}
\vspace{3\fill}

\begin{center}
\large accepted on the recommendation of \linebreak
\examinerstring, examiner \linebreak
\coexaminerstring, co-examiner
\end{center}
\vspace{0.55\fill}

\begin{center}
\large{\datestring}
\end{center}
}
\end{titlepage}

\cleardoublepage

\thispagestyle{empty}
\noindent\includegraphics[scale=.6]{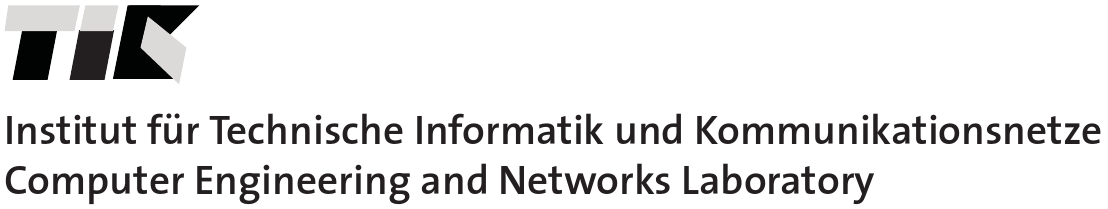}
\medskip
\hrule
\begin{flushright}
\vspace{0.5cm}
TIK-SCHRIFTENREIHE NR. 202 \\ \vspace{1cm} 
\large Zhongnan Qu \\ \vspace{1cm}
\end{flushright}
\begin{flushright}
\Large\bfseries
\titlestringNOBR
\end{flushright}
\vspace{\fill}
\hrule
\bigskip
\includegraphics[scale=0.7]{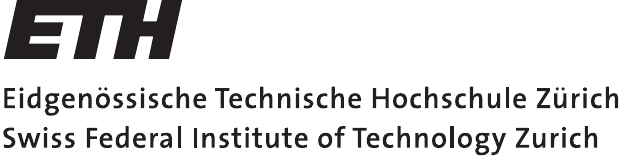}

\clearpage
\thispagestyle{empty}

\noindent
A dissertation submitted to\\
ETH Zurich\\
for the degree of Doctor of Sciences\\

\noindent DISS.\ ETH NO.\ \dissnumstring\\

\noindent \examinerstring, examiner\\
\coexaminerstring, co-examiner

\noindent Examination date: July 26, 2022\\
\vfill
\noindent

\cleardoublepage
\chapter*{}
\vfill
{\footnotesize
\begin{flushright}
\emph{}\\
\emph{To my family.\\}
\begin{CJK*}{UTF8}{gbsn}
\emph{致我的家人。}
\end{CJK*}
\end{flushright}
}
\vfill

\cleardoublepage
\frontmatter
\pagestyle{headings}
\chapter[Abstract]{Abstract}

Deep neural networks (DNNs) have succeeded in many different perception tasks, \eg computer vision, natural language processing, reinforcement learning, etc. 
The high-performed DNNs heavily rely on intensive resource consumption.
For example, training a DNN requires high dynamic memory, a large-scale dataset, and a large number of computations (a long training time); even inference with a DNN also demands a large amount of static storage, computations (a long inference time), and energy.  
Therefore, state-of-the-art DNNs are often deployed on a cloud server with a large number of super-computers, a high-bandwidth communication bus, a shared storage infrastructure, and a high power supplement. 

Recently, some new emerging intelligent applications, \eg AR/VR, mobile assistants, Internet of Things, require us to deploy DNNs on resource-constrained edge devices. 
Compare to a cloud server, edge devices often have a rather small amount of resources. 
To deploy DNNs on edge devices, we need to reduce the size of DNNs, \ie we target a better trade-off between the resource consumption and the model accuracy.

In this thesis, we study four edge intelligent scenarios and develop different methodologies to enable deep learning in each scenario. 
Since current DNNs are often over-parameterized, our goal is to find and reduce the redundancy of the DNNs in each scenario. 
We summarize the four studied scenarios as follows,

\begin{itemize}
    \item \fakeparagraph{Inference on Edge Devices} 
    Firstly, we enable efficient inference of DNNs given the fixed resource constraints on edge devices.
    Compared to cloud inference, inference on edge devices avoids transmitting the data to the cloud server, which can achieve a more stable, fast, and energy-efficient inference.
    Regarding the main resource constraints from storing a large number of weights and computation during inference, we proposed an Adaptive Loss-aware  Quantization (\alq) for multi-bit networks. 
    \alq reduces the redundancy on the quantization bitwidth.
    The direct optimization objective (\ie the loss) and the learned adaptive bitwidth assignment allow \alq to acquire extremely low-bit networks with an average bitwidth below 1-bit while yielding a higher accuracy than state-of-the-art binary networks.    
    
    \item \fakeparagraph{Adaptation on Edge Devices}
    Secondly, we enable efficient adaptation of DNNs when the resource constraints on the target edge devices dynamically change during runtime, \eg the allowed execution time and the allocatable RAM.
    To maximize the model accuracy during on-device inference, we develop a new synthesis approach, Dynamic REal-time Sparse Subnets (\dress) that can sample and execute sub-networks with different resource demands from a backbone network.  
    \dress reduces the redundancy among multiple sub-networks by weight sharing and architecture sharing, resulting in storage efficiency and re-configuration efficiency, respectively.
    The generated sub-networks have different sparsity, and thus can be fetched to infer under varying resource constraints by utilizing sparse tensor computations. 

    \item \fakeparagraph{Learning on Edge Devices}
    Thirdly, we enable efficient learning of DNNs when facing unseen environments or users on edge devices.
    On-device learning requires both data- and memory-efficiency. 
    We thus propose a new meta learning method \pMeta to enable memory-efficient learning with only a few samples of unseen tasks.
    \pMeta reduces the updating redundancy by identifying and updating structurewise adaptation-critical weights only, which saves the necessary memory consumption for the updated weights.
    
    \item \fakeparagraph{Edge-Server System}
    Finally, we enable efficient inference and efficient updating on edge-server systems. 
    In an edge-server system, several resource-constrained edge devices are connected to a resource-sufficient server with a constrained communication bus. 
    Due to the limited relevant training data beforehand, pretrained DNNs may be significantly improved after the initial deployment. 
    On such an edge-server system, on-device inference is preferred over cloud inference, since it can achieve a fast and stable inference with less energy consumption. 
    Yet retraining on the cloud server is preferred over on-device retraining (or federated learning) due to the limited memory and computing power on edge devices. 
    We proposed a novel pipeline Deep Partial Updating (\dpu) to iteratively update the deployed inference model.
    Particularly, when newly collected data samples from edge devices or from other sources are available at the server, the server smartly selects only a subset of critical weights to update and send to each edge device.
    This weightwise partial updating reduces the redundant updating by reusing the pretrained weights, which achieves a similar accuracy as full updating yet with a significantly lower communication cost.

\end{itemize}

\cleardoublepage
\begin{otherlanguage*}{german}
\chapter[Zusammenfassung]{Zusammenfassung}

Deep Neural Networks (DNNs) haben sich bei vielen verschiedenen Wahrnehmungsaufgaben bewährt, z. B. Computer Vision, Verarbeitung natürlicher Sprache, Verstärkungslernen usw.
Die leistungsstarken DNNs sind stark auf einen intensiven Ressourcenverbrauch angewiesen.
Beispielsweise erfordert das Training eines DNN einen hohen dynamischen Speicher, einen großen Datensatz und eine große Anzahl von Berechnungen (eine lange Trainingszeit); Selbst die Inferenz mit einem DNN erfordert auch eine große Menge an statischem Speicher, Berechnungen (eine lange Inferenzzeit) und Energie.
Daher werden moderne DNNs häufig auf einem Cloud-Server mit einer großen Anzahl von Supercomputern, einem Kommunikationsbus mit hoher Bandbreite, einer gemeinsam genutzten Speicherinfrastruktur und einem Hochleistungszusatz eingesetzt.

In letzter Zeit erfordern einige neu entstehende intelligente Anwendungen, z. B. AR/VR, mobile Assistenten, Internet of Things, den Einsatz von DNNs auf ressourcenbeschränkten Edge-Geräten.
Im Vergleich zu einem Cloud-Server verfügen Edge-Geräte oft über eine eher geringe Menge an Ressourcen.
Um DNNs auf Edge-Geräten einzusetzen, müssen wir die Größe von DNNs reduzieren, d. h. wir streben einen besseren Kompromiss zwischen dem Ressourcenverbrauch und der Modellgenauigkeit an.

In dieser Doktorarbeit untersuchen wir vier intelligente Edge-Szenarien und entwickeln verschiedene Methoden, um Deep Learning in jedem Szenario zu ermöglichen.
Da aktuelle DNNs oft überparametrisiert sind, ist unser Ziel, die Redundanz der DNNs in jedem Szenario zu finden und zu reduzieren.
Wir fassen die vier untersuchten Szenarien wie folgt zusammen,

\begin{itemize}
    \item \fakeparagraph{Inferenz auf Edge-Geräten} 
    Erstens ermöglichen wir eine effiziente Inferenz von DNNs angesichts der festen Ressourcenbeschränkungen auf Edge-Geräten.
    Im Vergleich zur Cloud-Inferenz wird bei der Inferenz auf Edge-Geräten die Übertragung der Daten an den Cloud-Server vermieden, wodurch eine stabilere, schnellere und energieeffizientere Inferenz erreicht werden kann.
    In Bezug auf die wichtigsten Ressourcenbeschränkungen, die sich aus der Speicherung einer großen Anzahl von Gewichten und Berechnungen während der Inferenz ergeben, haben wir eine Adaptive Loss-aware Quantization (\alq) für Multibit-Netzwerke vorgeschlagen. 
    \alq reduziert die Redundanz in der Quantisierungsbitbreite.
    Das direkte Optimierungsziel (d. h. der Verlust) und die erlernte adaptive Bitbreitenzuweisung ermöglichen es \alq, Netze mit extrem niedrigen Bits mit einer durchschnittlichen Bitbreite unter 1-Bit zu erfassen und gleichzeitig eine höhere Genauigkeit als moderne binäre Netze zu erzielen. 

    \item \fakeparagraph{Anpassung auf Edge-Geräten}
    Zweitens ermöglichen wir eine effiziente Anpassung von DNNs, wenn sich die Ressourcenbeschränkungen auf den Zielgeräten während der Laufzeit dynamisch ändern, z. B. die erlaubte Ausführungszeit und der zuweisbare RAM.
    Um die Modellgenauigkeit während der Inferenz auf dem Gerät zu maximieren, entwickeln wir einen neuen Syntheseansatz, Dynamic REal-time Sparse Subnets (\dress), der Subnetze mit unterschiedlichen Ressourcenanforderungen von einem Backbone-Netz abtasten und ausführen kann.
    \dress reduziert die Redundanz in mehreren Subnetzen durch gemeinsame Nutzung von Gewicht und Architektur, was zu Speichereffizienz bzw. Rekonfigurationseffizienz führt.
    Die erzeugten Subnetze weisen unterschiedliche Sparsamkeit auf und können daher abgerufen werden, um unter variierenden Ressourcenbeschränkungen durch Verwendung von spärliche Tensorberechnungen zu folgern.

    \item \fakeparagraph{Lernen auf Edge-Geräten}
    Drittens ermöglichen wir ein effizientes Lernen von DNNs, wenn Sie mit unsichtbaren Umgebungen oder Benutzern auf Edge-Geräten konfrontiert sind.
    Lernen auf dem Edge-Gerät erfordert sowohl Dateneffizienz als auch Speichereffizienz.
    Wir schlagen daher eine neue Meta-Lernmethode \pMeta vor, die speichereffizientes Lernen mit nur wenigen Datenbeispielen von unbekannten Aufgaben ermöglicht.
    \pMeta reduziert die Aktualisierungsredundanz, indem es nur strukturweise anpassungskritischen Gewichte identifiziert und aktualisiert, wodurch der notwendige Speicherverbrauch für die aktualisierten Gewichte eingespart wird.

    \item \fakeparagraph{Edge-Server-System}
    Schließlich ermöglichen wir effiziente Inferenz und effiziente Aktualisierung auf Edge-Server-Systemen.
    In einem Edge-Server-System sind mehrere ressourcenbeschränkte Edge-Geräte mit einem ressourcenstarken Server mit einem eingeschränkten Kommunikationsbus verbunden.
    Aufgrund der begrenzten Anzahl relevanter Trainingsdaten im Voraus können vortrainierte DNNs nach dem anfänglichen Einsatz erheblich verbessert werden.
    In einem solchen Edge-Server-System wird die Inferenz auf dem Gerät der Inferenz in der Cloud vorgezogen, da sie eine schnelle und stabile Inferenz mit weniger Energieverbrauch erreichen kann.
    Aufgrund des begrenzten Speichers und der begrenzten Rechenleistung auf Edge-Geräten wird jedoch die Re-Training in der Cloud gegenüber der Re-Training auf dem Gerät (oder föderiertem Lernen) bevorzugt.
    Wir haben eine neuartige Pipeline, Deep Partial Updating (\dpu) vorgeschlagen, um das eingesetzte Inferenzmodell iterativ zu aktualisieren.
    Insbesondere, wenn neu gesammelte Datenbeispielen von Edge-Geräten oder aus anderen Quellen auf dem Cloud-Server verfügbar sind, wählt der Server intelligenterweise nur eine Teilmenge kritischer Gewichte aus, um sie zu aktualisieren und an jedes Edge-Gerät zu senden.
    Diese gewichtsmäßige Teilaktualisierung reduziert die redundante Aktualisierung durch Wiederverwendung der vortrainierten Gewichtungen, wodurch eine ähnliche Genauigkeit wie bei der vollständigen Aktualisierung erreicht wird, jedoch mit deutlich geringeren Kommunikationskosten.

\end{itemize}
\end{otherlanguage*}

\cleardoublepage
\chapter[Acknowledgements]{\vspace{-2cm}Acknowledgements}

I was born in Nanyang China. Nanyang is a typical third-tier city in China, with a large number of citizens yet a rather low growth. It is not easy for children who are similar to me to finally receive a doctoral degree from a top-tier university. The entire study career was full of occasionality and uncertainty. There were many critical steps where you only had a small chance and the only thing you can do was to try your best and submitted to the will of god. But for the ones who are struggling like me in the past and happen to read my thesis (only a few ;-)), I would encourage you with a sentence: \textit{I strive to run, just to catch up with those who have been high hopes of their own} (originally from the show Total Soccer).

I am very lucky that at every crossroad, I could always wait till the best option that I believe was often far beyond my deserving. Therefore, I try my best to remember and appreciate all the people who recognized me, guided me, criticized me, supported me, and accompanied me in my 24-year study. I also appreciate the younger me who had the courage to choose the path that is made of difficulties but also leads me to explore the unknown beauties.

My interest in engineering was first inspired by my physics teacher in middle school. In high school, I started to systematically learn physics, and this unforgettable period with my friends built the foundation of my mathematical logic. In my first year of bachelor study at Tongji University, I had a serious ankle fracture in a football game. My peers provided me with countless help in both study and life. During my master study at TU Munich, I received kind supervision when I wrote my master theses in \textit{Computer Vision Group} and \textit{Robotics \& Artiﬁcial Intelligence Group}. These pleasant research experiences finally motivated and also helped me to proceed with my academic career at ETH Zurich as a doctoral student. 

The PhD study at ETH is my most memorable and enriched period. The four-year study taught me to view a new problem from an unprecedented width and depth, which I believe is even more beneficial to my future life than the knowledge gained from research. I sincerely thank Prof. Lothar Thiele for offering me this great opportunity, and for your supervision and guidance. Thanks for revising my work from the early morning until the late night. I can not imagine a better advisor than you. I wish you all the best in your retirement life. I appreciate my ETH colleagues in \textit{Computer Engineering Group}, \eg Prof. Rehan Ahmed, Prof. Jan Beutel, Andreas Biri, Dr. Yun Cheng, Reto Da Forno, Dr. Stefan Drašković, Tonio Gsell, Dr. Xiaoxi He, Dr. Romain Jacob, Prof. Cong Liu, Dr. Balz Maag, Dr. Matthias Meyer, Dr. Philipp Miedl, Dr. Lukas Sigrist, Naomi Stricker, Dr. Roman Trüb, etc. for the interesting discussion and happy daily working hours. Many thanks to Susann Arreghini and Beat Futterknecht for getting me familiar with Swiss life. I also appreciate all other collaborators for your patient discussion and constructive suggestions, \eg Hu Cao, Prof. Guang Chen, Xin Dong, Junfeng Guo, Lennart Heim, Dr. Shu Liu, Zhao Meng, Prof. Yongxin Tong, Prof. Ye Wang, etc.  Particular gratitude goes to Prof. Zimu Zhou. You played a major role in leading me into academic research when I was at the beginning of my PhD study. I hope our enormous audio calls on weekends were not too annoying for you. I also appreciate the team members, particularly my supervisor Dr. Syed Shakib Sarwar and Dr. Barbara De Salvo  as well as my peers Dominika Przewłocka-Rus and Dr. Peter Liu at \textit{Facebook Reality Labs} for providing me with a remote yet productive internship during the COVID pandemic.
In addition, I want to thank Prof. Olga Saukh for being the examiner of my doctoral defense. 

I also would like to thank my family and my friends in my personal life. I appreciate my mom and my dad for their guidance and support since my childhood, which shaped my body and my soul. 
Last but not least, I can not complete this journey alone without my wife, who tolerated, encouraged, and accompanied me through countless days and nights, and countless ups and downs.

\cleardoublepage
\phantomsection\pdfbookmark[0]{Contents}{Contents}
{\hypersetup{linkcolor=black} \tableofcontents}
\cleardoublepage
\phantomsection\addcontentsline{toc}{chapter}{List of Figures}
{\hypersetup{linkcolor=black} \listoffigures}
\cleardoublepage
\phantomsection\addcontentsline{toc}{chapter}{List of Tables}
{\hypersetup{linkcolor=black} \listoftables}
\cleardoublepage

\mainmatter
\pagenumbering{arabic}
\setcounter{page}{1}
\pagestyle{headings}

\chapter{Introduction}
\label{ch1:introduction}

Deep learning is a new disruptive technology that extremely drives the development of artificial intelligence. 
Deep neural networks (DNNs) are widely used in deep learning, which can make predictions according to the given inputs.
A DNN consists of a large number of cascaded layers, where each layer often comprises (\textit{i}) trainable weights that can perform matrix multiplication on the layer's input to output extracted features, (\textit{ii}) a non-linear function that can bring non-linear behaviors.
DNNs can often achieve superior performance than prior computational models or even human beings in many areas, \eg computer vision, natural language processing, mathematics, biochemistry, etc.

In image classification, AlexNet \cite{bib:NIPS12:Krizhevsky} automatically learns the features by training a deep convolutional neural network with GPUs, and the competition results on ImageNet Large Scale Visual Recognition Challenge (ILSVRC) show that AlexNet surpasses the prior classifiers that are built based on hand-crafted features \eg random forest and support vector machine, by a large margin (over 10\% accuracy gain).
AlphaGo Zero \cite{bib:Nature16:Silver} reinforce-learns a deep policy model to predict the movement on the Go board via playing games against itself, and the learned model can even defeat a human world champion of Go games. 
BERT \cite{bib:ACL19:Devlin} pretrains deep bidirectional representations from unlabeled text and then fine-tunes the pretrained model, which exhibits a better performance in language understanding on SQuAD test than humans. 
Recently, graph convolutional neural networks \cite{bib:Nature19:Eraslan} have also been applied to many biological and chemical problems \eg predicting protein function, predicting binarized gene expression, etc. 
As a result, DNNs not only can conduct some intelligent tasks that previously must rely on cumbersome human efforts in our daily life, but also may bring new scientific inspirations that are less explored in the long-term human history.   

\section{High Resource Demands of DNNs}
\label{ch1-sec:recource}

The high performance of state-of-the-art DNNs benefits from the intensive resource consumption during both the \textit{training} phase and the \textit{inference} phase. 

During the training phase, current DNNs are often optimized on high-performance cloud servers with a large-scale dataset over a long time, which may take (\textit{i}) many human labor resources to prepare the dataset or the training implementation, (\textit{ii}) a large amount of time and money cost, (\textit{iii}) a remarkable CO2 emission \cite{bib:ICLR20:Cai}.
For example, the widely-used DNN ResNet50 \cite{bib:CVPR16:He} needs to be trained with ImageNet dataset which contains $1.2$ million well-labeled internet images collected from 1000 balanced fine classes;
the GPT-3 model published by OpenAI \cite{bib:arXiv20:Brown} takes $3.14\times 10^{23}$ floating-point multiply-accumulate operations (FLOPs) for a single training run, which is equivalent to 355 GPU-years and 4.6M US dollars, according to the theoretical $2.8\times 10^{13}$ FLOPs of high-performed Nvidia V100 GPU and the lowest 3-year reserved cloud pricing we could find \cite{bib:GPT-3}.

Even during the inference phase, the pretrained DNNs still demand a rather significant amount of computing resources from \eg memory, computation, latency, and energy.
For example, the Faster-RCNN model \cite{bib:NIPS15:Ren} requires several hundreds of GFLOPs for a single inference, thus can only achieve around 5 frames per second for object detection on a state-of-the-art GPU; 
the current language models \cite{bib:arXiv20:Brown} contain billions of parameters which often require several GPUs with GB-level memory on a cloud server with high-bandwidth communication bus for the parallelism during inference. 
Note that the state-of-the-art GPU often has a minimal operation power requirement of around 500W \cite{bib:GPU_power}. 
All the examples mentioned above indicate the inherent resource-intensive characteristics of DNNs. 

\begin{figure}[tbp!]
    \centering
    \includegraphics[width=0.99\textwidth]{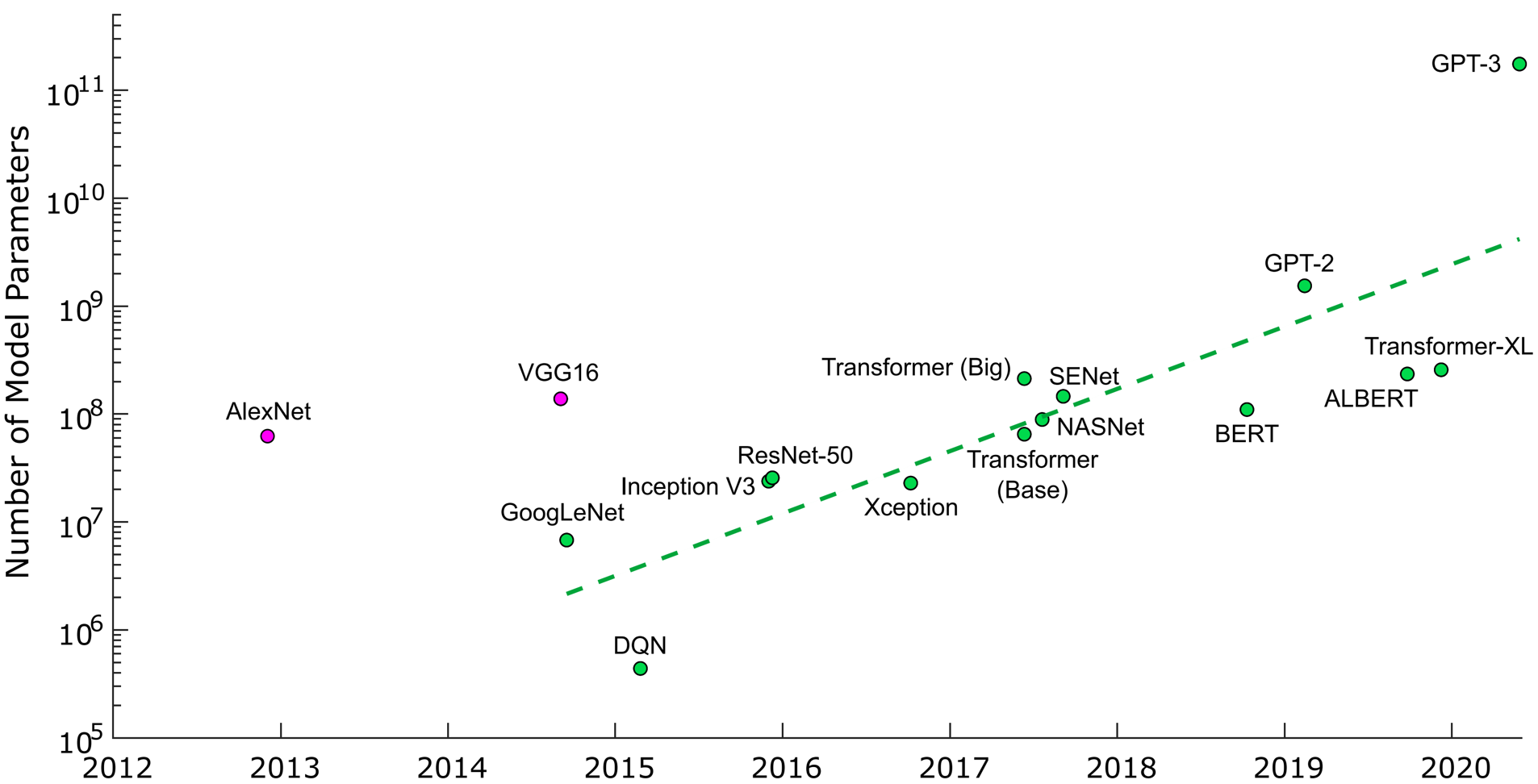}
    \caption[The number of parameters in different DNNs is exponentially increased along the years.]{The number of parameters in different DNNs is exponentially increased along the years. Note that the two outlying nodes (pink) are AlexNet and VGG16, now considered over-parameterized. The figure is originally from \cite{bib:Report21:Bernstein}.}
    \label{ch1-fig:dnn}
\end{figure}

However, the resource demands of DNNs still keep growing. 
As noted in \cite{bib:Report21:Bernstein}, although the state-of-the-art DNNs continuously improve the accuracy level, the number of parameters (as well as the number of FLOPs) in these DNNs also increases along the years, even with an \textit{exponential} increasing rate, as shown in \figref{ch1-fig:dnn}. 
On the other hand, the research and development of hardware often require a long cycle and a high investment.    
As a result, the growth rate of the model size is far larger than the growth rate of the computing power of the state-of-the-art high-performance computers, \eg GPUs. 
For example, the number of parameters has increased more than 2000 times from AlexNet in 2012 \cite{bib:NIPS12:Krizhevsky} to GPT-3 in 2020 \cite{bib:arXiv20:Brown}, whereas at the same time, the memory of Nvidia GPU has only increased 22 times from Geforce GTX 660 to Geforce RTX 3090, and the computing power (FLOPs/second) has increased around 17 times \cite{bib:wiki_Nvidia}.

\section{Cloud Intelligence}
\label{ch1-sec:cloud_intelligence}

As mentioned above, there exists a large gap between the computing power of available hardware and the resource demands of DNNs. 
The common solution to such a conflict is to gather multiple high-performance computers and build a cluster-based server in the cloud, also known as cloud computing \cite{bib:wiki_cloud}.
A cloud server is a group of two or more computers that can share the computing resource, communicate with others and distribute the workload of the same task according to the predefined scheduling system \cite{bib:cluster}. 
Some commercial cloud servers include Amazon Web Services (AWS), Google Cloud, Microsoft Azure, etc. 
These cloud servers may contain high-performance computers of CPUs, GPUs, TPUs, the communication bus with a high bandwidth, the on-demand shared storage infrastructures, and the high power supplement.

Particularly, a DNN can be deployed on a cloud server to perform some resource-intensive intelligent applications \eg gradient-based training, machine translation, question answering systems, etc. 
The high resource demands from these applications can be delegated to multiple computers, and if necessary the results from these computers are aggregated afterwards. 
Cloud intelligence has become a prevailing solution for many intelligent services, which require a large amount of resources (\eg memory, computation) whereas a single computer is often not unable to meet these requirements.

\section{Edge Intelligence}
\label{ch1-sec:edge_intelligence}

\begin{figure}[tbp!]
    \centering
    \includegraphics[width=0.99\textwidth]{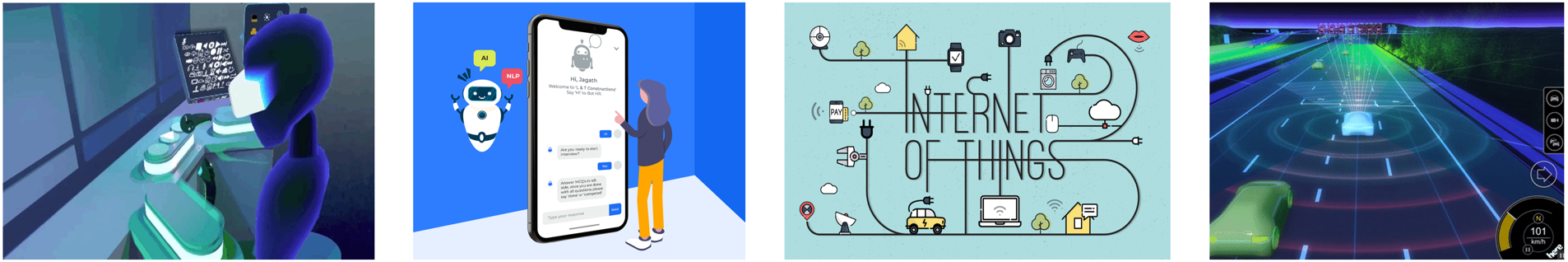}
    \caption[Example edge intelligence applications.]{Example edge intelligence applications. From Left to Right: Augmented/Virtual Reality, Mobile Assistants, Internet of Things, Autonomous Driving. The images are from Google.}
    \label{ch1-fig:edge_intelligent}
\end{figure}

In addition to cloud intelligence, some new emerging edge intelligent applications further require us to deploy DNNs on \textit{edge devices}. 
The term edge refers to an entry point \cite{bib:wiki_edge}.
Accordingly, the collected data (at the entry point) are processed by DNNs locally, \ie on devices. 
Edge devices have a large variety, including mobile phones, wearable devices, sensor nodes, etc.
Some example edge intelligent applications (see in \figref{ch1-fig:edge_intelligent}) include but are not limited to, 
\begin{itemize}
    \item 
    \fakeparagraph{Augmented/Virtual Reality} Augmented/Virtual reality (AR/VR) can visualize the digital information as the real world via wearable devices, \eg glasses \cite{bib:wiki_ar,bib:wiki_vr}. 
    To bridge the gap between the physical world and the virtual environment, many AR/VR tasks, \eg hand detection, eye tracking, digital humans, require deep learning methods to provide high-quality interaction.
   
    \item 
    \fakeparagraph{Mobile Assistants} Mobile assistants are software agents that can perform tasks or services on mobile platforms for an individual based on commands or questions \cite{bib:wiki_assistant}. 
    Individual users can input voice, images, or text to mobile assistants. 
    Given the inputs from users, DNNs are utilized to recognize, understand, and communicate with users.

    \item
    \fakeparagraph{Internet of Things} Internet of Things (IoT) describes physical objects with sensors, processing ability, software, and other technologies that connect with other devices over communication networks \cite{bib:wiki_iot}.
    IoT applications use DNNs for automatic sensing and reasoning, \eg detecting intruders in a ``smart home'' monitor system. 
    
    \item
    \fakeparagraph{Autonomous Driving} Autonomous cars can sense their surroundings and move safely with little or no human inputs \cite{bib:wiki_autonomous}. 
    Thanks to the rapid development of deep learning, many DNNs in computer vision tasks, \eg object detection, 3D localization, semantic segmentation, have been widely adopted to interpret sensory information and identify appropriate navigation paths.
\end{itemize}

In comparison to cloud intelligent applications, edge intelligent applications have the following advantages, (\textit{i}) it does not encounter privacy issues and can be used on sensitive/confidential data, as the data are processed locally; (\textit{ii}) it reduces the reliance on the cloud server, and can achieve a stable inference even with congested/interrupted communication channels; (\textit{iii}) it can realize a real-time inference if the communication bandwidth is limited; (\textit{iv}) it can save energy by avoiding to transfer data to the cloud server which often costs significant amounts of energy than sensing and computation \cite{bib:Book19:Warden,bib:arXiv18:Guo,bib:arXiv19:Lee}.

\begin{figure}[tbp!]
    \centering
    \includegraphics[width=0.99\textwidth]{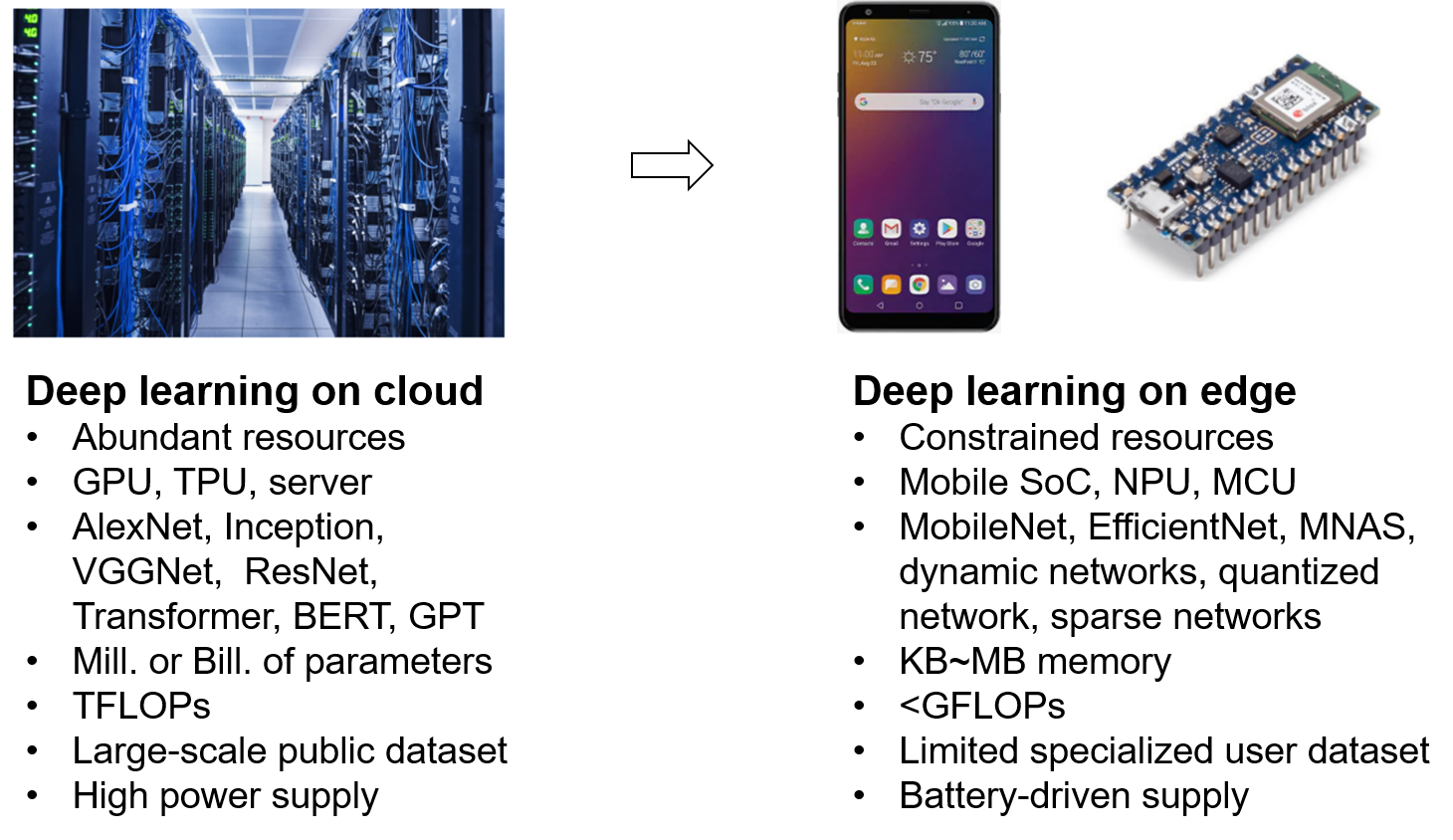}
    \caption[Comparison between deep learning on the cloud server and deep learning on edge devices.]{Comparison between deep learning on cloud and deep learning on edge. The figure is originally from \cite{bib:arXiv21:Soro}.}
    \label{ch1-fig:cloud_edge}
\end{figure}

Unfortunately, deploying DNNs on edge devices is not trivial, as current DNNs contradict the \textit{resource-constrained} nature of edge devices.
Unlike plenty of high-performance computers (\eg GPUs and TPUs) in the cloud server, the processors on edge devices are commonly mobile SoCs, NPUs, or even MCUs, which have a rather small amount of resources and limited scalability. 
We compare the difference between deep learning on the cloud server and deep learning on edge devices in \figref{ch1-fig:cloud_edge}.
The edge devices often use battery-driven energy and have only several $KB$ to $MB$ allocatable RAM. 
Their parallel computing capabilities are also relatively low due to the small number of computing cores. 
In addition, the number of user data collected on edge devices is also limited in comparison to the large-scale datasets used in cloud training.
To deploy DNNs on these edge devices, the complexity of DNNs needs to be trimmed down to fit the limited resource budget.

\section{Thesis Outline}
\label{ch1-sec:outline}

In this thesis, we will study how to \textit{enable deep learning on edge devices in different scenarios}. 
Deploying DNNs on edge devices always targets a trade-off between \textit{the resource demands} and \textit{the model accuracy}. 
Since DNNs often consume a large amount of resources, we hypothesize that there exists redundancy in the DNNs. 
Our goal is to identify and reduce the redundancy according to the main resource constraints in different scenarios. 
This thesis is partitioned into four separate scenarios.
In each scenario, we will (\textit{i}) analyze its main resource constraints, (\textit{ii}) review the drawbacks in the currently available solutions, (\textit{iii}) propose our solution to reduce the redundancy in the DNNs; (\textit{iv}) verify the effectiveness of our solution experimentally or theoretically.
The four studied scenarios are summarized as follows.

\subsection{Inference on Edge Devices (\chref{ch2:inference})}
\label{ch1-sec:inference}

\fakeparagraph{Scenario} 
We first enable an efficient inference on edge devices. 
Inference on edge devices does not rely on the connection to the cloud server, thus it is especially preferred if the communication is highly constrained, or a stable and fast inference is required. 
The main resource constraints of inference on edge devices are the limited static storage and the limited computational ability, as DNNs often contain a large number of parameters to be stored and require a large number of FLOPs for inference.
In this scenario, according to the given resource constraints on edge devices, we train a compressed DNN on a cloud server with a large-scale dataset collected beforehand.
The well-trained compressed network is then deployed on the edge devices and is able to conduct inference with limited resources.

\fakeparagraph{Related Work}
To reduce the storage cost and the computation cost, plenty of works propose to (\textit{i}) design efficient network architectures manually \cite{bib:arXiv17:Howard,bib:CVPR18:Sandler} or automatically using neural architecture search methods \cite{bib:ICLR20:Cai,bib:arXiv19:Yu,bib:ECCV20:Yu}; (\textit{ii}) quantize weights into lower bitwidth to use cheaper operations and reduce the storage consumption \cite{bib:NIPS15:Courbariaux,bib:ECCV16:Rastegari,bib:ECCV18:Zhang}; (\textit{iii}) structured \cite{bib:ICCV19:Liu,bib:ECCV20:Li}/unstructured \cite{bib:ICLR16:Han,bib:ICLR20:Renda,bib:ICML21:Evci} pruning unimportant weights as zeros to reduce the number of operations and the number of nonzero weights. 
We focus on quantizing a pretrained DNN into multi-bit form among others for the following reasons, (\textit{i}) it utilizes the cheaper operations of bitwise \texttt{xnor} and \texttt{popcount} to replace expensive FLOPs; (\textit{ii}) it achieves a high compression ratio without introducing irregular computations; (\textit{iii}) it explores the lower bound of quantized networks.
The state-of-the-art multi-bit networks \cite{bib:arXiv14:Gong,bib:CVPR17:Guo,bib:AAAI18:Hu,bib:NIPS17:Lin,bib:ICLR18:Xu,bib:ECCV18:Zhang} first assign an empirical global bitwidth across layers and then are optimized by minimizing the reconstruction error to the full precision weights, which often results in a subpar performance. 

\fakeparagraph{Our Solution}
To resolve the above drawbacks, we propose an adaptive loss-aware trained quantizer for multi-bit quantization, that (\textit{i}) allocates an adaptive bitwidth to different weights w.r.t. the loss, (\textit{ii}) optimizes the multi-bit quantizer by directly minimizing the loss.
We aim at reducing the \textit{redundant quantization bitwidth} of the weights that are less critical to the loss, to achieve a better trade-off between the model accuracy and the resource demands.  

\subsection{Adaptation on Edge Devices (\chref{ch3:adaptation})}
\label{ch1-sec:adaptation}

\fakeparagraph{Scenario} 
The compressed DNNs trained with the methods in \chref{ch2:inference} can achieve an efficient inference, if the available resources on edge devices are fixed and provided before training on the cloud server.
However, the resource constraints on the target edge devices may dynamically change during runtime \eg the allowed execution time, the allocatable RAM, and the battery energy.
To maximize the model accuracy during on-device inference, the deployed DNN should maintain a dynamic capacity, such that the DNN can be adapted and executed under varying resource constraints. 
In order to quantify the varying resource constraints mentioned earlier, we choose two proxies, (\textit{i}) the storage of weights, which affects the amount of memory fetching and static memory consumption, and (\textit{ii}) the number of operations for inference, which is relevant to the computing energy and the inference latency. 

\fakeparagraph{Related Work}
The most straightforward solution could be for example deploying multiple individual compressed DNNs with different resource demands on edge devices, yet it consumes several times more storage than a single DNN. 
Some prior works \cite{bib:ICLR18:Huang,bib:arXiv17:Hu,bib:ICLR19:Yu,bib:ICCV19:Yu,bib:ICLR20:Cai,bib:RTAS20:Lee} proposed to optimize a backbone network (a.k.a. supernet), such that different candidate sub-networks can be sampled from the backbone network while reaching a similar accuracy level as training them individually. However, these works often sample sub-networks along hand-crafted structured dimensions, \eg kernel size, width, depth, thus the generated sub-networks have different network architectures. This not only results in a sub-optimal performance but also leads to extra re-configuration overhead for storing multiple compiled network architectures.

\fakeparagraph{Our Solution}
We overcome the above disadvantages through sampling sub-networks in a row-based unstructured manner, and propose a novel compressed sparse row (CSR) format to efficiently execute different sub-networks on edge devices. 
Our solution reduces \textit{the architecture redundancy} by reusing a single compiled network architecture among multiple sparse sub-networks, achieving re-configuration efficiency. 
In addition, we also reduce \textit{the weight redundancy} by imposing nonzero weight sharing among sub-networks, achieving storage efficiency.

\subsection{Learning on Edge Devices (\chref{ch4:learning})}
\label{ch1-sec:learning}

\fakeparagraph{Scenario} 
In \chref{ch2:inference} and \chref{ch3:adaptation}, we train a compressed DNN on a cloud server with a large number of available data samples, such that this pretrained DNN can be deployed on edge devices to conduct inference under \textit{fixed} and \textit{varying} resource constraints, respectively.
However, the pretrained DNN may not achieve satisfactory performance when the inference environments on edge devices have a large variance in comparison to the prior environments used to collect data samples for cloud training. 
In other words, when facing unseen environments or users on edge devices, it is crucial to adapt the pretrained DNN to deliver consistent performance and customized services. 
New data samples collected by edge devices are often private and have a large diversity across users/devices. 
Hence, on-device learning is preferred over uploading the data to cloud server.
Compared to the number of data samples used in cloud training, the number of collected data on each edge device is significantly smaller (a.k.a. few-shot) due to the limited labor resources. 
Furthermore, training a DNN, \ie optimizing its weights, requires storing all the intermediate values of each layer, which often consumes several orders of magnitude more peak memory than inference. 
Thus, in this scenario, we target memory-efficient and data-efficient on-device learning.

\fakeparagraph{Related Work}
Meta learning is a prevailing solution to few-shot learning \cite{bib:arXiv20:Hospedales}, where the meta-trained model can learn an unseen task from a few training samples, \ie data-efficient learning.
However, most meta learning algorithms \cite{bib:ICLR19:Antreas, bib:ICML17:Finn, bib:NIPS21:Oswald} optimize the backbone network for better generalization yet ignore the workload if the meta-trained backbone is deployed on low-resource edge platforms for few-shot learning. 
Existing memory-efficient training schemes include for example, low-precision training \cite{bib:ICLR20:Cambier, bib:NIPS18:Wang}, trading memory with computation \cite{bib:arXiv16:Chen, bib:NIPS16:Gruslys}. 
However, they are mainly designed for high-throughput cloud training on large-scale datasets, which are not suitable for on-device learning with only a few data samples. 

\fakeparagraph{Our Solution}
We ground our work (\ie memory-efficient few-shot learning) on gradient-based meta learning methods for their wide applicability in various tasks.
To avoid the high dynamic memory cost in few-shot learning, we focus on reducing \textit{the updating redundancy}. 
In other words, we think not all weights in the learner are equally critical for adaptation. 
Thus, we propose to meta-train a selection mechanism, which can identify and update adaptation-critical weights only during few-shot learning.
This way, only the relevant subset of the intermediate values needs to be stored, leading to memory efficiency.

\subsection{Edge-Server-System (\chref{ch5:edgeserver})}
\label{ch1-sec:edgeserver}

\fakeparagraph{Scenario} 
In \chref{ch2:inference}, \chref{ch3:adaptation} and \chref{ch4:learning}, we explored enabling deep learning on a single edge platform in three different scenarios. 
In addition to a single edge device, edge-server system is another commonly used infrastructure for edge intelligent applications. 
In edge-server system, several edge devices are connected to a remote server, and some information is allowed to be communicated between edge devices and the server. 
In \chref{ch5:edgeserver}, we design a new pipeline to enable efficient inference and efficient updating for edge-server system.
On such an edge-server system, on-device inference is preferred over cloud inference, since it can achieve a fast and stable inference with less energy consumption. 
Due to a possible lack of relevant training data at the initial deployment, pretrained DNNs may either fail to perform satisfactorily or be significantly improved after the initial deployment. 
However, the resources on edge devices are often limited \eg memory, computing power, and energy; the wireless communication is also constrained, \eg limited bandwidth. 
An efficient updating/learning that satisfies the resource constraints mentioned above is needed.

\fakeparagraph{Related Work}
Communication-efficient federated learning \cite{bib:ICLR18:Lin,bib:arXiv19:Kairouz,bib:arXiv20:Li} studies how to compress multiple gradients (to be communicated to the server) calculated on different sets of non-\textit{i.i.d.} local data, such that the aggregation of these (compressed) gradients could result in a similar convergence performance as centralized training on all data.
However, federated learning (as well as other on-device retraining methods) has the following main shortages, (\textit{i}) it conducts resource-intensive gradient calculation on edge devices; (\textit{ii}) the collected data are continuously accumulated on memory-constrained edge devices; (\textit{iii}) it needs to label a large number of samples on edge devices.

\fakeparagraph{Our Solution}
We propose a two-stage iterative process for a continuous improvement of the deployed model's accuracy, (\textit{i}) at each round, edge devices collect new data samples and send them to the server, and (\textit{ii}) the server retrains the network using all collected data, and then sends the updates to each edge device. 
An essential challenge herein is that the transmissions in the server-to-edge stage are highly constrained by the limited communication resource (\eg bandwidth, energy) in comparison to the edge-to-server stage for the following reasons. 
(\textit{i}) A batch of samples that can lead to reasonable updates is relatively smaller in size than the DNN model, especially for the low-resource data type used on edge devices; (\textit{ii}) the server may also receive data from other sources, \eg through data augmentation or new data collection campaigns.
We reduce the communication cost in the server-to-edge stage by distinguishing \textit{the redundant updated weights} given newly collected samples.
In our proposed solution, the server only selects and updates a small subset of critical weights that have a large contribution to the loss reduction during the retraining.

In the rest of this thesis, we first present our four scenarios of enabling deep learning on edge devices, \ie inference on edge devices in \chref{ch2:inference}, adaptation on edge devices in \chref{ch3:adaptation}, learning on edge devices in \chref{ch4:learning}, edge-server-system in \chref{ch5:edgeserver}, respectively; finally conclude and discuss the future work in \chref{ch6:conclusion}.
\chapter[Inference on Edge Devices]{Inference on Edge Devices}
\label{ch2:inference}

We attempt to enable an efficient inference of DNNs on resource-constrained edge devices in this chapter.
Particularly, we focus on quantizing a pretrained DNN to fit the given resource constraints on edge devices while with the minimal accuracy drop.

\fakeparagraph{Main Resource Constraints}
State-of-the-art DNNs often contain a large number of floating-point weights and require a significant amount of floating-point multiply-accumulate operations, which are essential for conducting accurate inference. 
However, edge devices have neither powerful computational ability nor enormous storage.
Thus, for inference on edge devices, we consider that the main resource constraints are the \textit{limited static storage} and \textit{the limited computing power}. 

\fakeparagraph{Principles}
Unlike prior quantized networks that (\textit{i}) often assign an empirical global bitwidth across layers, (\textit{ii}) train the quantizer by minimizing the reconstruction error to the full precision weights, 
we propose an adaptive loss-aware trained quantizer for multi-bit quantization, that (\textit{i}) allocates an adaptive bitwidth to different weights w.r.t. the loss, (\textit{ii}) optimizes the multi-bit quantizer by minimizing the loss. 
The adaptive bitwidth assignment and the direct optimization objective allow our methods to find and remove more redundant bitwidth than prior works, thus achieving both storage efficiency and computation efficiency. 

The contents of this chapter are established mainly based on the paper ``Adaptive Loss-aware Quantization for Multi-bit Networks'' that is published on IEEE/CVF Conference on Computer Vision and Pattern Recognition (CVPR), 2020 \cite{bib:CVPR20:Qu}.

\section{Introduction}
\label{ch2-sec:introduction}

To take advantage of the various pretrained models for efficient inference on resource-constrained edge devices, it is common to compress the pretrained models via pruning \cite{bib:ICLR16:Han}, quantization \cite{bib:arXiv14:Gong,bib:CVPR17:Guo,bib:NIPS17:Lin,bib:ICLR18:Xu,bib:ECCV18:Zhang}, among others.
We focus on \emph{quantization}, especially quantizing both the full precision weights and activations of a deep neural network into binary encodes and the corresponding scaling factors \cite{bib:NIPS15:Courbariaux,bib:ECCV16:Rastegari}, which are also interpreted as binary basis vectors and floating-point coordinates in a geometry viewpoint \cite{bib:CVPR17:Guo}.
Neural networks quantized with binary encodes replace expensive floating-point operations by bitwise operations, which are supported even by microprocessors and often result in small memory footprints \cite{bib:ICLR18:Mishra}.
Since the space spanned by only one-bit binary basis and one coordinate is too sparse to optimize, many researchers suggest a multi-bit network (MBN) \cite{bib:arXiv14:Gong,bib:CVPR17:Guo,bib:AAAI18:Hu,bib:NIPS17:Lin,bib:ICLR18:Xu,bib:ECCV18:Zhang}, which allows to obtain a small size without notable accuracy loss and still leverages bitwise operations.
An MBN is usually obtained via quantization-aware training.
Recent studies~\cite{bib:ICLR18:Pedersoli} leverage bit-packing and bitwise computations for efficient deploying binary networks on a wide range of general devices, which also provides more flexibility to design multi-bit/binary networks. 

\fakeparagraph{Challenges}
Most MBN quantization schemes~\cite{bib:arXiv14:Gong,bib:CVPR17:Guo,bib:AAAI18:Hu,bib:NIPS17:Lin,bib:ICLR18:Xu,bib:ECCV18:Zhang} predetermine a global bitwidth, and learn a quantizer to transform the full precision parameters into binary bases and coordinates such that the quantized models do not incur a significant accuracy loss.
However, these approaches have the following drawbacks:
\begin{itemize}
    \item 
    A \textit{global bitwidth} may be sub-optimal.
    Recent studies on fixed-point quantization \cite{bib:ICLR18:Khoram,bib:ICML16:Lin} show that the optimal bitwidth varies across layers.
    \item 
    Previous efforts \cite{bib:NIPS17:Lin,bib:ICLR18:Xu,bib:ECCV18:Zhang} retain inference accuracy by minimizing \textit{the weight reconstruction error} rather than the loss function.
    Such an indirect optimization objective may lead to a notable loss in accuracy.
    Furthermore, they rely on approximated gradients, \eg straight-through estimators (STE) to propagate gradients through quantization functions during training.
    \item
    Many quantization schemes~\cite{bib:ECCV16:Rastegari,bib:ECCV18:Zhang} keep \textit{the first and last layer in full precision empirically}, because quantizing these layers to low bitwidth tends to dramatically decrease the inference accuracy \cite{bib:ECCV18:Wan,bib:ICLR18:Mishra2}. 
    However, these two full precision layers can be a significant storage overhead compared to other low-bit layers (see \secref{ch2-sec:experiment_imagenet}).
    Also, floating-point operations in both layers can take up the majority of computation in quantized networks~\cite{bib:ICLR19:Louizos}.
\end{itemize}

We overcome the above challenges and drawbacks via a novel \textbf{A}daptive \textbf{L}oss-aware \textbf{Q}uantization scheme (\alq).
Instead of using a uniform bitwidth, \alq assigns an adaptive different bitwidth to each group of weights.
More importantly, \alq directly minimizes the loss function w.r.t. the quantized weights, by iteratively learning a quantizer that (\textit{i}) smoothly reduces the number of binary bases (also the quantization bitwidth) and (\textit{ii}) alternatively optimizes the remaining binary bases and the corresponding coordinates.

\section{Related Work}
\label{ch2-sec:related}

\alq follows the trend to quantize the DNNs using discrete bases with lower bitwidth to reduce expensive floating-point operations as well as the static storage consumption.
Commonly used bases include fixed-point~\cite{bib:arXiv16:Zhou}, power of two \cite{bib:JMLR17:Hubara,bib:ICLR17:Zhou}, and $\{-1,0,+1\}$ \cite{bib:NIPS15:Courbariaux,bib:ECCV16:Rastegari}.
We focus on quantization with binary bases \ie $\{-1,+1\}$ among others for the following considerations. 
(\textit{i}) 
If both weights and activations are quantized with the same binary basis, it is possible to evaluate 32 floating-point multiply-accumulate operations (FLOPs) with only 3 instructions on a 32-bit microprocessor, \ie bitwise $\texttt{xnor}$, $\texttt{popcount}$, and accumulation.
This will significantly speed up the \texttt{conv} operations \cite{bib:JMLR17:Hubara,bib:ICLR18:Pedersoli}.
(\textit{ii})
Multi-bit quantization can be considered as the non-uniform counter-part of fixed-point (integer) quantization. 
A network quantized to fixed-point requires specialized integer arithmetic units and/or specialized integer storage units with various bitwidth for efficient computing~\cite{bib:MICRO17:Albericio,bib:ICLR18:Khoram}, whereas a network quantized with multiple binary bases adopts the same operations mentioned before as binary networks.
Therefore, multi-bit networks may also achieve a higher hardware efficiency than fixed-point network in adaptive bitwidth quantization. 
Popular networks quantized with binary bases include \textit{Binary Networks} and \textit{Multi-bit Networks}.

\subsection{Quantization for Binary Networks}
BNN \cite{bib:NIPS15:Courbariaux} is the first network with both binarized weights and activations.
It dramatically reduces the memory and computation but often with notable accuracy loss.
To resume the accuracy degradation from binarization, XNOR-Net \cite{bib:ECCV16:Rastegari} introduces a layerwise full precision scaling factor into BNN.
However, XNOR-Net leaves the first and last layers unquantized, which consumes more memory.
SYQ \cite{bib:CVPR18:Faraone} studies the efficiency of different structures during binarization/ternarization. 
LAB \cite{bib:ICLR17:Hou} is the first loss-aware quantization scheme which optimizes the weights by directly minimizing the loss function. 

\alq is inspired by recent loss-aware binary networks such as LAB \cite{bib:ICLR17:Hou}.
Loss-aware quantization has also been extended to fixed-point networks in \cite{bib:ICLR18:Hou}.
However, existing loss-aware quantization schemes proposed for binary and ternary networks \cite{bib:ICLR17:Hou,bib:ICLR18:Hou,bib:CVPR18:Zhou} are inapplicable for MBNs.
This is because multiple binary bases dramatically extend the optimization space with the same bitwidth (\ie an optimal set of binary bases rather than a single basis), which may be intractable. 
Some proposals \cite{bib:ICLR17:Hou,bib:ICLR18:Hou,bib:CVPR18:Zhou} still require full-precision weights and gradient approximation (backward STE and forward loss-aware projection), introducing undesirable errors when minimizing the loss.
In contrast, \alq is free from gradient approximation.

\subsection{Quantization for Multi-bit Networks}
MBNs denote networks that use multiple binary bases to trade-off storage and accuracy.
Gong \etal propose a residual quantization process, which greedily searches the next binary basis by minimizing the residual reconstruction error~\cite{bib:arXiv14:Gong}.
Guo \etal improve the greedy search with a least square refinement~\cite{bib:CVPR17:Guo}.
Xu \etal~\cite{bib:ICLR18:Xu} separate this search into two alternating steps, fixing coordinates then exhausted searching for optimal bases, and fixing the bases then refining the coordinates using the method in \cite{bib:CVPR17:Guo}.
LQ-Net~\cite{bib:ECCV18:Zhang} extends the scheme of~\cite{bib:ICLR18:Xu} with a moving average updating, which jointly quantizes weights and activations.
However, similar to XNOR-Net \cite{bib:ECCV16:Rastegari}, LQ-Net~\cite{bib:ECCV18:Zhang} does not quantize the first and last layers.
ABC-Net~\cite{bib:NIPS17:Lin} leverages the statistical information of all weights to construct the binary bases as a whole for all layers. 

All the state-of-the-art MBN quantization schemes minimize the weight reconstruction error rather than the loss function of the network. 
They also rely on the gradient approximation such as STE when back propagating the quantization function.
In addition, they all predetermine a uniform bitwidth for all parameters. 
The indirect objective, the approximated gradient, and the global bitwidth lead to a sub-optimal quantization.
\alq is the first scheme to explicitly optimize the loss function and incrementally train an adaptive bitwidth while without gradient approximation.

\section{Preliminaries and Notations}
\label{ch2-sec:notations}

We aim at multi-bit quantization with an adaptive bitwidth on a DNN consisting of $L$ convolutional (\texttt{conv}) layers or fully connected (\texttt{fc}) layers. 
To simplify the notation, we start the discussion with a single layer and extend to the entire network with $L$ layers in the implementation section \secref{ch2-sec:implementation}. 

For a \texttt{conv}/\texttt{fc} layer, its weights dominate the resource consumption of storage and computation than other parameters, \eg bias, batch normalization. 
We thus judiciously focus on quantizing the weight tensor of the \texttt{conv}/\texttt{fc} layer $l$.
To allow an adaptive bitwidth, we structure the weight tensor of the layer $l$ in \textit{disjoint groups}.
The weights in a single group will be quantized into the same bitwidth, whereas different group may have an adaptive different bitwidth.
Specifically, for the \textit{vectorized} weight tensor $\bm{w}_l\in\mathbb{R}^{N}$ of layer $l$, we divide $\bm{w}_l$ into $G$ disjoint groups.
For simplicity, we omit the subscript $l$ in the following discussion. 
Each group of weights is denoted by $\bm{w}_{g}$, where $\bm{w}_{g}\in\mathbb{R}^{n}$ and $N = n \times G$.
In other words, the overall $N$ weights in layer $l$ are evenly partitioned into $G$ groups, see more details in \secref{ch2-sec:experiment_group}.  
Then the multi-bit quantized weights $\hat{\bm{w}}_{g}$ of group $g$ are formulated as,
\begin{equation}
    \bm{\hat{w}}_g = \sum_{i=1}^{I_g}\alpha_i\bm{\beta}_i=\bm{B}_g\bm{\alpha}_g
    \label{ch2-eq:multi_bit}
\end{equation}
where $\bm{\beta}_i\in\{-1,+1\}^{n\times 1}$ and $\alpha_i\in\mathbb{R}_+$ are the $i$-th binary basis and the corresponding coordinate; $I_g$ represents the quantization bitwidth, \ie the number of binary bases, of group $g$.
$\bm{B}_g\in\{-1,+1\}^{n\times I_g}$ and $\bm{\alpha}_g\in\mathbb{R}_+^{I_g\times1}$ are the matrix forms of the binary bases and the coordinates.
We further denote $\bm{\alpha}=\bm{\alpha}_{1:G}$ as vectorized coordinates $\bm{\alpha}_g$ of all weight groups, and $\bm{B}=\bm{B}_{1:G}$ as concatenated binary bases $\bm{B}_g$ of all weight groups.
A layer $l$ quantized as above yields an average bitwidth
\begin{equation}
    I = \frac{1}{G}\sum_{g = 1}^G I_g
    \label{ch2-eq:avg_bit}
\end{equation}

\section{Adaptive Loss-Aware Quantization}
\label{ch2-sec:method}

\begin{figure}[tbp!]
    \centering
    \includegraphics[width=0.99\textwidth]{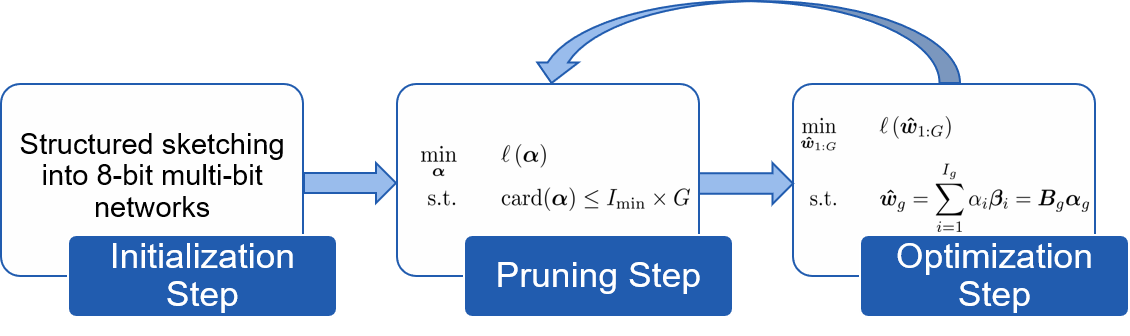}
    \caption[The overall approach of \alq.]{The figure depicts the overall approach of \alq. In Initialization Step, the pretrained full precision weights are separated into disjoint groups and then are quantized into an 8-bit multi-bit form. In Pruning Step, we search an adaptive different bitwidth for each group of weights by removing the unimportant $\alpha$'s w.r.t. the loss. Based on the searched bitwidth assignment, we further conduct an Optimization step to train the remaining binary bases $\bm{B}_g$ and coordinates $\bm{\alpha}_g$. Both Pruning Step and Optimization Step are conducted iteratively.}
    \label{ch2-fig:approach}
\end{figure}

\subsection{Weight Quantization Overview}
\label{ch2-sec:overview}

\fakeparagraph{Problem Formulation}
\alq quantizes weights by directly minimizing the loss function rather than the reconstruction error.
For layer $l$, the process can be formulated as the following optimization problem.
\begin{eqnarray}
  \min_{\bm{\hat{w}}_{1:G}} & & \ell\left(\bm{\hat{w}}_{1:G}\right) \label{ch2-eq:objective} \\
  \text{s.t.}           & & \bm{\hat{w}}_g = \sum_{i=1}^{I_g}\alpha_i\bm{\beta}_i = \bm{B}_g\bm{\alpha}_g \quad \forall g\in 1,...,G\label{ch2-eq:weights} \\
                        & & \mathrm{card}(\bm{\alpha}) = I\times G \leq I_\mathrm{min}\times G \label{ch2-eq:sum_Ig}
\end{eqnarray}
where $\ell$ is the loss; $\mathrm{card}(.)$ denotes the cardinality of the set, \ie the total number of elements in $\bm{\alpha}$; $I_\mathrm{min}$ is the desirable average bitwidth, which is determined by the storage constraints on edge devices.
Since the group size $n$ is the same in one layer, $\mathrm{card}(\bm{\alpha})$ is proportional to the storage consumption.

\fakeparagraph{Solution Pipeline}
The constrained domain of \equref{ch2-eq:weights} and \equref{ch2-eq:sum_Ig} are both discrete and non-convex. 
Directly conducting an exhaustive searching is NP-hard and infeasible on current DNNs. 
Therefore, we propose to narrow down the search space and disentangle the constraints into two sub-problems. 
Particularly, our \alq solves the optimization problem in \equref{ch2-eq:objective}-\equref{ch2-eq:sum_Ig} by three steps.  
The overall approach is shown in \figref{ch2-fig:approach}. 
The pseudocode of the entire pipeline is illustrated in \algoref{ch2-alg:pipeline} in \secref{ch2-sec:implementation_pipeline}.

\begin{itemize}
    \item 
    \underline{Initialization Step:} \textbf{Structured Sketching} (\secref{ch2-sec:implementation_initialization}).
    In this step, we adapt the network sketching in~\cite{bib:CVPR17:Guo}, and propose a structured sketching algorithm. 
    It first partitions the pretrained full precision weights $\bm{w}$ into $G$ groups; then quantizes each $\bm{w}_{g}$ into its 8-bit multi-bit form $\bm{\hat{w}}_{g}$ by greedily searching the optimal binary basis vector $\bm{\beta}_i$ and the optimal scaling factor $\alpha_i$. 
    This step not only provides a good initial point for the following steps, but also restricts each group to a maximal 8-bit to reduce the search space. 
   
    \item 
    \underline{Pruning Step:} \textbf{Pruning in $\bm{\alpha}$ Domain} (\secref{ch2-sec:pruning} and \secref{ch2-sec:implementation_pruning}).
    This step starts from the initialized 8-bit network obtained in Initialization Step, and then progressively reduces the average bitwidth $I$ by pruning the least important (w.r.t. the loss) coordinates in $\bm{\alpha}$ domain.
    Note that removing an element $\alpha_i$ will also lead to the removal of the binary basis $\bm{\beta}_i$, which in effect results in a smaller bitwidth $I_g$ for group $g$.
    This way, no sparse tensor is introduced. 
    Note that sparse tensors could lead to a detrimental irregular computation.
    Since the importance of each weight group differs, the resulting $I_g$ varies across groups, and thus contributes to an adaptive bitwidth $I_g$ for each group.
    In this step, we only set some elements of $\bm{\alpha}$ to zero (also remove them from $\bm{\alpha}$ leading to a reduced $I_g$) without changing the others.
    The sub-problem for Pruning Step is:
    \begin{eqnarray}
        \min_{\bm{\alpha}}  & & \ell\left(\bm{\alpha}\right) \label{ch2-eq:pruning_objective} \\
        \text{s.t.}         & & \mathrm{card}(\bm{\alpha}) \leq I_\mathrm{min}\times G \label{ch2-eq:pruning_constraint}
    \end{eqnarray}
    
    \item
    \underline{Optimization Step:} \textbf{Optimizing Binary Bases $\bm{B}_g$ and Coordinates $\bm{\alpha}_g$} (\secref{ch2-sec:optimization} and \secref{ch2-sec:implementaion_optimization}).
    In this step, we retrain the remaining binary bases and coordinates to recover the accuracy degradation induced by the bitwidth reduction.
    Similar to~\cite{bib:ICLR18:Xu}, we take an alternative approach for better accuracy recovery.
    Specifically, we first search for a new set of binary bases w.r.t. the loss given fixed coordinates.
    Then we optimize the coordinates by fixing the binary bases.
    The sub-problem for Optimization Step is:
    \begin{eqnarray}
        \min_{\bm{\hat{w}}_{1:G}}   & & \ell\left(\bm{\hat{w}}_{1:G}\right) \label{ch2-eq:optimization_objective} \\
        \text{s.t.}                 & & \bm{\hat{w}}_g = \sum_{i=1}^{I_g}\alpha_i\bm{\beta}_i=\bm{B}_g\bm{\alpha}_g \quad \forall g \in 1,...,G \label{ch2-eq:optimization_constraint}
    \end{eqnarray}
    
\end{itemize}

\noindent
For a higher accuracy, state-of-the-art unstructured pruning methods \cite{bib:ICLR16:Han,bib:ICLR19:Frankle} often conduct pruning and sparse fine-tuning iteratively rather than the one-shot manner. 
Similarly, we also conduct our Pruning Step and our Optimization Step \textit{iteratively} until the average bitwidth reaches the desired bitwidth. 
Namely, the original problem of \equref{ch2-eq:objective}-\equref{ch2-eq:sum_Ig} is decoupled into two sub-problems of \equref{ch2-eq:pruning_objective}-\equref{ch2-eq:pruning_constraint} and \equref{ch2-eq:optimization_objective}-\equref{ch2-eq:optimization_constraint}, and the two sub-problems are solved iteratively.

\fakeparagraph{Optimizer Framework}
We consider both Pruning Step and Optimization Step above as an optimization problem with \textit{domain constraints}, and solve them using the same optimization framework: subgradient methods with projection update \cite{bib:JMLR11:Duchi}.

The optimization problem in \equref{ch2-eq:optimization_objective}-\equref{ch2-eq:optimization_constraint} imposes domain constraints on $\bm{B}_g$ because they can only be discrete binary bases.
The optimization problem in \equref{ch2-eq:pruning_objective}-\equref{ch2-eq:pruning_constraint} can be considered as with a trivial domain constraint: the output $\bm{\alpha}$ should be a subset (subvector) of the input $\bm{\alpha}$.
Furthermore, the feasible sets for both $\bm{B}_g$ and $\bm{\alpha}$ are bounded.

Subgradient methods with projection update are effective to solve problems in the form of $\min_{\bm{\theta}}(\ell(\bm{\theta}))$ s.t. $\bm{\theta}\in\Theta$ \cite{bib:JMLR11:Duchi}.
We apply AMSGrad~\cite{bib:ICLR18:Reddi}, an adaptive stochastic subgradient method with projection update, as the common optimizer framework in Pruning Step and Optimization Step.
At training iteration $s$, AMSGrad generates the next update as,
\begin{equation}
    \begin{split}
    \bm{\theta}^{s+1} & = \Pi_{\Theta,\sqrt{\bm{\hat{V}}^s}}(\bm{\theta}^s-a^s\bm{m}^s/\sqrt{\bm{\hat{v}}^s}) \\
                 & = \underset{{\bm{\theta}\in\Theta}}{\mathrm{argmin}}~\|(\sqrt{\bm{\hat{V}}^s})^{1/2}(\bm{\theta}-(\bm{\theta}^s-\frac{a^s\bm{m}^s}{\sqrt{\bm{\hat{v}}^s}}))\|
    \end{split}
    \label{ch2-eq:amsgrad_theta}
\end{equation}
where $\Pi$ is a projection operator; $\Theta$ is the feasible domain of $\bm{\theta}$; $a^s$ is the learning rate; $\bm{m}^s$ is the (unbiased) first momentum; $\bm{\hat{v}}^s$ is the (unbiased) maximum second momentum; and $\bm{\hat{V}}^s$ is the diagonal matrix of $\bm{\hat{v}}^s$.

In our context, \equref{ch2-eq:amsgrad_theta} can be written as,
\begin{equation}
    \bm{\hat{w}}_g^{s+1} = \underset{\bm{\hat{w}}_g\in\mathbb{W}}{\mathrm{argmin}} f^s(\bm{\hat{w}}_g)
    \label{ch2-eq:amsgrad_w1}
\end{equation}
\begin{equation}
    f^s=(a^s\bm{m}^s)^{\mathrm{T}}(\bm{\hat{w}}_g-\bm{\hat{w}}_g^s)+\frac{1}{2}(\bm{\hat{w}}_g-\bm{\hat{w}}_g^s)^{\mathrm{T}}\sqrt{\bm{\hat{V}}^s}(\bm{\hat{w}}_g-\bm{\hat{w}}_g^s)
    \label{ch2-eq:amsgrad_w2}
\end{equation}
where $\mathbb{W}$ is the feasible domain of $\bm{\hat{w}}_g$.

Pruning Step and Optimization Step have different feasible domains of $\mathbb{W}$ according to their objective (see details in~\secref{ch2-sec:pruning} and~\secref{ch2-sec:optimization}).
\equref{ch2-eq:amsgrad_w2} approximates the loss increment incurred by $\bm{\hat{w}}_g$ around the current point $\bm{\hat{w}}_g^s$ as a quadratic model function under domain constraints \cite{bib:NIPS15:Dauphin,bib:JMLR11:Duchi,bib:ICLR18:Reddi}.
For simplicity, we replace $a^s\bm{m}^s$ with $\bm{g}^s$ and replace $\sqrt{\bm{\hat{V}}^s}$ with $\bm{H}^s$.
$\bm{g}^s$ and $\bm{H}^s$ are updated by the loss gradient of $\bm{\hat{w}}_g^s$. 
Thus, the required input of each AMSGrad step is $\partial\ell^s/\partial {\bm{\hat{w}}_g}^s$.
It can be directly obtained during the backward propagation, since $\bm{\hat{w}}_g^s$ is used as an intermediate value during the forward propagation.

\subsection{Pruning in $\bm{\alpha}$ Domain}
\label{ch2-sec:pruning}

As introduced in \secref{ch2-sec:overview}, we reduce the average bitwidth $I$ by pruning the elements in $\bm{\alpha}$ w.r.t. the resulting loss.
If one element $\alpha_i$ in $\bm{\alpha}$ is pruned, the corresponding dimension $\bm{\beta}_i$ is also removed from $\bm{B}$.
Now we explain how to instantiate the optimizer in \equref{ch2-eq:amsgrad_w1} to solve \equref{ch2-eq:pruning_objective}-\equref{ch2-eq:pruning_constraint} of Pruning Step.

As discussed above, pruning in $\bm{\alpha}$ domain is regarded as an optimization problem solved in multiple training iterations.  
Thus, the cardinality of the chosen subset (\ie the average bitwidth) is uniformly reduced over training iterations. 
For example, assume there are $T$ training iterations in total, the initial average bitwidth is $I^0$ and the desired average bitwidth after $T$ iterations $I^{T}$ is $I_\mathrm{min}$.
Then at each iteration $t$, ($M_p = (I^{0}-I_\mathrm{min})\times G/T$) of $\alpha_i^t$'s are pruned. 
This way, the cardinality after $T$ iterations will be smaller than $I_\mathrm{min}\times G$. 

When pruning in the $\bm{\alpha}$ domain, $\bm{B}$ is considered as invariant. 
Hence \equref{ch2-eq:amsgrad_w1} and \equref{ch2-eq:amsgrad_w2} become,
\begin{equation}
    \bm{\alpha}^{t+1} = \underset{\bm{\alpha}\in\mathbb{A}}{\mathrm{argmin}}~f_{\bm{\alpha}}^t(\bm{\alpha})
    \label{ch2-eq:amsgrad_alpha1}
\end{equation}

\begin{equation}
    f_{\bm{\alpha}}^t=(\bm{g}_{\bm{\alpha}}^t)^{\mathrm{T}}(\bm{\alpha}-\bm{\alpha}^t)+\frac{1}{2}(\bm{\alpha}-\bm{\alpha}^t)^{\mathrm{T}}\bm{H_\alpha}^t(\bm{\alpha}-\bm{\alpha}^t)
    \label{ch2-eq:amsgrad_alpha2}
\end{equation}
where $\bm{g}_{\bm{\alpha}}^t$ and $\bm{H_\alpha}^t$ are similar to the ones in \equref{ch2-eq:amsgrad_w2} but are in the $\bm{\alpha}$ domain.
If $\alpha_i^t$ is pruned, the $i$-th element in $\bm{\alpha}$ is set to $0$ in the above~\equref{ch2-eq:amsgrad_alpha1} and~\equref{ch2-eq:amsgrad_alpha2}. 
Thus, the constrained domain $\mathbb{A}$ is taken as all possible vectors with $M_p$ zero elements in $\bm{\alpha}^t$. 

AMSGrad uses a diagonal matrix of $\bm{H_\alpha}^t$ in the quadratic model function, which decouples each element in $\bm{\alpha}^t$.
This means the loss increment caused by several $\alpha_i^t$ equals the sum of the increments caused by them individually, which are calculated as,
\begin{equation}
    f_{\bm{\alpha},i}^t = -g_{\bm{\alpha},i}^t~\alpha_i^t+\frac{1}{2}~H_{\bm{\alpha},{ii}}^t~({\alpha_i^t})^2
    \label{ch2-eq:taylor_pruning}
\end{equation}
All items of $f_{\bm{\alpha},i}^t$ are sorted in ascending.
Then the first $M_p$ items ($\alpha_i^t$) in the sorted list are removed from $\bm{\alpha}^t$, and results in a smaller cardinality $I^{t}\times G$. 
The input of the AMSGrad step in $\bm{\alpha}$ domain is the loss gradient of $\bm{\alpha}_g^t$, which can be computed with the chain rule,
\begin{equation}
    \frac{\partial\ell^t}{\partial\bm{\alpha}_g^t}={\bm{B}_g^t}^{\mathrm{T}} \frac{\partial\ell^t}{\partial {\bm{\hat{w}}_g}^t}
    \label{ch2-eq:gradient_pruning}
\end{equation}
\begin{equation}
    \bm{\hat{w}}_g^t=\bm{B}_g^t \bm{\alpha}_g^t
\end{equation}

Our pipeline allows to reduce the bitwidth smoothly, since the average bitwidth can be floating-point.
In \alq, since different layers have a similar group size (see in \secref{ch2-sec:experiment_group}), the loss increment caused by pruning is sorted among all layers, such that only a global pruning number needs to be determined.
More details are explained in \secref{ch2-sec:implementation_pipeline}.
This Pruning Step not only provides a loss-aware adaptive bitwidth, but also seeks a better initialization for the successive Optimization Step, since low-bit quantized weights may be relatively far from their original full precision values.

\subsection{Optimizing Binary Bases and Coordinates}
\label{ch2-sec:optimization}

After pruning, the loss degradation needs to be recovered. 
Following~\equref{ch2-eq:amsgrad_w1}, the objective in Optimization Step is
\begin{equation}
    \bm{\hat{w}}_g^{s+1} = \underset{\bm{\hat{w}}_g\in\mathbb{W}}{\mathrm{argmin}}~f^s(\bm{\hat{w}}_g)
\end{equation}
The constrained domain $\mathbb{W}$ is decided by, both binary bases and full precision coordinates.
Hence directly searching for the optimal $\bm{\hat{w}}_g$ is NP-hard.
Instead, we optimize $\bm{B}_g$ and $\bm{\alpha}_g$ in an alternative manner, as prior multi-bit quantization works \cite{bib:ICLR18:Xu,bib:ECCV18:Zhang} that minimize the reconstruction error.

\fakeparagraph{Optimizing $\bm{B}_g$}
We directly search for the optimal bases with AMSGrad.
In each training iteration $q$, we fix $\bm{\alpha}_g^q$, and update $\bm{B}_g^q$.
We find the optimal increment for each group of weights, such that it converts to a new set of binary bases, $\bm{B}_g^{q+1}$.
This Optimization Step searches a new space spanned by $\bm{B}_g^{q+1}$ based on the loss reduction, which prevents the pruned space to be always a subspace of the previous one.

According to~\equref{ch2-eq:amsgrad_w1} and~\equref{ch2-eq:amsgrad_w2}, the optimal $\bm{B}_g$ w.r.t. the loss is updated by,
\begin{equation}
    \bm{B}_g^{q+1} = \underset{\bm{B}_g\in\{-1,+1\}^{n\times I_g}}{\mathrm{argmin}}~f^q(\bm{B}_g)
    \label{ch2-eq:amsgrad_B1}
\end{equation}
\begin{equation}
    f^q=(\bm{g}^q)^{\mathrm{T}}(\bm{B}_g\bm{\alpha}_g^{q}-\bm{\hat{w}}_g^q)+\frac{1}{2}(\bm{B}_g\bm{\alpha}_g^{q}-\bm{\hat{w}}_g^q)^{\mathrm{T}}\bm{H}^q (\bm{B}_g\bm{\alpha}_g^{q}-\bm{\hat{w}}_g^q)
    \label{ch2-eq:amsgrad_B2}
\end{equation}
where $\bm{\hat{w}}_g^q = \bm{B}_g^{q}\bm{\alpha}_g^{q}$.

Recall that $\bm{B}_g^q\in\{-1,+1\}^{n\times I_g}$. 
Since $\bm{H}^q$ is diagonal in AMSGrad, each row vector in $\bm{B}_g^{q+1}$ can be independently determined.
For example, the $j$-th row is computed as,
\begin{equation}
    \bm{B}_{g,j}^{q+1} = \underset{\bm{B}_{g,j}}{\mathrm{argmin}}~\|\bm{B}_{g,j}\bm{\alpha}_{g}^q-(\hat{w}_{g,j}^q-g^q_j/H_{jj}^q)\|,\quad j \in 1,...,n
    \label{ch2-eq:row}
\end{equation}
Since in general $n>>I_g$, to reduce the computation complexity, we firstly compute all $2^{I_g}$ possible values of
\begin{equation}
    \bm{b}^{\mathrm{T}}\bm{\alpha}_{g}^q~,~~~ \bm{b}^{\mathrm{T}}\in\{-1,+1\}^{1\times I_g}
    \label{ch2-eq:comb}
\end{equation}
Then each row vector $\bm{B}_{g,j}^{q+1}$ can be directly substituted with the optimal $\bm{b}^{\mathrm{T}}$ through an exhaustive searching in $2^{I_g}$ values.

\fakeparagraph{Optimizing $\bm{\alpha}_g$}
The above obtained set of binary bases $\bm{B}_g$ spans a new $I_g$-dim linear space, which is a subspace of original $n$-dim full space.
The current $\bm{\alpha}_g$ is unlikely to be the optimal point in this $I_g$-dim space, so now we optimize $\bm{\alpha}_g$.
Since $\bm{\alpha}_g$ is in full precision, \ie $\bm{\alpha}_g\in\mathbb{R}^{I_g\times1}$, there is no domain constraint and thus no need for projection updating.
Similar to optimizing full precision $\bm{w}_g$, conventional training strategies can be directly used to optimize $\bm{\alpha}_g$.

Similar to~\equref{ch2-eq:amsgrad_alpha1} and~\equref{ch2-eq:amsgrad_alpha2}, we use AMSGrad optimizer in $\bm{\alpha}$ domain without projection updating, for each group in the $p$-th training iteration as,
\begin{equation}
    \bm{\alpha}_g^{p+1} = \bm{\alpha}_g^p-a_{\bm{\alpha}}^p\bm{m}_{\bm{\alpha}}^p/\sqrt{\bm{\hat{v}_\alpha}^p}
    \label{ch2-eq:optimizing_alpha}
\end{equation}

We also add an L2-norm regularization on $\bm{\alpha}_g$ to enforce unimportant coordinates to zero. 
If there is a negative value in $\bm{\alpha}_{g}$, the corresponding basis is set to its negative complement, to keep $\bm{\alpha}_{g}$ semi-positive definite. Optimizing $\bm{B}_g$  and $\bm{\alpha}_g$ does not influence the number of binary bases $I_g$.

\fakeparagraph{Optimization Speedup}
Since $\bm{\alpha}_g$ is full precision, updating $\bm{\alpha}_g^q$ is much cheaper than exhaustively search $\bm{B}_g^{q+1}$. 
Even if the main purpose of the first step in Optimization Step is optimizing bases, we also add an updating process for $\bm{\alpha}_g^q$ in each training iteration $q$.

We fix $\bm{B}_{g}^{q+1}$, and update $\bm{\alpha}_{g}^{q}$.
The overall increment of quantized weights from both updating processes is,
\begin{equation}
    \bm{\hat{w}}^{q+1}_g - \bm{\hat{w}}^q_{g} = \bm{B}_{g}^{q+1}\bm{\alpha}_{g}^{q+1}-\bm{B}_{g}^{q}\bm{\alpha}_{g}^{q}
    \label{ch2-eq:W}
\end{equation}
Substituting~\equref{ch2-eq:W} into~\equref{ch2-eq:amsgrad_w1} and~\equref{ch2-eq:amsgrad_w2}, we have,
\begin{equation}
    \bm{\alpha}_{g}^{q+1}=-((\bm{B}_{g}^{q+1})^{\mathrm{T}} \bm{H}^q \bm{B}_{g}^{q+1})^{-1}\times((\bm{B}_{g}^{q+1})^{\mathrm{T}}(\bm{g}^q-\bm{H}^q\bm{B}^q_{g}\bm{\alpha}_{g}^{q}))
    \label{ch2-eq:alpha}
\end{equation}
To ensure the inverse in~\equref{ch2-eq:alpha} exists, we add a small diagonal matrix $\lambda \mathbf{I}$ to \equref{ch2-eq:alpha},
\begin{equation}
    \bm{\alpha}_{g}^{q+1}=-((\bm{B}_{g}^{q+1})^{\mathrm{T}} \bm{H}^q \bm{B}_{g}^{q+1}+\lambda \mathbf{I})^{-1}\times((\bm{B}_{g}^{q+1})^{\mathrm{T}}(\bm{g}^q-\bm{H}^q\bm{B}^q_{g}\bm{\alpha}_{g}^{q}))
    \label{ch2-eq:alpha_lambda}
\end{equation}
where $\lambda=10^{-6}$.

\subsection{Implementation}
\label{ch2-sec:implementation}

In this section, we discuss the detailed implementation of \alq. 
We elaborate the pseudocodes of three steps and analyze their complexity. 
Note that the discussion in this section is extended to the entire networks with $L$ layers, thus we reintroduce the layer index $l$ for clarity reasons.

\subsubsection{Implementation of Initialization Step}
\label{ch2-sec:implementation_initialization}

We adapt the network sketching in~\cite{bib:CVPR17:Guo}, and propose a structured sketching algorithm for Initialization Step, see \algoref{ch2-alg:sketching}\footnote{Circled operation in \algoref{ch2-alg:sketching} means elementwise operations.}. 
This algorithm partitions the pretrained full precision weights $\bm{w}_l$ of the $l$-th layer into $G_l$ groups. 
We study the different structures of grouping in \secref{ch2-sec:experiment_group}. 
The vectorized weights $\bm{w}_{l,g}$ of each group are quantized with $I_{l,g}$ linear independent binary bases (\ie column vectors in $\bm{B}_{l,g}$) and corresponding coordinates $\bm{\alpha}_{l,g}$ to minimize the reconstruction error. 
This algorithm initializes the matrix of binary bases $\bm{B}_{l,g}$, the vector of floating-point coordinates $\bm{\alpha}_{l,g}$, and the scalar of integer bitwidth $I_{l,g}$ in each group across layers.
The initial reconstruction error is upper bounded by a threshold $\sigma$. 
In addition, a maximum bitwidth of each group is defined as $I_\mathrm{max}$.
Both of these two parameters determine the initial bitwidth $I_{l,g}$.
We discuss the choice of group size $n$, and the maximum bitwidth $I_\mathrm{max}$ in \secref{ch2-sec:experiment_initialization}.

\begin{algorithm}[!htbp]
    \caption{Structured sketching of weights}\label{ch2-alg:sketching}
    \KwIn{$\bm{w}_{1:L}$, $G_{1:L}$, $I_\mathrm{max}$, $\sigma$}
    \KwOut{$\{\{\bm{\alpha}_{l,g},\bm{B}_{l,g}, I_{l,g}\}_{g=1}^{G_l}\}_{l=1}^{L}$}
    \For {$l\leftarrow 1$ \KwTo $L$} {
        \For {$g \leftarrow 1$ \KwTo $G_l$} {
            Fetch and vectorize $\bm{w}_{l,g}$ from $\bm{w}_l$\;
            Initialize $\bm{\epsilon} = \bm{w}_{l,g}$, $i=0$\;
            $\bm{B}_{l,g} = [~]$\;
            \While{$\|\bm{\epsilon}\oslash\bm{w}_{l,g}\|_2^2>\sigma$ \texttt{\textup{and}} $i<I_\mathrm{max}$} {
                $i = i+1$\;
                $\bm{\beta}_{i} = \mathrm{sign}(\bm{\epsilon})$\;
                $\bm{B}_{l,g} = [\bm{B}_{l,g}, \bm{\beta}_{i}]$\;
                \tcc{Find the optimal point spanned by $\bm{B}_{l,g}$}
                $\bm{\alpha}_{l,g} = (\bm{B}_{l,g}^\mathrm{T}\bm{B}_{l,g})^{-1}\bm{B}_{l,g}^\mathrm{T}\bm{w}_{l,g}$\;
                \tcc{Update the residual reconstruction error}
                $\bm{\epsilon} = \bm{w}_{l,g}-\bm{B}_{l,g}\bm{\alpha}_{l,g}$\;
            }
            $I_{l,g}=i$\;
        }
    }
\end{algorithm}

\begin{theorem}
    The column vectors in $\bm{B}_{l,g}$ are linear independent.
\end{theorem}

\begin{proof}
    The instruction $\bm{\alpha}_{l,g} = (\bm{B}_{l,g}^\mathrm{T}\bm{B}_{l,g})^{-1}\bm{B}_{l,g}^\mathrm{T}\bm{w}_{l,g}$ ensures $\bm{\alpha}_{l,g}$ is the optimal point in $\mathrm{span}(\bm{B}_{l,g})$ regarding the least square reconstruction error $\bm{\epsilon}$. 
    Thus, $\bm{\epsilon}$ is orthogonal to $\mathrm{span}(\bm{B}_{l,g})$. 
    The new basis is computed from the next iteration by $\bm{\beta}_{i}= \mathrm{sign}(\bm{\epsilon})$. 
    Since $\mathrm{sign}(\bm{\epsilon})\cdot\bm{\epsilon}>0, \forall\bm{\epsilon}\ne\bm{0}$, we have $\bm{\beta}_{i}\notin \mathrm{span}(\bm{B}_{l,g})$. 
    Thus, the iteratively generated column vectors in $\bm{B}_{l,g}$ are linear independent.
    This also means the square matrix of $\bm{B}_{l,g}^\mathrm{T}\bm{B}_{l,g}$ is invertible.
\end{proof}

\subsubsection{Implementation of Pruning Step}
\label{ch2-sec:implementation_pruning}

As discussed in \secref{ch2-sec:pruning}, $\alpha_i$'s are pruned iteratively in mini-batches. 
During each Pruning Step, for example, $30\%$ of $\alpha_i$'s are iteratively pruned in one epoch. 
Due to the high complexity of sorting all $f_{\bm{\alpha},i}$, sorting is firstly executed in each layer, and the top-$k\%$ $f_{\bm{\alpha}_l,i}$ of the $l$-th layer are selected to resort again for pruning.
Recall that $l$ stands for the layer index.
$k$ is generally small, \eg $1$ or $0.5$, which ensures that the pruned $\alpha_i$'s in one iteration do not always come from a single layer.
There are $n_l$ weights in each group, and $G_l$ groups in the $l$-th layer.
The sorting complexity mainly depends on the sorting in the most critical layer that has the largest $\mathrm{card}(\bm{\alpha}_l)$.

The Pruning Step is elaborated in \algoref{ch2-alg:pruning}. 
Here, assume that there are altogether $T$ pruning (training) iterations in each execution of Pruning Step; the total number of $\alpha_i$'s across all layers is $M_0$ before pruning, \ie
\begin{equation}
    M_0 = \underset{l}{\sum}{\underset{g}{\sum}{\mathrm{card}(\bm{\alpha}_{l,g})}}
    \label{ch2-eq:m0}
\end{equation}
and the desired total number of $\alpha_i$'s after pruning is $M_T$.

\begin{algorithm}[tbp!]
    \caption{Pruning in $\alpha$ domain}\label{ch2-alg:pruning}
    \KwIn{$T$, $M_T$, $k$, $\{\{\bm{\alpha}_{l,g},\bm{B}_{l,g}, I_{l,g}\}_{g=1}^{G_l}\}_{l=1}^L$, training dataset}
    \KwOut{$\{\{\bm{\alpha}_{l,g},\bm{B}_{l,g}, I_{l,g}\}_{g=1}^{G_l}\}_{l=1}^L$}
    Compute $M_0$ with \equref{ch2-eq:m0}\;
    Compute the pruning number per iteration $M_p = \mathrm{round}(\frac{M_0-M_T}{T})$\;
    \For {$t \leftarrow 1$ \KwTo $T$} {
        \For {$l\leftarrow 1$ \KwTo $L$} {
            Update $\bm{\hat{w}}_{l,g}^t = \bm{B}_{l,g}^t\bm{\alpha}_{l,g}^t$\;
            Forward propagate\;
        }
        Compute the loss $\ell^t$\;
        \For {$l\leftarrow L$ \KwTo $1$} {
            Backward propagate gradient $\partial\ell^t/\partial\bm{\hat{w}}_{l,g}^t$\;
            Compute $\partial\ell^t/\partial\bm{\alpha}_{l,g}^t$ with \equref{ch2-eq:gradient_pruning}\;
            Update momentums of AMSGrad in $\bm{\alpha}$ domain\; 
            \For {$\alpha_{l,i}^t$ \textup{in} $\bm{\alpha}_l^t$} {
                Compute $f_{\bm{\alpha}_l,i}^t$ with \equref{ch2-eq:taylor_pruning}\;
            }
            Sort and select Top-$k\%$ $f_{\bm{\alpha}_l,i}^t$ in ascending order\;
        }
        Resort the selected $\{f_{\bm{\alpha}_l,i}^t\}_{l=1}^{L}$ in ascending order\;
        Remove Top-$M_p$ $\alpha_{l,i}^t$ and their binary bases\;
        Update $\{\{\bm{\alpha}_{l,g}^{t+1},\bm{B}_{l,g}^{t+1}, I_{l,g}^{t+1}\}_{g=1}^{G_l}\}_{l=1}^L$\;
    }
\end{algorithm}

\subsubsection{Implementation of Optimization Step}
\label{ch2-sec:implementaion_optimization}

Optimization Step is also executed in batch training.
Since $\bm{\alpha}_g$ is floating-point value, the complexity of optimizing $\bm{\alpha}_g$ is the same as the conventional optimization (see \algoref{ch2-alg:coordinates}).
Assume that there are altogether $P$ training iterations. 
It is worth noting that both the bitwidth $I_{l,g}$ and the binary bases $\bm{B}_{l,g}$ do not change in this step; only the coordinates $\bm{\alpha}_{l,g}$ are updated over $P$ iterations. 

\begin{algorithm}[tbp!]
\caption{Optimizing $\bm{\alpha}_g$}\label{ch2-alg:coordinates}
\KwIn{$P$, $\{\{\bm{\alpha}_{l,g},\bm{B}_{l,g}, I_{l,g}\}_{g=1}^{G_l}\}_{l=1}^L$, training dataset}
\KwOut{$\{\{\bm{\alpha}_{l,g},\bm{B}_{l,g}, I_{l,g}\}_{g=1}^{G_l}\}_{l=1}^L$}
\For {$p \leftarrow 1$ \KwTo $P$} {
    \For {$l\leftarrow 1$ \KwTo $L$} {
        Update $\bm{\hat{w}}_{l,g}^p = \bm{B}_{l,g}\bm{\alpha}_{l,g}^p$\;
        Forward propagate\;
    }
    Compute the loss $\ell^p$\;
    \For {$l\leftarrow L$ \KwTo $1$} {
        Backward propagate gradient $\partial\ell^p/\partial\bm{\hat{w}}_{l,g}^p$\;
        Compute $\partial\ell^p/\partial\bm{\alpha}_{l,g}^p$ with \equref{ch2-eq:gradient_pruning}\;
        Update momentums of AMSGrad in $\bm{\alpha}$ domain\;
        \For {$g \leftarrow 1$ \KwTo $G_l$} {
            Update $\bm{\alpha}_{l,g}^{p+1}$ with \equref{ch2-eq:optimizing_alpha}\;
        }
    }
}
\end{algorithm}

Optimizing $\bm{B}_g$ with speedup is presented in \algoref{ch2-alg:bases}. 
Assume that there are altogether $Q$ training iterations. 
It is worth noting that the bitwidth $I_{l,g}$ does not change in this step; only the binary bases $\bm{B}_{l,g}$ and the coordinates $\bm{\alpha}_{l,g}$ are updated over $Q$ iterations.  

\begin{algorithm}[tbp!]
    \caption{Optimizing $\bm{B}_g$ with speedup}\label{ch2-alg:bases}
    \KwIn{$Q$, $\{\{\bm{\alpha}_{l,g},\bm{B}_{l,g}, I_{l,g}\}_{g=1}^{G_l}\}_{l=1}^L$, training dataset}
    \KwOut{$\{\{\bm{\alpha}_{l,g},\bm{B}_{l,g}, I_{l,g}\}_{g=1}^{G_l}\}_{l=1}^L$}
    \For {$q \leftarrow 1$ \KwTo $Q$} {
        \For {$l\leftarrow 1$ \KwTo $L$} {
            Update $\bm{\hat{w}}_{l,g}^q = \bm{B}_{l,g}^q\bm{\alpha}_{l,g}^q$\;
            Forward propagate\;
        }
        Compute the loss $\ell^q$ \;
        \For {$l\leftarrow L$ \KwTo $1$} {
            Backward propagate gradient $\partial\ell^q/\partial\bm{\hat{w}}_{l,g}^q$\;
            Update momentums of AMSGrad\; 
            \For {$g \leftarrow 1$ \KwTo $G_l$} {
                Compute all values of \equref{ch2-eq:comb}\;
                \For {$j \leftarrow 1$ \KwTo $n_l$} {
                    Update $\bm{B}_{l,g,j}^{q+1}$ with \equref{ch2-eq:row}\;
                }
                Update $\bm{\alpha}_{l,g}^{q+1}$ with \equref{ch2-eq:alpha_lambda}\;
            }
        }
    }
\end{algorithm}

The extra complexity related to the original AMSGrad mainly comes from two parts, \equref{ch2-eq:row} and \equref{ch2-eq:alpha_lambda}.
\equref{ch2-eq:row} is also the most resource-hungry step of the whole pipeline, since it requires an exhaustive search.
For each group, \equref{ch2-eq:row} takes both time and storage complexities of $O(n\cdot2^{I_g})$, and in general $n>>I_g\geq1$.
Since $\bm{H}^q$ is a diagonal matrix, most of the matrix-matrix multiplications in \equref{ch2-eq:alpha_lambda} is avoided through matrix-vector multiplication and matrix-diagonalmatrix multiplication.
Thus, the time complexity trims down to $O(nI_g+nI_g^2+I_g^3+nI_g+n+n+nI_g+I_g^2) \doteq O(n(I_g^2+3I_g+2))$.

\subsubsection{Implementation of the Pipeline}
\label{ch2-sec:implementation_pipeline}

The entire pipeline of \alq is demonstrated in \algoref{ch2-alg:pipeline}.
For Initialization Step, the pretrained full precision weights $\bm{w}_{1:L}$ are required. 
Then, we need to specify the structure used in each layer, \ie the structure of grouping $G_{1:L}$.
In addition, a maximum bitwidth $I_\mathrm{max}$ and a threshold $\sigma$ for the residual reconstruction error also need to be determined (see more details in \secref{ch2-sec:implementation_initialization}).
After initialization, we might need to retrain the model with several epochs of \algoref{ch2-alg:bases} to recover the accuracy degradation caused by the initialization. 

Then, we need to determine the number of outer iterations $R$, \ie how many times the Pruning Step is executed.
A pruning schedule $M^{1:R}$ is also required. 
$M^r$ determines the total number of remaining $\alpha_i$'s (across all layers) after the $r$-th Pruning Step, which is also taken as the input $M_T$ in \algoref{ch2-alg:pruning}. 
For example, we can build this schedule by pruning $30\%$ of $\alpha_i$'s during each execution of Pruning Step, as,
\begin{equation}
    M^{r+1} = M^{r}\times(1-0.3)
    \label{ch2-eq:mr}
\end{equation}
with $r\in 0,1,2,...,R-1$. $M^0$ represents the total number of $\alpha_i$'s (across all layers) after initialization.

For Pruning Step, other individual inputs include the total number of iterations $T$, and the selected percentages $k$ for sorting (see \algoref{ch2-alg:pruning}).
For Optimization Step, the individual inputs includes the total number of iterations $Q$ in optimizing $\bm{B}_g$ (see \algoref{ch2-alg:bases}), and the total number of iterations $P$ in optimizing $\bm{\alpha}_g$ (see \algoref{ch2-alg:coordinates}).

\begin{algorithm}[htbp!]
    \caption{Adaptive Loss-aware Quantization for multi-bit networks} \label{ch2-alg:pipeline}
    \KwIn{Pretrained full precision weights $\bm{w}_{1:L},$ structures $G_{1:L}$, $I_\mathrm{max}$, $\sigma$, $T$, pruning schedule $M^{1:R}$, $k$, $P$, $Q$, $R$, training dataset}
    \KwOut{$\{\{\bm{\alpha}_{l,g},\bm{B}_{l,g}, I_{l,g}\}_{g=1}^{G_l}\}_{l=1}^L$}
    \tcc{Initialization Step: }
    Initialize $\{\{\bm{\alpha}_{l,g},\bm{B}_{l,g}, I_{l,g}\}_{g=1}^{G_l}\}_{l=1}^L$ with \algoref{ch2-alg:sketching}\;  
    \For {$r \leftarrow 1$ \KwTo $R$} {
        \tcc{Pruning Step: }
        Assign $M^r$ to the input $M_T$ of \algoref{ch2-alg:pruning}\;
        Prune in $\bm{\alpha}$ domain with \algoref{ch2-alg:pruning}\;
        \tcc{Optimization Step: }
        Optimize binary bases with \algoref{ch2-alg:bases}\;
        Optimize coordinates with \algoref{ch2-alg:coordinates}\;
    }
\end{algorithm}

\section{Activation Quantization}
\label{ch2-sec:activation}

To leverage bitwise operations for speedup, the inputs of each layer (\ie the activation output of the last layer) also need to be quantized into the multi-bit form.
We quantize activations with the same binary basis (\ie $\{-1,+1\}$) as the aforementioned weight quantization.

Our activation quantization follows the idea proposed in \cite{bib:arXiv18:Choi}, \ie a parameterized clipping for fixed-point activation quantization, but it is adapted to the multi-bit form.
Specially, we replace ReLU with a step activation function.
The vectorized activation $\bm{x}$ of the $l$-th layer is quantized as,
\begin{equation}
    \bm{x}\doteq\bm{\hat{x}}=x_{\mathrm{ref}}+\bm{D}\bm{\gamma}=\bm{D}'\bm{\gamma}'
    \label{ch2-eq:act}
\end{equation}
where $\bm{D}\in\{-1,+1\}^{N_x\times I_x}$, and $\bm{\gamma}\in\mathbb{R}_+^{I_x\times1}$.
$\bm{\gamma}'$ is a column vector formed by $[x_{\mathrm{ref}},\bm{\gamma}^\mathrm{T}]^{\mathrm{T}}$; $\bm{D}'$ is a matrix formed by $[\bm{1}^{{N_x\times 1}}, \bm{D}]$.
$N_x$ is the dimension of $\bm{x}$, and $I_x$ is the quantization bitwidth for activations.
$x_{\mathrm{ref}}$ is the introduced layerwise (positive floating-point) reference to fit the output range of ReLU.
During inference, $x_{\mathrm{ref}}$ is convoluted with the weights of the next layer and added to the bias.
Hence the introduction of $x_{\mathrm{ref}}$ does not lead to extra computations. 
The output of the last layer is not quantized, as it does not involve computations anymore.
For other settings, we mainly follow the ones used in \cite{bib:ECCV18:Zhang}. 
$\bm{\gamma}$ and $x_{\mathrm{ref}}$ are updated during the forward propagation with a running average to minimize the squared reconstruction error as,
\begin{equation}
    \bm{\gamma}'_{\text{new}} = (\bm{D'}^{\mathrm{T}}\bm{D}')^{-1}\bm{D'}^{\mathrm{T}}\bm{x}
\end{equation}
\begin{equation}
    \bm{\gamma}' = 0.9\bm{\gamma}'+(1-0.9)\bm{\gamma}'_{\text{new}}
\end{equation}

The (quantized) weights are also further fine-tuned with our optimizer to resume the accuracy drop.
Here, we only set a global bitwidth for all layers in activation quantization.

\section{Experiments}
\label{ch2-sec:experiment}

In this section, we implement \alq with Pytorch~\cite{bib:NIPSWorkshop17:Paszke}, and evaluate its performance on MNIST~\cite{bib:MNIST}, CIFAR10~\cite{bib:CIFAR}, and ImageNet~\cite{bib:ILSVRC15} using LeNet5~\cite{bib:PIEEE98:LeCun}, VGGNet~\cite{bib:ICLR17:Hou,bib:ECCV16:Rastegari}, and ResNet18/34~\cite{bib:CVPR16:He}, respectively.
The Top-1 test accuracy is reported, when the validation dataset has the highest accuracy during training. 
We first conduct the experiments on Initialization Step (\secref{ch2-sec:experiment_initialization}), Pruning Step (\secref{ch2-sec:experiment_adaptive}) and Optimization Step (\secref{ch2-sec:experiment_convergence}) individually to study their impacts. 
Then, we benchmark \alq on different datasets and compare \alq with different state-of-the-art network compression methods.

\subsection{Benchmarking Details}
\label{ch2-sec:experiment_benchmark}

\fakeparagraph{LeNet5 on MNIST}
The MNIST dataset~\cite{bib:MNIST} consists of $28\times28$ gray scale images from 10 digit classes. 
We use 50000 samples in the training set for training, the rest 10000 for validation, and the 10000 samples in the test set for testing. 
We use a mini-batch with size of 128. 
We use the default hyperparameters proposed in~\cite{bib:torchLeNet5} to train LeNet5 for 100 epochs as the baseline of full precision version.
The network architecture is presented as, 20C5 - MP2 - 50C5 - MP2 - 500FC - 10SVM.

\fakeparagraph{VGGNet on CIFAR10}
The CIFAR-10 dataset~\cite{bib:CIFAR} consists of 60000 $32\times32$ color images in 10 object classes. 
We use 45000 samples in the training set for training, the rest 5000 for validation, and the 10000 samples in the test set for testing. 
We use a mini-batch with size of 128. 
We use the default Adam optimizer provided by Pytorch to train full precision parameters for 200 epochs as the baseline of the full precision version. 
The initial learning rate is $0.01$, and it decays with 0.2 every $30$ epochs.
The network architecture is presented as, 2$\times$128C3 - MP2 - 2$\times$256C3 - MP2 - 2$\times$512C3 - MP2 - 2$\times$1024FC - 10SVM.

\fakeparagraph{ResNet18/34 on ImageNet}
The ImageNet dataset~\cite{bib:ILSVRC15} consists of $1.28$ million high-resolution images for classifying in 1000 object classes. 
The validation set contains 50k images, which are used to report the accuracy level.
We use mini-batch with size of 256. The used ResNet18/34 is from~\cite{bib:CVPR16:He}.
We use the ResNet18/34 provided by Pytorch as the baseline of full precision version. 
The network architecture is the same as "resnet18/resnet34" in~\cite{bib:torchResNet}.

\subsection{Experiments on Initialization}
\label{ch2-sec:experiment_initialization}

As mentioned in \secref{ch2-sec:implementation_initialization}, we propose a structured sketching for Initialization Step.
Some important parameters in \algoref{ch2-alg:sketching} are discussed as below. 

\subsubsection{Group Size $n$} 
\label{ch2-sec:experiment_group}

Researchers propose different structures \eg layerwise, channelwise, to partition weights, and then quantize the weights in one structured group with the same bitwidth. 
To explore the redundancy among weights, we conduct experiments on the different structures of grouping.
Certainly, the weights in one layer can be arbitrarily selected to gather a group.
However, due to the extra indexing cost, the weights are often sliced along the tensor dimensions and uniformly grouped. 

According to~\cite{bib:CVPR17:Guo}, the squared reconstruction error of a single group decays with~\equref{ch2-eq:error_decay}, where $\lambda\ge0$.
\begin{equation}
    \|\bm{\epsilon}\|_2^2 \le \|\bm{w}_{g}\|_2^2 (1-\frac{1}{n-\lambda})^{I_g}
    \label{ch2-eq:error_decay}
\end{equation}
If full precision values are stored in floating-point, \ie $32$-bit, the storage compression ratio in one layer can be written as,
\begin{equation}
    r_s = \frac{N\times32}{I\times N+I\times32\times \frac{N}{n}}
    \label{ch2-eq:r_s}
\end{equation}
where $N$ is the total number of weights in one layer; $n$ is the number of weights in each group, \ie $n = N/G$; $I$ is the average bitwidth, $I = \frac{1}{G}\sum_{g = 1}^G I_g$.

We analyse the trade-off between the reconstruction error and the storage compression ratio of different group size $n$.
We choose the pretrained AlexNet~\cite{bib:NIPS12:Krizhevsky} and VGGNet~\cite{bib:ICLR15:Simonyan}, and plot the curves of the average (per weight) reconstruction error related to the storage compression ratio of each layer under different sliced structures.
We also randomly shuffle the weights in each layer, then partition them into groups with different sizes.
We select one example plot which comes from the last \texttt{conv} layer ($256\times256\times3\times3$) of AlexNet~\cite{bib:NIPS12:Krizhevsky} (see~\figref{ch2-fig:conv_alexnet}).
The pretrained full precision weights are provided by Pytorch~\cite{bib:NIPSWorkshop17:Paszke}. 
\begin{figure}[htbp!]
    \centering
    \includegraphics[width=0.75\textwidth]{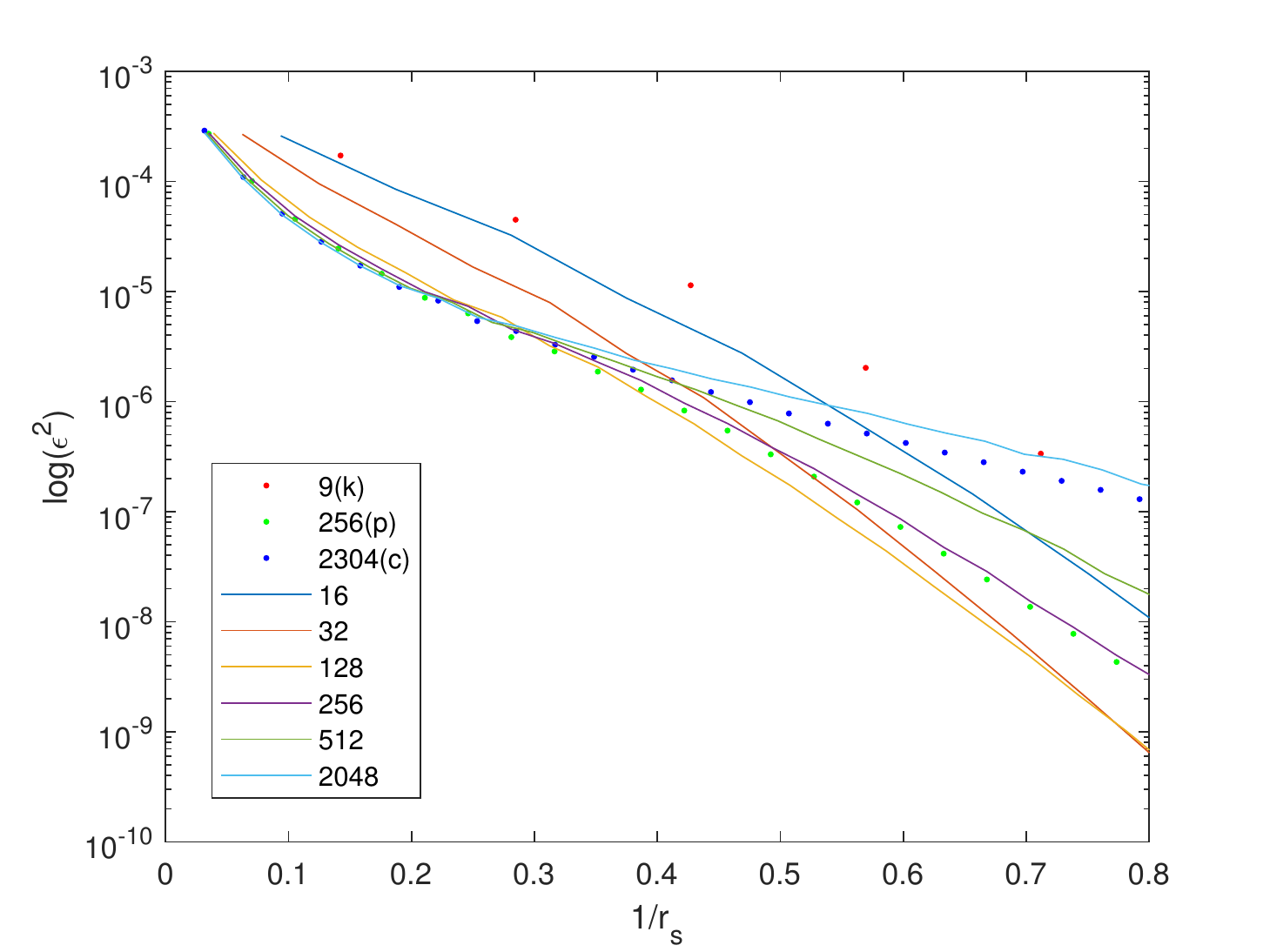}
    \caption[The curves of the average reconstruction error in different group size] {The curves about the logarithmic L2-norm of the average reconstruction error $\mathrm{log}(\|\bm{\epsilon}\|_2^2)$ related to the reciprocal of the storage compression ratio $1/r_s$. The pretrained full precision weights are from the last \texttt{conv} layer of AlexNet. The legend demonstrates the corresponding group sizes. `k' stands for kernelwise; `p' stands for pointwise; `c' stands for channelwise. }
    \label{ch2-fig:conv_alexnet}
\end{figure}

We found that there is not a significant difference between random groups and sliced groups along tensor dimensions.
Only the group size influences the trade-off.
We think the reason is that one layer always contains thousands of groups, such that the points presented by these groups are roughly scattered in the $n$-dim space.
Furthermore, regarding the deployment on a 32-bit general microprocessor, the group size should be larger than 32 for efficient computation.
In short, a group size from $32$ to $512$ achieves relatively good trade-off between the weight reconstruction error and the storage compression ratio. 
Accordingly, for \texttt{conv} layers, grouping in channelwise ($\bm{w}_{c,:,:,:}$), kernelwise ($\bm{w}_{c,d,:,:}$), and pointwise ($\bm{w}_{c,:,h,w}$) appears to be appropriate.
Channelwise $\bm{w}_{c,:}$ and subchannelwise $\bm{w}_{c,d:d+n}$ grouping are suited for \texttt{fc} layers.
For example, if each channel is sliced into 2 groups with the same size, we denote it as subchannelwise(2). 
In addition, the most frequently used structures in this chapter are pointwise (\texttt{conv} layers) and (sub)channelwise (\texttt{fc} layers), which align with the bit-packing approach in~\cite{bib:ICLR18:Pedersoli}, and could result in a more efficient deployment.
Since many network architectures choose an integer multiple of 32 as the number of output channels in each layer, pointwise and (sub)channelwise are also efficient for the current storage format in 32-bit microprocessors.

\subsubsection{Maximum Bitwidth $I_\mathrm{max}$}
\label{ch2-sec:experiment_max_bit}

The initial $I_g$ is decided by a predefined initial reconstruction precision or a maximum bitwidth.
We notice that the accuracy degradation caused by the initialization can be fully recovered after several optimization epochs of \algoref{ch2-alg:bases}, if the maximum bitwidth is $8$.
For example, ResNet18 on ImageNet after such an initialization can be retrained to a Top-1/5 accuracy of $70.3\%$/$89.4\%$, even higher than its full precision counterpart ($69.8\%$/$89.1\%$). 
For smaller networks, \eg VGGNet on CIFAR10, a maximum bitwidth of $6$ is already sufficient.

\subsection{Convergence Analysis of Optimization Step}
\label{ch2-sec:experiment_convergence}

In this section, we conduct the ablation studies on our Optimization Step in \secref{ch2-sec:optimization}.
We show the advantages of our optimizer in terms of convergence.
We mainly studied the convergence performance of \algoref{ch2-alg:bases} (\ie optimizing $\bm{B}_g$ with speedup) for two reasons, (\textit{i}) it involves the domain constraints of binarization and takes the majority of computation complexity; (\textit{ii}) it conducts a similar alternative process as prior works~\cite{bib:ICLR18:Xu,bib:ECCV18:Zhang}. 
Recall that our optimizer in \algoref{ch2-alg:bases} (\textit{i}) has no gradient approximation and (\textit{ii}) directly minimizes the loss. 
We developed the following two baselines for comparison.

\begin{itemize}
    \item 
    \textit{STE with rec. error:}
    This baseline quantizes the maintained full precision weights by minimizing the reconstruction error (rather than the loss) during forward and approximates gradients via STE during backward.
    This approach is adopted in some of the best-performing quantization schemes such as LQ-Net \cite{bib:ECCV18:Zhang}.
    \item 
    \textit{STE with loss-aware:}
    This baseline approximates gradients via STE but performs a loss-aware projection updating (adapted from our \alq).
    It can be considered as a multi-bit extension of prior loss-aware quantizers for binary and ternary networks \cite{bib:ICLR17:Hou,bib:ICLR18:Hou}.
    See \secref{ch2-sec:lossaware_ste} below for more details.
\end{itemize}

\subsubsection{The Optimizer of ``STE with Loss-Aware''}
\label{ch2-sec:lossaware_ste}

In this section, we provide the details of the proposed \textit{STE with loss-aware} optimizer.
The training scheme of \textit{STE with loss-aware} is similar to \algoref{ch2-alg:bases}, except that it maintains the full precision weights $\bm{w}_g$.
See the pseudocode of \textit{STE with loss-aware} in \algoref{ch2-alg:ste}.

\begin{algorithm}[t!]
\caption{STE with loss-aware}\label{ch2-alg:ste}
\KwIn{$Q$, $\{\{\bm{\alpha}_{l,g},\bm{B}_{l,g}, I_{l,g}\}_{g=1}^{G_l}\}_{l=1}^L$, training dataset}
\KwOut{$\{\{\bm{\alpha}_{l,g},\bm{B}_{l,g}, I_{l,g}\}_{g=1}^{G_l}\}_{l=1}^L$}
\For {$q \leftarrow 1$ \KwTo $Q$} {
    \For {$l\leftarrow 1$ \KwTo $L$} {
        Update $\bm{\hat{w}}_{l,g}^q = \bm{B}_{l,g}^q\bm{\alpha}_{l,g}^q$\;
        Forward propagate\;
    }
  Compute the loss $\ell^q$ \;
  \For {$l\leftarrow L$ \KwTo $1$} {
    Backward propagate gradient $\partial\ell^q/\partial\bm{\hat{w}}_{l,g}^q$\;
    Directly approximate $\partial\ell^q/\partial\bm{w}_{l,g}^q$ with $\partial\ell^q/\partial {\bm{\hat{w}}_{l,g}}^q$\;
    Update momentums of AMSGrad\; 
    \For {$g \leftarrow 1$ \KwTo $G_l$} {
      Update $\bm{w}_{l,g}^{q+1}$ with \equref{ch2-eq:ste_W}\;
      Compute all values of \equref{ch2-eq:comb}\;
      \For {$j \leftarrow 1$ \KwTo $n_l$} {
        Update $\bm{B}_{l,g,j}^{q+1}$ with \equref{ch2-eq:ste_row}\;
      }
      Update $\bm{\alpha}_{l,g}^{q+1}$ with \equref{ch2-eq:ste_alpha}\;
    }
  }
}
\end{algorithm}

For the layer $l$, the quantized weights $\bm{\hat{w}}_g$ is used during forward propagation.
During backward propagation, the loss gradients to the full precision weights $\partial\ell/\partial\bm{w}_{g}$ are directly approximated with $\partial\ell/\partial {\bm{\hat{w}}_{g}}$, \ie via STE in the $q$-th training iteration as, 
\begin{equation}
    \frac{\partial\ell^q}{\partial\bm{w}_g^q}=\frac{\partial\ell^q}{\partial {\bm{\hat{w}}_g}^q}
\end{equation}
Then the first and second momentums in AMSGrad are updated with $\partial\ell^q/\partial\bm{w}_{g}^q$.
Accordingly, the loss increment around $\bm{w}_g^q$ is modeled as,
\begin{equation}
    f_{\text{ste}}^q=(\bm{g}^q)^{\mathrm{T}}(\bm{w}_g-\bm{w}_g^q)+\frac{1}{2} (\bm{w}_g-\bm{w}_g^q)^{\mathrm{T}} \bm{H}^q (\bm{w}_g-\bm{w}_g^q)
    \label{ch2-eq:ste_B}
\end{equation}
Since $\bm{w}_g$ is full precision, $\bm{w}_g^{q+1}$ can be directly obtained through the above AMSGrad step without projection updating,
\begin{equation}
    \bm{w}_g^{q+1} = \bm{w}_g^q-({\bm{H}^q})^{-1}\bm{g}^q = \bm{w}_g^q-a^q\bm{m}^q/\sqrt{\bm{\hat{v}}^q}
    \label{ch2-eq:ste_W}
\end{equation}
Similarly, the loss increment caused by $\bm{B}_g$ (see \equref{ch2-eq:amsgrad_B1} and \equref{ch2-eq:amsgrad_B2}) is formulated as,
\begin{equation}
    f_{\text{ste},\bm{B}}^q=(\bm{g}^q)^{\mathrm{T}}(\bm{B}_g\bm{\alpha}_g^{q}-\bm{w}_g^q)+\frac{1}{2}(\bm{B}_g\bm{\alpha}_g^{q}-\bm{w}_g^q)^{\mathrm{T}}\bm{H}^q (\bm{B}_g\bm{\alpha}_g^{q}-\bm{w}_g^q)
    \label{ch2-eq:ste_B2}
\end{equation}
Thus, the $j$-th row in $\bm{B}_g^{q+1}$ is updated by,  
\begin{equation}
    \bm{B}_{g,j}^{q+1} = \underset{\bm{B}_{g,j}}{\mathrm{argmin}}~\|\bm{B}_{g,j}\bm{\alpha}_{g}^q-(w_{g,j}^q-g^q_j/H_{jj}^q)\|
    \label{ch2-eq:ste_row}
\end{equation}
In addition, the speedup of \equref{ch2-eq:alpha_lambda} is changed accordingly as,
\begin{equation}
    \bm{\alpha}_{g}^{q+1}=-((\bm{B}_{g}^{q+1})^{\mathrm{T}}\bm{H}^q\bm{B}_{g}^{q+1}+\lambda \mathbf{I})^{-1}\times((\bm{B}_{g}^{q+1})^{\mathrm{T}}(\bm{g}^q-\bm{H}^q\bm{w}^q_{g}))
    \label{ch2-eq:ste_alpha}
\end{equation}
So far, the quantized weights are updated in a loss-aware manner as,
\begin{equation}
    \bm{\hat{w}}_{g}^{q+1} = \bm{B}_{g}^{q+1}\bm{\alpha}_{g}^{q+1}
\end{equation}

\subsubsection{Ablation Results}
\label{ch2-sec:experiment_convergence_results}

\fakeparagraph{Settings}
To show the convergence performance of our Optimization Step, we compare \algoref{ch2-alg:bases} with the above two baselines \textit{STE with rec. error} and \textit{STE with loss-aware} mentioned above. 
The three optimizers are used to train the networks quantized with a uniform bitwidth.
We use AMSGrad\footnote{AMSGrad can also optimize full precision parameters.} as the optimization framework for all optimizers and adopt a learning rate of 0.001.

\begin{figure}[tbp!]
    \centering
    \includegraphics[width=0.95\textwidth]{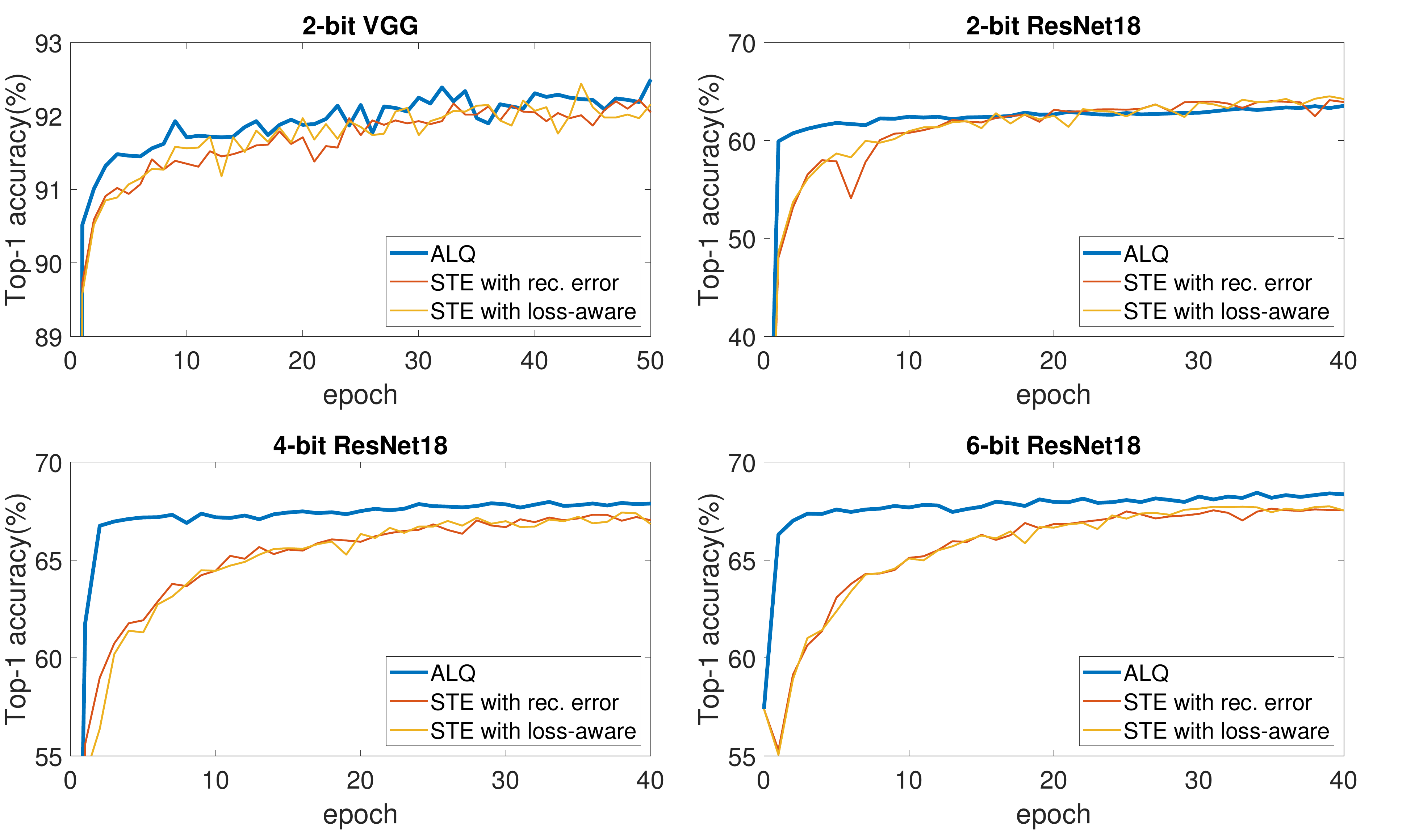}
    \caption[Validation accuracy trained with \alq and other STE-based baselines.]{Validation accuracy trained with \alq and other STE-based baselines along the training epochs.}
    \label{ch2-fig:convergence}
\end{figure}

\fakeparagraph{Results}
\figref{ch2-fig:convergence} shows the Top-1 validation accuracy of different optimizers, with increasing epochs on uniform bitwidth MBNs.
\alq exhibits not only a more stable and faster convergence, but also a higher accuracy.
The exception is 2-bit ResNet18. 
\alq converges faster, but the validation accuracy trained with STE gradually exceeds \alq after about 20 epochs.
For training a large network with $\leq2$ bitwidth, the positive effect brought from the high precision trace may compensate certain negative effects caused by gradient approximation.
In this case, keeping full precision parameters will help calibrate some aggressive steps of quantization, resulting in a slow oscillating convergence to a better local optimum.
This also encourages us to add several epochs of STE based optimization (\eg \textit{STE with loss-aware}) after low bitwidth quantization to further regain the accuracy.

\begin{table}[tbp!]
    \centering
    \caption[Comparison between uniform bitwidth and adaptive bitwidth in \alq.]{Comparison between uniform bitwidth and adaptive bitwidth in \alq.}
    \label{ch2-tab:adapt}
    \small
    \begin{tabular}{ccc}
        \toprule
        Method                              & $I_W$                 & Top-1              \\ \hline
        Baseline VGGNet (uniform)           & 1                     & 91.8\%             \\
        \textbf{\alq VGGNet}                & \textbf{0.66}         & \textbf{92.0}\%    \\
        Baseline ResNet18 (uniform)         & 2                     & 66.2\%             \\
        \textbf{\alq ResNet18}              & \textbf{2.00}         & \textbf{68.9}\%    \\
        \bottomrule
    \end{tabular}
\end{table}

\subsection{Ablation Studies on Adaptive Bitwidth}
\label{ch2-sec:experiment_adaptive}

\fakeparagraph{Settings}
This experiment demonstrates the performance of incrementally trained adaptive bitwidth in \alq, \ie our Pruning Step in \secref{ch2-sec:pruning}.
Uniform bitwidth quantization (an equal bitwidth allocation across all groups in all layers) is taken as the baseline.
The baseline is trained with the same number of epochs as the sum of all epochs during the bitwidth reduction.
Both \alq and the baseline are trained with the same learning rate decay schedule.

\begin{figure}[tbp!]
    \centering
    \includegraphics[width=0.8\textwidth]{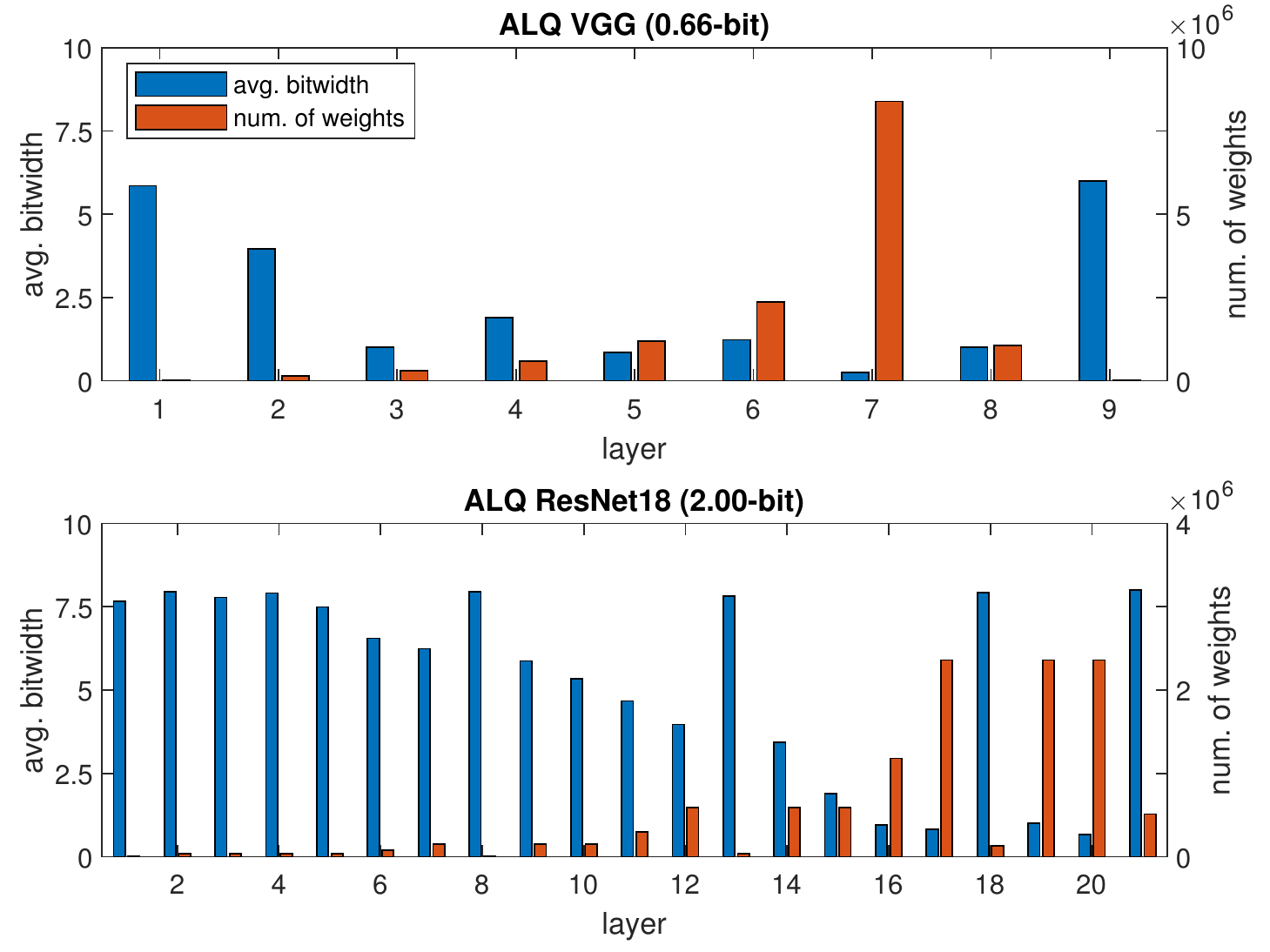}
    \caption[Distribution of the average bitwidth and the number of weights across layers.]{Distribution of the average bitwidth and the number of weights across layers.}
    \label{ch2-fig:adapt}
\end{figure}

\fakeparagraph{Results}
\tabref{ch2-tab:adapt} shows that there is a large Top-1 accuracy gap between an adaptive bitwidth trained with \alq and a uniform bitwidth.
In addition to the overall average bitwidth, we also plot the distribution of the average bitwidth and the number of weights across layers (both models in \tabref{ch2-tab:adapt}) in \figref{ch2-fig:adapt}.
Generally, the first several layers and the last layer are more sensitive to the loss, thus require a higher bitwidth. 
The shortcut layers in ResNet architecture (\eg the $8$-th, $13^{\text{rd}}$, $18$-th layers in ResNet18) also need a higher bitwidth.
We think this is due to the fact that the shortcut pass helps the information forward/backward propagate through the blocks. 
Since the average of adaptive bitwidth can have a decimal part, \alq can achieve a compression ratio with a much higher resolution than a uniform bitwidth, which not only controls a more precise trade-off between storage and accuracy, but also benefits our incremental bitwidth reduction scheme.

It is worth noting that both the Optimization Step and the Pruning Step in \alq follow the same metric, \ie the loss increment modeled by a quadratic function, allowing them to work in synergy. 
We replace the step of optimizing $\bm{B}_g$ in \alq with an STE step (with the reconstruction forward, see in \secref{ch2-sec:experiment_convergence}), and keep other steps unchanged in the pipeline. 
When the VGGNet model is reduced to an average bitwidth of $0.66$-bit, the simple combination of an STE step with our Pruning Step can only reach $90.7\%$ Top-1 accuracy, which is significantly worse than \alq's $92.0\%$. 

\subsection{Comparison with State-of-the-Art Methods}
\label{ch2-sec:experiment_comparison}

\subsubsection{Unstructured Pruning on MNIST}
\label{ch2-sec:experiment_mnist}

\fakeparagraph{Settings}
Since \alq can be considered a structured pruning scheme (\ie pruning in $\bm{\alpha}$ domain), we first compare \alq with two widely used unstructured pruning schemes: Deep Compression (DC) \cite{bib:ICLR16:Han} and ADMM-Pruning (ADMM) \cite{bib:ECCV18:Zhang2}, \ie pruning in the original $\bm{w}$ domain.
For a fair comparison, we implement a modified LeNet5 model as in \cite{bib:ICLR16:Han,bib:ECCV18:Zhang2} on MNIST dataset~\cite{bib:MNIST} and compare the Top-1 prediction accuracy and the compression ratio.

The structures of each layer chosen for \alq are kernelwise, kernelwise, subchannelwise(2), channelwise, respectively. 
After each pruning, the network is retrained to recover the accuracy degradation with 20 epochs of optimizing $\bm{B}_g$ and 10 epochs of optimizing $\bm{\alpha}_g$. 
The pruning ratio is 80\%, and 4 times of Pruning Step are executed after initialization in the reported experiment in \tabref{ch2-tab:lenet5}. 
After the last Pruning Step, we conduct 50 epochs of Optimizing Step to further increase the final accuracy (also applied in the following experiments of VGGNet and ResNet18/34).

\alq can fast converge in the training. 
However, we observed that even after the convergence, the accuracy still continues increasing slowly along the training, which is similar to the behavior of STE-based optimizer. 
During the Optimization Step after each Pruning Step, as long as the training loss is almost converged with a few epochs, we can further proceed the next Pruning Step. 
We found that the final accuracy level is approximately the same whether we add plenty of epochs each time to slowly recover the accuracy to the original level or not.
Thus, we choose a fixed modest number of retraining epochs after each Pruning Step to save the overall training time.
In fact, this benefits from the feature of \alq, which leverages the true gradient w.r.t. the loss to result in a fast and stable convergence.
The final added 50 training epochs aim to further slowly regain the final accuracy level, where we use a gradually decayed learning rate, \eg $10^{-4}$ decays with 0.98 in each epoch.

Note that the storage consumption only counts the weights, since the weights take the most majority of the storage (even after quantization) in comparison to others, \eg bias, activation quantizer, batch normalization, \etc
The storage consumption of weights in \alq includes the look-up-table for the resulting $I_g$ in each group.

\begin{table}[tbp!]
    \centering
    \caption[Comparison with unstructured pruning methods (LeNet5 on MNIST)]{Comparison with state-of-the-art unstructured pruning methods (LeNet5 on MNIST). ``FP'' denotes the full precision baseline. ``CR'' denotes the compression ratio related to full precision. }
    \label{ch2-tab:lenet5}
    \small
    \begin{tabular}{ccc}
        \toprule
        Method                                  & Weights~(CR)                                      & Top-1                 \\ \hline
        FP                                      & 1720KB~(1$\times$ )                               & 99.19\%               \\
        DC~\cite{bib:ICLR16:Han}                & 44.0KB~(39$\times$)                               & \textbf{99.26\%}      \\
        ADMM~\cite{bib:ECCV18:Zhang2}           & 24.2KB~(71$\times$)                               & 99.20\%               \\
        \textbf{\alq}                           & \textbf{22.7KB}~(\textbf{76}$\bm{\times}$)        & 99.12\%      \\ 
        \bottomrule
    \end{tabular}
\end{table}

\fakeparagraph{Results}
\alq shows the highest compression ratio (\textbf{76}$\bm{\times}$) while keeping acceptable Top-1 accuracy compared to the two other pruning methods (see \tabref{ch2-tab:lenet5}). 
FP stands for full precision, and the weights in the original full precision LeNet5 consume $1720$KB \cite{bib:ICLR16:Han}.
CR denotes the compression ratio of static weight storage.

Note that both DC \cite{bib:ICLR16:Han} and ADMM \cite{bib:ECCV18:Zhang2} rely on sparse tensors, which need special libraries or hardwares for efficient execution \cite{bib:ICLR17:Li}.
Their operands (the shared quantized values) are still floating-point.
Hence they hardly utilize bitwise operations for speedup.
In contrast, \alq achieves a higher compression ratio without sparse tensors, which is more suited for general off-the-shelf platforms.

The average bitwidth of \alq is below $1.0$-bit ($1.0$-bit corresponds to a compression ratio slightly below $32$), indicating some groups are fully removed. 
In fact, this process leads to a new network architecture containing less output channels of each layer, and thus the corresponding input channels of the next layers can be safely removed.
The original configuration $20-50-500-10$ is now $18-45-231-10$.

\subsubsection{Binary Networks on CIFAR10}
\label{ch2-sec:experiment_cifar10}
\fakeparagraph{Settings}
In this experiment, we compare the performance of \alq with state-of-the-art binary networks \cite{bib:NIPS15:Courbariaux,bib:ECCV16:Rastegari,bib:ICLR17:Hou}.
A binary network is an MBN with the lowest bitwidth, \ie single-bit. 
Thus, the storage consumption of a binary network can be regarded as the lower bound of a (uniform) quantized network. 
We implement a small version of VGGNet from~\cite{bib:ICLR15:Simonyan} on CIFAR10 dataset~\cite{bib:CIFAR}, as in many state-of-the-art binary networks~\cite{bib:NIPS15:Courbariaux,bib:ICLR17:Hou,bib:ECCV16:Rastegari}.

The structures of each layer chosen for \alq are channelwise, pointwise, pointwise, pointwise, pointwise, pointwise, subchannelwise(16), subchannelwise(2), subchannelwise(2) respectively.
After each pruning, the network is retrained to recover the accuracy degradation with 20 epochs of optimizing $\bm{B}_g$ and 10 epochs of optimizing $\bm{\alpha}_g$. 
The pruning ratio is 40\%, and 5 or 6 times of Pruning Step are executed after initialization in the reported experiment (\tabref{ch2-tab:cifar10}). 

\begin{table}[tbp!]
    \centering
    \caption[Comparison with binary networks (VGGNet on CIFAR10)]{Comparison with state-of-the-art binary networks (VGGNet on CIFAR10). ``FP'' denotes the full precision baseline. ``CR'' denotes the compression ratio related to full precision. $I_w$ denotes the average bitwidth of weights. } 
    \label{ch2-tab:cifar10}
    \small
    \begin{tabular}{cccc}
        \toprule
        Method                            & $I_W$             & Weights~(CR)                          & Top-1             \\  \hline
        FP                                & 32                & 56.09MB~(1$\times$)                   & 92.8\%            \\
        BC~\cite{bib:NIPS15:Courbariaux}  & 1                 & 1.75MB~(32$\times$)                   & 90.1\%            \\
        BWN~\cite{bib:ECCV16:Rastegari}*  & 1                 & 1.82MB~(31$\times$)                   & 90.1\%            \\
        LAB~\cite{bib:ICLR17:Hou}         & 1                 & 1.77MB~(32$\times$)                   & 89.5\%            \\
        AQ~\cite{bib:ICLR18:Khoram}       & 0.27              & 1.60MB~(35$\times$)                   & 90.9\%            \\
        \textbf{\alq}                     & \textbf{0.66}     & \textbf{1.29MB}~(\textbf{43$\times$}) & \textbf{92.0\%}   \\
        \textbf{\alq}                     & \textbf{0.40}     & \textbf{0.82MB}~(\textbf{68$\times$}) & \textbf{90.9\%}   \\ 
        \bottomrule
    \end{tabular}
    \begin{tablenotes}
    \item
    *: both first and last layers are unquantized.
    \end{tablenotes}
\end{table}

\fakeparagraph{Results}
\tabref{ch2-tab:cifar10} shows the performance comparison to popular binary networks.
$I_W$ stands for the quantization bitwidth for weights.
Since \alq has an adaptive quantization bitwidth, the reported bitwidth of \alq is an average bitwidth of all weights.

\alq allows to compress the network to under $1$-bit, which remarkably reduces the storage and computation. 
\alq achieves the smallest weight storage and the highest accuracy compared to all weights binarization methods BC~\cite{bib:NIPS15:Courbariaux}, BWN~\cite{bib:ECCV16:Rastegari}, LAB~\cite{bib:ICLR17:Hou}.
Similar to results on LeNet5, \alq generates a new network architecture with fewer output channels per layer, which further reduces our models in \tabref{ch2-tab:cifar10} to $1.01$MB ($0.66$-bit) or even $0.62$MB ($0.40$-bit).
The computation and the run-time memory can also decrease. 

Furthermore, we also compare with AQ~\cite{bib:ICLR18:Khoram}, the state-of-the-art adaptive fixed-point quantizer.
It assigns a different bitwidth for each parameter based on its sensitivity, and also realizes a pruning for 0-bit parameters. 
Our \alq not only consumes less storage, but also acquires a higher accuracy than AQ~\cite{bib:ICLR18:Khoram}.
Besides, the non-standard quantization bitwidth in AQ cannot efficiently run on general hardware due to the irregularity~\cite{bib:ICLR18:Khoram}, which is not the case for \alq.

In order to demonstrate the affects from different steps in \alq, we plot the training loss curve of quantizing VGGNet on CIFAR10 with \alq. 
Different steps in \alq are marked with different colors, see~\figref{ch2-fig:loss}. 
The results show that (\textit{i}) Initialization Step does not bring any performance drop; (\textit{ii}) Optimization Step can fast converge in a few epochs and may recover the performance drop from Pruning Step as long as the average bitwidth is not extremely low.

\begin{figure}[tbp!]
    \centering
    \includegraphics[width=0.5\textwidth]{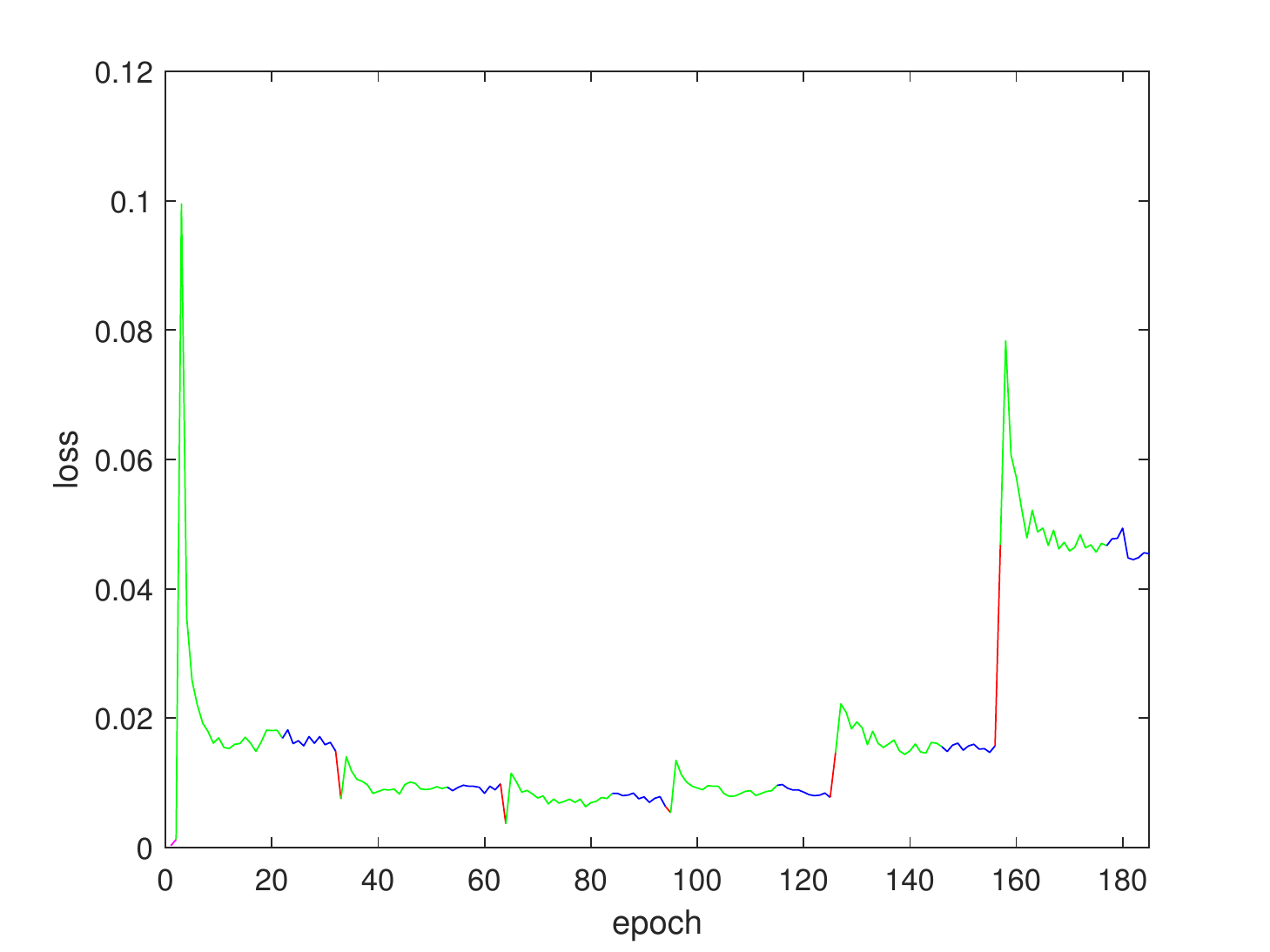}
    \caption[The training loss curves of different steps in \alq.]{The training loss curves of different steps in \alq (VGGNet on CIFAR10). 'Magenta' stands for Initialization Step; 'Green' stands for optimizing $\bm{B}_g$ with speedup; 'Blue' stands for optimizing $\bm{\alpha}_g$; 'Red' stands for Pruning Step. Please see this figure in color.}
    \label{ch2-fig:loss}
\end{figure}

\subsubsection{MBNs on ImageNet}
\label{ch2-sec:experiment_imagenet}

\fakeparagraph{Settings}
We quantize both the weights and the activations of ResNet18/34~\cite{bib:CVPR16:He} with a low bitwidth ($\leq2$-bit) on ImageNet dataset~\cite{bib:ILSVRC15}, and compare our results with state-of-the-art multi-bit networks. 
The results for the full precision version are provided by Pytorch~\cite{bib:NIPSWorkshop17:Paszke}.
We choose ResNet18, as it is a popular model on ImageNet used in the previous quantization schemes. 
ResNet34 is a deeper network used more in recent quantization papers.  

The structures of each layer chosen for \alq are all pointwise except for the first layer (kernelwise) and the last layer (subchannelwise(2)).
After each pruning, the network is retrained to recover the accuracy degradation with 10 epochs of optimizing $\bm{B}_g$ and 5 epochs of optimizing $\bm{\alpha}_g$. 
The pruning ratio is 15\%. 

\begin{table}[tbp!]
    \centering
    \caption[Comparison with quantized networks (ResNet18/34 on ImageNet).] {Comparison with state-of-the-art quantized networks (ResNet18/34 on ImageNet). ``FP'' denotes the full precision baseline. ``CR'' denotes the compression ratio related to full precision. $I_w$ denotes the average bitwidth of weights. $I_A$ denotes the bitwidth of activations. }
    \label{ch2-tab:ResNet}
    \footnotesize
    \begin{tabular}{cccc}
        \toprule
        Method                                      & $I_W$/$I_A$                 & Weights                 & Top-1                     \\ \hline
        \multicolumn{4}{c}{ResNet18}                                                                                                    \\ \hdashline
        FP~\cite{bib:NIPSWorkshop17:Paszke}         & 32/32                       & 46.72MB                 & 69.8\%                    \\
        TWN~\cite{bib:NIPS16:Li}                    & 2/32                        & 2.97MB                  & 61.8\%                    \\
        LR~\cite{bib:ICLR18:Shayer}                 & 2/32                        & 4.84MB                  & 63.5\%                    \\
        LQ~\cite{bib:ECCV18:Zhang}*                 & 2/32                        & 4.91MB                  & 68.0\%                    \\
        QIL~\cite{bib:CVPR19:Jung}*                 & 2/32                        & 4.88MB                  & 68.1\%                    \\
        INQ~\cite{bib:ICLR17:Zhou}                  & 3/32                        & 4.38MB                  & 68.1\%                    \\
        ABC~\cite{bib:NIPS17:Lin}                   & 5/32                        & 7.41MB                  & 68.3\%                    \\
        \textbf{\alq}                               & \textbf{2.00/32}            & \textbf{3.44MB}         & \textbf{68.9\%}           \\
        \textbf{\alq}$^\mathrm{e}$                  & \textbf{2.00/32}            & \textbf{3.44MB}         & \textbf{70.0\%}           \\
        BWN~\cite{bib:ECCV16:Rastegari}*            & 1/32                        & 3.50MB                  & 60.8\%                    \\
        LR~\cite{bib:ICLR18:Shayer}*                & 1/32                        & 3.48MB                  & 59.9\%                    \\
        DSQ~\cite{bib:ICCV19:Gong}*                 & 1/32                        & 3.48MB                  & 63.7\%                    \\
        \textbf{\alq}                               & \textbf{1.01/32}            & \textbf{1.77MB}         & \textbf{65.6\%}           \\
        \textbf{\alq}$^\mathrm{e}$                  & \textbf{1.01/32}            & \textbf{1.77MB}         & \textbf{67.7\%}           \\
        LQ~\cite{bib:ECCV18:Zhang}*                 & 2/2                         & 4.91MB                  & 64.9\%                    \\
        PACT~\cite{bib:arXiv18:Choi}*               & 2/2                         & 4.88MB                  & 64.4\%                    \\
        QIL~\cite{bib:CVPR19:Jung}*                 & 2/2                         & 4.88MB                  & 65.7\%                    \\
        DSQ~\cite{bib:ICCV19:Gong}*                 & 2/2                         & 4.88MB                  & 65.2\%                    \\
        GroupNet~\cite{bib:CVPR19:Zhuang}*          & 4/1                         & 7.67MB                  & 66.3\%                    \\
        RQ~\cite{bib:ICLR19:Louizos}                & 4/4                         & 5.93MB                  & 62.5\%                    \\
        ABC~\cite{bib:NIPS17:Lin}                   & 5/5                         & 7.41MB                  & 65.0\%                    \\
        \textbf{\alq}                               & \textbf{2.00/2}             & \textbf{3.44MB}         & \textbf{66.4\%}           \\
        SYQ~\cite{bib:CVPR18:Faraone}*              & 1/8                         & 3.48MB                  & 62.9\%                    \\
        LQ~\cite{bib:ECCV18:Zhang}*                 & 1/2                         & 3.50MB                  & 62.6\%                    \\
        PACT~\cite{bib:arXiv18:Choi}*               & 1/2                         & 3.48MB                  & 62.9\%                    \\
        \textbf{\alq}                               & \textbf{1.01/2}             & \textbf{1.77MB}         & \textbf{63.2\%}           \\ \hline
        \multicolumn{4}{c}{ResNet34}                                                                                                    \\ \hdashline
        FP~\cite{bib:NIPSWorkshop17:Paszke}         & 32/32                       & 87.12MB                 & 73.3\%                    \\
        \textbf{\alq}$^\mathrm{e}$                  & \textbf{2.00/32}            & \textbf{6.37MB}         & \textbf{73.6\%}           \\
        \textbf{\alq}$^\mathrm{e}$                  & \textbf{1.00/32}            & \textbf{3.29MB}         & \textbf{72.5\%}           \\
        LQ~\cite{bib:ECCV18:Zhang}*                 & 2/2                         & 7.47MB                  & 69.8\%                    \\
        QIL~\cite{bib:CVPR19:Jung}*                 & 2/2                         & 7.40MB                  & 70.6\%                    \\
        DSQ~\cite{bib:ICCV19:Gong}*                 & 2/2                         & 7.40MB                  & 70.0\%                    \\
        GroupNet~\cite{bib:CVPR19:Zhuang}*          & 5/1                         & 12.71MB                 & 70.5\%                    \\
        ABC~\cite{bib:NIPS17:Lin}                   & 5/5                         & 13.80MB                 & 68.4\%                    \\
        \textbf{\alq}                               & \textbf{2.00/2}             & \textbf{6.37MB}         & \textbf{71.0\%}           \\ 
        TBN~\cite{bib:ECCV18:Wan}*                  & 1/2                         & 4.78MB                  & 58.2\%                    \\
        LQ~\cite{bib:ECCV18:Zhang}*                 & 1/2                         & 4.78MB                  & 66.6\%                    \\
        \textbf{\alq}                               & \textbf{1.00/2}             & \textbf{3.29MB}         & \textbf{67.4\%}           \\ 
        \bottomrule
    \end{tabular}
    \begin{tablenotes}
        \item
        *: both first and last layers are unquantized.
        \item
        $^\mathrm{e}$: adding extra epochs of \textit{STE with loss-aware} in the end.
    \end{tablenotes}
\end{table}

\fakeparagraph{Results}
\tabref{ch2-tab:ResNet} shows that \alq obtains the highest accuracy with the smallest network size on ResNet18/34, in comparison with other weight and weight+activation quantization approaches.
$I_W$ and $I_A$ are the quantization bitwidth for weights and activations respectively.

Several schemes (marked with *) are not able to quantize the first and last layers, since quantizing both layers as other layers will cause a huge accuracy degradation \cite{bib:ECCV18:Wan,bib:ICLR18:Mishra2}. 
It is worth noting that the first and last layers with floating-point values occupy $2.09$MB storage in ResNet18/34, which is still a significant storage consumption on such a low-bit network. 
We can simply observe this enormous difference between TWN~\cite{bib:NIPS16:Li} and LQ-Net~\cite{bib:ECCV18:Zhang} in~\tabref{ch2-tab:ResNet} for example. 
The evolved floating-point computations in both layers can hardly be accelerated with bitwise operations either. 

For reported \alq models in~\tabref{ch2-tab:ResNet}, as several layers have already been pruned to an average bitwidth below $1.0$-bit (\eg see in \figref{ch2-fig:adapt}), we add extra 50 epochs of our \textit{STE with loss-aware} in the end as discussed in \secref{ch2-sec:experiment_convergence}.
The learning rate is $10^{-4}$, and gradually decays with $0.98$ per epoch. 
The final accuracy is further boosted by around $1\%\sim2\%$, see the results marked with $^\mathrm{e}$. 
With such an extremely low bitwidth, maintained full precision weights help to calibrate some aggressive steps of quantization, which slowly converges to a local optimum with a higher accuracy for a large network.
Recall that maintaining full precision parameters means STE is required to approximate the gradients, since the true-gradients only relate to the quantized parameters used in the forward propagation.
However, for the quantization bitwidth higher than two ($>2.0$-bit), the quantizer can take smooth steps, and the gradient approximation due to STE damages the training inevitably.
Thus in this case, the true-gradient optimizer, \ie \algoref{ch2-alg:bases}, can converge to a better local optimum, faster and more stable.

\alq can quantize ResNet18/34 with 2.00-bit (across all layers) \textit{without any accuracy loss}.
To the best of our knowledge, this is the first time that the 2-bit weight-quantized ResNet18/34 can achieve the accuracy level of its full precision version, even if some prior schemes keep the first and last layers unquantized.
These results further demonstrate the high-performance of the pipeline in \alq.

\section{Summary}
\label{ch2-sec:summary}

In this chapter, we propose \alq, an adaptive loss-aware trained quantizer for multi-bit networks. 
\alq enables efficient inference on edge devices. 
\alq tries to reduce the redundancy on the quantization bitwidth to achieve both storage efficiency and computation efficiency.
Unlike prior quantized networks that (\textit{i}) often assign an empirical global bitwidth across layers, (\textit{ii}) train the quantizer by minimizing the reconstruction error to the full precision weights, 
\alq (\textit{i}) allocates an adaptive bitwidth to different weights w.r.t. the loss, (\textit{ii}) optimizes the multi-bit quantizer by minimizing the loss as well. 
The adaptive bitwidth assignment and the direct optimization objective allow \alq to find and remove more redundant bitwidth, thus achieving a better trade-off between the resource constraints and the model accuracy.
The main contributions are summarized as follows, 
\begin{itemize}
    \item 
    \alq introduces a multi-bit network with adaptive quantization bitwidth across different groups of weights. 
    Such an adaptive multi-bit network not only achieves a high compression ratio on static weight storage by only assigning a high bitwidth to loss-critical weights, 
    but also replaces the expensive floating-point operations with a single set of cheaper operations from \texttt{xnor}, \texttt{popcount} and accumulations.
    
    \item
    \alq trains the multi-bit quantized weights by directly minimizing the loss function. 
    This loss-aware quantization results in a faster convergence rate as well as a higher final accuracy than state-of-the-art STE-based quantization training that minimizes the reconstruction error.
    
    \item
    Via entirely pruned groups (\ie 0-bit weights in some groups), \alq enables extremely low-bit networks with an average bitwidth below 1-bit yet with \textit{dense tensor form}.
    It breaks the traditional lower bound of the quantized network, \ie binary network, thus providing more visions and possibilities for the network compression. 
    Experiments on CIFAR10 show that \alq can compress VGGNet to an average bitwidth of $0.4$-bit, while yielding a higher accuracy than other binary networks~\cite{bib:ECCV16:Rastegari,bib:NIPS15:Courbariaux}.
    
    \item
    \alq is the first loss-aware quantization scheme for multi-bit networks and eliminates the need for approximating gradients and retaining full precision weights.
    \alq is also able to quantize the first and last layers without incurring a notable accuracy loss.
    
\end{itemize}

This chapter studied how to compress the network for efficient inference given the fixed on-device resource constraints. 
In the next chapter, we will further study how to adapt the network on edge devices when the resource constraint is varied along the lifetime. 
Although we may deploy multiple \alq-quantized multi-bit networks with different average bitwidth to execute under different resource budgets, this naive solution can only result in a subpar performance, as it requires several times more storage consumption in comparison to a single (multi-bit) network. 
However, the solution proposed in the next chapter can meet the varying resource constraints without incurring extra storage overhead. 


\newcommand\blfootnote[1]{%
  \begingroup
  \renewcommand\thefootnote{}\footnote{#1}%
  \addtocounter{footnote}{-1}%
  \endgroup
}

\chapter[Adaptation on Edge Devices]{Adaptation on Edge Devices}
\label{ch3:adaptation}

In \chref{ch2:inference}, we explored how DNNs can be compressed while respecting resource constraints. 
However, the resource constraints on the target edge devices may change dynamically during runtime. 
To maximize model accuracy during on-device inference, in this chapter we deploy a DNN that can adapt to the different resource constraints on the edge device.

\fakeparagraph{Main Resource Constraints}
The different resource constraints during on-device inference may be due to for example the available battery power or the allowed inference time.
Similar to \chref{ch2:inference}, we mainly adopt two widely used proxies to quantify the (varying) resource consumption, (\textit{i}) \textit{the storage of weights}, which affects the amount of memory fetching and static memory consumption, and (\textit{ii}) \textit{the number of operations for inference}, which is relevant to the computing energy and the inference latency. 

\fakeparagraph{Principles}
Faced with the varying resource constraints on edge devices, existing synthesis methods require either deploying multiple individual networks with different resource demands or sampling sub-networks along structured dimensions, which leads to poor performance. 
However, we propose to sample sub-networks from the backbone network through row-based unstructured sparsity, and propose a novel compressed sparse row (CSR) format for efficient sparse inference. 
Our synthesis methods reduce redundancy among multiple sub-networks through weight sharing and architecture sharing, resulting in storage efficiency and re-configuration efficiency.

The contents of this chapter are established mainly based on the paper ``DRESS: Dynamic REal-time Sparse Subnets'' that is published on Efficient Deep Learning for Computer Vision CVPRWorkshop (ECV), 2022 \cite{bib:CVPRWorkshop22:Qu}.\footnote{This work was done when Zhongnan Qu was a research intern at Meta, and it was collaborated with the colleagues at Meta Reality Labs Research. }

\section{Introduction}
\label{ch3-sec:introduction}

Extensive synthesis works \cite{bib:ICLR16:Han,bib:ECCV16:Rastegari,bib:ACMTrans21:Sunny,bib:ACMTrans21:Oh} have proposed to first compress a pretrained model according to the given resource constraints, and then compile the compressed model to deploy on target edge devices. 
However, the time constraints of many practical embedded systems may dynamically change at run-time.
For example, when detecting hand positions on a workbench in real-time, the allowed inference time varies during the entire manipulation. 
In comparison to general movement, engineers will slow down the hand movement if performing some critical tasks, \eg grasping objects, which gives DNNs a longer execution time when requiring higher perceptive precision.
Some similar scenarios also include autonomous vehicles' reaction time on city roads and highways due to different operating speeds. 
On the other hand, the available resources on the target edge device may also vary along the lifetime, \eg the battery energy, the allocatable RAM.
All considerations mentioned above indicate that the deployed inference model should maintain a dynamic capacity, such that the model can be adapted and executed under different resource constraints.

\fakeparagraph{Challenges}
Making DNNs adaptable on resource-constrained devices is even more challenging.
Existing synthesis methods either fail to compile DNNs that can adapt to varying resource constraints, or result in subpar performance. 
Traditional compression techniques, \eg pruning, quantization, only result in a static inference model. 
Although the compressed model is mapped onto target devices, it can not meet various resource requirements. 
As an alternative, we may compile for example multiple networks with different sparsity levels, which however need several times more storage consumption in comparison to a single sparse network.
Recent works \cite{bib:ICLR19:Yu,bib:ICLR20:Cai} show that sub-networks from a pretrained backbone network can reach a decent performance compared to the sub-networks trained individually from scratch.
Nevertheless, they only sample sub-network architectures along hand-crafted structured dimensions, \eg width, kernel size, which leads to sub-optimal results.
Switching among multiple compiled architectures on edge devices may also cause extra re-configuration overhead.

In this chapter, we propose a novel synthesis technique, \textbf{D}ynamic \textbf{RE}al-time \textbf{S}parse \textbf{S}ubnets (DRESS).
\dress samples sub-networks from the backbone network through row-based unstructured sparsity, while ensuring that nonzero weights of the higher sparsity networks are reused by the lower sparsity networks.
This way, the overall memory consumption is bounded by the network with the lowest sparsity and does not depend on the number of networks, resulting in \textit{memory efficiency}; all sparse sub-networks leverage the same architecture as the backbone network, leading to \textit{re-configuration efficiency}.
The sub-network with a higher sparsity (\ie fewer nonzero weights) needs a smaller amount of on-device memory fetching and fewer FLOPs, thus shall be adopted to inference under more severe resource constraints, \eg lower energy budget, limited inference time.
Specifically, we (\textit{i}) sample weights w.r.t. their magnitudes in a row-based unstructured manner; (\textit{ii}) train all sampled sparse sub-networks with weighted loss in parallel; 
(\textit{iii}) further fine-tune batch normalization for each sub-network individually.

\section{Related Work}
\label{ch3-sec:related}

\subsection{Network Compression \& Deployment}
\label{ch3-sec:related_compression}

Network compression focuses on trimming down the DNN model size with negligible performance degradation. 
Commonly used compression techniques can be divided into three categories, (\textit{i}) designing efficient network architectures manually \cite{bib:arXiv17:Howard,bib:CVPR18:Sandler} or automatically using neural architecture search \cite{bib:ICLR20:Cai,bib:arXiv19:Yu,bib:ECCV20:Yu,bib:ACMTrans21:Mendis}; (\textit{ii}) quantizing weight values into lower bitwidth to use cheaper operations and reduce the storage consumption \cite{bib:ECCV16:Rastegari,bib:ACMTrans19:Yu,bib:ACMTrans21:Sunny}; (\textit{iii}) structured \cite{bib:ECCV20:Li,bib:IEEETrans20:Li,bib:IEEETrans20:Wu}/unstructured \cite{bib:ICLR16:Han,bib:ICLR20:Renda,bib:ICML21:Evci,bib:NIPS21:Peste,bib:ACMTrans21:Oh,bib:IEEETrans20:Ahmad} pruning unimportant weights as zeros to reduce the number of operations and the number of nonzero weights.
The compressed model is further optimized by some compilation libraries in order to speed up inference on target edge platforms, \eg CMSIS-NN for Arm Cortex-M CPUs \cite{bib:CMSIS-NN}, XNNPACK for Arm64 and ArmV7 CPUs \cite{bib:XNNPACK}, Vela for Ethos-U NPU \cite{bib:Vela}.
Note that the compiled model often only supports a static computation graph due to the limited resources on edge devices \cite{bib:CMSIS-NN,bib:XNNPACK,bib:Vela}.
In this chapter, we focus on unstructured pruning among others, since (\textit{i}) it often yields a high compression ratio \cite{bib:ICLR20:Renda}; (\textit{ii}) the networks with different unstructured sparsity may share the same network architecture, \ie the same compiled computation graph.
Furthermore, some recent libraries \eg XNNPACK include fast kernels for sparse matrix-dense matrix multiplication, which enables sparse DNN acceleration on edge platforms \cite{bib:CVPR20:Elsen,bib:IEEETrans20:Li2}.

\subsection{Dynamic Networks}
\label{ch3-sec:related_dynamic}

Dynamic networks aim at a better trade-off between inference accuracy and average inference efficiency, by adapting network structures or network parameters according to the inputs during inference \cite{bib:PAMI21:Han}. 
Among them, some works propose allocating less computation on those canonical data samples, through skipping layers \cite{bib:ICLR18:Huang}, pruning unimportant channels \cite{bib:CVPR21:Li,bib:IEEETrans20:Wu}, selecting a subset of salience pixels \cite{bib:CVPR20:Verelst}.
Although these sample-wise dynamic networks may achieve a smaller inference cost averaged over different samples, they cannot adapt the model to fit different resource budgets. 
In addition, to achieve data-dependent adaptiveness, they often bring additional computation burden, \eg hard attention, gater, etc. \cite{bib:PAMI21:Han}

\subsection{Anytime Networks (Sub-networks)}
\label{ch3-sec:related_anytime}

Anytime networks refer to the network whose sub-networks can be executed separately with less resource consumption while achieving a satisfactory performance. 
\dress falls into the same scope of anytime networks.
MSDNet \cite{bib:ICLR18:Huang} densely connects multiple convolutional layers in both depth direction and scale direction, such that the computation can be saved by early-exiting from a certain layer.
\cite{bib:arXiv17:Hu} introduces an adaptive weighted loss to optimize the network with various depths.
Slimmable networks \cite{bib:ICLR19:Yu,bib:ICCV19:Yu} propose to train a single model which supports multiple width multipliers (\ie number of channels) in each layer.
\cite{bib:ICLR20:Cai} suggests to search network architectures with different kernel size, depth, and width, in a single pretrained once-for-all network.
Subflow \cite{bib:RTAS20:Lee} executes only a sub-graph of the full DNN by activating partial neurons given the varying time constraints.
State-of-the-art anytime networks always sample sub-networks from the backbone network along hand-crafted structured dimensions, \eg depth, width, kernel size, neuron.
As zero weights have no effects on the calculation, anytime networks actually perform structured pruning on the backbone network, which could yield a subpar performance in comparison to unstructured sampling.
In addition, resulted sub-networks often have different network architectures, \eg different kernel sizes.
When adopting these sub-networks on edge devices, the re-configuration of the computation graph may bring extra overhead.
On the other hand, SP-Net \cite{bib:arXiv20:Guerra} suggests adjusting the quantization bitwidth on demand, which however requires specialized integer arithmetic units for efficient computing.

\subsection{Weight Sharing}
\label{ch3-sec:related_sharing}

Sub-networks rely on weight reusing (sharing). 
Except for sub-networks, weight sharing among different networks is also widely used in other settings. 
Multi-task learning \cite{bib:NIPS18:Sener} reuses partial weights of networks performing diverse tasks to reduce memory consumption. However, these methods are inapplicable in our scenarios, which target a single task with varying resource constraints.  
Neural architecture search (NAS) applied in \cite{bib:ICLR20:Cai,bib:arXiv19:Yu,bib:ECCV20:Yu,bib:ICCV21:Chu} maintains a single set of shared weights (also known as \textit{supernet}) when searching different architectures to reduce the training effort. 
Note that NAS is orthogonal to our method since the searched optimal architecture can be used as our backbone network.

\section{Dynamic Real-time Sparse Subnets}
\label{ch3-sec:method}

\subsection{Problem Definition}
\label{ch3-sec:problem}

We aim at sampling multiple subnets from a backbone network.
The backbone network is a conventional DNN consisting of $L$ convolutional (\texttt{conv}) layers or fully connected (\texttt{fc}) layers.
These subnets have different resource demands and thus can be adapted to different resource availabilities. 
Since the subnets have the same architecture as the backbone network, they can share a single compiled architecture to achieve re-configuration efficiency;
the nonzero weights of the subnet with a higher sparsity are reused by the subnet with a lower sparsity to achieve memory efficiency.
This way, we only need to store a table for the lowest sparsity network, including its nonzero weights sorted w.r.t. importance and corresponding indices.
Accordingly, the other networks can be build from the top important weights through a pre-defined sparsity together with the compiled architecture.
Assume that we sample $K$ sparse subnets, then the preliminary problem is defined as,
\begin{eqnarray}
    \min_{\bm{w},\bm{m}_k}      & \ell(\bm{w}\odot\bm{m}_k)             & \forall k \in 1,...,K       \label{ch3-eq:loss} \\
    \text{s.t.}                 & \|\bm{m}_k\|_0 = (1-s_k) \cdot I      & \forall k \in 1,...,K       \label{ch3-eq:constraints1} \\
                                & \bm{m}_i \odot \bm{m}_j = \bm{m}_j    & \forall 1 \le i<j \le K   \label{ch3-eq:constraints2} 
\end{eqnarray}
where $\bm{w}$ stands for the weights of the (dense) backbone network; $\bm{m}_k$ stands for the binary mask of the $k$-th subnet; $s_k$ stands for the pre-defined sparsity level. 
$\ell(.)$ denotes the loss function, $\|.\|_0$ denotes the L0-norm, $\odot$ denotes the element-wise multiplication. 
Note that $\bm{w}\in\mathbb{R}^I$, $\bm{m}_k\in\{0,1\}^I$, where $I$ is the total number of weights. 
Clearly, we have $0<s_1<s_2<...<s_K<1$, \ie the first subnet bounds the overall static storage consumption. 
$\bm{w}_k$ is denoted as nonzero weights of the $k$-th sparse subnet, \ie $\bm{w}_k=\bm{w}\odot\bm{m}_k$.

\begin{algorithm}[!ht]
\caption{Dynamic REal-time Sparse Subnets}\label{ch3-alg:DRESS}
\KwIn{Initial random weights $\bm{w}$, training dataset $\mathcal{D}_\text{tr}$, validation dataset $\mathcal{D}_\text{val}$, overall sparsity $\{s_k\}_{k=1}^K$, normalized loss weights $\{\pi_k\}_{k=1}^K$} 
\KwOut{Optimized weights $\bm{w}$, binary masks $\{\bm{m}_k\}_{k=1}^K$}
\tcc{Dense pre-training}
Train dense network $\bm{w}$ with traditional optimizer\;
\tcc{DRESS training}
Allocate layer-wise sparsity $\{s_{k,l}\}_{l=1}^L$ for each $s_k$\;
Initiate $\bm{w}^0=\bm{w}$\;
\For {$q \leftarrow 1$ \KwTo $Q$} { 
    \tcp{The $q$-th training iteration}
    Fetch mini-batch from $\mathcal{D}_\text{tr}$\;
    Initialize backbone-net gradient $\bm{g}(\bm{w}^{q-1})=\bm{0}$\;
    \For {$k \leftarrow 1$ \KwTo $K$} {
        Sample a subnet with sparsity $\{s_{k,l}\}_{l=1}^L$ and get its mask $\bm{m}_k$\;
        Get sparse subnet $\bm{w}^{q-1}_k=\bm{w}^{q-1}\odot\bm{m}_k$\;
        Back-propagate subnet gradient $\bm{g}(\bm{w}^{q-1}_k)=\pi_k\cdot\frac{\partial\ell(\bm{w}^{q-1}_k)}{\partial\bm{w}^{q-1}_k}$\;
        Accumulate backbone-net gradient $\bm{g}(\bm{w}^{q-1})=\bm{g}(\bm{w}^{q-1})+\bm{g}(\bm{w}^{q-1}_k)\odot\bm{m}_k$\;
    }
    Compute optimization step $\Delta\bm{w}^{q}$ with $\bm{g}(\bm{w}^{q-1})$\;
    Update $\bm{w}^{q} = \bm{w}^{q-1} + \Delta\bm{w}^{q}$\;
    \uIf{Higher average epoch accuracy on $\mathcal{D}_\text{val}$}
    {Save $\bm{w}=\bm{w}^q$ and $\{\bm{m}_k\}_{k=1}^K$\;}
    \Else{Re-allocate layer-wise sparsity $\{s_{k,l}\}_{l=1}^L$ for each $s_k$\;}
}
\tcc{Post-training on batch normalization (BN)}
\For {$k \leftarrow 1$ \KwTo $K$} {
    Load $\bm{w}$ and $\bm{m}_k$\;
    Fine-tune BN layers of subnet $\bm{w}\odot\bm{m}_k$\;
}
\end{algorithm}

In the following sections, we detail how to solve \equref{ch3-eq:loss}-\equref{ch3-eq:constraints2} in our \dress synthesis approach. 
\dress consists of three training stages as discussed below. 
The overall pipeline is shown in \algoref{ch3-alg:DRESS}.

\begin{itemize}
    \item 
    \underline{Dense Pre-Training:} The backbone network is trained from scratch with a traditional optimizer to provide a good initialization for the following sparse training. 
    \item 
    \underline{\dress Training:} The multiple sparse subnets are sampled from the backbone network (\secref{ch3-sec:sample}, \secref{ch3-sec:store}) and are jointly trained in parallel with weighted loss (\secref{ch3-sec:optimize}). 
    \item 
    \underline{Post-Training on Batch Normalization:} Batch normalization (BN) layers are further optimized individually for each subnet to better reveal the statistical information (\secref{ch3-sec:boost}).  
\end{itemize}

\subsection{How to Sample Sparse Subnets}
\label{ch3-sec:sample}

\begin{figure}[!t]
	\centering
	\includegraphics[width=0.99\textwidth]{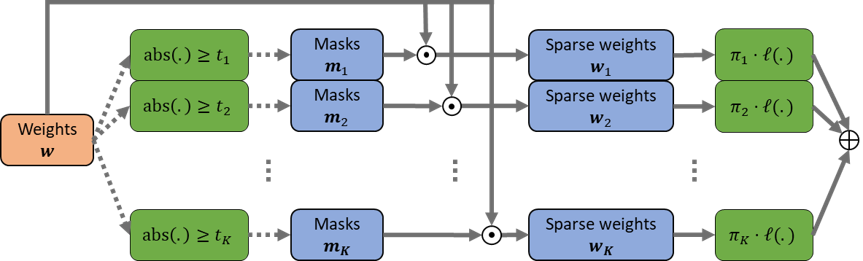}
	\caption[The computation graph of \dress.]{The computation graph used in parallel training multiple subnets. 
	The orange block stand for the leaf variable to be optimized; the blue block stand for the intermediate variable; the green block stand for the computation unit.}
	\label{ch3-fig:approach}
\end{figure}

\begin{figure}[!t]
	\centering
	\includegraphics[width=0.8\textwidth]{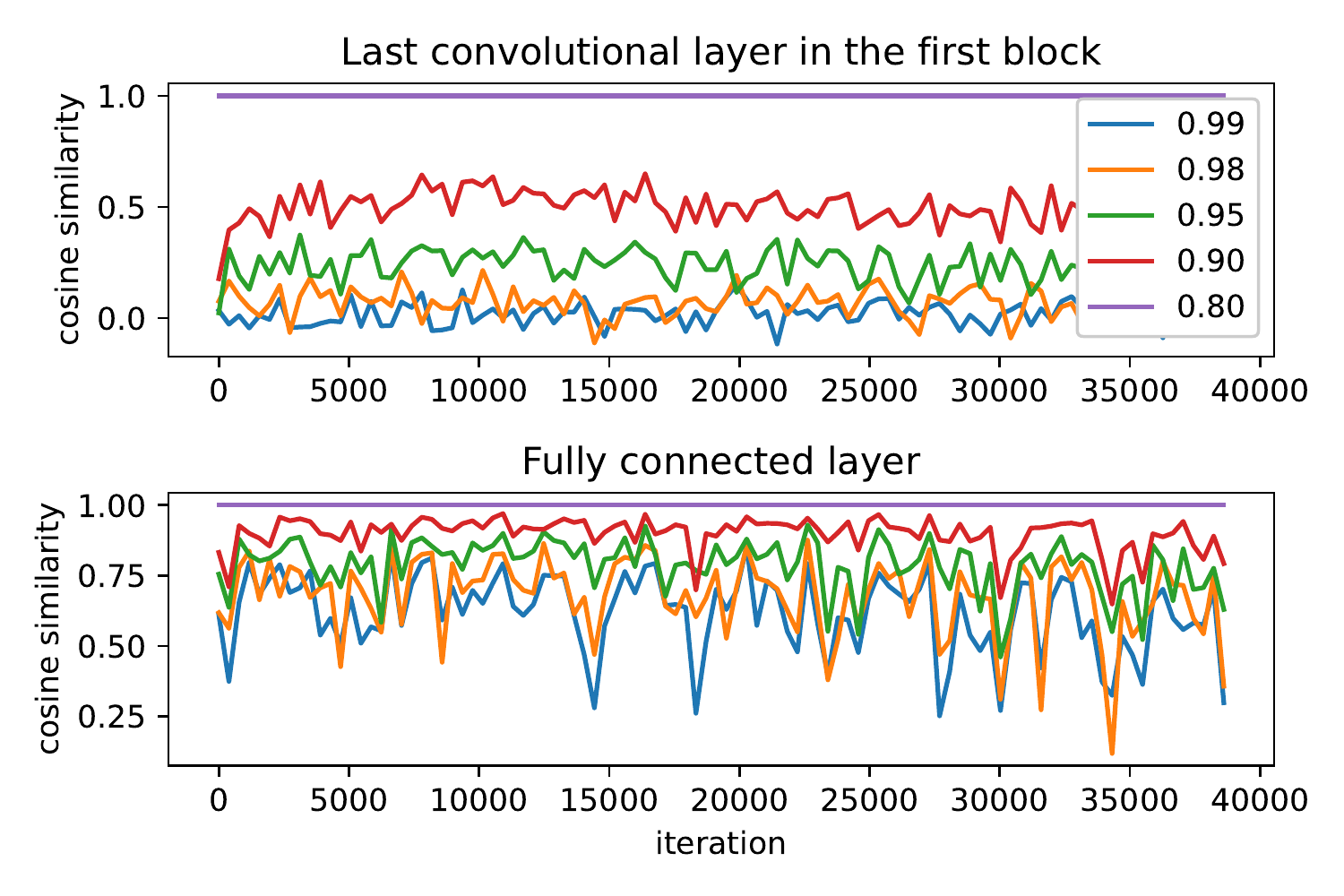}
	\caption[The cosine similarity between the loss gradients of different subnets.]{The cosine similarity between the loss gradients of 5 subnets (with sparsity 0.8,0.9,0.95,0.98,0.99) and that of the lowest sparsity subnet (with sparsity 0.8) along the training iterations.
    We show two typical layers in ResNet20, the last \texttt{conv} layer of the first block and the \texttt{fc} layer}
	\label{ch3-fig:cosine}
\end{figure}

Unlike traditional anytime networks that sample subnets along structured dimensions, \dress samples subnets weight-wise which extremely enlarges the sampling space.
Recall that we introduce $K$ binary masks $\bm{m}_{1:K}$ to indicate if the weight is selected in each subnet. 
The naive approach could be iterative sampling $K$ subnets where each iteration exhaustively searches the best-performed subnet inside the current subnet. 
Yet this naive approach can be either conducted rarely or infeasible due to the high complexity.

To reduce the complexity, we propose to greedily sample the subnet based on the importance of weights.
Following the prior pruning works \cite{bib:ICLR16:Han,bib:ICLR19:Frankle,bib:ICLR20:Renda}, the importance is measured by the weight magnitudes. 
Given an overall sparsity level $s_k$, the $(1-s_k)\cdot I$ weights with the largest magnitudes will be sampled and are used to build the subnet.
However, it is still infeasible to conduct such a global sorting across all layers in each training iteration.
Instead, the weights are only sorted and sampled inside each layer according to a layer-wise sparsity $s_{k,l}$, where $l$ denotes the layer index.
The global sorting, also the (re-)allocation of layer-wise sparsity, is conducted only if the average accuracy of subnets does not improve anymore, see in \algoref{ch3-alg:DRESS}.
During (re-)allocation, the weights of all \texttt{conv} and \texttt{fc} layers with the largest magnitudes will be selected in sequence until reaching the overall sparsity $s_k$, and the layer-wise sparsity can be then calculated accordingly. 

Note that the (dense) backbone network is maintained and continuously updated when training sampled subnets. 
In comparison to traditional pruning, where only the nonzero weights at fixed locations are fine-tuned, our sparse subnets are re-sampled from the backbone network in each training iteration.
This flexible mechanism is crucial to acquire multiple high-performed subnets, see the ablation results in \secref{ch3-sec:ablation_sampling}. 

\subsection{How to Optimize Subnets}
\label{ch3-sec:optimize}

With sampled binary masks, we can now build and train the subnets. 
Our concept for optimizing subnets is based on the key insight: \textit{in comparison to iterative training of subnets in progressively decreased/increased sparsity, parallel training allows multiple subnets to be sampled and optimized jointly thus yields higher performance.}

Experimental results in \secref{ch3-sec:ablation_iterative} show that parallel training multiple subnets always yields a higher prediction accuracy than iterative training. 
As a possible explanation, the optimizer may be stuck into a bad local optimum around the previous subnet during iterative training, whereas parallel training searches multiple subnets jointly. 
We thus adopt parallel training in \dress as in \algoref{ch3-alg:DRESS}.

In parallel training, \equref{ch3-eq:loss} can be re-written as,  
\begin{equation}
    \min_{\bm{w},\bm{m}_k}~~~\sum_{k=1}^K \pi_k \cdot \ell(\bm{w}\odot\bm{m}_k) 
    \label{ch3-eq:loss_parallel}
\end{equation}
where $\pi_k$ is the normalized scale ($\sum_{k=1}^K \pi_k = 1$) used to weight $K$ loss items, which will be discussed later. 
In fact, this process determines a threshold $t_k$ for the $k$-th sparsity level, the mask value $m_{k,i}=1$ if $\text{abs(}w_i\text{)}\ge t_k$, otherwise 0, $\forall i\in1,...,I$.
$t_k$ is set to the value such that $(1-s_k)$ of weights have a larger absolute value than $t_k$.
Clearly, we have $t_1<t_2<...<t_K$ due to the constraints of \equref{ch3-eq:constraints1}-\equref{ch3-eq:constraints2}.
In each training iteration, we sample $K$ sparse subnets $\bm{w}_{1:K}$ from the backbone network $\bm{w}$.
Each subnet's loss function is weighted by $\pi_k$ and summed together.
This weighted sum is to be minimized and thus is used to compute the gradients of $\bm{w}$, see the optimization graph in \figref{ch3-fig:approach}.

When parallel training multiple subnets, the gradients of the backbone network are accumulated by the (weighted) loss gradients back-propagated through all $K$ sparse subnets, as,
\begin{equation}
    \bm{g}(\bm{w}) = \sum_{k=1}^K \pi_k \frac{\partial\ell(\bm{w}_k)}{\partial\bm{w}_k} \odot \bm{m}_k
\end{equation}

We parallelly train 5 subnets of ResNet20 \cite{bib:CVPR16:He} with sparsity $s_{1:5}=0.8,0.9,0.95,0.98,0.99$ on CIFAR10, and let the 5 loss items weighted equally, \ie $\pi_{1:5}=0.2$. 
We plot the cosine similarity between the loss gradients of 5 subnets (\ie $(\partial\ell(\bm{w}_k)/\partial\bm{w}_k) \odot \bm{m}_k$ with $k=1,...,5$) and that of the lowest sparsity subnet (\ie $(\partial\ell(\bm{w}_1)/\partial\bm{w}_1) \odot \bm{m}_1$) along with the training iterations, see in \figref{ch3-fig:cosine}.
It shows that the loss gradients of different subnets are always positively correlated with each other.
The results also verify that multiple subnets are jointly trained towards the optimal point in the loss landscape.
Because of \equref{ch3-eq:constraints2}, the nonzero weights in higher sparsity subnets (\eg $\bm{w}_5$) are also selected by other subnets, which means these weights are optimized with a larger step size than other weights.  
To balance the step size, the subnet with a higher sparsity (fewer trainable weights) shall be assigned to a smaller weight $\pi_k$ on its loss. 
We propose to weight the loss items by the ratio of trainable weights (\ie $1-s_k$) together with a correction factor $\gamma$.

\begin{equation}
    \alpha_k = (1-s_k)^\gamma
    \label{ch3-eq:gamma}
\end{equation}
\begin{equation}
    \pi_k = \alpha_k/\sum_{k=1}^K\alpha_k
    \label{ch3-eq:alpha}
\end{equation}
The normalized weights $\pi_k$ provide control over the significance of subnets in parallel training.
$\gamma=0$ means weighting loss items equally. 
Experimentally, we find that $\gamma\in[0.5,1]$ often yields a satisfactory performance, see in \secref{ch3-sec:ablation_gamma}.

\subsection{How to Store Subnets}
\label{ch3-sec:store}

\begin{figure}[!t]
	\centering
	\includegraphics[width=0.9\textwidth]{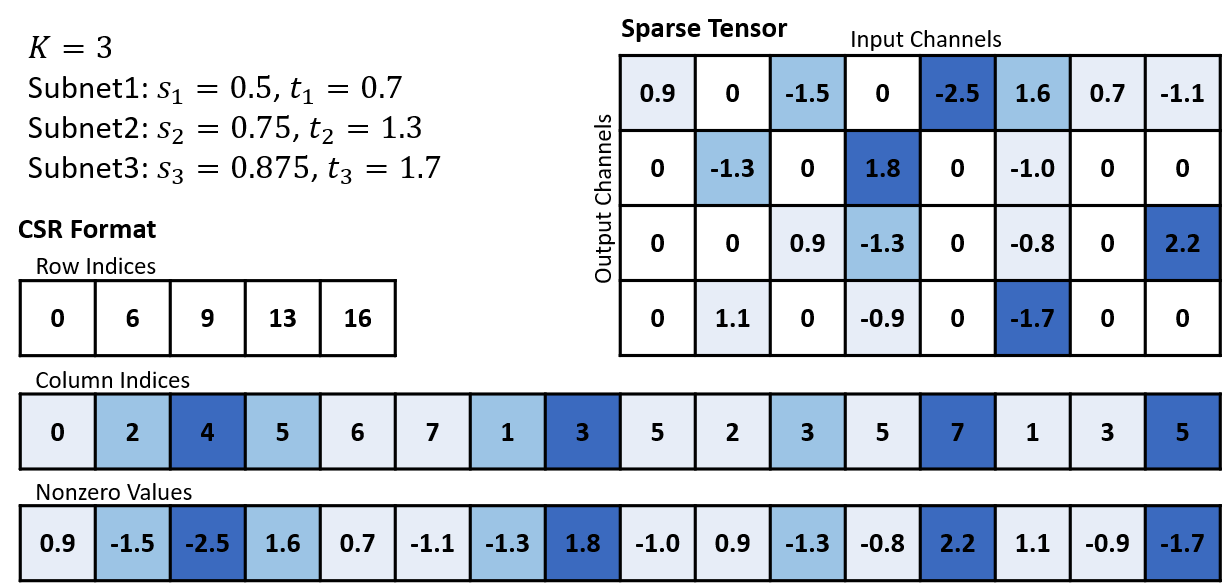}
	\caption[Traditional CSR format of unstructured sparse tensor.]{Traditional CSR format of unstructured sparse tensor. 
	The example weight tensor is from a $1\times1$ \texttt{conv} layer with 8 input channels and 4 output channels. 
	The sparse tensor has a row dimension along each output channel, \ie each \texttt{conv} filter.
	There are 3 sparse subnets with sparsity 0.5, 0.75, and 0.875. 
	Each subnet corresponds to a threshold value of 0.7, 1.3, and 1.7, respectively.
	}
	\label{ch3-fig:csr1}
\end{figure}

To realize efficient storage and computation, current compilation libraries often encode sparse tensors in compressed sparse row (CSR) format (or some similar formats \eg Block-CSR) \cite{bib:XNNPACK,bib:CVPR20:Elsen,bib:IEEETrans20:Li2}.
An example CSR format of sparse tensor for a \texttt{conv} layer is depicted in \figref{ch3-fig:csr1}.
When adopting traditional CSR format to store subnets generated by \dress, we need to store (\textit{i}) the subnet with the lowest sparsity including the row indices, the column indices, and the nonzero values, (\textit{ii}) $K$ threshold values $t_{1:K}$.
However, when selecting the $k$-th subnet for inference, all nonzero weights need to be fetched and compared with $t_k$.
Although we may build specialized indexing for every subnet individually, it in turn results in more memory cost depending on the number of subnets. 

To achieve an efficient inference on different sparse subnets while without extra memory overhead, we adopt a \textit{row-based unstructured} sparsity (a.k.a. \textit{N:M fine-grained structure sparsity} \cite{bib:ICLR21:Zhou,bib:NIPS21:Hubara,bib:NIPS21:Sun}), where different rows leverage the same sparsity level.
We denote $N$ as the row size, also the number of weights in each row.
Especially, for sparsity $s_k$, all rows have exactly $(1-s_k)\cdot N$ nonzero weights.
In comparison to conventional unstructured sparsity, this kind of sparsity can also be accelerated with sparse tensor cores of Nvidia A100 GPUs \cite{bib:arXiv21:Mishra} for both training and inference, and thus becomes prevailing recently. 
\textit{To our best knowledge, this is the first work that builds multiple sub-networks via fine-grained structure of weight sharing.}
The column indices are stored according to the descending order of the importance (also weight magnitudes) in a two-dimensional table.
The nonzero weights are stored in another table with the same order as the column indices.
This \dress CSR format needs to store (\textit{i}) the subnet with the lowest sparsity including the table of the column indices and the table of nonzero weights, (\textit{ii}) $K$ integers $\{(1-s_k)\cdot N\}_{k=1}^K$. 
It has overall a similar memory cost as traditional CSR format.
When adopting the $k$-th subnet, we fetch the first $(1-s_k)\cdot N$ columns from both tables as shown in \figref{ch3-fig:csr2}.
The row indices can be built with $(1-s_k)\cdot N$.
Note that all fetched subnets follow the CSR format and utilize the same compiled network architecture, which allows us to leverage available libraries to achieve a fast inference without re-configuration overhead.

\begin{figure}[!t]
	\centering
	\includegraphics[width=0.99\textwidth]{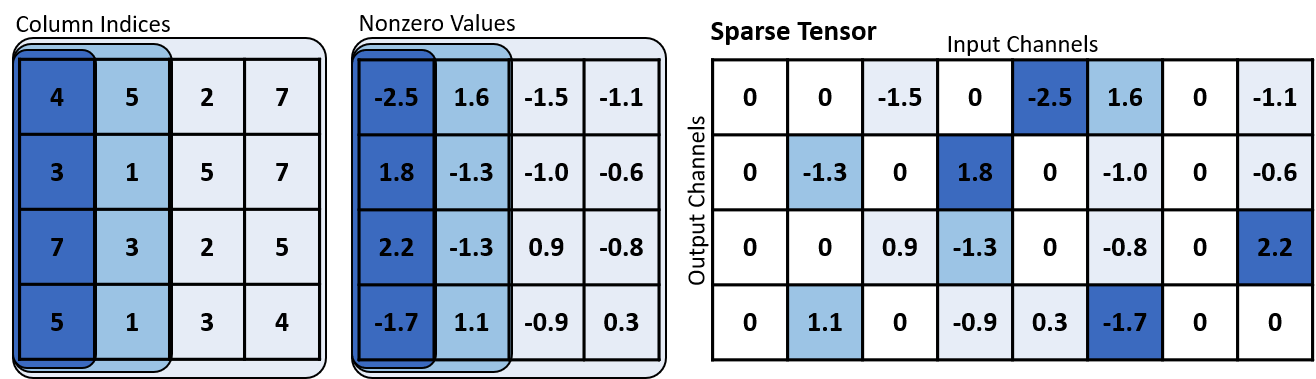}
	\caption[\dress CSR format of row-based unstructured sparse tensor.]{\dress CSR format of row-based unstructured sparse tensor. 
	The example weight tensor has the same size as \figref{ch3-fig:csr1}. 
	Each row has 8 weights in total, also the row size $N=8$.
	There are 3 sparse subnets with sparsity 0.5, 0.75, and 0.875, as \figref{ch3-fig:csr1}. 
	Each subnet has 4, 2, and 1 nonzero weights per row, respectively.}
	\label{ch3-fig:csr2}
\end{figure}

To obtain \dress CSR format, the sampling process needs to be adjusted accordingly.  
Especially, for layer $l$, we first pre-define a row size $N_l$ and reshape the weight tensor into rows. 
For example, each \texttt{conv} filter corresponds to one row in accordance with the compilation libraries \cite{bib:XNNPACK,bib:CVPR20:Elsen}. 
When sampling a subnet in layer $l$, we conduct unstructured sampling in each row individually, while each row has the same sparsity $s_{k,l}$ as in \figref{ch3-fig:csr2}.
Note that when $N_l$ equals the total number of weights, \ie a single row in \dress CSR, it turns back into the original unstructured sampling discussed in \secref{ch3-sec:sample}.
In this case, unstructured sampling is conducted in the entire tensor.
Although the resulted sparse tensor can still be stored in the traditional CSR format as \figref{ch3-fig:csr1}, it can not perform an efficient inference due to extra comparison computation discussed above.

The algorithm for generating the mask with row-based unstructured sampling is presented in \algoref{ch3-alg:rowsampling}.
We focus on sampling in a weight tensor $\bm{w}$ with a predefined row size $N$.
Note that $N$ must be divisible by the total number of weights in $\bm{w}$.
The weight tensor $\bm{w}$ is then reshaped into the form of $\mathbb{R}^{G \times N}$, \ie $N$ weights per row and $G$ rows in total.
Given $K$ sparsity levels $s_{1:K}$, $K$ binary masks with the form of $\{0,1\}^{G \times N}$ are generated.
Binary masks can be reshaped into the original form of the weight tensor accordingly. 

\begin{algorithm}[!t]
    \caption{Row-based unstructured sampling}\label{ch3-alg:rowsampling}
    \KwIn{Weight tensor $\bm{w}\in\mathbb{R}^{G \times N}$, row size $N$, sparsity $\{s_k\}_{k=1}^K$}
    \KwOut{Binary masks $\{\bm{m}_k\}_{k=1}^K$}
    \For {$k \leftarrow 1$ \KwTo $K$} {
        Initiate binary mask $\bm{m}_k=0^{G \times N}$\;
        Get the number of nonzero weights per row $N_k^\text{nz}=N\cdot (1-s_k)$\;
    }
    \For {$g \leftarrow 1$ \KwTo $G$} { 
        Sort the weight magnitudes of row $\bm{w}_{g,:}$ in descending order\;
        \For {$k \leftarrow 1$ \KwTo $K$} {
            Set the mask values of $\bm{m}_{k,g,:}$ as 1 for Top-$N_k^\text{nz}$ indices\;
        }
    }
\end{algorithm}

\subsection{How to Further Boost Subnets}
\label{ch3-sec:boost}

Batch normalization (BN) layers are critical for the stable training of state-of-the-art DNNs. 
Previous synthesis works \cite{bib:ICLR19:Yu,bib:ICCV19:Yu} find that subnets with different width may cause an accumulated error on batch statistics, and propose to switch BN layers for different subnets.
Multiple subnets in \dress share a single architecture, and thus are capable of being optimized in synergy with a shared BN layer.
However, post-training BN layers for each subnet can better calibrate the running statistics, which in turn increases the accuracy.
As BN layers often only require a rather smaller amount of memory and computation in comparison to \texttt{conv}/\texttt{fc} layers, we propose to further fine-tune BN layers for each subnet individually after parallel training, as shown in \algoref{ch3-alg:DRESS}.

\section{Evaluation}
\label{ch3-sec:experiments}

With our design-flow mentioned above, we are now synthesizing the algorithm to map onto resource-constrained edge platforms.
To better understand the effectiveness of our algorithm, we first evaluate our algorithm on widely used vision benchmarks in this section.
Then, we compile and deploy the generated subnets on an edge platform in \secref{ch3-sec:deployment} to see the actual performance of the entire synthesis.

\subsection{Benchmarking Details}
\label{ch3-sec:experiment_benchmark}

We implement our algorithm with Pytorch \cite{bib:NIPSWorkshop17:Paszke}, and evaluate on image classification and object detection/instance segmentation tasks.
As prior works \cite{bib:ICLR18:Huang,bib:ICLR19:Yu,bib:ICCV19:Yu,bib:ICLR20:Renda,bib:NIPS21:Peste,bib:ICLR21:Zhou,bib:NIPS21:Sun}, for image classification, we benchmark VGGNet \cite{bib:ICLR15:Simonyan} and ResNet20 \cite{bib:CVPR16:He} on CIFAR10 \cite{bib:CIFAR}, and benchmark ResNet50 \cite{bib:CVPR16:He} and MobileNetV1/V2 \cite{bib:arXiv17:Howard,bib:CVPR18:Sandler} on ImageNet \cite{bib:ILSVRC15}; for object detection, we benchmark Faster-RCNN with ResNet50-FPN on COCO \cite{bib:COCO}; for instance segmentation, we benchmark Mask-RCNN with ResNet50-FPN on COCO \cite{bib:COCO}.
We use Nesterov SGD optimizer with the cosine schedule for learning rate decay. 
We report the Top-1 test accuracy for the subnets of the epoch when the validation dataset achieves the highest average accuracy over all subnets.
For all pre-processing and random initialization, we apply the tools provided in Pytorch.

In our experiments, the row-based unstructured sampling is conducted in all \texttt{conv}/\texttt{fc} layers, except for the depthwise \texttt{conv} layers in MobileNetV1/V2.
We found that sparse depthwise \texttt{conv} layers lead to substantially lower accuracy. 
As depthwise \texttt{conv} layers only consume a rather small amount of memory and computation \cite{bib:ICML21:Evci,bib:ACMTrans21:Cho}, different subnets share the same dense depthwise \texttt{conv} layers in \dress.  
In addition, we keep BN layers dense as in \cite{bib:ICLR19:Yu,bib:ICCV19:Yu}. 
We set the overall sparsity levels $s_{1:5}=0.95,0.98,0.99,0.995,0.998$ for VGGNet, $s_{1:5}=0.8,0.9,0.95,0.98,0.99$ for ResNet20, $s_{1:4}=0.5,0.8,0.9,0.95$ for ResNet50 and MobileNetV1/V2.
The sparsity levels are averaged over all \texttt{conv}/\texttt{fc} layers. 

\subsubsection{VGGNet/ResNet20 on CIFAR10}
\label{ch3-sec:experiment_cifar}

CIFAR10 \cite{bib:CIFAR} is an image classification dataset, which consists of $32\times32$ color images in 10 object classes. 
We use the original training dataset with 50000 samples for training, and randomly select 2000 samples in the original test dataset (10000 samples in total) for validation, and the rest 8000 samples for testing.
We train on 1 Nvidia V100 GPU with a batch size of 128.

\fakeparagraph{VGGNet}
The used VGGNet is widely adopted in many previous compression works \cite{bib:NIPS15:Courbariaux,bib:ICLR17:Hou,bib:ECCV16:Rastegari}, which is a modified version of the original VGG~\cite{bib:ICLR15:Simonyan}.
The used VGGNet architecture is presented as, 2$\times$128C3 - MP2 - 2$\times$256C3 - MP2 - 2$\times$512C3 - MP2 - 2$\times$1024FC - 10SVM/100SVM.
The initial learning rate is set as 0.1; the momentum is set as 0.9; the weight decay is set as 0.0005; the number of training epochs is set as 100.
Note that we use the same training hyperparameters for all three stages in \algoref{ch3-alg:DRESS}.
This also holds true for the following experiments.

\fakeparagraph{ResNet20}
The network architecture is the same as ResNet-20 in the original paper~\cite{bib:CVPR16:He}.
The initial learning rate is set as 0.1; the momentum is set as 0.9; the weight decay is set as 0.0005; the number of training epochs is set as 100.

\subsubsection{ResNet50/MobileNetV1/MobileNetV2 on ImageNet}
\label{ch3-sec:experiment_imagenet}

ImageNet \cite{bib:ILSVRC15} is a large-scale image classification dataset, which consists of high-resolution color images in 1000 object classes. 
We use the original training dataset with 1.28 million samples for training, and randomly select 10000 samples in the original validation dataset (50000 samples in total) for validation, and the rest 40000 samples for testing.
We train on 4 Nvidia V100 GPUs with a batch size of 1024.

\fakeparagraph{ResNet50}
We use pytorch-style ResNet50, which is slightly different than the original Resnet-50~\cite{bib:CVPR16:He}.
The down-sampling (stride=2) is conducted in $3\times3$ \texttt{conv} layer instead of $1\times1$ \texttt{conv} layer.
The network architecture is the same as ``resnet50'' in~\cite{bib:torchResNet}.
The initial learning rate is set as 0.5; the momentum is set as 0.9; the weight decay is set as 0.0001; the number of training epochs is set as 100.

\fakeparagraph{MobileNetV1}
The network architecture is the same as $1.0\times$ MobileNet-224 in the original paper~\cite{bib:arXiv17:Howard}.
The initial learning rate is set as 0.5; the momentum is set as 0.9; the weight decay is set as 0.00001; the number of training epochs is set as 150.

\fakeparagraph{MobileNetV2}
The network architecture is the same as $1.0\times$ MobileNetV2 in the original paper~\cite{bib:CVPR18:Sandler}.
The initial learning rate is set as 0.1; the momentum is set as 0.9; the weight decay is set as 0.00004; the number of training epochs is set as 300.

\subsubsection{ResNet50-FPN on COCO}
\label{ch3-sec:supplementary_coco}

MS COCO \cite{bib:COCO} is object detection, segmentation, key-point detection, and captioning dataset.
We use COCO 2017 dataset, which consists of high-resolution annotated images in 80 object classes.
It contains a training dataset with 118000 annotated samples, and a validation dataset with 5000 data samples. 
We focus on object detection and instance segmentation.
We report the standard COCO metrics, average precision (AP), which is averaged over Intersection-over-Union (IoU) thresholds $\in0.5:0.05:0.95$.
The bounding box level AP and the mask level AP are adopted in object detection and the instance segmentation, respectively.
We follow the official reference training scripts provided by Pytorch \cite{bib:torchdetection} to set up our experiments. 
We distributed train on 8 Nvidia V100 GPUs with a batch size of 16 (2 per GPU). 
The final AP is reported on the validation dataset after the entire training. 

\fakeparagraph{ResNet50-FPN}
We adopt Faster-RCNN \cite{bib:NIPS15:Ren} in object detection and Mask-RCNN \cite{bib:ICCV17:He} in instance segmentation.
The overall network architecture consists of two parts, the basic network\footnote{To avoid confusion, we use \textit{the basic model} to refer to \textit{the backbone network} mentioned in the Faster-RCNN and Mask-RCNN papers \cite{bib:NIPS15:Ren,bib:ICCV17:He}.
The backbone network only stands for the original dense network in \dress in this chapter.} and the head architecture.
We use ResNet50 pretrained on ImageNet dataset as the basic network.
As suggested by \cite{bib:ICCV17:He}, the feature extractor, feature pyramid network (FPN) \cite{bib:CVPR17:Lin}, is connected to ResNet50 in lateral.
The bounding-box head and the mask head will then use the extracted feature to detect objects and segment instances.
Especially, our network architectures are the same as the ones provided in the Pytorch reference training scripts \cite{bib:torchdetection}. 
As the batch size used in Faster-RCNN training and Mask-RCNN training is relatively small, we freeze the BN layers of ResNet50 as in \cite{bib:ICCV17:He,bib:detectron2018}. 

Following \cite{bib:ICLR19:Yu}, we first pretrain ResNet50 with \algoref{ch3-alg:DRESS} on ImageNet, \ie obtain 4 subnets of ResNet50 with sparsity $0.5,0.8,0.9,0.95$ as in \figref{ch3-fig:benchmark}.
The lateral FPN and the head architecture are added into ResNet50. 
We then train the overall network on COCO dataset with \algoref{ch3-alg:DRESS} while fixing BN layers for each subnet.

\subsection{Ablation Studies}
\label{ch3-sec:ablation}

We first implement a set of ablation experiments to study the effect of different components/parameters in \dress.
The ablation experiments are mainly conducted with ResNet20 on CIFAR10 and MobileNetV1 on ImageNet.

\subsubsection{Row Size $N$}
\label{ch3-sec:ablation_rowsize}

\fakeparagraph{Settings}
In \secref{ch3-sec:store}, we restrict different rows in a CSR sparse tensor to have the same number of nonzero weights, and subnets are sampled in a row-based unstructured manner.
To study the impact of row size $N$, we select three methodical ways to reshape the weight tensor for row-based sampling, 
(\textit{i}) unstructured, where unstructured sampling is conducted in the entire weight tensor, \ie each layer only contains a single row in \dress CSR format as discussed in \secref{ch3-sec:store}; 
(\textit{ii}) filterwise, where unstructured sampling is conducted in each filter for \texttt{conv} layers or in each output-neuron for \texttt{fc} layers;
(\textit{iii}) 256, where each row contains 256 elements in a \texttt{conv} filter, or the entire filter if the filter has less than 256 weights.
$\gamma$ is set as 1 in these experiments.
The row size $N$ used in CSR format is tightly related to the memory cost to store the column indices of nonzero weights, \eg 256 means 8-bit for each column index. 

\begin{figure}[tbp]
	\centering
	\includegraphics[width=0.8\textwidth]{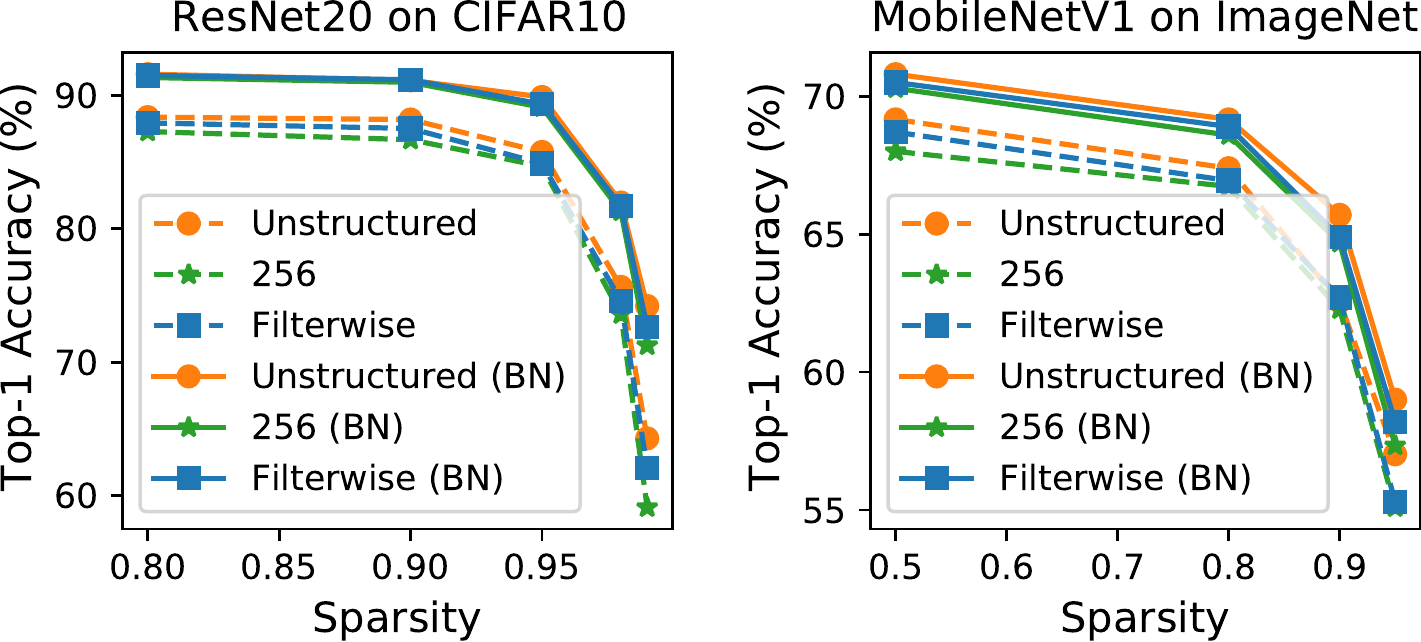}
	\caption[Ablation studies on different row sizes.]{Ablation studies on different row sizes. ``BN'' means further fine-tuning BN layers for each subnet.}
	\label{ch3-fig:rowsize}
\end{figure}

\fakeparagraph{Results}
The results in \figref{ch3-fig:rowsize} show that when choosing a relatively large row size \eg filterwise or 256, our proposed row-based unstructured sampling can yield a similar accuracy as totally unstructured sampling in the entire tensor.
Especially, for both ResNet20 and MobileNetV1, the accuracy difference between ``Unstructured'' and ``Filterwise'' is less than 0.5\% on average. 
In the following experiments, we mainly adopt filterwise unstructured sampling due to its high accuracy and efficient \dress CSR format. 
The dashed curves and the solid curves with the same marker in \figref{ch3-fig:rowsize} can be viewed as the ablation study of the third training stage, \ie further fine-tuning BN layers for each subnet (see \secref{ch3-sec:boost}). 
Particularly, fine-tuning BN layers can calibrate the statistic discrepancy between different subnets and thus consistently improves the performance of subnets.

\begin{figure}[tbp]
	\centering
	\includegraphics[width=0.7\textwidth]{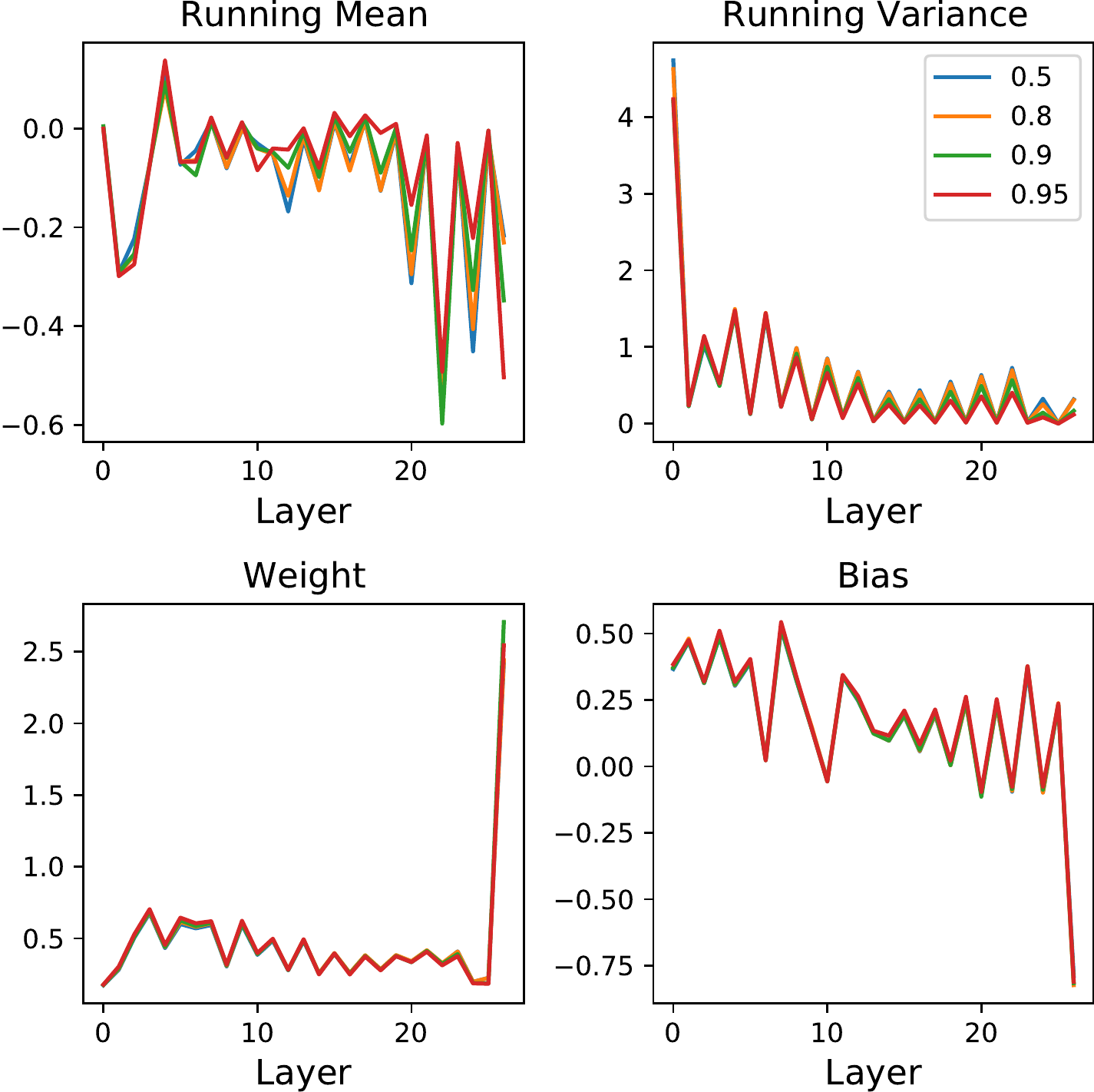}
	\caption[The BN statistics of different subnets across layers.]{The BN statistics of different subnets across layers. The subnets of MobileNetV1 with different sparsity levels are plotted with different colors.}
	\label{ch3-fig:bn}
\end{figure}

The BN statistic information of 4 subnets of MobileNetV1 is shown in \figref{ch3-fig:bn}.
For each layer in subnets, we plot the average value of ``running mean'', ``running variance'', ``weight'', and ``bias'' over all channels after the third training stage.
The results show that the BN statistic information is closed among different subnets, which allows multiple subnets to be optimized in synergy in the second training stage (\dress training) of \algoref{ch3-alg:DRESS}.
On the other hand, the third training stage can calibrate the small discrepancy between different subnets, which in turn improves the accuracy of each subnet.

\subsubsection{With/Without Sampling}
\label{ch3-sec:ablation_sampling}

\fakeparagraph{Settings}
To explore the efficacy of our sampling process, we compare \dress with sampling (\ie \algoref{ch3-alg:DRESS}) and \dress without sampling. 
\dress without sampling has a similar process as traditional unstructured magnitude pruning \cite{bib:ICLR16:Han,bib:ICLR20:Renda}, where $K$ binary masks are built after the dense pre-training and then are fixed, and only the nonzero weights (with mask value equals 1) are sparsely fine-tuned.
In other words, the subnets will not be re-sampled in the \dress training stage of \algoref{ch3-alg:DRESS}.
We set $\gamma$ as 1 and use filterwise unstructured sampling in these experiments.

\fakeparagraph{Results}
As shown in \figref{ch3-fig:sampling}, our (re-)sampling process improves the accuracy of subnets by a large margin than without sampling, especially under a high sparsity, \eg increasing by around 7.4\% on average (up to 16.1\%) on MobileNetV1. 
Re-sampling provides more flexibility to re-select the sparse subnets that are abandoned before the parallel training.

\begin{figure}[!t]
	\centering
	\includegraphics[width=0.8\textwidth]{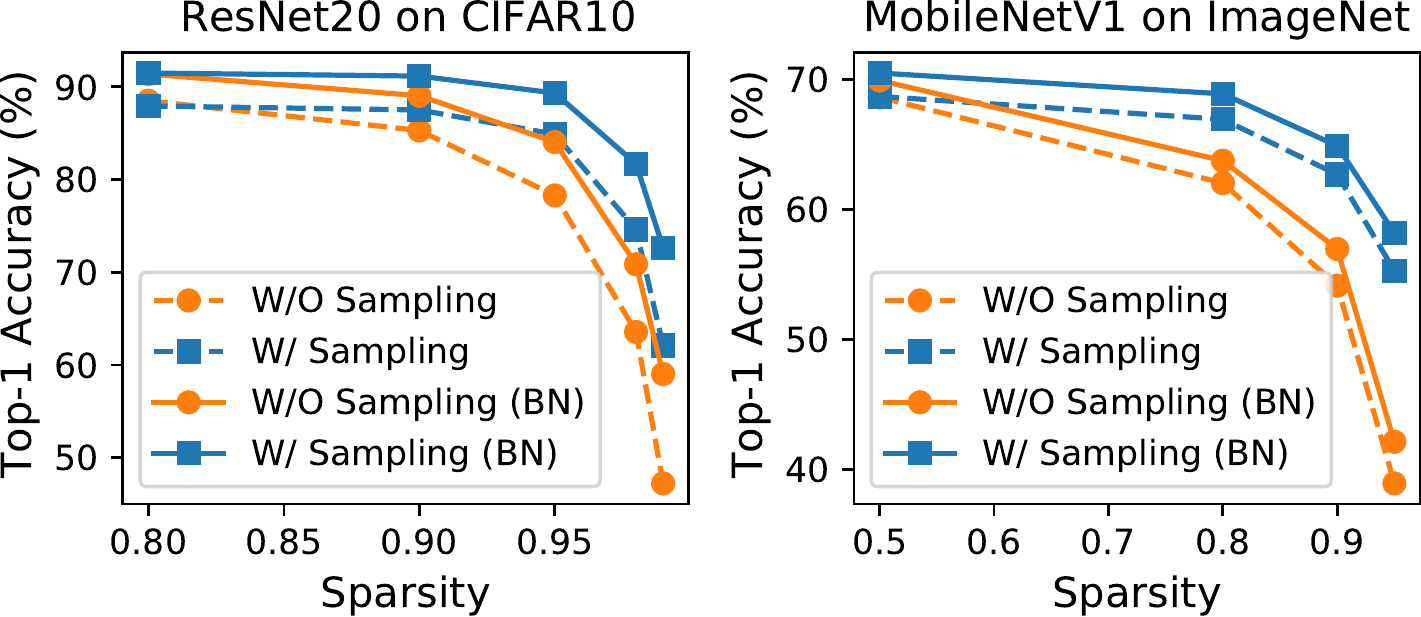}
	\caption[Comparison between \dress with sampling and \dress without sampling.]{Comparison between \dress with sampling and \dress without sampling.}
	\label{ch3-fig:sampling}
\end{figure}

\subsubsection{Iterative vs. Parallel}
\label{ch3-sec:ablation_iterative}

In this part, we compare \dress (parallel training) with iterative training multiple subnets.
We first elaborate the iterative training methods with progressively increased/decreased sparsity mentioned in \secref{ch3-sec:optimize}. 

Recall that there are $K$ sparsity levels, and $0<s_1<s_2<...<s_K<1$.
In iterative training, each subnet is optimized separately, also altogether $K$ iterations. 
In each iteration, we mainly adopt the idea of traditional unstructured pruning \cite{bib:ICLR20:Renda}, which is the current best-performed pruning method aiming at the trade-off between the model accuracy and the number of zero’s weights.
\cite{bib:ICLR20:Renda} conducts iterative pruning with a pruning scheduler $p^{1:R}$. 
The network progressively reaches the desired sparsity $s$ until the $R$-th pruning iteration.
We choose $p^{1:5}=0.5,0.8,0.9,0.95,1$. 
Accordingly, the sparsity is set to $0.5s,0.8s,0.9s,0.95s,s$ in 5 pruning iterations, respectively.
During each pruning iteration, the network is pruned with the corresponding sparsity, and the remaining nonzero weights are sparsely fine-tuned with learning rate rewinding.

\begin{algorithm}[tbp!]
    \caption{Iterative training with increased sparsity}\label{ch3-alg:increased}
    \KwIn{Initial random weights $\bm{w}$, training dataset $\mathcal{D}_\text{tr}$, validation dataset $\mathcal{D}_\text{val}$, sparsity $\{s_k\}_{k=1}^K$, pruning scheduler $\{p^r\}_{r=1}^{R}$} 
    \KwOut{Optimized weights $\bm{w}$, binary masks $\{\bm{m}_k\}_{k=1}^K$}
    \tcc{Dense pre-training}
    Train dense network $\bm{w}$ with traditional optimizer\;
    \tcc{Traditional pruning, also k=1}
    \For {$r \leftarrow 1$ \KwTo $R$} { 
        \tcp{The $r$-th pruning iteration}
        Prune with sparsity $s_1\cdot p^r$ and get mask $\bm{m}_1^r$\;
        Sparsely fine-tune nonzero weights $\bm{w}\odot\bm{m}_1^r$ on $\mathcal{D}_\text{tr}$\;
    }
    Get mask $\bm{m}_1 = \bm{m}_1^R$\;
    \tcc{Iterative (training)}
    \For {$k \leftarrow 2$ \KwTo $K$} { 
        Get the previous subnet $\bm{w}_{k-1} = \bm{w}\odot\bm{m}_{k-1}$\; 
        Sample a subnet from $\bm{w}_{k-1}$ with sparsity $s_k$ and get mask $\bm{m}_k$\;
        \tcp{Note that no training here.}
    }
\end{algorithm}

The pseudocode of training subnets iteratively with increased sparsity is shown in \algoref{ch3-alg:increased}.
With progressively increased sparsity (from $s_1$ to $s_K$), the first optimized subnet of $\bm{w}\odot\bm{m}_1$ already contains all subsequent subnets with higher sparsity due to \equref{ch3-eq:constraints2}.
The first sparse subnet $\bm{w}\odot\bm{m}_1$ is trained by unstructured pruning \cite{bib:ICLR20:Renda} as discussed above.
In the following iteration $k$ ($k\in2,...,K$), the subnet with sparsity $s_k$ is directly sampled from the previous subnet without any retraining regarding the constraint of \equref{ch3-eq:constraints2}.

The pseudocode of training multiple subnets iteratively with decreased sparsity is shown in \algoref{ch3-alg:decreased}.
For progressively decreased sparsity (from $s_K$ to $s_1$), the sampling and training process only happen in the complementary part of the previous subnet due to \equref{ch3-eq:constraints2}. 
Particularly, in iteration $k$ ($k\in K,...,1$), we should (\textit{i}) sample the new subnet from the backbone network with sparsity $s_k$ that contains the subnet of $\bm{w}\odot\bm{m}_{k+1}$; (\textit{ii}) freeze the subnet of $\bm{w}\odot\bm{m}_{k+1}$ and only update the other weights.
We still adopt the iterative pruning when training each subnet, \ie the sparsity of the $k$-th subnet gradually approaches the target sparsity $s_k$.
Note that the dense backbone network is maintained and updated during the training.
Note that the (re-)sampling process is only conduct in each pruning iteration instead of each training iteration.

\begin{algorithm}[tbp!]
    \caption{Iterative training with decreased sparsity}\label{ch3-alg:decreased}
    \KwIn{Initial random weights $\bm{w}$, training dataset $\mathcal{D}_\text{tr}$, validation dataset $\mathcal{D}_\text{val}$, sparsity $\{s_k\}_{k=1}^K$, pruning scheduler $\{p^r\}_{r=1}^{R}$}
    \KwOut{Optimized weights $\bm{w}$, binary masks $\{\bm{m}_k\}_{k=1}^K$}
    \tcc{Dense pre-training}
    Train dense network $\bm{w}$ with traditional optimizer\;
    \tcc{Iterative training}
    Set $s_{K+1}=1$ and $\bm{m}_{K+1}=\bm{0}$\;
    \For {$k \leftarrow K$ \KwTo $1$} { 
        Get the complementary subnet $\bm{w}^\text{cs}=\bm{w}\odot(1-\bm{m}_{k+1})$\;
        \For {$r \leftarrow 1$ \KwTo $R$} { 
            \tcp{The $r$-th pruning iteration}
            Sample a subnet from $\bm{w}^\text{cs}$ with sparsity $(1-(s_{k+1}-s_k))\cdot p^r$ and get mask $\bm{m}_k^{\text{cs},r}$\;
            Merge mask $\bm{m}_k^{r}=\bm{m}_{k+1}+\bm{m}_k^{\text{cs},r}$\;
            Initiate $\bm{w}^0=\bm{w}$\;
            \For {$q \leftarrow 1$ \KwTo $Q$} { 
                \tcp{The $q$-th training iteration}
                Fetch mini-batch from $\mathcal{D}_\text{tr}$\;
                Get sparse subnet $\bm{w}^{r,q-1}_k=\bm{w}^{q-1}\odot\bm{m}_k^{r}$\;
                Back-propagate subnet gradient $\bm{g}(\bm{w}^{r,q-1}_k)=\frac{\partial\ell(\bm{w}^{r,q-1}_k)}{\partial\bm{w}^{r,q-1}_k}$\;
                Compute optimization step $\Delta\bm{w}^{q}$ with $\bm{g}(\bm{w}^{r,q-1}_k)\odot\bm{m}_k^{\text{cs},r}$\;
                Update $\bm{w}^{q} = \bm{w}^{q-1} + \Delta\bm{w}^{q}$\;
            }
            Save $\bm{w}=\bm{w}^Q$
        }
        Save mask $\bm{m}_k=\bm{m}_k^R$\;
    }
\end{algorithm}

\fakeparagraph{Settings}
We implement the two iterative training methods of \algoref{ch3-alg:increased} and \algoref{ch3-alg:decreased}.
The loss of each subnet is optimized separately in iterative training. 
Thus for a fair comparison, we do not re-weight loss in the parallel training of \dress, \ie $\gamma=0$.
Also in all experiments, we conduct unstructured sampling in the entire tensor, and allow BN layers to be fine-tuned individually for each subnet to avoid other side effects.

\fakeparagraph{Results}
The comparison results are plotted in \figref{ch3-fig:parallel} Left. 
Parallel training substantially outperforms iterative training. 
Iterating over increased sparsity does not provide any space to optimize subnets with higher sparsity. 
Therefore, the accuracy drops quickly along iterations.
Although iterating over decreased sparsity may yield a well-performed high sparsity network, the accuracy does not improve significantly afterwards. 
We argue this is due to the fact that iterative training causes the optimizer to end in a hard to escape region around the previous subnet in the loss landscape. 
On the contrary, parallel training allows multiple subnets to be sampled and optimized jointly, which may especially benefit highly sparse networks, see \figref{ch3-fig:parallel} Left. 

\begin{figure}[!t]
	\centering
	\includegraphics[width=0.8\textwidth]{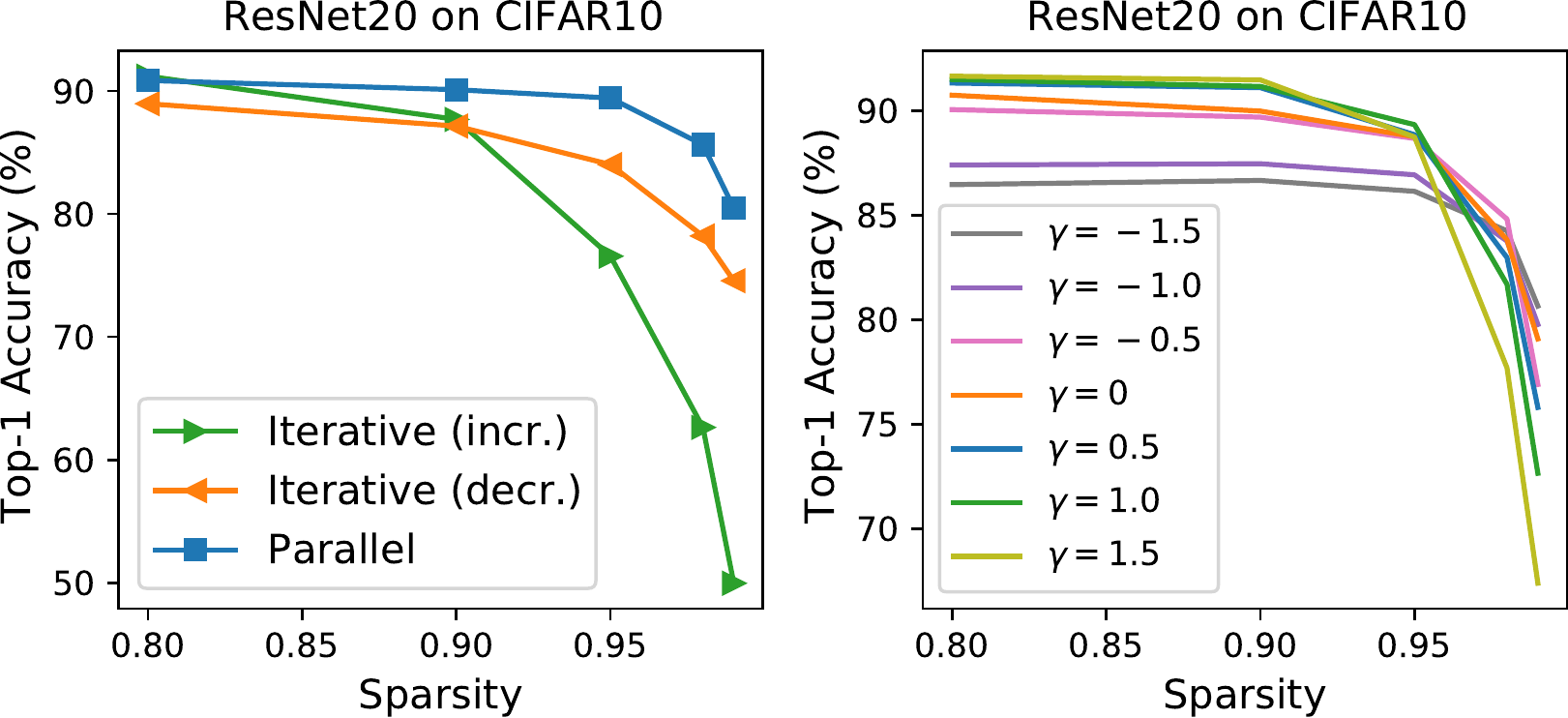}
	\caption[Left: Comparing parallel training with iterative training. Right: Ablation studies on the correction factor $\gamma$.]{Left: Comparing parallel training with iterative training. Right: Ablation studies on the correction factor $\gamma$.}
    \label{ch3-fig:parallel}
\end{figure}

\subsubsection{Correction Factor $\gamma$}
\label{ch3-sec:ablation_gamma}

\fakeparagraph{Settings}
The loss weights $\pi_k$ used in the parallel training may influence the final accuracy of different subnets.
In \secref{ch3-sec:optimize}, we introduce a correction factor $\gamma$ to control $\pi_k$ (see \equref{ch3-eq:gamma} and \equref{ch3-eq:alpha}). 
We thus conduct a set of experiments with different $\gamma$.
$\gamma=0$ means all loss items are weighted equally; $\gamma>0$ means the loss of the lower sparsity subnets is weighted larger, and vice versa. 
For example, for ResNet20 with $s_{1:5}=0.8,0.9,0.95,0.98,0.99$, when $\gamma=0.5$, $\pi_{1:5}\approx0.36,0.26,0.18,0.12,0.08$; $\gamma=-1.0$, $\pi_{1:5}\approx0.03,0.05,0.11,0.27,0.54$.

\fakeparagraph{Results}
The results in \figref{ch3-fig:parallel} Right show that the high sparsity subnets generally yield a higher final accuracy with a smaller $\gamma$. 
This is intuitive since a smaller $\gamma$ assigns a larger weight on the high sparsity subnets.
However, the downside is that the most powerful subnet (with the lowest sparsity) can not reach its top accuracy.
Note that the most powerful subnet is often adopted either under the critical case requiring high accuracy or in the commonly used scenario with standard resource constraints, see in \secref{ch3-sec:introduction}.
Also as discussed in \secref{ch3-sec:optimize}, low sparsity subnets should be weighted more, since they are implicitly optimized with a smaller step size.
Experimentally, we find that $\gamma\in[0.5,1]$ in parallel training allows us to train a group of subnets where the most powerful subnet can reach a similar accuracy as training in separately. 
We set $\gamma=0.5$ in the following experiments.

\begin{figure}[tbp]
	\centering
	\includegraphics[width=0.75\textwidth]{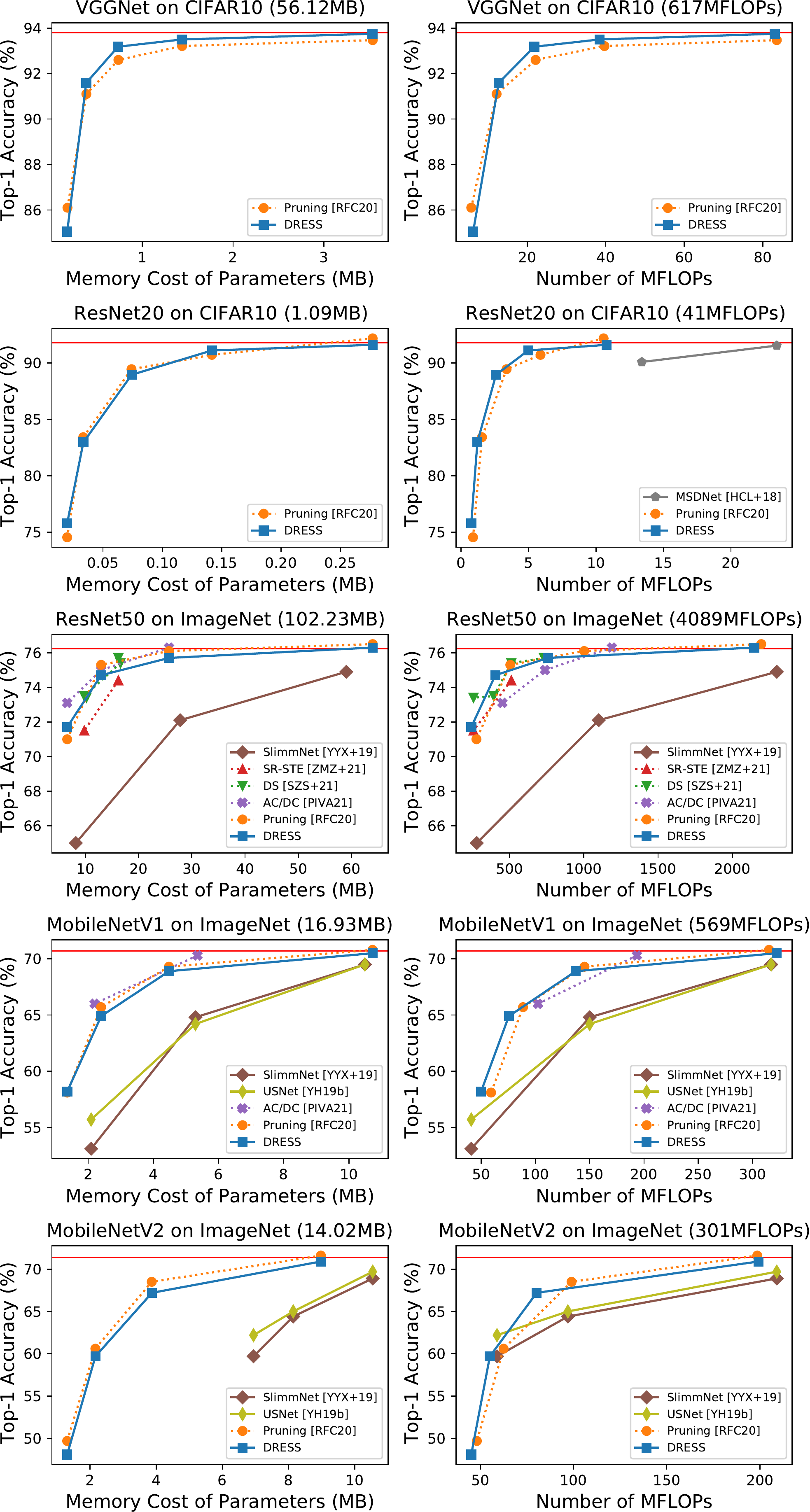}
	\caption[Comparing \dress with other baselines on image classification.]{Comparing \dress with other baselines on image classification. The methods that do not involve weight sharing among different networks are plotted with dotted curves. The memory cost and the number of MFLOPs of the original backbone networks are reported in the parentheses in titles; their accuracy is shown as red horizontal lines. The sparsity levels of \dress and ``Pruning'' are the same as the ones discussed in \secref{ch3-sec:experiments}.}
	\label{ch3-fig:benchmark}
\end{figure}

\subsection{Evaluation on Image Classification}
\label{ch3-sec:experiment_classification}

\fakeparagraph{Settings}
In this section, we benchmark \dress on public image classification datasets including CIFAR10 and ImageNet with different backbone networks discussed earlier in \secref{ch3-sec:experiments}. 
We compare the performance of the subnets generated by \dress with various methods, including (\textit{i}) anytime networks \cite{bib:ICLR18:Huang,bib:ICLR19:Yu,bib:ICCV19:Yu}, where the sub-networks with different width or depth can be cropped from the backbone network; 
(\textit{ii}) unstructured pruning \cite{bib:ICLR20:Renda,bib:NIPS21:Peste}, where \cite{bib:ICLR20:Renda} is re-implemented under our settings for a fair comparison; 
(\textit{iii}) N:M fine-grained structure pruning \cite{bib:ICLR21:Zhou,bib:NIPS21:Sun}.
We choose two metrics for comparison, the memory cost of parameters and the number of MFLOPs ($10^6$ FLOPs).
Both metrics are widely used proxies of resource consumption.
FLOPs dominate in the entire computation burden, thus fewer FLOPs can (but does not necessarily) result in a smaller computation time.   
The memory cost of parameters not only represents the static storage consumption but also relates to the amount of memory fetching when on-device inference with different (sub-)networks \cite{bib:IEEETrans20:Ahmad}.
Note that memory access often consumes more time and more energy than computation \cite{bib:ISSCC14:Horowitz}.
Assume that each parameter uses 32-bit for floating point values.
\dress, (\textit{ii}), and (\textit{iii}) generate sparse tensors, thus their memory cost also includes the indices of nonzero weights.
Following the suggestions of \cite{bib:tensorflow,bib:XNNPACK}, each index of nonzero weights is encoded into 8-bit in \dress and (\textit{ii}), whereas the binary mask is stored for indexing in \cite{bib:ICLR21:Zhou,bib:NIPS21:Sun}.  

\begin{table*}[tbp!]
    \centering
    \caption{The average test accuracy over all sub-networks, the theoretical storage (MB) required by all sub-networks and the average number of theoretical MFLOPs over all sub-networks.} 
    \label{ch3-tab:storage}
    \footnotesize
    \begin{tabular}{ccccccc}
    \toprule
    \multirow{2}{*}{Model}  & \multicolumn{2}{c}{Average Accuracy}    & \multicolumn{2}{c}{Overall Storage (MB)}    & \multicolumn{2}{c}{Average MFLOPs}      \\ \cmidrule(lr){2-3}  \cmidrule(lr){4-5} \cmidrule(lr){6-7} 
                            & \dress          & Pruning             & \dress          & Pruning                 & \dress          & Pruning             \\ \hline
    
    VGGNet                  & \textbf{91.4\%}   & 91.1\%              & \textbf{3.54}     & 6.27                    & \textbf{32}       & 33                  \\
    ResNet20                & \textbf{86.1\%}   & 85.9\%              & \textbf{0.28}     & 0.55                    & \textbf{4}        & \textbf{4}          \\
    ResNet50                & 74.5\%    & \textbf{74.6\%}             & \textbf{63.97}    & 109.25                  & \textbf{887}      & 994                 \\
    MobileNetV1             & 65.6\%    & \textbf{65.9\%}             & \textbf{10.72}    & 18.96                   & \textbf{146}      & 152                 \\
    MobileNetV2             & 61.5\%    & \textbf{62.4\%}             & \textbf{8.98}     & 16.32                   & \textbf{95}       & 101                 \\
    \bottomrule
    \end{tabular}
\end{table*}

\fakeparagraph{Results}
The results are plotted in \figref{ch3-fig:benchmark}.
In comparison to other anytime networks, the subnets generated by \dress require significantly less memory fetching and fewer FLOPs under the same accuracy level.
In addition, the sub-networks of conventional anytime networks \cite{bib:ICLR18:Huang,bib:ICLR19:Yu,bib:ICCV19:Yu} have different network architectures, while current compilation libraries (\eg TensorFlowLite) may not support to adopt a dynamic architecture on-device. 
The extra re-configuration overhead \eg storing various compiled architectures could be necessary for on-device inference.  However, this is avoided in \dress, since different subnets of \dress leverage the same architecture as the backbone network.
Like traditional unstructured pruning, \dress does not explicitly reduce the number of operations, \ie the networks with the same sparsity can require different numbers of FLOPs to perform inference as shown in \figref{ch3-fig:benchmark}. 
Thanks to the weight sharing, the static storage is only determined by the largest network for both \dress and anytime networks \cite{bib:ICLR18:Huang,bib:ICLR19:Yu,bib:ICCV19:Yu}.
The methods of (\textit{ii})-(\textit{iii}) do not involve weight sharing, thus they need more memory to store all networks separately. 

We further compare \dress with the unstructured pruning method \cite{bib:ICLR20:Renda}, in terms of the test accuracy, the theoretical storage, and the theoretical FLOPs, see \tabref{ch3-tab:storage}.
\dress reaches a similar average accuracy and computation complexity while only requiring 50\%-60\% of storage as pruning. 

\begin{figure}[tbp]
	\centering
	\includegraphics[width=0.8\textwidth]{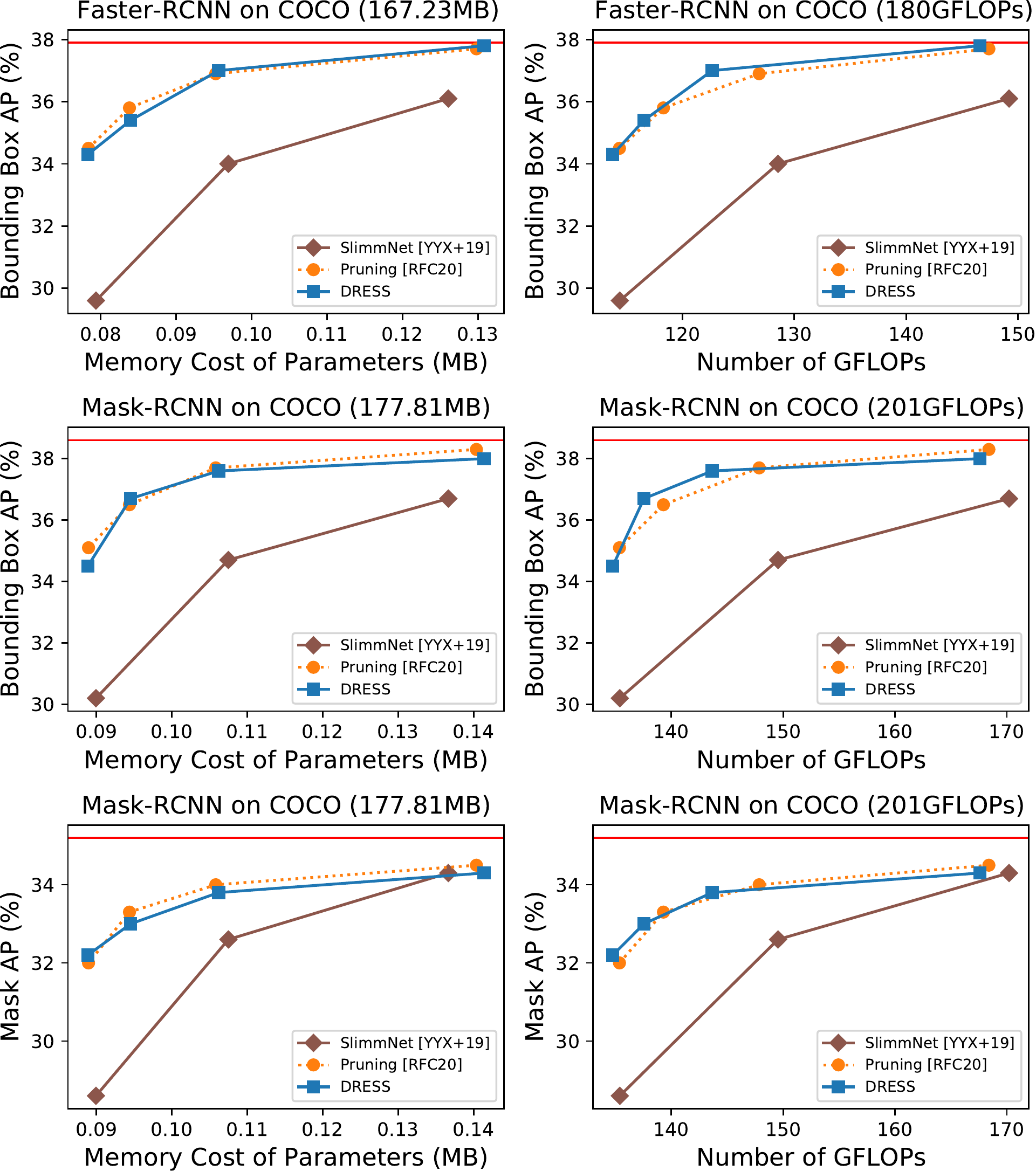}
	\caption[Comparing \dress with other baselines on object detection and instance segmentation.]{Comparing \dress with other baselines on object detection and instance segmentation. The methods that do not involve weight sharing among different networks are plotted with dotted curves. The memory cost and the number of GFLOPs of the original backbone networks are reported in the parentheses in titles; their average precision is shown as red horizontal lines.}
	\label{ch3-fig:benchmark_appendix}
\end{figure}

\subsection{Evaluation on Object Detection/Instance Segmentation}
\label{ch3-sec:experiment_object}

\fakeparagraph{Settings}
To show the versatility of our synthesis technique, we further benchmark \dress on other vision tasks, object detection and instance segmentation.
We compare \dress with other baselines mentioned in \secref{ch3-sec:experiment_classification} on MS COCO 2017 dataset.
We adopt Faster-RCNN with ResNet50-FPN \cite{bib:NIPS15:Ren} in object detection and Mask-RCNN \cite{bib:ICCV17:He} with ResNet50-FPN in instance segmentation.

\fakeparagraph{Results}
Since the number of FLOPs for Faster-RCNN and Mask-RCNN depends on the number of proposals in each image \cite{bib:arXiv20:Carion}, we report the average number of FLOPS for the randomly selected 100 images in COCO 2017 validation dataset. 
We compute the FLOPS with the tool flop count operators from Detectron2 \cite{bib:detectron2019}. 
For Faster-RCNN, we report its bounding box AP; for Mask-RCNN, we report its bounding box AP and its mask AP. 
The results are plotted in \figref{ch3-fig:benchmark_appendix}.
Similar to the results in \figref{ch3-fig:benchmark}, the subnets generated by \dress require a significantly lower memory cost and fewer GFLOPs ($10^9$ FLOPs) than other anytime networks \cite{bib:ICLR19:Yu}.
In addition, in comparison to the unstructured pruning \cite{bib:ICLR20:Renda} that does not involve weight sharing, \dress can also achieve a similar precision level. 

\begin{figure}[t!]
	\centering
	\includegraphics[width=0.78\textwidth]{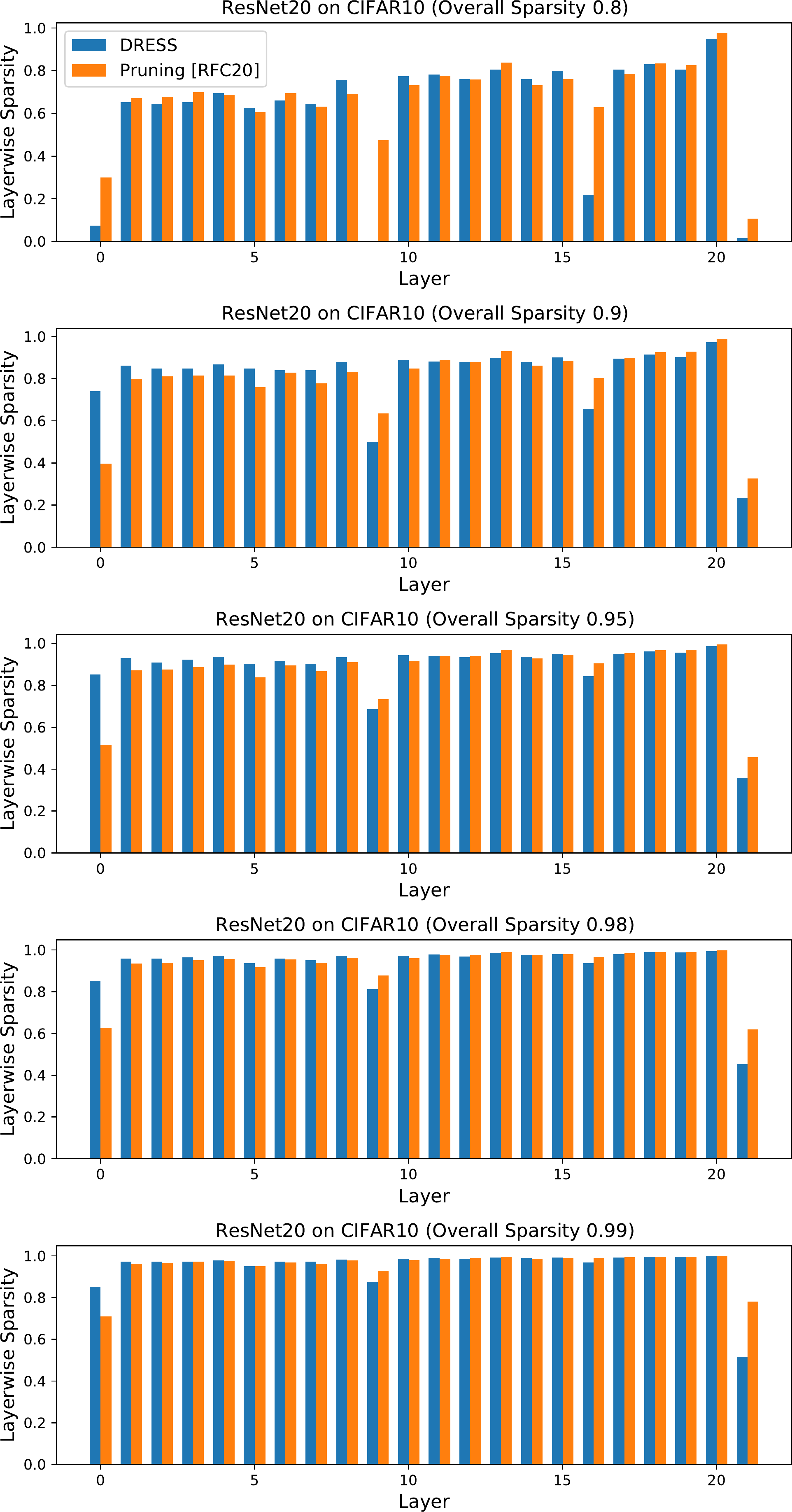}
	\caption[Comparing \dress with traditional pruning on ResNet20 in terms of the layerwise sparsity.]{Comparing \dress with traditional pruning on ResNet20 (CIFAR10) in terms of the layerwise sparsity. The (sub-)networks with different overall sparsity levels ($0.8,0.9,0.95,0.98,0.99$) are plotted in different subplots.}
	\label{ch3-fig:sparsity_resnet20}
\end{figure}

\begin{figure}[t!]
	\centering
	\includegraphics[width=0.8\textwidth]{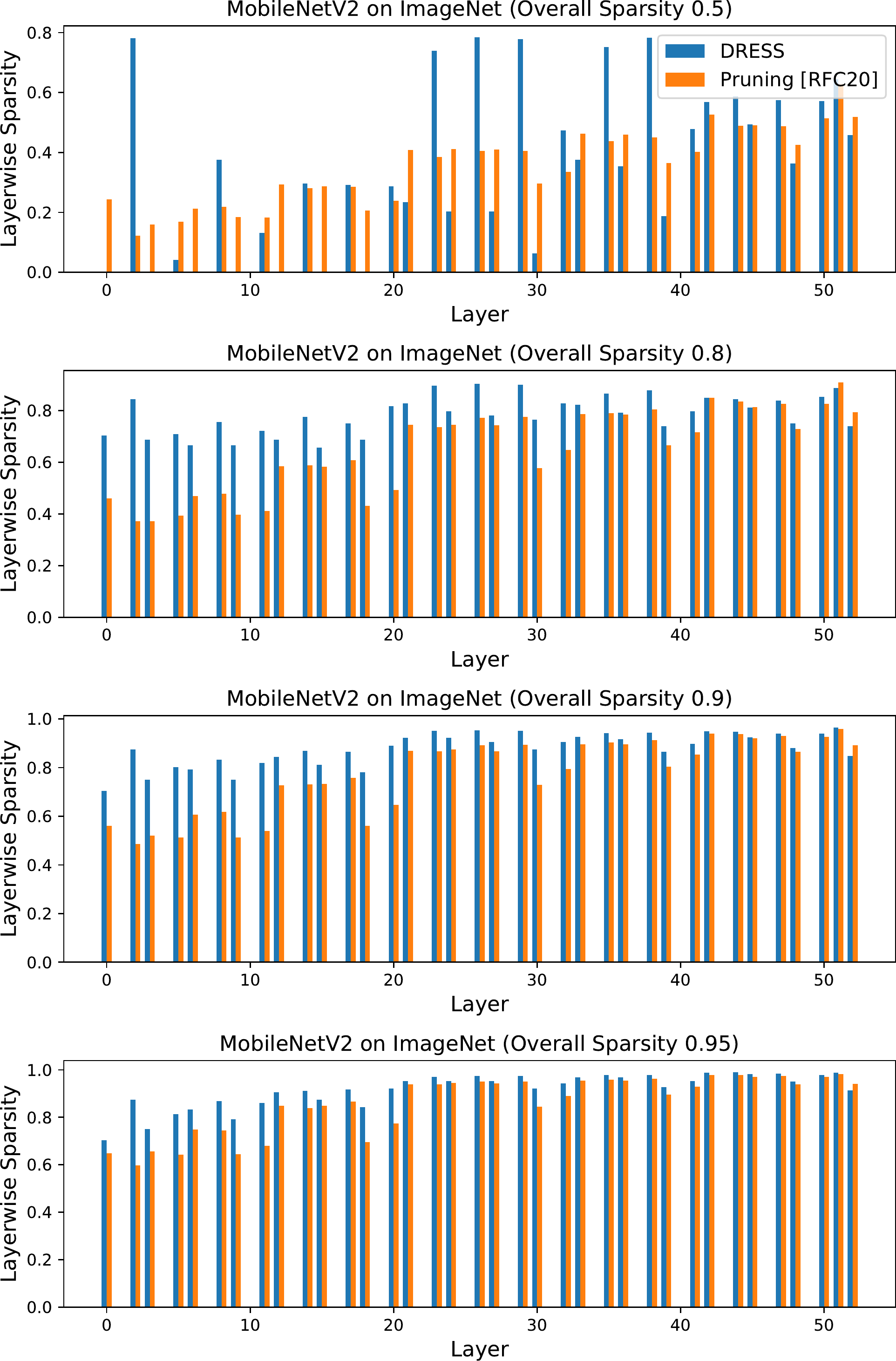}
	\caption[Comparing \dress with traditional pruning on MobileNetV2 in terms of the layerwise sparsity.]{Comparing \dress with traditional pruning on MobileNetV2 (ImageNet) in terms of the layerwise sparsity. The (sub-)networks with different overall sparsity levels $0.5,0.8,0.9,0.95$ are plotted in different subplots.}
	\label{ch3-fig:sparsity_mobilenetv2}
\end{figure}

\subsection{Sparsity across Layers}
\label{ch3-sec:sparsity}

To further explore the different impact from \dress and traditional pruning, we compare their layerwise sparsity.
Recall that the main differences between \dress and traditional pruning are, (\textit{i}) the nonzero weights of the higher sparsity subnets are reused by the lower sparsity subnets in \dress, whereas different sparse networks generated by traditional pruning are independent; (\textit{ii}) \dress maintains an unstructured sparse pattern in a row-based manner (\ie fine-grained structure sparsity \cite{bib:ICLR21:Zhou,bib:NIPS21:Hubara,bib:NIPS21:Sun}), whereas traditional pruning yields an unstructured sparse pattern in the entire tensor.

We plot the layerwise sparsity of the sparse (sub-)networks generated by \dress and traditional pruning \cite{bib:ICLR20:Renda}, for ResNet20 on CIFAR10 in \figref{ch3-fig:sparsity_resnet20} and for MobileNetV2 on ImageNet in \figref{ch3-fig:sparsity_mobilenetv2}.
In general, both methods have a similar layerwise sparsity in each subplot. 
However, there exists a diversity under a low sparsity level, \eg MobileNetV2 with sparsity 0.5.
Jointly training with weight sharing in \dress enforces the low sparsity network to be optimized towards a certain region that has been less explored in the individual training of pruning, as the low sparsity network often has relatively looser constraints.

\section{Deployments}
\label{ch3-sec:deployment}

To measure actual performance and compare it to the benchmark evaluation, we used \dress-generated subnets on a RaspberryPi 4 edge platform (with off-the-shelf Arm Cortex-A72 quad-core CPUs) for on-device inference.
The optimized Pytorch model is compiled by TensorFlow Lite \cite{bib:tensorflow} with XNNPACK \cite{bib:XNNPACK} delegate for deployment. 
We use multi-threading with 4 threads for acceleration. 

The inference latency when adopting different subnets of MobileNetV1 and MobileNetV2 on RaspberryPi 4 is reported in \tabref{ch3-tab:raspberry}.
The original dense models and the sparse models generated from unstructured pruning methods \cite{bib:ICLR20:Renda} are also added in \tabref{ch3-tab:raspberry} for comparison.
The reported latency is averaged over 100 randomly selected samples from ImageNet dataset.
By using the fast kernels for sparse matrix-dense matrix multiplication provided by XNNPACK, \dress can dynamically select its subnets to satisfy the various inference latency constraints. 
Note that the sparse (sub-)networks of \dress and pruning \cite{bib:ICLR20:Renda} have a similar number of theoretical FLOPs (see \figref{ch3-fig:benchmark} and \tabref{ch3-tab:storage}), yet \dress often yields a lower inference time.
This is due to the fact that row-based unstructured sparsity leads to regular computation among different rows, which speeds up inference \cite{bib:ICLR21:Zhou}.

\begin{table}[tbp!]
    \centering
    \caption[The average inference time (ms) on RaspberryPi 4.]{The average inference time (ms) on RaspberryPi 4.} 
    \label{ch3-tab:raspberry}
    \footnotesize
    \begin{tabular}{cccccc}
    \toprule
                                & Sparsity  & \multicolumn{2}{c}{MobileNetV1}           & \multicolumn{2}{c}{MobileNetV2}       \\ \hline
    Model                       &           & \multicolumn{2}{c}{Dense}                 & \multicolumn{2}{c}{Dense}             \\ 
    Time (ms)                   & 0\%       & \multicolumn{2}{c}{83}                    & \multicolumn{2}{c}{52}                \\ \hdashline
    Model                       &           & \dress            & Pruning               & \dress            & Pruning           \\ 
    \multirow{4}{*}{Time (ms)}  & 50\%      & \textbf{77}       & 80                    & \textbf{47}       & 48                \\
                                & 80\%      & \textbf{45}       & 55                    & \textbf{36}       & 41                \\
                                & 90\%      & \textbf{31}       & 35                    & \textbf{29}       & 32                \\
                                & 95\%      & \textbf{25}       & 27                    & \textbf{26}       & \textbf{26}       \\ 
    \bottomrule
    \end{tabular}
\end{table}

Note also that although the inference time decreases when adopting the subnets with a higher sparsity, the realistic speedup of sparse inference is not proportional to the reduction in theoretical FLOPs. 
For example, the theoretical FLOPs decrease by a factor of 6.4 when the sparsity of \dress MobileNetV1 subnets increases from 50\% to 95\%, while the inference is only accelerated by a factor of 3.1.
A similar phenomenon can also be observed in MobileNetV2 and pruned models. 
We suspect that the reason is that sparse computational cores of XNNPACK have a larger fraction of cache miss at a higher sparsity level, see also in \cite{bib:CVPR20:Elsen}.

\section{Summary}
\label{ch3-sec:summary}

This chapter develops a novel synthesis approach \dress that can adapt the sub-networks for on-device inference to maximize the model performance with different resource budgets.
\dress enables efficient adaptation on edge devices under varying resource constraints.
Prior synthesis methods either require deploying multiple individual networks, or sample sub-network architectures along structured dimensions leading to subpar performance. 
However, \dress utilizes nonzero-weight sharing and architecture sharing to reduce the redundancy among multiple unstructured sub-networks, resulting in both storage efficiency and re-configuration efficiency.
The main contributions of \dress are summarized as follows,
\begin{itemize}
    \item 
    \dress can adapt different sub-networks sampled from the backbone network on edge devices. 
    These optimized sub-networks have different sparsity, and thus can infer under various resource constraints, \eg the inference latency, and the battery energy.
    
    \item 
    \dress samples sub-networks in a row-based unstructured sparsity (a.k.a. fine-grained structure sparsity) and introduces a novel compressed sparse row (CSR) format for storing the sub-networks.
    This way, multiple sub-networks can be efficiently fetched and executed for on-device inference, by using the fast kernels of sparse tensor computation provided by recent compilation libraries. 
    To our best knowledge, this is the first work that builds multiple sub-networks via a fine-grained structure of weight sharing.

    \item 
    \dress enables weight sharing and architecture sharing among multiple sub-networks, resulting in (static) storage efficiency and re-configuration efficiency, respectively.
    
    \item
    Experimental results show \dress reaches a similar accuracy while only requiring 50\%-60\% of static storage as unstructured pruned networks, and can result in various distinct inference latency on off-the-shelf edge platforms according to different sparsity levels. 
\end{itemize}

This chapter studied how to adapt the network on edge devices to maximize the inference accuracy under varying resource constraints. 
In the next chapter, we will study how to conduct learning on edge devices given a few data samples of new tasks.
The different sub-networks generated by \dress are used for the same inference task, thus \dress is inapplicable to adapt its network given a new task.

\chapter[Learning on Edge Devices]{Learning on Edge Devices}
\label{ch4:learning}

In \chref{ch2:inference} and \chref{ch3:adaptation}, we studied how to compress a pretrained DNN for on-device inference under \textit{fixed} and \textit{varying} resource constraints.
However, when facing \textit{unseen} environments, users, or tasks, it is crucial to adapt\footnote{In this chapter, the adaptation is referred to as (re-)training on new data samples.} the pretrained DNN to deliver consistent performance and customized services. 
Sometimes, data collected by edge devices are private and have a large diversity across users/devices. 
Hence, \textit{on-device learning} is preferred over uploading the data to cloud servers for adaptation.

\fakeparagraph{Main Resource Constraints}
For on-device learning, neither abundant \textit{user data} nor \textit{computing resources} are applicable.
On the one hand, the amount of user data collected on a single edge device is rather small due to the limited labor resources. 
On the other hand, edge devices often have a small amount of available resources from memory and computation. 

\fakeparagraph{Principles}
Existing memory-efficient training approaches are not able to optimize a DNN given only a few training samples, whereas current meta learning methods require a significant amount of dynamic memory to few-shot learn unseen tasks.
Therefore, we introduce a memory-efficient on-device few-shot learning setting, and propose a novel meta learning scheme that can (\textit{i}) fast learn new unseen tasks given a few training samples, resulting in data efficiency, (\textit{ii}) avoid redundant training by distinguishing and learning adaptation-critical weights only, leading to memory efficiency.

The contents of this chapter are established mainly based on the paper ``p-Meta: Towards On-device Deep Model Adaptation'' that is published on ACM Conference on Knowledge Discovery and Data Mining (SIGKDD), 2022 \cite{bib:KDD22:Qu}.

\section{Introduction}
\label{ch4-sec:introduction}

The excellent accuracy of contemporary DNNs is attributed to training with high-performance computers on large-scale datasets \cite{bib:Book16:Goodfellow}.
For example, it takes $29$ hours to complete a $90$-epoch ResNet50 \cite{bib:CVPR16:He} training on ImageNet ($1.28$ million training images) \cite{bib:ILSVRC15} with $8$ Nvidia Tesla P100 GPUs \cite{bib:arXiv17:Goyal}.
However, on-device learning/adaptation of a DNN demands both \textit{data efficiency} and \textit{memory efficiency}.
A personal voice assistant, for example, may learn to adapt to users' accent and dialect within a few sentences, while a home robot should learn to recognize new object categories with few labelled images to navigate in new environments.
Furthermore, such learning is expected to be conducted on low-resource platforms such as smart portable devices, home hubs, and other IoT devices, with only several $KB$ to $MB$ memory.

\begin{figure}[t]
  \centering
  \includegraphics[width=0.8\textwidth]{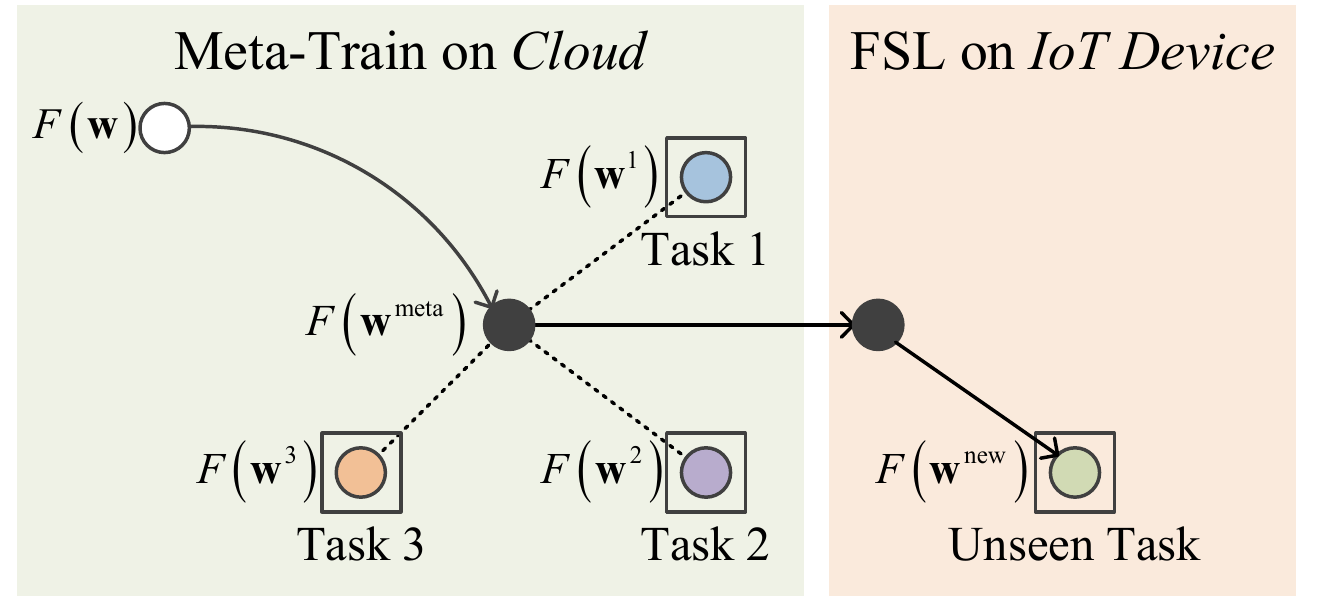} 
  \caption[Meta learning and few-shot learning in the context of on-device learning.]{Meta learning and few-shot learning (FSL) in the context of on-device learning. The backbone $F(\bm{w})$ is meta-trained into $F(\bm{w}^\mathrm{meta})$ on the cloud and is deployed to IoT devices to learn unseen tasks as $F(\bm{w}^\mathrm{new})$ via FSL.} 
  \label{ch4-fig:meta}
\end{figure}

For \textit{data-efficient} DNN training, we resort to \textit{meta learning}, a paradigm that learns to fast generalize to unseen tasks \cite{bib:arXiv20:Hospedales}.
Of our particular interest is \textit{gradient-based} meta learning \cite{bib:ICLR19:Antreas, bib:ICML17:Finn, bib:ICLR20:Raghu, bib:ICLR21:Oh} for its wide applicability in classification, regression and reinforcement learning, as well as the availability of gradient-based training frameworks for low-resource devices, \eg TensorFlow Lite \cite{bib:tfliteOndeviceTraining}.
\figref{ch4-fig:meta} explains major terminologies in the context of on-device learning.
Given a backbone, its weights are \textit{meta-trained} on \textit{many} tasks, to output a \textit{model} that is expected to fast learn new \textit{unseen} tasks.
The process of learning is also known as \textit{few-shot learning}, where the meta-trained model is further retrained by standard stochastic gradient decent (SGD) on \textit{few new} samples only.

However, existing gradient-based meta learning schemes \cite{bib:ICLR19:Antreas, bib:ICML17:Finn, bib:ICLR20:Raghu, bib:ICLR21:Oh} fail to support \textit{memory-efficient} training.
Although \textit{meta training} is conducted in the cloud, \textit{few-shot learning} of the meta-trained model is performed on IoT devices.
Consider to retrain a common backbone ResNet12 in a 5-way (5 new classes) 5-shot (5 samples per class) scenario.  
One round of SGD consumes 370.44MB peak dynamic memory, since the inputs of all layers must be stored to compute the gradients of these layers' weights in the backward path. 
In comparison, inference only needs 3.61MB.
The necessary dynamic memory is a key bottleneck for on-device learning due to cost and power constraints, even though the meta-trained model only needs to be retrained with a few data.

Prior efficient DNN training solutions mainly focus on parallel and distributed training on data centers \cite{bib:arXiv16:Chen, bib:ICLR21:Chen, bib:NIPS17:Greff, bib:NIPS16:Gruslys, bib:NIPS20:Raihan}.
On-device training has been explored for \textit{vanilla supervised training} \cite{bib:ICASSP20:Gooneratne, bib:MLSys21:Mathur, bib:SenSys19:Lee}, where training and testing are performed on the \textit{same} task. 
A pioneer study \cite{bib:NIPS20:Cai} investigated on-device learning to new tasks via memory-efficient \textit{transfer learning}. 
Yet transfer learning is prone to overfitting when only a few samples are available \cite{bib:ICML17:Finn}. 

In this paper, we propose \textbf{\pMeta}, a new meta learning method for data- and memory-efficient DNN training.
The key idea is to enforce \textit{structured partial parameter updates} while ensuring \textit{fast generalization to unseen tasks}. 
The idea is inspired by recent advances in understanding gradient-based meta learning \cite{bib:ICLR21:Oh, bib:ICLR20:Raghu}.
Empirical evidence shows that only the \textit{head} (the last output layer) of a DNN needs to be updated to achieve reasonable few-shot classification accuracy \cite{bib:ICLR20:Raghu} whereas the \textit{body} (the layers closed to the input) needs to be updated for cross-domain few-shot classification \cite{bib:ICLR21:Oh}.
These studies imply that certain weights are more important than others when generalizing to unseen tasks. Hence, we propose to automatically identify these \textit{adaptation-critical weights} to minimize the memory demand in few-shot learning.

Particularly, the critical weights are determined in two structured dimensionalities as, 
(\textit{i}) layer-wise: we meta-train a layer-by-layer learning rate that enables a \textit{static} selection of critical layers for updating; 
(\textit{ii}) channel-wise: we introduce meta attention modules in each layer to select critical channels \textit{dynamically}, \ie depending on samples from new tasks. 
Partial updating of weights means that (structurally) sparse gradients are generated, reducing memory requirements to those for computing nonzero gradients.
In addition, the computation demand for calculating zero gradients can be also saved.
To further reduce the memory, we utilize \textit{gradient accumulation} in few-shot learning and \textit{group normalization} in the backbone.
Although weight importance metrics and SGD with sparse gradients have been explored in vanilla training \cite{bib:NIPS20:Raihan, bib:PIEEE20:Deng, bib:ICASSP20:Gooneratne, bib:ICLR16:Han}, it is unknown \textit{(i)} how to identify adaptation-critical weights and \textit{(ii)} whether meta learning is robust to sparse gradients, where the objective is to fast learn \textit{unseen} tasks.

\section{Related Work}
\label{ch4-sec:related}

\subsection{Meta Learning for Few-Shot Learning}
\label{ch4-sec:related_meta}

Meta learning is a prevailing solution to few-shot learning \cite{bib:arXiv20:Hospedales}, where the meta-trained model can learn an unseen task from a few training samples, \ie data-efficient training.
The majority of meta learning methods can be divided into two categories, (\textit{i}) embedding-based methods \cite{bib:NIPS16:Vinyals, bib:NIPS17:Snell, bib:CVPR18:Sung} that learn an embedding for classification tasks to map the query samples onto the classes of labeled support samples, (\textit{ii}) gradient-based methods \cite{bib:ICLR19:Antreas, bib:ICML17:Finn, bib:ICLR20:Raghu, bib:ICLR21:Oh, bib:NIPS21:Oswald} that learn an initial model (and/or optimizer parameters) such that it can be trained with gradient information calculated on the new few samples.   
Among them, we focus on gradient-based meta learning methods for their applicability in various learning tasks and the availability of gradient-based training frameworks for low-resource devices \cite{bib:tfliteOndeviceTraining}.

Particularly, we aim at meta training a DNN that allows fast learning on memory-constrained devices.
Most meta learning algorithms \cite{bib:ICLR19:Antreas, bib:ICML17:Finn, bib:NIPS21:Oswald} optimize the backbone network for better generalization yet ignore the workload if the meta-trained backbone is deployed on low-resource platforms for few-shot learning. 
Manually fixing certain layers during on-device few-shot learning \cite{bib:ICLR20:Raghu, bib:ICLR21:Oh, bib:AAAI21:Shen} may also reduce memory and computation, but to a much lesser extent as shown in our evaluations.

\subsection{Efficient DNN Training}
\label{ch4-sec:related_efficient}

Existing efficient training schemes are mainly designed for high-throughput GPU training on large-scale datasets.
\cite{bib:ICLR20:Cambier, bib:NIPS18:Wang} conduct 8-bit floating point low precision training which requires specialized hardware for efficient execution. 
A general strategy is to trade memory with computation \cite{bib:arXiv16:Chen, bib:NIPS16:Gruslys}, which is unfit for IoT device with a limited computation capability. 
An alternative is to sparsify the computational graphs in backpropagation \cite{bib:NIPS20:Raihan}.
Yet it relies on massive training iterations on large-scale datasets.
Other techniques include layer-wise local training \cite{bib:NIPS17:Greff} and reversible residual module \cite{bib:NIPS17:Gomez}, but they often incur notable accuracy drops. 

There are a few studies on DNN training on low-resource platforms, such as updating the last several layers only \cite{bib:MLSys21:Mathur}, reducing batch sizes \cite{bib:SenSys19:Lee}, and gradient approximation \cite{bib:ICASSP20:Gooneratne}.
However, they are designed for vanilla supervised training, \ie train and test on the same task. 
One recent study proposes to update the bias parameters only for memory-efficient transfer learning \cite{bib:NIPS20:Cai}, yet transfer learning is prone to overfitting when trained with limited data \cite{bib:ICML17:Finn}.

\section{Preliminaries and Challenges}
\label{ch4-sec:preliminary}

In this section, we first motivate on-device few-shot learning via example applications, then provide the basics on meta learning for few-shot learning and highlight the challenges to enable on-device learning.

\subsection{Example Application Scenarios}
\label{ch4-sec:preliminary_example}

On-device few-shot learning is essential for model adaptation in some intelligent applications, when the new data collected on edge devices tend to relate to personal habits and lifestyle.
For instance, activity recognition with smartphone sensors should adapt to countless walking patterns and sensor orientation \cite{bib:SenSys19:Gong}.
Gaze tracking with smart glasses requires calibration to personal gaze conditions for cognitive context recognition \cite{bib:SenSys20:Lan}.
Human motion prediction with home robots needs fast learning of unseen poses for seamless human-robot interaction \cite{bib:ECCV18:Gui}.
We detail two representative applications below and summarize their resource utilization in \tabref{ch4-tab:examples}.

\fakeparagraph{Home Surveillance Customization}
Household camera systems are pervasively deployed to detect intruders and monitor pets, where suspicious images are uploaded to a smart gateway for further investigation such as object classification.
Due to the \textit{countless object classes} of interest across individuals, the image classification model needs post-deployment customization. 
Fast model adaptation (\eg pre-trained on dog breeds such as Komondor, Poodle and Saluki, and re-trained to recognize Malamute) at the smart gateway delivers more targeted surveillance services without leaking images of family members or private locations.

\fakeparagraph{Robot Locomotion Control}
Robots that walk and run as humans have been a long-standing challenge in robotics \cite{bib:IJRR21:Ibarz}.
Deep reinforcement learning (DRL) advances the development of naturally behaved robots for new applications such as police robotic dogs and unmanned last-mile delivery \cite{bib:SR19:Hwangbo}.
It is important that the robots fast learn their locomotion policies to new goals and environments since there is often a gap between the training and deployment environments.
Naive DRL can take millions of data samples to learn meaningful locomotion gaits \cite{bib:IJRR21:Ibarz}.
Conversely, on-robot few-shot DRL enables rapid control policy acquisition with few new experience.

\subsection{Meta Learning for Few-Shot Learning}
\label{ch4-sec:preliminary_meta}

Meta learning is a prevailing solution to adapt a DNN to unseen tasks with limited training samples, \ie few-shot learning \cite{bib:arXiv20:Hospedales}.
We ground our work on model-agnostic meta learning (MAML) \cite{bib:ICML17:Finn}, a generic meta learning framework which supports classification, regression and reinforcement learning.
\tabref{ch4-tab:notations} lists the major notations.

\begin{table}[!t]
 	\centering
 	\caption[Summary of major notations.]{Summary of major notations.}
 	\label{ch4-tab:notations}
 	\footnotesize
 	\begin{tabular}{ll}
 		\toprule
 		Notation & Description \\
 		\midrule
 		$l$ & Layer index, $l \in 1,2,...,L$ \\
 		$\bm{x}_{l-1}$, $\bm{w}_l$, $\bm{y}_l$  & Input, weight, and intermediate tensors \\
 		$C_l$, $H_l$, $W_l$ & Output channel number, height and width \\
 		$\bm{x}_L = F(\bm{w};\bm{x}_0)$ & A model (backbone) with parameter $\bm{w}$, its input and output \\
 		$\mathsf{T}^{i}$ & Sampled task $i$ from distribution $p(\mathsf{T})$ during meta training \\
 		$\mathcal{D}^i=\{\mathcal{S}^i,\mathcal{Q}^i\}$ & Dataset with support set $\mathcal{S}^i$ and query set $\mathcal{Q}^i$ for task $\mathsf{T}^{i}$ \\
 		$\mathsf{T}^{\mathrm{new}}$ & New unseen task during on-device few-shot learning \\
 		\multirow{2}{*}{$\mathcal{D}^{\mathrm{new}}=\{\mathcal{S}^{\mathrm{new}},\mathcal{Q}^{\mathrm{new}}\}$} & Dataset for unseen task $\mathsf{T}^{\mathrm{new}}$ \\
 		                                                                                                        & $\mathcal{S}^{\mathrm{new}}$ for few-shot learning and $\mathcal{Q}^{\mathrm{new}}$ for evaluation \\
 		$\ell(\bm{w};\mathcal{D})$ & loss function over model $F(\bm{w})$ and dataset $\mathcal{D}$ \\
 		$\bm{w}^\mathrm{meta}$, $\bm{w}^\mathrm{new}$ & parameters after meta training and few-shot learning \\
 		$\bm{w}^{i,k}$ & model parameters $\bm{w}^{i}$ at step $k$ on task $i$ in inner loop \\
 		$\bm{\alpha}$, $\beta$ & inner and outer step size \\
 		$\bm{g}(\cdot)$ & the loss gradients w.r.t. the given tensor \\
 		$\sigma(\cdot)$,  $\sigma'(\cdot)$ & Non-linear function and its derivative\\
 		$m(\cdot)$ & memory consumption of the given tensor in words \\
 		\bottomrule
 	\end{tabular}
 \end{table}
 
Given the dataset $\mathcal{D}=\{\mathcal{S}, \mathcal{Q}\}$ of an unseen few-shot task, where $\mathcal{S}$ (support set) and $\mathcal{Q}$ (query set) are for training and testing, MAML trains a model $F(\bm{w})$ with weights $\bm{w}$ such that it yields high accuracy on $\mathcal{Q}$ even when $\mathcal{S}$ only contains a few samples.
This is enabled by simulating the few-shot learning experiences over abundant few-shot tasks sampled from a task distribution $p(\mathsf{T})$.
Specifically, it meta-trains a backbone $F$ over few-shot tasks $\mathsf{T}^{i}\sim p(\mathsf{T})$, where each $\mathsf{T}^{i}$ has dataset $\mathcal{D}^i=\{\mathcal{S}^i,\mathcal{Q}^i\}$, and then generates $F(\bm{w}^\mathrm{meta})$, an initialization for the unseen few-shot task $\mathsf{T}^{\mathrm{new}}$ with dataset $\mathcal{D}^{\mathrm{new}}=\{\mathcal{S}^{\mathrm{new}},\mathcal{Q}^{\mathrm{new}}\}$.
Training from $F(\bm{w}^\mathrm{meta})$ over $\mathcal{S}^{\mathrm{new}}$ is expected to achieve a high test accuracy on $\mathcal{Q}^{\mathrm{new}}$.

MAML achieves fast learning via two-tier optimization. 
In the \textit{inner loop}, a task $\mathsf{T}^i$ and its dataset $\mathcal{D}^i$ are sampled. 
The weights $\bm{w}$ are updated to $\bm{w}^{i}$ on support dataset $\mathcal{S}^{i}$ via $K$ gradient descent steps, where $K$ is usually small, compared to vanilla training:
\begin{equation}\label{ch4-eq:maml_inner}
    \bm{w}^{i,k} = \bm{w}^{i,k-1}-\alpha\nabla_{\bm{w}}~\ell\left(\bm{w}^{i,k-1};\mathcal{S}^i\right)\quad\mathrm{for}~k=1,...,K
\end{equation}
where $\bm{w}^{i,k}$ are the weights at step $k$ in the inner loop, and $\alpha$ is the inner step size.
Note that $\bm{w}^{i,0}=\bm{w}$ and $\bm{w}^{i}=\bm{w}^{i,K}$.
$\ell(\bm{w};\mathcal{D})$ is the loss function on dataset $\mathcal{D}$. 
In the \textit{outer loop}, the weights are optimized to minimize the sum of loss at $\bm{w}^{i}$ on query dataset $\mathcal{Q}^i$ across tasks. 
The gradients to update weights in the outer loop are calculated w.r.t. the starting point $\bm{w}$ of the inner loop.
\begin{equation}\label{ch4-eq:maml_outer}
    \bm{w} \leftarrow \bm{w}-\beta \nabla_{\bm{w}} \sum_i \ell\left(\bm{w}^{i};\mathcal{Q}^i\right)
\end{equation}
where $\beta$ is the outer step size. 
 
\begin{table}[t]
    \centering
 	\caption[Memory and total computation of inference and training in example few-shot learning.]{Memory and total computation (GFLOPs $=10^9$FLOPs) of inference and training in example few-shot learning. For image classification (``4Conv on MiniImageNet'' and ``ResNet12 on MiniImageNet''), we use batch size $=25$, \ie 5-way 5-shot. For robot locomotion (``MLP'' on MuJoCo), we use rollouts $=20$, horizon $=200$; each sample corresponds to a rollouted episode, and the case for an observation is reported in brackets. The calculation is based on \secref{ch4-sec:analysis}.}
 	\label{ch4-tab:examples}
 	\footnotesize
 	\begin{tabular}{lccc}
 		\toprule
 		\multirow{2}{*}{Benchmark}      & 4Conv                 & ResNet12                  & MLP \\ 
 		                                & MiniImageNet          & MiniImageNet              & MuJoCo  \\ \midrule
 		Model Static Storage (MB)       & $0.13$                & $32.0$                    & $0.05$ \\
 		Sample Static Storage (MB)      & $0.53$                & $0.53$                    & $0.016 (0.00008)$ \\ \hdashline
 		Inference Peak Memory (MB)      & $0.90$                & $3.61$                    & $0.08 (0.0004)$ \\
 		Training Peak Memory (MB)       & $48.33$               & $370.44$                  & $3.72$ \\ \hdashline
 		Inference GFLOPs                & $0.72$                & $62.08$                   & $0.05$ \\
 		Training GFLOPs                 & $1.96$                & $185.42$                  & $0.15$ \\
 		\bottomrule
 	\end{tabular}
\end{table}
 
The meta-trained weights $\bm{w}^\mathrm{meta}$ are then used as initialization for few-shot learning into $\bm{w}^\mathrm{new}$ by $K$ gradient descent steps over $\mathcal{S}^{\mathrm{new}}$.
Finally we assess the accuracy of $F(\bm{w}^\mathrm{new})$ on $\mathcal{Q}^{\mathrm{new}}$.

\subsection{Memory Bottleneck of On-Device Learning}
\label{ch4-sec:preliminary_memory}

As mentioned above, the meta-trained model $F(\bm{w}^\mathrm{meta})$ can learn unseen tasks via $K$ gradient descent steps.
Each step is the same as the inner loop of meta-training \equref{ch4-eq:maml_inner}, but on dataset $\mathcal{S}^{\mathrm{new}}$.
\begin{equation}\label{ch4-eq:fsl}
    \bm{w}^{\mathrm{new},k} = \bm{w}^{\mathrm{new},k-1}-\alpha\nabla_{\bm{w}^{\mathrm{new}}}~\ell\left(\bm{w}^{\mathrm{new},k-1};\mathcal{S}^{\mathrm{new}}\right)
\end{equation}
where $\bm{w}^{\mathrm{new},0}=\bm{w}^\mathrm{meta}$.
For brevity, we omit the superscripts of model adaption in \equref{ch4-eq:fsl} and use $\bm{g}(\cdot)$ as the loss gradients w.r.t. the given tensor.
Hence, without ambiguity, we simplify the notations of \equref{ch4-eq:fsl} as follows:
\begin{equation}\label{ch4-eq:fsl_s}
    \bm{w} \leftarrow \bm{w}-\alpha \bm{g}(\bm{w})
\end{equation}

Let us now understand where the main memory cost for iterating \equref{ch4-eq:fsl_s} comes from. 
For the sake of clarity, we focus on a feed forward DNNs that consist of $L$ convolutional (\texttt{conv}) layers or fully-connected (\texttt{fc}) layers.
A typical layer (see \figref{ch4-fig:layer}) consists of two operations: (\textit{i}) a linear operation with trainable parameters, \eg convolution or affine; (\textit{ii}) a parameter-free non-linear operation, where we consider max-pooling or ReLU-styled (ReLU, LeakyReLU) activation functions in this paper.
Note that the non-linear operation unit may not exist in some layers; some layers may also have more than one non-linear units (\eg both max-pooling and ReLU activation function), and all corresponding intermediate tensors should be stored.

\begin{figure}[t]
    \centering
    \includegraphics[width=0.65\textwidth]{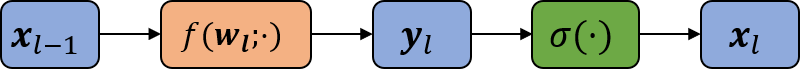} 
    \caption[A typical layer in DNNs.]{A typical layer $l$ in DNNs. $\bm{x}_{l-1}$ is the input tensor; $\bm{x}_l$ is the output tensor, also the input tensor of layer $l+1$; $\bm{y}_l$ is the intermediate tensor; $\bm{w}_l$ is the weight tensor. }
    \label{ch4-fig:layer}
\end{figure}

Take a network consisting of \texttt{conv} layers only as an example. 
The memory requirements for storing the activations $\bm{x}_l\in\mathbb{R}^{C_l\times H_l \times W_l}$ as well as the convolution weights $\bm{w}_l \in \mathbb{R}^{C_l\times C_{l-1}\times S_l \times S_l} $ of layer $l$ in words can be determined as
\[
    m(\bm{x}_l)= C_l H_l W_l \; , \; \; m(\bm{w}_l) = C_l C_{l-1} S_l^2
\]
where $C_{l-1}$, $C_l$, $H_l$, and $W_l$ stand for input channel number, output channel number, height and width of layer $l$, respectively; $S_l$ stands for the kernel size.
The detailed memory and computation demand analysis as provided in \secref{ch4-sec:analysis} reveals that the by far largest memory requirement is neither attributed to determining the activations $\bm{x}_l$ in the forward path nor to determining the gradients of the activations $\bm{g}(\bm{x}_l)$ in the backward path. 
Instead, the memory bottleneck lies in the computation of the weight gradients $\bm{g}(\bm{w}_l)$, which requires the availability of the activations $\bm{x}_{l-1}$ from the forward path. 
Following \equref{ch4-eq:memorySimple} in \secref{ch4-sec:analysis}, the necessary memory in words is
\begin{equation}
    \label{ch4-eq:MemoryContrib}
    \sum_{1 \leq l \leq L} m(\bm{x}_{l-1})
\end{equation}

\tabref{ch4-tab:examples} summarizes the peak memory and the total computation of the commonly used few-shot learning backbone models \cite{bib:ICML17:Finn, bib:NIPS18:Oreshkin}. 
The requirements are based on the detailed analysis in \secref{ch4-sec:analysis}.  
We can draw two intermediate conclusions. 
\begin{itemize}
    \item 
    The total computation of training is approximately $2.7\times$ to $3\times$ larger compared to inference.
    Yet the peak memory of training is far larger, $47 \times$ to $103 \times$ over inference. 
    \item
    To enable training on memory-constrained IoT devices, we need to find some way of getting rid of the major dynamic memory contribution in \equref{ch4-eq:MemoryContrib}.
\end{itemize}

\section{p-Meta}
\label{ch4-sec:method}

This section presents \pMeta, a new meta learning scheme that enables memory-efficient few-shot learning on unseen tasks.
\pMeta is a novel meta training algorithm that not only learns the weights of the initialized backbone but also learns to identify adaptation-critical weights for memory-efficient few-shot learning. 

\subsection{\pMeta Overview}
\label{ch4-sec:overview}
We first provide an overview of \pMeta and introduce its main concepts, namely selecting critical gradients, using a hierarchical approach to determine adaption-critical layers and channels, and using a mixture of static and dynamic selection mechanisms. 

We impose \textit{structured sparsity} on the \textit{gradients} $\bm{g}(\bm{w}_l)$ such that the corresponding tensor dimensions of $\bm{x}_l$ do not need to be saved. 
There are other options to reduce the dominant memory demand in \equref{ch4-eq:MemoryContrib}.
They are inapplicable for the reasons below.
\begin{itemize}
    \item 
    One may trade-off computation and memory by recomputing activations $\bm{x}_{l-1}$ when needed for determining $\bm{w}_l$, see for example \cite{bib:arXiv16:Chen, bib:NIPS16:Gruslys}. 
    Due to the limited processing abilities of IoT devices, we exclude this option. 
    \item
    It is also possible to prune activations $\bm{x}_{l-1}$.
    Yet based on our experimental results in \tabref{ch4-tab:sparse}, imposing sparsity on $\bm{x}_{l-1}$ hugely degrades few-shot learning accuracy as this causes error accumulation along the propagation, see also \cite{bib:NIPS20:Raihan}.
    \item
    Note that unstructured sparsity, as proposed in \cite{bib:CIKM21:Gao, bib:NIPS21:Oswald}, does not in general lead to memory savings, since there is a very small probability that all weight gradients for which an element of $\bm{x}_{l-1}$ is necessary have been pruned. 
    Furthermore, their weight selection is fixed after meta training, whereas \pMeta allows dynamic weight selection when few-shot learning on different tasks. Such runtime weight selection is essential for few-shot model training.
\end{itemize}

\noindent
We impose sparsity on the gradients in a hierarchical manner.

\begin{itemize}
    \item 
    \underline{Selecting Adaption-Critical Layers:} 
    We first impose layer-by-layer sparsity on $\bm{g}(\bm{w}_l)$.
    It is motivated by previous results showing that manual freezing of certain layers does no harm to few-shot learning accuracy \cite{bib:ICLR20:Raghu,bib:ICLR21:Oh}. 
    Layer-wise sparsity reduces the number of layers whose weights need to be updated. 
    We determine the adaptation-critical layers from the meta-trained \textit{layer-wise sparse learning rates}. 
    \item
    \underline{Selecting Adaption-Critical Channels:}
    In addition to imposing layer-wise sparsity of weight gradients, 
    We further reduce the memory demand by imposing sparsity on $\bm{g}(\bm{w}_l)$ within each layer. Noting that calculating $\bm{g}(\bm{w}_l)$ needs both the input channels $\bm{x}_{l-1}$ and the output channels $\bm{g}(\bm{y}_{l})$, we enforce sparsity on both of them.
    Input channel sparsity decreases memory and computation overhead, whereas output channel sparsity improves few-shot learning accuracy and reduces computation.
    We design a novel \textit{meta attention mechanism} to \textit{dynamically} determine adaptation-critical channels. 
    They take as inputs $\bm{x}_{l-1}$ and $\bm{g}(\bm{y}_{l})$ and determine adaptation-critical channels during few-shot learning, based on the given data samples from new unseen tasks. 
    Dynamic channel-wise learning rates as determined by meta attention yield a significantly higher accuracy than a static channel-wise learning rate (see \secref{ch4-sec:ablation}).
\end{itemize}

\fakeparagraph{Memory Reduction} 
The reduced memory demand due to our hierarchical approach can be formulated in high level as,
\begin{equation}
    \sum_{1 \leq l \leq L} \hat{\alpha}_l \mu^{\mathrm{fw}}_l m(\bm{x}_{l-1})
\end{equation}
where $\hat{\alpha}_l \in \{ 0, 1\}$ is the mask from the static selection of adaptation-critical layers and $0 \leq \mu^{\mathrm{fw}}_l \leq 1$ denotes the relative amount of dynamically chosen adaptation-critical input channels. 
A more detailed analysis of memory demands can be found in \secref{ch4-sec:analysis}.

Next, we explain how \pMeta selects adaptation-critical layers (\secref{ch4-sec:LR}) and channels within layers (\secref{ch4-sec:attention}) as well as the deployment optimizations (\secref{ch4-sec:others}) for memory-efficient training.

\subsection{Selecting Adaption-Critical Layers by Learning Sparse Inner Step Sizes}
\label{ch4-sec:LR}
This subsection introduces how \pMeta meta-learns adaptation-critical layers to reduce the number of updated layers during few-shot learning.
Particularly, instead of manual configuration as in \cite{bib:ICLR21:Oh, bib:ICLR20:Raghu}, we propose to automate the layer selection process.
During meta training, we identify adaptation-critical layers by learning layer-wise sparse inner step sizes (\secref{ch4-sec:LR_ml}).
Only these critical layers with nonzero step sizes will be updated during on-device learning to new tasks (\secref{ch4-sec:LR_fsl}).

\subsubsection{Learning Sparse Inner Step Sizes in Meta Training}
\label{ch4-sec:LR_ml}

Prior work \cite{bib:ICLR19:Antreas} suggests that instead of a global fixed inner step size $\alpha$, learning the inner step sizes $\bm{\alpha}$ for each layer and each gradient descent step improves the generalization of meta learning, where $\bm{\alpha} = \alpha_{1:L}^{1:K} \succeq \bm{0}$.
We utilize such learned inner step sizes to infer layer importance for adaptation.
We learn the inner step sizes $\bm{\alpha}$ in the outer loop of meta-training while fixing them in the inner loop.

\fakeparagraph{Learning Layer-wise Inner Step Sizes}
We change the inner loop of \equref{ch4-eq:maml_inner} to incorporate the per-layer inner step sizes:
\begin{equation}\label{ch4-eq:inner_alpha}
    \bm{w}_l^{i,k} = \bm{w}_l^{i,k-1}-\alpha^k_l\nabla_{\bm{w}_l}~\ell\left(\bm{w}^{i,k-1}_{1:L};\mathcal{S}^i\right)
\end{equation}
where $\bm{w}_l^{i,k}$ is the weights of layer $l$ at step $k$ optimized on task $i$ (dataset $\mathcal{S}^{i}$).
In the outer loop, weights $\bm{w}$ are still optimized as
\begin{equation}\label{ch4-eq:outer_alpha}
    \bm{w} \leftarrow \bm{w}-\beta \nabla_{\bm{w}} \sum_i \ell\left(\bm{w}^{i};\mathcal{Q}^i\right)
\end{equation}
where $\bm{w}^{i}=\bm{w}^{i,K}=\bm{w}^{i,K}_{1:L}$, which is a function of $\bm{\alpha}$.
The inner step sizes $\bm{\alpha}$ are then optimized as
\begin{equation}\label{ch4-eq:alpha}
    \bm{\alpha} \leftarrow \bm{\alpha}-\beta \nabla_{\bm{\alpha}} \sum_i \ell\left(\bm{w}^{i};\mathcal{Q}^i\right)
\end{equation}

\fakeparagraph{Imposing Sparsity on Inner Step Sizes}
To facilitate layer selection, we enforce sparsity in $\bm{\alpha}$, \ie encouraging a subset of layers to be selected for updating.
Specifically, we add a Lasso regularization term in the loss function of \equref{ch4-eq:alpha} when optimizing $\bm{\alpha}$.
Hence, the final optimization of $\bm{\alpha}$ in the outer loop is formulated as
\begin{equation}\label{ch4-eq:regularization}
    \bm{\alpha} \leftarrow \bm{\alpha}-\beta \nabla_{\bm{\alpha}} (\sum_i \ell\left(\bm{w}^{i};\mathcal{Q}^i\right)+\lambda \sum_{l,k} m(\bm{x}_{l-1})\cdot|\alpha_l^k|) \
\end{equation}
where $\lambda$ is a positive scalar to control the ratio between two terms in the loss function.
We empirically set $\lambda=0.001$.
$|\alpha_l^k|$ is re-weighted by $m(\bm{x}_{l-1})$, which denotes the necessary memory in \equref{ch4-eq:MemoryContrib} if only updating the weights in layer $l$. 

\subsubsection{Exploiting Sparse Inner Step Sizes for on-device learning}
\label{ch4-sec:LR_fsl}
We now explain how to apply the learned $\bm{\alpha}$ to save memory during on-device learning.
After deploying the meta-trained model to IoT devices for few-shot learning, at updating step $k$, for layers with $\alpha_l^k=0$, the activations (\ie their inputs) $\bm{x}_{l-1}$ need not be stored, see \equref{ch4-eq:memoryAll} and \equref{ch4-eq:memorySimple} in \secref{ch4-sec:analysis}. 
In addition, we do not need to calculate the corresponding weight gradients $\bm{g}(\bm{w}_l)$, which saves computation, see \equref{ch4-eq:MAC} in \secref{ch4-sec:analysis}.

\subsection{Selecting Adaption-Critical Channels within Layers via Sparse Meta Attention}
\label{ch4-sec:attention}

\begin{figure}[t]
    \centering
    \includegraphics[width=0.99\textwidth]{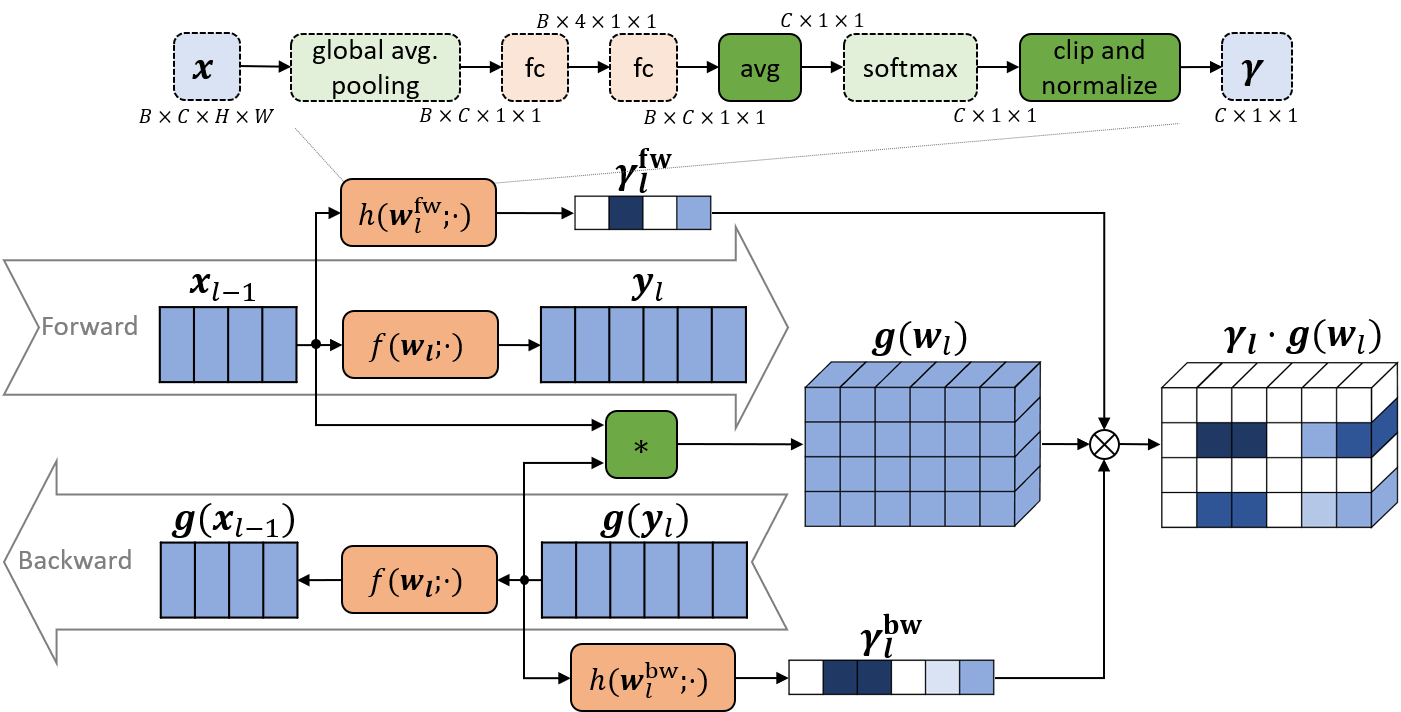} 
    \caption[Meta attention during meta-training.]{Meta attention of layer $l$ during meta-training. The blue blocks correspond to tensors; the orange blocks correspond to computation units with parameters, and the green ones without. Each column of a tensor corresponds to one channel. The input tensor $\bm{x}_{l-1}$ has 4 channels; the output tensor $\bm{y}_{l}$ has 6 channels. The other dimensions (\eg height, width and batch) are omitted here. The green block with $*$ stands for the operations involved to compute $\bm{g}(\bm{w}_l)$. In order to compute the gradients of the parameters in meta attention, \ie $\bm{w}_l^{\mathrm{fw}}$ and $\bm{w}_l^{\mathrm{bw}}$, the full dense gradients $\bm{g}(\bm{w}_l)$ are computed during meta-training, and then are masked by $\bm{\gamma}_l$. An example meta attention module for a \texttt{conv} layer is shown in the upper part. $B$ denotes the batch size. The newly added blocks related to the inference attention in \cite{bib:CVPR20:Chen} are marked with solid lines.}
    \label{ch4-fig:metaattention}
\end{figure}

This subsection explains how \pMeta learns a novel meta attention mechanism in each layer to dynamically select adaptation-critical channels for further memory saving in few-shot learning.
Despite the widespread adoption of channel-wise attention for inference \cite{bib:CVPR18:Hu, bib:CVPR20:Chen}, we make the first attempt to use attention for memory-efficient training (few-shot learning in our case).
For each layer, its meta attention outputs a dynamic channel-wise sparse attention score based on the samples from new tasks.
The sparse attention score is used to re-weight (also sparsify) the weight gradients.
Therefore, by calculating only the nonzero gradients of critical weights within a layer, we can save both memory and computation.
We first present our meta attention mechanism during meta training (\secref{ch4-sec:attention_ml}) and then show its usage for on-device model training (\secref{ch4-sec:attention_fsl}).

\subsubsection{Learning Sparse Meta Attention in Meta Training}
\label{ch4-sec:attention_ml}

Since mainstream backbones in meta learning use small kernel sizes (1 or 3), we design the meta attention mechanism channel-wise.
\figref{ch4-fig:metaattention} illustrates the attention design during meta-training.

\fakeparagraph{Learning Meta Attention}
The attention mechanism is as follows.
\begin{itemize}
    \item 
    We assign an attention score to the weight gradients of layer $l$ in the inner loop of meta training.
    The attention scores are expected to indicate which weights/channels are important and thus should be updated in layer $l$.
    \item
    The attention score is obtained from two attention modules: one taking $\bm{x}_{l-1}$ as input in the forward pass, and the other taking $\bm{g}(\bm{y}_l)$ as input during the backward pass.
    We use $\bm{x}_{l-1}$ and $\bm{g}(\bm{y}_l)$ to calculate the attention scores because they are used to compute the weight gradients $\bm{g}(\bm{w}_{l})$.
\end{itemize}
Concretely, we define the forward and backward attention scores for a \texttt{conv} layer as,
\begin{equation}\label{ch4-eq:atten_fw}
    \bm{\gamma}^{\mathrm{fw}}_{l} = h(\bm{w}^{\mathrm{fw}}_l;\bm{x}_{l-1})\in\mathbb{R}^{C_{l-1}\times 1 \times 1}
\end{equation}
\begin{equation}\label{ch4-eq:atten_bw}
    \bm{\gamma}^{\mathrm{bw}}_{l} = h(\bm{w}^{\mathrm{bw}}_l;\bm{g}(\bm{y}_l))\in\mathbb{R}^{C_{l}\times 1 \times 1}
\end{equation}
where $h(\cdot;\cdot)$ stands for the meta attention module, and $\bm{w}^{\mathrm{fw}}_l$ and $\bm{w}^{\mathrm{bw}}_l$ are the parameters of the meta attention modules.
The overall (sparse) attention scores $\bm{\gamma}_l\in\mathbb{R}^{C_{l}\times C_{l-1} \times 1 \times 1}$ and is computed as, 
\begin{equation}\label{ch4-eq:atten_fwbw}
    \gamma_{l,ba11} = \gamma^{\mathrm{fw}}_{l,a11} \cdot \gamma^{\mathrm{bw}}_{l,b11} 
\end{equation}
In the inner loop, for layer $l$, step $k$ and task $i$, $\bm{\gamma}_l$ is (broadcasting) multiplied with the dense weight gradients to get the sparse ones,
\begin{equation}\label{ch4-eq:inner_incr}
    \bm{\gamma}_l^{i,k}\odot\nabla_{\bm{w}_l}~\ell\left(\bm{w}^{i,k-1}_{1:L};\mathcal{S}^i\right)
\end{equation}
The weights are then updated by,
\begin{equation}\label{ch4-eq:inner_atten}
    \bm{w}_l^{i,k} = \bm{w}_l^{i,k-1}-\alpha^k_l(\bm{\gamma}_l^{i,k}\odot\nabla_{\bm{w}_l}~\ell\left(\bm{w}^{i,k-1}_{1:L};\mathcal{S}^i\right))
\end{equation}
Let all attention parameters be $\bm{w}^{\mathrm{atten}} = \{\bm{w}^{\mathrm{fw}}_l,\bm{w}^{\mathrm{bw}}_l\}_{l=1}^{L}$. 
The attention parameters $\bm{w}^{\mathrm{atten}}$ are optimized in the outer loop as,
\begin{equation}\label{ch4-eq:outer_atten}
    \bm{w}^{\mathrm{atten}} \leftarrow \bm{w}^{\mathrm{atten}}-\beta \nabla_{\bm{w}^{\mathrm{atten}}} \sum_i \ell\left(\bm{w}^{i};\mathcal{Q}^i\right)
\end{equation}
Note that we use a dense forward path and a dense backward path in both meta-training and on-device learning, as shown in \figref{ch4-fig:metaattention}.
That is, the attention scores $\bm{\gamma}^{\mathrm{fw}}_{l}$ and $\bm{\gamma}^{\mathrm{bw}}_{l}$ are only calculated locally and will not affect $\bm{y}_l$ during forward and $\bm{g}(\bm{x}_{l-1})$ during backward.
Based on our experimental results in \tabref{ch4-tab:sparse}, using either sparse $\bm{x}_{l-1}$ during forward or sparse $\bm{g}(\bm{y}_l)$ during backward will cause a dramatic performance degradation.  

\begin{algorithm}[t]
    \caption{Clip and normalization}\label{ch4-alg:clip}
    \KwIn{softmax output (normalized) $\bm{\pi}\in\mathbb{R}^{C}$, clip ratio $\rho$} 
    \KwOut{sparse $\bm{\gamma}$}
    Sort $\bm{\pi}$ in ascending order and get sorted indices $d_{1:C}$\; 
    Find the smallest $c$ such that $\sum_{i=1}^c\pi_{d_i}\ge\rho$\;
    Set $\pi_{d_1:d_c}$ as 0 \tcp*{if $\rho=0$, do nothing}
    Normalize $\bm{\gamma} = \bm{\pi}/\sum\bm{\pi}$\;
    Re-scale $\bm{\gamma} = \bm{\gamma}\cdot C$ \tcp*{keeping step sizes' magnitude}
\end{algorithm}

\begin{algorithm}[t]
    \caption{p-Meta}\label{ch4-alg:pMeta}
    \KwIn{meta-training task distribution $p(\mathsf{T})$, backbone $F$ with initial weights $\bm{w}$, meta attention parameters $\bm{w}^{\mathrm{atten}}$, inner step sizes $\bm{\alpha}$, outer step sizes $\beta$} 
    \KwOut{meta-trained weights $\bm{w}$, meta-trained meta attention parameters $\bm{w}^{\mathrm{atten}}$, meta-trained sparse inner step sizes $\bm{\alpha}$}
    \While {not done} {
        Sample a batch of $I$ tasks $\mathsf{T}^i\sim p(\mathsf{T})$\;
        \For {$i \leftarrow 1$ \KwTo $I$} {
            Update $\bm{w}^i$ in $K$ gradient descent steps with \eqref{ch4-eq:inner_atten}\;
        }
        Update $\bm{w}$ with \eqref{ch4-eq:outer_alpha}\;
        Update inner step sizes $\bm{\alpha}$ with \eqref{ch4-eq:regularization}\; 
        Update attention parameters $\bm{w}^{\mathrm{atten}}$ with \equref{ch4-eq:outer_atten}\;
    }
\end{algorithm}

\fakeparagraph{Meta Attention Module Design}
\figref{ch4-fig:metaattention} (upper part) shows an example meta attention module.
We adapt the inference attention modules used in \cite{bib:CVPR18:Hu, bib:CVPR20:Chen}, yet with the following modifications.
\begin{itemize}
    \item 
    Unlike inference attention that applies to a single sample, training may calculate the averaged loss gradients based on a batch of samples.
    Since $\bm{g}(\bm{w}_l)$ does not have a batch dimension, the input to softmax function is first averaged over the batch data, see in \figref{ch4-fig:metaattention}.
    \item
    We enforce sparsity on the meta attention scores such that they can be utilized to save memory and computation in few-shot learning.
    The original attention in \cite{bib:CVPR18:Hu, bib:CVPR20:Chen} outputs normalized scales in $[0,1]$ from softmax.
    We clip the output with a clip ratio $\rho\in[0,1]$ to create zeros in $\bm{\gamma}$.
    This way, our meta attention modules yield batch-averaged sparse attention scores $\bm{\gamma}^{\mathrm{fw}}_{l}$ and $\bm{\gamma}^{\mathrm{bw}}_{l}$. 
    \algoref{ch4-alg:clip} shows this clipping and re-normalization process.
    Note that \algoref{ch4-alg:clip} is not differentiable.
    Hence we use the straight-through-estimator for its backward propagation in meta training. 
\end{itemize}

\begin{figure}[t]
    \centering
    \includegraphics[width=0.99\textwidth]{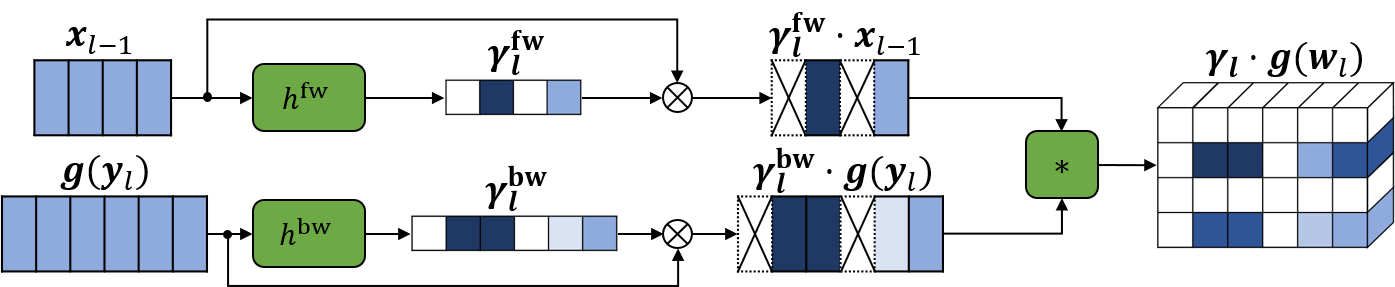} 
    \caption[Meta attention during on-device few-shot learning.]{Meta attention of layer $l$ during on-device few-shot learning. Note that ``Forward'' part and ``Backward'' part are the same as \figref{ch4-fig:metaattention}, which are omitted for simplicity. Meta attention modules are not optimized during few-shot learning, thus are expressed as parameter-free functions $h^{\mathrm{fw}}$ and $h^{\mathrm{bw}}$. The input $\bm{x}_{l-1}$ stored during forward path is a sparse re-weighted tensor. }
    \label{ch4-fig:metaattention_fewshot}
\end{figure}

\subsubsection{Exploiting Meta Attention for on-device learning}
\label{ch4-sec:attention_fsl}
We now explain how to apply the meta attention to save memory during on-device few-shot learning.
Note that the parameters in the meta attention modules are fixed during few-shot learning.
Assume that at step $k$, layer $l$ has a nonzero step size $\alpha_l^k$.
In the forward pass, we only store a sparse tensor $\bm{\gamma}_l^{\mathrm{fw}}\cdot\bm{x}_{l-1}$, \ie its channels are stored only if they correspond to nonzero entries in $\bm{\gamma}_l^{\mathrm{fw}}$.
This reduces memory consumption as shown in \equref{ch4-eq:memorySimple} in \secref{ch4-sec:analysis}.
Similarly, in the backward pass, we get a channel-wise sparse tensor $\bm{\gamma}_l^{\mathrm{bw}}\cdot\bm{g}(\bm{y}_{l})$.
Since both sparse tensors are used to calculate the corresponding nonzero gradients in $\bm{g}(\bm{w}_l)$, the computation cost is also reduced, see \equref{ch4-eq:MAC} in \secref{ch4-sec:analysis}.
We plot the meta attention during on-device learning in \figref{ch4-fig:metaattention_fewshot}.

\subsection{Summary of \pMeta}
\label{ch4-sec:overall_alg}

\algoref{ch4-alg:pMeta} shows the overall process of \pMeta during meta-training. 
The final meta-trained weights $\bm{w}$ from \algoref{ch4-alg:pMeta} are assigned to $\bm{w}^{\mathrm{meta}}$, see \secref{ch4-sec:preliminary_meta}.
The meta-trained backbone model $F(\bm{w}^{\mathrm{meta}})$, the sparse inner step sizes $\bm{\alpha}$, and the meta attention modules will be then deployed on edge devices and used to conduct a memory-efficient few-shot learning.

\subsection{Deployment Optimization}
\label{ch4-sec:others}
To further reduce the memory during few-shot learning, we propose gradient accumulation during backpropagation and replace batch normalization in the backbone with group normalization.

\subsubsection{Gradient Accumulation}
\label{ch4-sec:gradacc}
In standard few-shot learning, all the new samples (\eg $25$ for $5$-way $5$-shot) are fed into the model as one batch.
To reduce the peak memory due to large batch sizes, we conduct few-shot learning with gradient accumulation (GA).

GA is a technique that (\textit{i}) breaks a large batch into smaller partial batches; (\textit{ii}) sequentially forward/backward propagates each partial batches through the model; (\textit{iii}) accumulates the loss gradients of each partial batch and get the final averaged gradients of the full batch. 
Note that GA does not increase computation, which is desired for low-resource platforms with constrained memory and limited parallelism. 
Accordingly, our meta attention module should be also modified. 
Particularly, the input to the softmax is averaged over all samples in the batch (see \figref{ch4-fig:metaattention}), \ie $\bm{\gamma}^{\mathrm{fw}}_{l}$ and $\bm{\gamma}^{\mathrm{bw}}_{l}$ are the batch-averaged scores.
We evaluate the impact of different sample batch sizes in GA in \secref{ch4-sec:experiment_batchsize}.

\subsubsection{Group Normalization}
\label{ch4-sec:norm}
Mainstream backbones in meta learning typically adopt batch normalization layers.
Batch normalization layers compute the statistical information in each batch, which is dependent on the sample batch size.
When using GA with different sample batch sizes, the inaccurate batch statistics can degrade the training performance (see \secref{ch4-sec:experiment_pool}).
As a remedy, we use group normalization \cite{bib:ECCV18:Wu}, which does not rely on batch statistics (\ie independent of the sample batch size).
We also apply meta attention on group normalization layers when updating their weights.
The only difference w.r.t. \texttt{conv} and \texttt{fc} layers is that the stored input tensor (also the one used for the meta attention) is not $\bm{x}_{l-1}$, but its normalized version.

\section{Theoretical Analysis on Memory and Computation}
\label{ch4-sec:analysis}

In this section, we derive the memory requirement and computation workload for inference and training. 
We further analyze the reduced consumption of memory and computation due to \pMeta.

Recall that we focus on a feed forward DNN that consists of $L$ convolutional (\texttt{conv}) layers or fully-connected (\texttt{fc}) layers. 
Note that our analysis focuses on 2D \texttt{conv} layers but can apply to other \texttt{conv} layer types as well. We assume the ReLU activation function for all layers, denoted as $\sigma(\cdot)$. For simplicity, we omit the bias, normalization layers, pooling or strides. 
We use the notation $m(\bm{x})$ to denote the memory demand in words to store tensor $\bm{x}$. The wordlength is denoted as $\mathit{T}$. 

For representing indexed summations we use the Einstein notation. If index variables appear in a term on the right hand side of an equation and are not otherwise defined (free indices), it implies summation of that term over the range of the free indices. If indices of involved tensor elements are out of range, the values of these elements are assumed to be 0.

\subsection{Single Layer}

We start with a single layer and accumulate the memory and computation for networks with several layers afterwards. Assume the input tensor of a layer is $\bm{x}$, the weight tensor is $\bm{w}$, the result after the linear transformation is $\bm{y}$, and the layer output after the non-linear operator is $\bm{z}$ which is also the input to the next layer. 

For convolutional layers, we have $\bm{x} \in \mathbb{R}^{C_I \times H_I \times W_I}$ and elements $x_{cij}$, where $C_I$, $H_I$, and $W_I$ denote the number of input channels, height and width, respectively. In a similar way, we have $\bm{z} \in \mathbb{R}^{C_O \times H_O \times W_O}$ with elements $x_{fij}$ where $C_O$, $H_O$, and $W_O$ denote the number of output channels, height and width, respectively. Moreover, $\bm{w} \in \mathbb{R}^{C_O \times C_I \times S \times S}$ with elements $w_{fcmn}$. Therefore,
\[
    m(\bm{x})= C_I H_I W_I \; , \; \; m(\bm{y})= m(\bm{z}) = C_O H_O W_O \; , \; \; m(\bm{w}) = C_O C_I S^2
\]
For fully connected layers we have $\bm{x} \in \mathbb{R}^{C_I}$, $\bm{y}, \bm{z} \in \mathbb{R}^{C_O}$, and $\bm{w} \in \mathbb{R}^{C_O \times C_I}$ with memory demand
\[
    m(\bm{x})= C_I  \; , \; \; m(\bm{y})= m(\bm{z}) = C_O  \; , \; \; m(\bm{w}) = C_O C_I 
\]

\subsubsection{Fully Connected Layer}
For inference we derive the relations $y_{f} = w_{fc} x_c$ and $z_{f} = \sigma(y_{f})$ for all admissible indices $f \in [1, C_O]$. The necessary dynamic memory has a size of about $m(\bm{x}) + m(\bm{y})$ words and we need about $m(\bm{w})$ FLOPs. 

For training, we suppose that $\frac{\partial \ell}{\partial z_i}$ is already provided from the next layer. We find $\frac{\partial \ell}{\partial y_i} = \sigma'(y_i) \cdot \frac{\partial \ell}{\partial z_i}$ with $\sigma'(y_i) = \begin{cases} 1 & \mbox{if } y_i > 0 \\ 0 & \mbox{if } y_i < 0  \end{cases}$ which leads to $\frac{\partial \ell}{\partial x_i} = w_{ji} \cdot \frac{\partial \ell}{\partial y_j}$. The necessary dynamic memory is about $m(\bm{x}) + m(\bm{y}) \cdot (1 + \frac{1}{\mathit{T}})$ words, where the last term comes from storing $\sigma'(y_i)$ single bits from the forward path. We need about $m(\bm{w})$ FLOPs.

According to the approach described in the paper we are only interested in the partial derivatives $\frac{\partial \ell}{\partial w_{fc}}$ if $\alpha > 0$ for this layer, and if scales $\gamma^{\mathrm{bw}}_{f} > 0$ and $\gamma^{\mathrm{fw}}_{c} > 0$ for indices $f$, $c$. To simplify the notation, let us define the critical ratios 
\begin{gather}
    \mu^{\mathrm{fw}} = \frac{\text{number of nonzero elements of } \gamma^{\mathrm{fw}}_{c}}{C_I} \\
    \mu^{\mathrm{bw}} = \frac{\text{number of nonzero elements of } \gamma^{\mathrm{bw}}_{f}}{C_O}
\end{gather}
which are 1 if all channels are determined to be critical for weight adaptation, and 0 if none of them. 

We find $\gamma^{\mathrm{bw}}_{f} \frac{\partial \ell}{\partial w_{fc}}  \gamma^{\mathrm{fw}}_{c}  =  (\gamma^{\mathrm{bw}}_{f} \frac{\partial \ell}{\partial y_{f}}) \cdot  (\gamma^{\mathrm{fw}}_{c} x_c)$. Therefore, we need $\mu^{\mathrm{fw}} \mu^{\mathrm{bw}} m(\bm{w}) + \mu^{\mathrm{fw}} m(\bm{x})$ words dynamic memory if $\alpha > 0$ where the latter term considers the information needed from the forward path. We require about $\mu^{\mathrm{fw}} \mu^{\mathrm{bw}} m(\bm{w})$ FLOPs if $\alpha > 0$. 

\subsubsection{Convolutional Layer}
The memory analysis for a convolutional layer is very similar, just replacing matrix multiplication by convolution. For inference we find $y_{fij} = w_{fcmn} x_{c, i+m-1, j+n-1}$ and $z_{fij} = \sigma(y_{fij})$ for all admissible indices $f$, $i$, $j$. The necessary dynamic memory has a size of about $\max \{ m(\bm{x}), m(\bm{y}) \}$ words when using memory sharing between input and output tensors. We need about $H_O W_O \cdot m(\bm{w})$ FLOPs. 

For training, we again suppose that $\frac{\partial \ell}{\partial z_{fij}}$ is provided from the next layer. We find $\frac{\partial \ell}{\partial y_{fij}} = \sigma'(y_{fij}) \cdot \frac{\partial \ell}{\partial z_{fij}}$ and get $\frac{\partial \ell}{\partial x_{cij}} = w_{fcmn}  \cdot \frac{\partial \ell}{\partial y_{f, i + m - 1, j + n - 1}}$. The necessary memory is about $\max \{ m(\bm{x}), m(\bm{y}) \} +  \frac{ m(\bm{y})}{\mathit{T}}$ words, where the last term comes from storing $\sigma'(y_{fij})$ single bits from the forward path. We need about $H_I W_I \cdot m(\bm{w})$ multiply and accumulate operations.

For determining the weight gradients we find $\frac{\partial \ell}{\partial w_{fcmn}}  =  \frac{\partial \ell}{\partial y_{fij}} \cdot x_{c, i + m -1, j + n - 1}$. When considering the scales for filtering, we yield $\gamma^{\mathrm{bw}}_{f} \frac{\partial \ell}{\partial w_{fcmn}}  \gamma^{\mathrm{fw}}_{c}  =  (\gamma^{\mathrm{bw}}_{f} \frac{\partial \ell}{\partial y_{fij}}) \cdot  (\gamma^{\mathrm{fw}}_{c} x_{c, i + m -1, j + n - 1})$. As a result, we need $\mu^{\mathrm{fw}} \mu^{\mathrm{bw}} m(\bm{w}) +  \mu^{\mathrm{fw}} m(\bm{x})$ words of dynamic memory if $\alpha > 0$ where the latter term considers the information needed from the forward path. We require about $\mu^{\mathrm{fw}} \mu^{\mathrm{bw}} H_O W_O m(\bm{w})$ FLOPs if $\alpha > 0$. 

Finally, let us determine the required memory and computation to determine the scales $\gamma^{\mathrm{fw}}_{c}$ and $\gamma^{\mathrm{bw}}_{f}$. According to \figref{ch4-fig:metaattention}, we find as an upper bound for the memory $B \cdot ( C_I + C_O)$ and $( C_I H_I W_I + 2 C_I^2 + C_O H_O W_O + 2 C_O^2 )$ FLOPs.

\subsection{All Layers}

The above relations are valid for a single layer. The following relations hold for the overall network. In order to simplify the notation, we consider a network that consists of convolution layers only. Extensions to mixed layers can simply be done.

We suppose $L$ layers with sizes $C_l$, $H_l$, $W_l$ and $S_l$ for the number of output channels, output width, output height and kernel size, respectively. We assume that the step-sizes $\alpha_l$ for some iteration of the training are given. The memory requirement in words is
\[
    m(\bm{x}_l)= C_l H_l W_l \; , \; \; m(\bm{w}_l) = C_l C_{l-1} S_l^2
\]
and the word-length is again denoted as $\mathit{T}$.
We define as $\hat{\alpha}_l = \begin{cases} 1 & \mbox{if } \alpha_l > 0 \\ 0 & \mbox{if } \alpha_l = 0  \end{cases}$ the mask that determines whether the weight adaptation for this layer is necessary or not.

Let us first look at the forward path. The necessary dynamic memory is about $\max_{0 \leq l \leq L} \{ m(\bm{x}_l) \}$ words. The number of FLOPs is $\sum_{1 \leq l \leq L} H_l W_l m(\bm{w}_l)$. 

The backward path needs only to be evaluated until we reach the first layer where we require the computation of the gradients. We define $l_\mathit{min} = \min \{ l \, | \, \hat{\alpha}_{l} = 1 \}$. For the calculation of the partial derivatives of the activations we need dynamic memory of $\max_{l_{\mathit{min}} \leq l \leq L} \{ m(\bm{x}_l) \} + \frac{1}{\mathit{T}} \sum_{l_\mathit{min} \leq l \leq L} m(\bm{x}_l)$ words where the last term is due to storing the derivatives of the ReLU operations. We need about $\sum_{l_\mathit{min} + 1 \leq l \leq L} H_{l-1} W_{l-1} m(\bm{w}_l) $ FLOPs.

The second contribution of the backward path is for computing the weight gradients. The memory and computation demand of the scales will be neglected as they are much smaller than other contributions. We can determine the necessary dynamic memory as $\max_{1 \leq l \leq L} \{ \hat{\alpha}_l \mu^{\mathrm{fw}}_l \mu^{\mathrm{bw}}_l m(\bm{w}_l)\} + \sum_{1 \leq l \leq L} \hat{\alpha}_l \mu^{\mathrm{fw}}_l m(\bm{x}_{l-1})$, and we need $\sum_{1 \leq l \leq L}  \hat{\alpha}_l \mu^{\mathrm{fw}}_l \mu^{\mathrm{bw}}_l H_l W_l m(\bm{w}_l)$ FLOPs. 

Considering all necessary dynamic memory with memory reuse for a gradient-based training step, we get an estimation of memory in words
\begin{multline}
    \max_{0 \leq l \leq L} \{ m(\bm{x}_l) \}  + \sum_{1 \leq l \leq L} \hat{\alpha}_l m(\bm{w}_l) + \sum_{1 \leq l \leq L} \hat{\alpha}_l \mu^{\mathrm{fw}}_l m(\bm{x}_{l-1}) + \frac{1}{\mathit{T}} \sum_{l_\mathit{min} \leq l \leq L} m(\bm{x}_l)
    \label{ch4-eq:memoryAll}
\end{multline}
if we accumulate the weight gradients before doing an SGD step and re-use some memory during back-propagation. More elaborate memory re-use can be used to slightly sharpen the bounds without a major improvement.
For conventional training, each parameter is in 32-bit floating point format, \ie one word corresponds to 32-bit. As discussed in \secref{ch4-sec:preliminary_memory}, we only consider max-pooling and ReLU-styled activation as the $\sigma$ function. The wordlength $T$ in \equref{ch4-eq:memoryAll} is set as 16 for max-pooling , and 32 for ReLU-styled activation.
One can see that under the typical assumptions for network parameters, the above memory requirement in words is dominated by 
\begin{equation}
    \sum_{1 \leq l \leq L} \hat{\alpha}_l \mu^{\mathrm{fw}}_l m(\bm{x}_{l-1})
    \label{ch4-eq:memorySimple}
\end{equation}
The necessary storage between the forward and backward path is reduced proportionally to $\mu^{\mathrm{fw}}_l$ with factor $m(\bm{x}_{l-1})$. 

Finally, the amount of FLOPs can be estimated as 
\begin{equation}
    \sum_{1 \leq l \leq L} H_l W_l m(\bm{w}_l) ( 1 + \hat{\alpha}_l \mu^{\mathrm{fw}}_l \mu^{\mathrm{bw}}_l) + \sum_{l_\mathit{min} \leq l \leq L} H_{l-1} W_{l-1} m(\bm{w}_l)
    \label{ch4-eq:MAC}
\end{equation}
while neglecting lower order terms. Here it is important to note that all terms are of similar order. The approach used in the paper does not determine a trade-off between computation and memory, but reduces the amount of FLOPs. This reduction is less than the reduction in required dynamic memory.

\section{Experiments}
\label{ch4-sec:experiment}

This section presents the evaluations of \pMeta on standard few-shot image classification and reinforcement learning benchmarks.

\subsection{General Experimental Settings}
\label{ch4-sec:experiment_settings}

\fakeparagraph{Compared Methods}
We test the meta learning algorithms below. 
\begin{itemize}
    \item MAML \cite{bib:ICML17:Finn}: the original model-agnostic meta learning.
    \item ANIL \cite{bib:ICLR20:Raghu}: update the last layer only in few-shot learning.
    \item BOIL \cite{bib:ICLR21:Oh}: update the body except the last layer.
    \item MAML++ \cite{bib:ICLR19:Antreas}: learn a per-step per-layer step sizes $\bm{\alpha}$.
    \item \pMeta (\ref{ch4-sec:LR}): can be regarded as a sparse version of MAML++, since it learns a sparse $\bm{\alpha}$ with our methods in \secref{ch4-sec:LR}. 
    \item \pMeta (\ref{ch4-sec:LR}+\ref{ch4-sec:attention}): the full version of our methods which include the meta attention modules in \secref{ch4-sec:attention}. 
\end{itemize}
For fair comparison, all the algorithms are re-implemented with the deployment optimization in \secref{ch4-sec:others}. 

\fakeparagraph{Implementation}
The experiments are conducted with tools provided by TorchMeta \cite{bib:torchmeta, bib:torchmetarl}.
Particularly, the backbone is meta-trained with full sample batch size (\eg 25 for 5-way 5-shot) on meta training dataset.
After each meta training epoch, the model is tested (\ie few-shot learned) on meta validation dataset.
The model with the highest validation performance is used to report the final few-shot learning results on meta test dataset.
We follow the same process as TorchMeta \cite{bib:torchmeta, bib:torchmetarl} to build the dataset. 
During few-shot learning, we adopt a sample batch size of 1 to verify the model performance under the most strict memory constraints.

In \pMeta, meta attention is applied to all \texttt{conv}, \texttt{fc}, and group normalization layers, except the last output layer, because (\textit{i}) we find modifying the last layer's gradients may decrease accuracy; (\textit{ii}) the final output is often rather small in size, resulting in little memory saving even if imposing sparsity on the last layer.  
Without further notations, we set $\rho=0.3$ in forward attention, and $\rho=0$ in backward attention across all layers, as the sparsity of $\bm{\gamma}_l^{\mathrm{bw}}$ almost has no effect on the memory saving.

\fakeparagraph{Metrics}
We compare the peak memory and FLOPs of different algorithms.
Note that the reported peak memory and FLOPs for \pMeta also include the consumption from meta attention, although they are rather small related to the backward propagation.

\subsection{Benchmarking Details}
\label{ch4-sec:experiment_benchmark}

\subsubsection{4Conv/ResNet12 on MiniImageNet/TieredImageNet/CUB}

MiniImageNet \cite{bib:NIPS16:Vinyals} is an image classification dataset from ImageNet dataset \cite{bib:ILSVRC15}, which consists of $84\times84$ color images in 100 classes.
Following the splitting in \cite{bib:NIPS16:Vinyals}, 64 classes are used for meta-training, 16 classes are used for meta-validation, and the rest 20 classes are used as unseen tasks for meta-testing (\ie few-shot learning). 
We train on 1 Nvidia V100 GPU.
We experiment in both 5-way 1-shot and 5-way 5-shot settings.
The task batch size is set to 4 in general, except for ResNet12 under 5-way 5-shot settings where we use 2.

TieredImageNet \cite{bib:ICLR18:Ren} is an image classification dataset from ImageNet dataset \cite{bib:ILSVRC15}, which consists of $84\times84$ color images in 34 categories (608 classes).
Following the splitting in \cite{bib:ICLR18:Ren}, 20 categories (351 classes) are used for meta-training, 6 categories (97 classes) are used for meta-validation, and the rest 8 categories (160 classes) are used as unseen tasks for meta-testing (\ie few-shot learning). 

CUB \cite{bib:CUB} is an image classification dataset, which consists of $84\times84$ color images of bird species in 200 classes.
Following the splitting in \cite{bib:torchmeta}, 100 classes are used for meta-training, 50 classes are used for meta-validation, and the rest 50 classes are used as unseen tasks for meta-testing (\ie few-shot learning). 
 
\fakeparagraph{4Conv} 
The ``4Conv'' \cite{bib:ICML17:Finn} backbone has 4 \texttt{conv} blocks.
Each \texttt{conv} block includes a \texttt{conv} layer with 32 channels, a group normalization layer (as discussed in \secref{ch4-sec:norm}), a ReLU activation, and a max-pooling with stride 2.

\fakeparagraph{ResNet12}
The ``ResNet12'' \cite{bib:NIPS18:Oreshkin} backbone has 4 residual blocks with $\{64,128,256,512\}$ channels in each block respectively.
Each residual block consists of 3 \texttt{conv} blocks followed by max-pooling with stride 2. 
Each \texttt{conv} layer is followed by a group normalization layer and a LeakyReLU activation with slope 0.1.
Refer to \cite{bib:NIPS18:Oreshkin} for more detailed structure.

\subsubsection{MLP on MuJoCo}

MuJoCo is an advanced simulator for multi-body dynamics with contact. 
For all experiments, we mainly adopt the experimental setup in \cite{bib:ICML17:Finn,bib:torchmetarl}.
We run the MuJoCo environment as well as the policy model training on 8 CPUs.

\fakeparagraph{MLP} 
We use a neural network as the policy model. The neural network is a MLP with two hidden \texttt{fc} layers of size 100 and the ReLU activation.

\subsection{Experiments on Image Classification}
\label{ch4-sec:image}

\begin{table}[t!]
    \centering
    \caption[5-Way 1-shot few-shot image classification results on 4Conv and ResNet12.]{5-Way 1-shot few-shot image classification results on 4Conv and ResNet12. All methods are meta-trained on MiniImageNet, and are few-shot learned on the reported datasets: MiniImageNet, TieredImageNet, and CUB (denoted by Mini, Tiered, and CUB in the table). The total computation (\# GFLOPs) and the peak memory (MB) during few-shot learning are reported based on the theoretical analysis in \secref{ch4-sec:analysis}. }
    \label{ch4-tab:image_1shot}
    \footnotesize
    \begin{tabular}{llccccc}
        \toprule                                               
        \multicolumn{2}{l}{\textbf{5-Way 1-Shot}}                                           & \multicolumn{3}{c}{Accuracy}                          & GFLOPs    & Memory        \\
                                                                                            \cmidrule(lr){3-5}                                      \cmidrule(lr){6-6} \cmidrule(lr){7-7}
        \multicolumn{2}{l}{Benchmarks}                                                      & Mini   & Tiered & CUB                                 & Mini      & Mini          \\ \midrule
        \multirow{6}{*}{4Conv}      & MAML \cite{bib:ICML17:Finn}                           & 46.2\% & 51.4\% & 39.7\%                              & 0.39      & 2.06          \\
                                    & ANIL \cite{bib:ICLR20:Raghu}                          & 46.4\% & 51.5\% & 39.2\%                              & 0.14      & 0.92          \\
                                    & BOIL \cite{bib:ICLR21:Oh}                             & 44.7\% & 51.3\% & 42.3\%                              & 0.39      & 2.05          \\
                                    & MAML++ \cite{bib:ICLR19:Antreas}                      & 48.2\% & 53.2\% & \textbf{43.2\%}                     & 0.39      & 2.06          \\
                                    & \pMeta (\ref{ch4-sec:LR})                             & 47.1\% & 52.3\% & 41.8\%                              & 0.16      & 1.00          \\
                                    & \pMeta (\ref{ch4-sec:LR}+\ref{ch4-sec:attention})     & \textbf{48.8\%} & \textbf{53.9\%} & 42.6\%            & 0.15      & 0.99          \\ \midrule
        \multirow{6}{*}{ResNet12}   & MAML \cite{bib:ICML17:Finn}                           & 51.7\% & 57.4\% & 41.3\%                              & 37.08     & 54.69         \\
                                    & ANIL \cite{bib:ICLR20:Raghu}                          & 50.3\% & 56.7\% & 40.6\%                              & 12.42     & 3.62          \\
                                    & BOIL \cite{bib:ICLR21:Oh}                             & 42.7\% & 47.7\% & 44.2\%                              & 37.08     & 54.69         \\
                                    & MAML++ \cite{bib:ICLR19:Antreas}                      & 53.1\% & 58.6\% & 45.1\%                              & 37.08     & 54.69         \\
                                    & \pMeta (\ref{ch4-sec:LR})                             & 51.8\% & 58.3\% & 40.6\%                              & 25.84     & 17.66         \\
                                    & \pMeta (\ref{ch4-sec:LR}+\ref{ch4-sec:attention})     & \textbf{53.6\%} & \textbf{59.4\%} & \textbf{45.4\%}   & 24.02     & 16.01         \\
    \bottomrule
    \end{tabular}
\end{table}

\begin{table}[t!]
    \centering
    \caption[5-Way 5-shot few-shot image classification results on 4Conv and ResNet12.]{5-Way 5-shot few-shot image classification results on 4Conv and ResNet12. All methods are meta-trained on MiniImageNet, and are few-shot learned on the reported datasets: MiniImageNet, TieredImageNet, and CUB (denoted by Mini, Tiered, and CUB in the table). The total computation (\# GFLOPs) and the peak memory (MB) during few-shot learning are reported based on the theoretical analysis in \secref{ch4-sec:analysis}. }
    \label{ch4-tab:image_5shot}
    \footnotesize
    \begin{tabular}{llccccc}
        \toprule
        \multicolumn{2}{l}{\textbf{5-Way 5-Shot}}                                           & \multicolumn{3}{c}{Accuracy}                          & GFLOPs    & Memory        \\
                                                                                            \cmidrule(lr){3-5}                                      \cmidrule(lr){6-6} \cmidrule(lr){7-7}
        \multicolumn{2}{l}{Benchmarks}                                                      & Mini   & Tiered & CUB                                 & Mini      & Mini          \\ \midrule
        \multirow{6}{*}{4Conv}      & MAML \cite{bib:ICML17:Finn}                           & 61.4\% & 66.5\% & 55.6\%                              & 1.96      & 2.06          \\
                                    & ANIL \cite{bib:ICLR20:Raghu}                          & 60.6\% & 64.5\% & 54.2\%                              & 0.72      & 0.92          \\
                                    & BOIL \cite{bib:ICLR21:Oh}                             & 60.5\% & 65.3\% & 58.3\%                              & 1.96      & 2.05          \\
                                    & MAML++ \cite{bib:ICLR19:Antreas}                      & 63.7\% & \textbf{68.5\%} & 59.1\%                     & 1.96      & 2.06          \\
                                    & \pMeta (\ref{ch4-sec:LR})                             & 62.9\% & 68.3\% & 59.3\%                              & 1.34      & 1.09          \\
                                    & \pMeta (\ref{ch4-sec:LR}+\ref{ch4-sec:attention})     & \textbf{65.0\%} & \textbf{68.5\%} & \textbf{60.2\%}   & 1.11      & 1.04          \\
        \midrule
        \multirow{6}{*}{ResNet12}   & MAML \cite{bib:ICML17:Finn}                           & 64.7\% & 69.6\% & 53.8\%                              & 185.42    & 54.69         \\
                                    & ANIL \cite{bib:ICLR20:Raghu}                          & 62.3\% & 68.7\% & 54.0\%                              & 62.08     & 3.62          \\
                                    & BOIL \cite{bib:ICLR21:Oh}                             & 53.6\% & 59.8\% & 53.7\%                              & 185.42    & 54.69         \\
                                    & MAML++ \cite{bib:ICLR19:Antreas}                      & 68.6\% & \textbf{73.4\%} & 63.9\%                     & 185.42    & 54.69         \\
                                    & \pMeta (\ref{ch4-sec:LR})                             & 68.8\% & 72.6\% & 65.9\%                              & 124.15    & 18.95         \\
                                    & \pMeta (\ref{ch4-sec:LR}+\ref{ch4-sec:attention})     & \textbf{69.7\%} & 73.3\% & \textbf{66.6\%}            & 116.79    & 17.17         \\
        \bottomrule
    \end{tabular}
\end{table}

\fakeparagraph{Settings}
We test on standard few-shot image classification tasks (both in-domain and cross-domain). 
We adopt two common backbones ``4Conv'' \cite{bib:ICML17:Finn} and ``ResNet12'' \cite{bib:NIPS18:Oreshkin}.
The batch normalization layers are replaced with group normalization layers, as discussed in \secref{ch4-sec:norm}.
We train the model on MiniImageNet \cite{bib:NIPS16:Vinyals} (both meta training and meta validation dataset) with 100 meta epochs.
In each meta epoch, 1000 random tasks are drawn from the task distribution. 

The model is updated with 5 gradient steps (\ie $K=5$) in both inner loop of meta-training and few-shot learning.  
We use Adam optimizer with cosine learning rate scheduling as \cite{bib:ICLR19:Antreas} for all outer loop updating.
The (initial) inner step size $\bm{\alpha}$ is set to 0.01.
The meta-trained model is then tested on three datasets MiniImageNet \cite{bib:NIPS16:Vinyals}, TieredImageNet \cite{bib:ICLR18:Ren}, and CUB \cite{bib:CUB} to verify both \textit{in-domain} and \textit{cross-domain} performance.

\fakeparagraph{Results}
\tabref{ch4-tab:image_1shot} and \tabref{ch4-tab:image_5shot} show the accuracy of few-shot learned models for 5-way 1-shot and 5-way 5-shot scenarios respectively.
The reported accuracy is averaged over $5000$ new unseen tasks randomly drawn from the meta test dataset
We also report the average number of GFLOPs and the average peak memory per task according to \secref{ch4-sec:analysis}.
Clearly, \pMeta almost always yields the highest accuracy in all settings. 
Note that the comparison between ``\pMeta (\ref{ch4-sec:LR})'' and ``MAML++'' can be considered as the ablation studies on learning sparse layer-wise inner step sizes proposed in \secref{ch4-sec:LR}. 
Thanks to the imposed sparsity on $\bm{\alpha}$, ``\pMeta (\ref{ch4-sec:LR})'' significantly reduces the peak memory ($2.5\times$ saving on average and up to $3.1\times$) and the computation burden ($1.7\times$ saving on average and up to $2.4\times$) over ``MAML++''. 
Note that the imposed sparsity also cause a moderate accuracy drop.
However, with the meta attention, ``\pMeta (\ref{ch4-sec:LR}+\ref{ch4-sec:attention})'' not only notably improves the accuracy but also further reduces the peak memory ($2.7\times$ saving on average and up to $3.4\times$) and computation ($1.9\times$ saving on average and up to $2.6\times$) over ``MAML++''.  
Note that ``ANIL'' only updates the last layer, and therefore consumes less memory but also yields a substantially lower accuracy.

\subsection{Experiments on Reinforcement Learning}
\label{ch4-sec:reinforcement}

\begin{table}[t!]
    \centering
    \caption[Few-shot reinforcement learning results on 2D navigation and robot locomotion tasks.]{Few-shot reinforcement learning results on 2D navigation and robot locomotion tasks (larger return means better). A MLP with two hidden layers of size 100 is used as the policy model. The total computation (\# GFLOPs) and the peak memory (MB) during few-shot learning are reported based on the theoretical analysis in \secref{ch4-sec:analysis}.}
    \label{ch4-tab:reinforce}
    \footnotesize
    \begin{tabular}{lcccccc}
        \toprule
        \textbf{20 Rollouts}                                    & \multicolumn{3}{c}{Half-Cheetah Velocity}    & \multicolumn{3}{c}{2D Navigation}             \\
                                                                \cmidrule(lr){2-4}                              \cmidrule(lr){5-7}
        Benchmarks                                              & Return            & GFLOPs    & Memory        & Return            & GFLOPs    & Memory        \\ \midrule
        MAML \cite{bib:ICML17:Finn}                             & -82.2             & 0.15      & 0.24          & -13.3             & 0.12      & 0.21          \\
        ANIL \cite{bib:ICLR20:Raghu}                            & -78.8             & 0.06      & 0.09          & -13.8             & 0.04      & 0.08          \\
        BOIL \cite{bib:ICLR21:Oh}                               & -76.4             & 0.15      & 0.23          & -12.4             & 0.12      & 0.21          \\
        MAML++ \cite{bib:ICLR19:Antreas}                        & -69.6             & 0.15      & 0.24          & -17.6             & 0.12      & 0.21          \\
        p-Meta (\ref{ch4-sec:LR})                               & -65.5             & 0.11      & 0.12          & \textbf{-11.2}    & 0.09      & 0.09          \\
        p-Meta (\ref{ch4-sec:LR}+\ref{ch4-sec:attention})       & \textbf{-64.0}    & 0.11      & 0.11          & -11.8             & 0.09      & 0.09          \\
        \bottomrule
    \end{tabular}
\end{table}

\fakeparagraph{Settings}
To show the versatility of \pMeta, we experiment with two few-shot reinforcement learning problems: 2D navigation and Half-Cheetah robot locomotion simulated with MuJoCo library \cite{bib:IROS12:Todorov}.
We adopt vanilla policy gradient \cite{bib:ML1992:Williams} for the inner loop and trust-region policy optimization \cite{bib:ICML15:Schulman} for the outer loop.
During the inner loop as well as few-shot learning, the agents rollout 20 episodes with a horizon size of 200 and are updated for one gradient step.
The policy model is trained for 500 meta epochs, and the model with the best average return during training is used for evaluation.
The task batch size is set to 20 for 2D navigation, and 40 for robot locomotion.
The (initial) inner step size $\bm{\alpha}$ is set to 0.1.
Each episode is considered as a data sample, and thus the gradients are accumulated 20 times for a gradient step.

\fakeparagraph{Results}
\tabref{ch4-tab:reinforce} lists the average return averaged over 400 new unseen tasks randomly drawn from simulated environments.
We also report the average number of GFLOPs and the average peak memory per task according to \secref{ch4-sec:analysis}.
Note that the reported computation and peak memory do not include the estimations of the advantage \cite{bib:ICML16:Duan}, as they are relatively small and could be done during the rollout. 
\pMeta consumes a rather small amount of memory and computation, while often obtains the highest return in comparison to others. 
Therefore, \pMeta can fast adapt its policy to reach the new goal in the environment with less on-device resource demand. 

\subsection{Ablation Studies on Meta Attention}
\label{ch4-sec:ablation}
We study the effectiveness of our meta attention via the following two ablation studies.
The experiments are conducted on ``4Conv'' in both 5-way 1-shot and 5-way 5-shot as \secref{ch4-sec:image}.

\fakeparagraph{Sparsity in Meta Attention}
\tabref{ch4-tab:ablation} shows the few-shot classification accuracy with different sparsity settings in the meta attention. 

We first do not impose sparsity on $\bm{\gamma}_l^{\mathrm{fw}}$ and $\bm{\gamma}_l^{\mathrm{bw}}$ (\ie set both $\rho$'s as 0), and adopt forward attention and backward attention separately. 
In comparison to no meta attention at all, enabling either forward or backward attention improves accuracy. 
With both attention enabled, the model achieves the best performance.

We then test the effects when imposing sparsity on $\bm{\gamma}_l^{\mathrm{fw}}$ or $\bm{\gamma}_l^{\mathrm{bw}}$ (\ie set $\rho>0$). 
We use the same $\rho$ for all layers.
We observe a sparse $\bm{\gamma}_l^{\mathrm{bw}}$ often cause a larger accuracy drop than a sparse $\bm{\gamma}_l^{\mathrm{fw}}$.
Since a sparse $\bm{\gamma}_l^{\mathrm{bw}}$ does not bring substantial memory or computation saving (see \secref{ch4-sec:analysis}), we use $\rho=0$ for backward attention and $\rho=0.3$ for forward attention.
Note that $\rho=1$ means that the resulted $\bm{\gamma}_l$ are all zeros and the layers are not updated, which can be realized by imposing sparsity on layerwise learning rate in \secref{ch4-sec:LR}. 

Attention scores $\bm{\gamma}_l$ introduce a dynamic channel-wise learning rate according to the new data samples. 
We further compare meta attention with a static channel-wise learning rate, where the channel-wise learning rate $\bm{\alpha}^{\mathrm{Ch}}$ is meta-trained as the layer-wise inner step sizes in \secref{ch4-sec:LR} while without imposing sparsity.
By comparing ``$\bm{\alpha}^{\mathrm{Ch}}$'' with ``0, 0'' in \tabref{ch4-tab:ablation}, we conclude that the dynamic channel-wise learning rate yields a significantly higher accuracy.

\begin{table}[t]
    \centering
    \caption[Ablation results of meta attention on 4Conv.]{Ablation results of meta attention on 4Conv.}
    \label{ch4-tab:ablation}
    \small
    \begin{tabular}{cccccccc}
        \toprule
        \multicolumn{2}{c}{$\rho$}                          & \multicolumn{3}{c}{5-way 1-shot}                      & \multicolumn{3}{c}{5-way 5-shot}              \\
        \cmidrule(lr){1-2}                                  \cmidrule(lr){3-5}                                      \cmidrule(lr){6-8}
        fw      & bw                                        & Mini   & Tiered & CUB                                 & Mini   & Tiered & CUB                         \\
        \midrule
        x       & x                                         & 47.1\% & 52.3\% & 41.8\%                              & 62.9\% & 68.3\% & 59.3\%                      \\
        0       & x                                         & 48.1\% & 53.2\% & 41.7\%                              & 64.1\% & 68.4\% & 59.0\%                      \\
        x       & 0                                         & 47.8\% & 53.1\% & 40.9\%                              & 63.9\% & 68.5\% & 60.0\%                      \\
        0       & 0                                         & \textbf{49.0\%} & \textbf{54.2\%} & \textbf{43.1\%}   & 64.5\% & \textbf{69.2\%} & \textbf{60.2\%}    \\
        0       & 0.3                                       & 48.5\% & 53.4\% & 42.2\%                              & 64.7\% & 68.2\% & 59.3\%                      \\
        0.3     & 0                                         & 48.8\% & 53.9\% & 42.6\%                              & \textbf{65.0\%} & 68.5\% & \textbf{60.2\%}    \\
        0.3     & 0.3                                       & 48.7\% & 53.7\% & 42.3\%                              & 64.5\% & 68.3\% & 59.5\%                      \\
        0.5     & 0.5                                       & 48.2\% & 53.4\% & 42.7\%                              & 64.8\% & 68.1\% & 59.1\%                      \\
        \multicolumn{2}{c}{$\bm{\alpha}^{\mathrm{Ch}}$}     & 47.8\% & 52.8\% & 41.0\%                              & 63.6\% & 68.1\% & 58.1\%                      \\
    \bottomrule
    \end{tabular}
    \begin{tablenotes}
        \footnotesize
        \item
        x: no forward/backward meta attention, \ie $\bm{\gamma}_l^{\mathrm{fw}}=1$ or $\bm{\gamma}_l^{\mathrm{bw}}=1$.
        \item
        $\bm{\alpha}^{\mathrm{Ch}}$: introduce an input- and output-channel inner step sizes $\bm{\alpha}^{\mathrm{Ch}}$ per layer. We use $\bm{\alpha}\cdot\bm{\alpha}^{\mathrm{Ch}}$ as inner step sizes. $\bm{\alpha}^{\mathrm{Ch}}$ is meta-trained as $\bm{\alpha}$ while without sparsity.
    \end{tablenotes}
\end{table}

\fakeparagraph{Layer-wise Updating Ratios}
To study the resulted updating ratios across layers, \ie the layer-wise sparsity of weight gradients, we randomly select 100 new tasks and plot the layer-wise updating ratios, see \figref{ch4-fig:ratio}. 
The ``4Conv'' backbone has 9 layers ($L=9$), \ie 8 alternates of \texttt{conv} and group normalization layers, and an \texttt{fc} output layer.
As mentioned in \secref{ch4-sec:experiment_settings}, we do not apply meta attention to the output layer, \ie $\bm{\gamma}_9=1$.
The used backbone is updated with 5 gradient steps ($K=5$). 
We use $\rho=0.3$ for forward attention, and $\rho=0$ for backward. 
Note that \algoref{ch4-alg:clip} adaptively determines the sparsity of $\bm{\gamma}_l$, which also means different samples may result in different updating ratios even with the same $\rho$ (see \figref{ch4-fig:ratio}). 
The size of $\bm{x}_{l-1}$ often decreases along the layers in current DNNs.
As expected, the latter layers are preferred to be updated more, since they need a smaller amount of memory for updating.
Interestingly, even if with a small $\rho$($=0.3$), the ratio of updated weights is rather small, \eg smaller than $0.2$ in step 3 of 5-way 5-shot.
It implies that the outputs of softmax have a large discrepancy, \ie only a few channels are adaptation-critical for each sample, which in turn verifies the effectiveness of our meta attention mechanism. 

We also randomly pair data samples and compute the cosine similarity between their attention scores $\bm{\gamma}_l$.  
We plot the cosine similarity of step 1 in \figref{ch4-fig:similarity}. 
The results show that there may exist a considerable variation on the adaptation-critical weights selected by different samples, which is consistent with our observation in \tabref{ch4-tab:ablation}, \ie dynamic learning rate outperforms the static one.

\begin{figure}[t]
    \centering
    \includegraphics[width=0.99\textwidth]{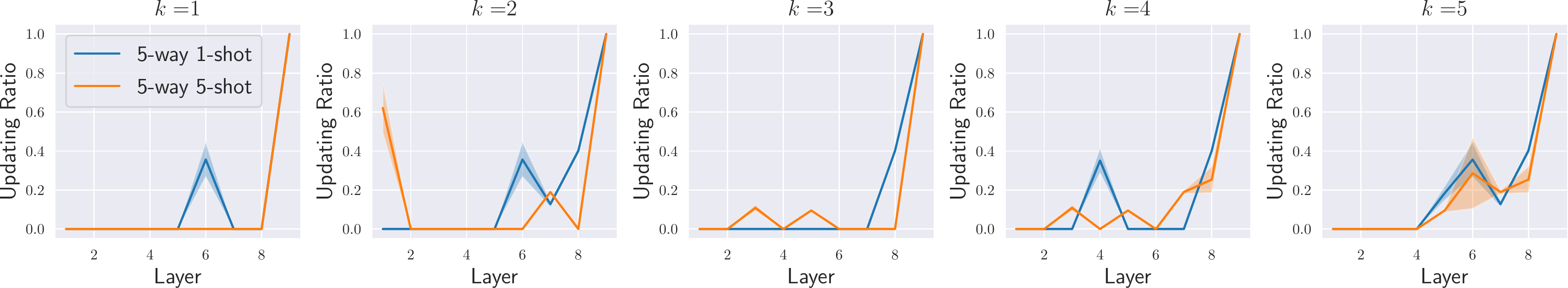} 
    \caption[Layer-wise updating ratios in each updating step.]{Layer-wise updating ratios (mean $\pm$ standard deviation) in each updating step. Note that the ratio of updated weights is determined by both static layer-wise inner step sizes $\alpha_{1:L}^{1:K}$ and the dynamic meta attention scores $\bm{\gamma}_{1:L}$. The layer with an updating ratio of 0 means its $\alpha=0$.}
    \label{ch4-fig:ratio}
\end{figure}

\begin{figure}[t]
    \centering
    \includegraphics[width=0.4\textwidth]{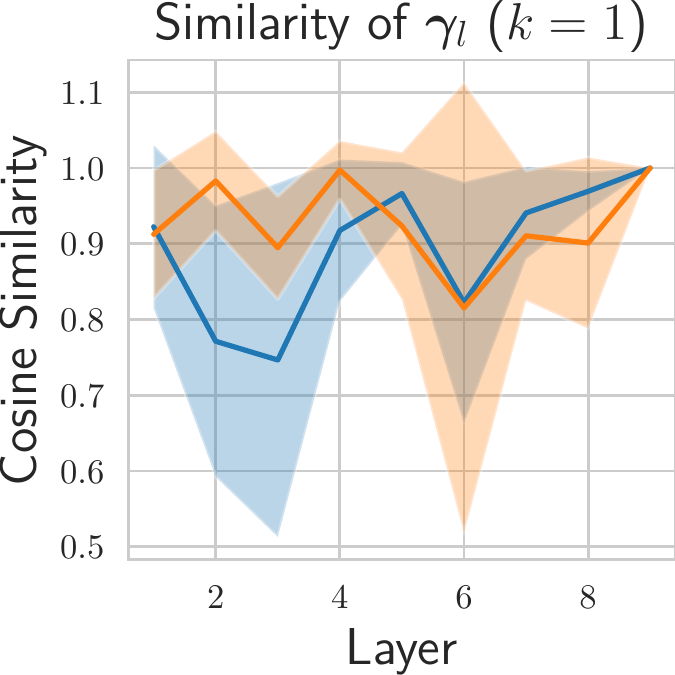} 
    \caption[Cosine similarity of $\bm{\gamma}_{1:L}$ between random pair of data samples.]{Cosine similarity (mean $\pm$ standard deviation) of $\bm{\gamma}_{1:L}$ between random pair of data samples. The results are reported in step 1, because all samples are fed into the same initial model in step 1. }
    \label{ch4-fig:similarity}
\end{figure}

\begin{table}[t]
    \centering
    \small
    \caption[Ablation results of sparse $\bm{x}_{l-1}$ and sparse $\bm{g}(\bm{y}_l)$.]{Ablation results of sparse $\bm{x}_{l-1}$ and sparse $\bm{g}(\bm{y}_l)$.}
    \label{ch4-tab:sparse}
    \begin{tabular}{cccccccc}
        \toprule
        \multicolumn{2}{c}{$\rho=0.3$}                  & \multicolumn{3}{c}{5-way 1-shot}              & \multicolumn{3}{c}{5-way 5-shot}              \\
        \cmidrule(lr){1-2}                              \cmidrule(lr){3-5}                              \cmidrule(lr){6-8}
        fw                  & bw                        & Mini   & Tiered & CUB                         & Mini   & Tiered & CUB                         \\
        \midrule
        x                   & x                         & 47.1\% & 52.3\% & 41.8\%                      & 62.9\% & 68.3\% & 59.3\%                      \\
        $\bm{g}(\bm{w}_l)$  & x                         & 48.2\% & 53.6\% & 41.2\%                      & 63.6\% & 69.0\% & 59.0\%                      \\
        $\bm{x}_{l-1}$      & x                         & 37.4\% & 37.9\% & 35.4\%                      & 47.9\% & 49.3\% & 42.5\%                      \\
        x                   & $\bm{g}(\bm{w}_l)$        & 48.0\% & 53.0\% & 42.6\%                      & 64.0\% & 67.8\% & 59.9\%                      \\
        x                   & $\bm{g}(\bm{y}_{l})$      & 22.8\% & 21.1\% & 20.6\%                      & 20.7\% & 21.0\% & 20.4\%                      \\
    \bottomrule
    \end{tabular}
    \begin{tablenotes}
        \footnotesize
        \item
        x: no forward/backward (sparse) meta attention, \ie $\bm{\gamma}_l^{\mathrm{fw}}=1$ or $\bm{\gamma}_l^{\mathrm{bw}}=1$.
    \end{tablenotes}
\end{table}

\fakeparagraph{Sparse \textit{x} and Sparse \textit{g}(\textit{y})}
Our meta attention modules take $\bm{x}_{l-1}$ and $\bm{g}(\bm{y}_l)$ as inputs, and output attention scores which are used to create sparse $\bm{g}(\bm{w}_l)$.
However, applying the resulted sparse attention scores on $\bm{x}_{l-1}$ and $\bm{g}(\bm{y}_l)$ can also bring memory and computation benefits, as discussed in \secref{ch4-sec:overview}.
We conduct the ablations when multiplying attention scores $\bm{\gamma}_l^{\mathrm{fw}}$ and $\bm{\gamma}_l^{\mathrm{bw}}$ on $\bm{g}(\bm{w}_l)$ (also the one used in the main text), or on $\bm{x}_{l-1}$ and $\bm{g}(\bm{y}_{l})$ respectively.
The results in \tabref{ch4-tab:sparse} show that a channel-wise sparse $\bm{x}_{l-1}$ hugely degrades the performance, in comparison to only imposing sparsity on $\bm{g}(\bm{w}_l)$ while using a dense $\bm{x}_{l-1}$ in the forward pass.  
In addition, directly adopting a sparse $\bm{g}(\bm{y}_{l})$ in backpropagation may even cause non-convergence in few-shot learning. 
We think this is due to the fact that the error accumulates along the propagation when imposing sparsity on $\bm{x}_{l-1}$ or $\bm{g}(\bm{y}_{l})$.

\begin{table}[t]
    \centering
    \small
    \caption[Comparison between different pooling and normalization layers.]{Comparison between different pooling and normalization layers.}
    \label{ch4-tab:pool}
    \begin{tabular}{ccccc}
        \toprule
        \multicolumn{2}{c}{4Conv}                   & \multicolumn{3}{c}{5-way 1-shot}      \\
        \cmidrule(lr){1-2}                          \cmidrule(lr){3-5}                              
        Pooling             & Normalization         & Mini   & Tiered & CUB                 \\
        \midrule
        Average-pooling     & Batch normalization   & 25.3\% & 27.2\% & 26.1\%              \\
        Average-pooling     & Group normalization   & 45.8\% & 50.3\% & 40.2\%              \\
        Max-pooling         & Batch normalization   & 27.6\% & 28.9\% & 26.5\%              \\
        Max-pooling         & Group normalization   & 46.2\% & 51.4\% & 39.9\%              \\
    \bottomrule
    \end{tabular}
\end{table}

\subsection{Ablation Studies on Pooling \& Normalization Layers}
\label{ch4-sec:experiment_pool}

In this section, we test the backbone network with different types of pooling and normalization. 
Without further notations in the following experiments, we meta-train our ``4Conv'' backbone on MiniImageNet with full batch sizes, and conduct few-shot learning with gradient accumulation with a batch size of 1, as in \secref{ch4-sec:experiment_settings}.
Here, we report the results with the original ``MAML'' method \cite{bib:ICML17:Finn} in \tabref{ch4-tab:pool}. 
Clearly, the discrepancy of batch statistics between meta-training phase and few-shot learning phase causes a large accuracy loss in batch normalization layers.
Batch normalization works only if few-shot learning uses full batch sizes, \ie without gradient accumulation, which however does not fit in our memory-constrained scenarios (see \secref{ch4-sec:gradacc}).
In addition, max-pooling performs better than average-pooling.
We thus use group normalization and max-pooling in our backbone model, see \secref{ch4-sec:experiment_settings}.

\begin{table}[t]
    \centering
    \small
    \caption[Ablation results of sample batch sizes.]{Ablation results of sample batch sizes.}
    \label{ch4-tab:batchsize}
    \begin{tabular}{lcccccc}
        \toprule
                                    & \multicolumn{3}{c}{5-way 1-shot}              & \multicolumn{3}{c}{5-way 5-shot}              \\
                                    \cmidrule(lr){2-4}                              \cmidrule(lr){5-7}
        Batch Size                  & 1      & 2      & 5                           & 1      & 5      & 25                          \\
        \midrule
        Mini                        & 48.8\% & 48.7\% & 48.3\%                      & 65.0\% & 65.1\% & 64.7\%                      \\
        Tiered                      & 53.9\% & 53.6\% & 54.3\%                      & 68.5\% & 68.9\% & 68.1\%                      \\
        CUB                         & 42.6\% & 42.1\% & 42.4\%                      & 60.2\% & 59.5\% & 60.6\%                      \\
    \bottomrule
    \end{tabular}
\end{table}

\subsection{Ablation Studies on Sample Batch Size}
\label{ch4-sec:experiment_batchsize}

In this section, we show the effects brought from different sample batch sizes.
As the setting mentioned in \secref{ch4-sec:experiment_settings}, the full batch sizes is adopted in meta-training phase with our \pMeta.
During the few-shot learning phase, gradient accumulation is applied to fit different on-device memory constraints. 
We report the accuracy when adopting different sample batch sizes in gradient accumulation. 
Although group normalization eliminates the variance of batch statistics, adopting different batch sizes may still result in diverse performance due to the batch-averaged scores in meta attention.
The results in \tabref{ch4-tab:batchsize} show that different batch sizes yield a similar accuracy level, which indicates that our meta attention module is relatively robust to batch sizes.

\section{Summary}
\label{ch4-sec:summary}

In this chapter, we propose a new meta learning method \pMeta for memory-efficient few-shot learning on unseen tasks.
\pMeta enables efficient learning on edge devices.
On-device learning of a DNN requires both data efficiency and memory efficiency. 
However, on the one hand, existing low memory training methods fail to learn a DNN given only a few training samples; on the other hand, current few-shot learning methods require a significant amount of dynamic memory.
\pMeta addresses these challenges by (\textit{i}) meta-training an initial backbone that can fast adapt to unseen tasks with only a few samples, (\textit{ii}) meta-training a selection mechanism that can identify structurewise adaptation-critical weights to reduce the training memory.
The main contributions of \dress are summarized as follows,
\begin{itemize}
    \item 
    \pMeta enables data- and memory-efficient DNN (re-)training given new unseen tasks.
    \pMeta utilizes gradient-based model adaptation, and thus is applicable to various tasks, \eg classification, regression, and reinforcement learning.
    \item 
    \pMeta adopts structured partial parameter updates for low-memory training, which is realized by automatically identifying adaptation-critical weights both layer-wise and channel-wise. 
    This hierarchical approach combines static selection of layers and dynamic selection of channels whose weights are critical for few-shot learning on the given new task, and avoids the redundant updating of non-critical weights.
    This way, the necessary memory consumption required for optimizing adaptation-critical weights decreases.
    To the best of our knowledge, \pMeta is the first meta learning method designed for on-device few-shot learning.
    \item
    Evaluations on few-shot image classification and reinforcement learning show that \pMeta not only improves the accuracy but also reduces the peak dynamic memory by a factor of 2.5 on average over the state-of-the-art few-shot learning methods. 
    \pMeta can also simultaneously reduce the computation by a factor of 1.7 on average.

\end{itemize}

This chapter studied how to conduct learning on edge devices with limited dynamic memory and limited training data. 
Note that the methods proposed in the previous chapters solely target the application scenarios on a single edge platform. 
Edge-server-system is another common scenario of edge intelligence, where multiple resource-constrained edge nodes are remotely connected with a resource-sufficient central server.   
In the next chapter, we will study how to deploy DNNs on edge-server-system to achieve an efficient inference and an efficient updating.


\newcommand\imgtabcolsep{0.2em} 
\newcommand{\tabimga}[1]{\raisebox{-.5\height}{\includegraphics[width=0.29\textwidth,height=0.21\textwidth]{{#1}.pdf}}}
\newcommand{\tabimgc}[1]{\raisebox{-.5\height}{\includegraphics[width=0.29\textwidth]{{#1}.pdf}}}
\newcommand{\tabimgb}[1]{{\includegraphics[width=0.30\textwidth]{{#1}.pdf}}}
\newcommand{\tabimgd}[1]{{\includegraphics[width=0.30\textwidth,height=0.22\textwidth]{{#1}.pdf}}}

\chapter[Edge-Server System]{Edge-Server System}
\label{ch5:edgeserver}

In \chref{ch2:inference}, \chref{ch3:adaptation} and \chref{ch4:learning}, we studied how to conduct inference, adaptation, and learning on a single edge device, respectively.
Edge-server system is another commonly used infrastructure for edge intelligent applications. 
In edge-server system, several resource-constrained edge devices are connected to a remote server with sufficient resources, and some public information is allowed to be communicated between edge devices and the server. 
In this chapter, we design a new pipeline to enable efficient inference and efficient updating for edge-server system.

\fakeparagraph{Main Resource Constraints}
The main resource constraints on edge-server system comprise two aspects, (\textit{i}) the limited resources on edge devices \eg from memory, computing power, and energy, as discussed in \chref{ch2:inference} and \chref{ch3:adaptation}, (\textit{ii}) the limited communication resources \eg from bandwidth. 

\fakeparagraph{Principles}
On-device inference is preferred over cloud inference, since it can achieve a fast and stable inference with less energy consumption. 
Due to a possible lack of relevant training data at the initial deployment, pretrained DNNs may either fail to perform satisfactorily or be significantly improved after the initial deployment. 
On such an edge-server system, the remote server retrains the DNNs with newly collected data from edge devices or from other sources and sends the updates to the edge device is preferred over on-device re-training (or federated learning), because of the limited memory and computing power on edge devices.
To reduce the communication cost for sending the updated models, we propose a deep partial updating paradigm, where the server only selects and sends a small subset of critical weights that have a large contribution to the loss reduction during the retraining.

The contents of this chapter are established mainly based on the paper ``Deep Partial Updating: Towards Communication Efficient Updating for On-device Inference'' that is published on European Conference on Computer Vision (ECCV), 2022 \cite{bib:ECCV22:Qu}.

\section{Introduction}
\label{ch5-sec:introduction}

Compared to traditional cloud inference, on-device inference is subject to severe limitations in terms of storage, energy, computing power and communication. 
On the other hand, it has many advantages, \eg it enables fast and stable inference even with low communication bandwidth or interrupted communication, and can save energy by avoiding the transfer of data to the cloud, which often costs significant amounts of energy than sensing and computation \cite{bib:Book19:Warden,bib:arXiv18:Guo,bib:arXiv19:Lee}.
To deploy deep neural networks (DNNs) on resource-constrained edge devices, extensive research has been done to compress a well pre-trained model via pruning \cite{bib:ICLR16:Han,bib:ICLR19:Frankle,bib:ICLR20:Renda} and quantization \cite{bib:NIPS15:Courbariaux,bib:ECCV16:Rastegari}.
During on-device inference, compressed DNNs may achieve a good balance between model performance and resource demand.

However, due to a possible lack of relevant training data at the time of initial deployment or due to an unknown sensing environment, pre-trained DNNs may either fail to perform satisfactorily or be significantly improved after the initial deployment. 
In other words, re-training the models by using newly collected data (from \emph{edge devices} or \emph{other sources}) is typically required to achieve the desired performance during the lifetime of devices. 

Because of the resource-constrained nature of edge devices in terms of memory and computing power, on-device re-training (or federated learning) is typically restricted to tiny batch size, small inference (sub-)networks or limited optimization steps, all resulting in a performance degradation. 
Instead, retraining often occurs on a remote server with sufficient resources.
One possible strategy to allow for a continuous improvement of the model performance is a two-stage iterative process: (\textit{i}) at each round, edge devices collect new data samples and send them to the server, and (\textit{ii}) the server retrains the model using all collected data, and sends updates to each edge device \cite{bib:CS06:Brown}. 
The first stage may be even not necessary if new data are collected in other ways and made directly available to the server. 

\begin{figure}[tb!]
    \centering
	\includegraphics[width=0.99\textwidth,height=0.42\textwidth]{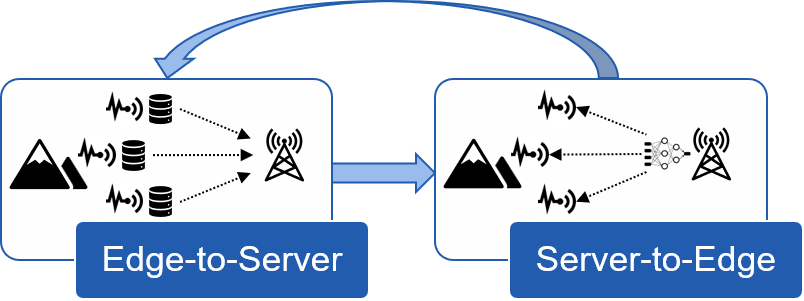}
	\caption[The iterative process of edge-to-server communication and server-to-edge communication.]{The iterative process for updating the deployed inference model on edge devices via a wireless communication. Edge-to-server communication: edge devices collect new data samples and send them to the server. Server-to-edge communication: the server retrains the model and then sends the updates to each edge device. The edge-to-server communication may not be necessary if new training data is collected from other sources and made directly available to the server.}
	\label{ch5-fig:communication}
\end{figure}

\fakeparagraph{Example Scenarios}
Example application scenarios of relevance include vision robotic sensing in an unknown environment (\eg Mars) \cite{bib:RAL17:Meng}, local translators of low-resource languages on mobile phones \cite{bib:ICMLWorkshop19:Bhandare,bib:arXiv20:Wang}, and sensor networks mounted in alpine areas \cite{bib:IPSN19:Meyer}, automatic wildlife monitoring \cite{bib:MEE18:Stowell}.
We detail two specific scenarios.
\textit{Hazard alarming on mountains: }
Researchers in \cite{bib:IPSN19:Meyer} mounted tens of sensor nodes at different scarps in high alpine areas with cameras, geophones and high-precision GPS. The purpose is to achieve fast, stable, and energy-efficient hazard monitoring for early warning to protect people and infrastructure. 
To this end, a DNN is deployed on each node to on-device detect rockfalls and debris flows. 
The nodes regularly collect and send data to the server for labeling and retraining, and the server sends the updated model back through a low-power wireless network. Retraining during deployment is essential for a highly reliable hazard warning.  
\textit{Endangered species monitoring: }
To detect endangered species, researchers often deploy some audio or image sensor nodes in virgin rainforests \cite{bib:MEE18:Stowell}.
Edge nodes are supposed to classify the potential signal from endangered species and send these relevant data to the server.
Due to the limited prior information from environments and species, retraining the initially classifier with received data or data from other sources (\eg other areas) is necessary.

\fakeparagraph{Challenges}
An essential challenge herein is that the transmissions in the server-to-edge stage are highly constrained by the limited communication resource (\eg bandwidth, energy \cite{bib:Sensors16:Augustin}) in comparison to the edge-to-server stage, if necessary at all. 
Typically, state-of-the-art DNNs often require tens or even hundreds of mega-Bytes (MB) to store parameters, whereas a single batch of data samples (a number of samples that lead to reasonable updates in batch training) needs a relatively smaller amount of data. 
For example, for CIFAR10 dataset \cite{bib:CIFAR}, the weights of a popular VGGNet require $56.09$MB storage, while one batch of 128 samples only uses around $0.40$MB \cite{bib:ICLR15:Simonyan,bib:ECCV16:Rastegari}. 
As an alternative approach, the server sends a full update of the inference model once or rarely. But in this case, every node will suffer from a low performance until such an update occurs.
Besides, edge devices could decide on and send only critical samples by using active learning schemes \cite{bib:ICLR20:Ash}.
The server may also receive data from other sources, \eg through data augmentation based on the data collected in previous rounds or new data collection campaigns.
These considerations indicate that the updated weights that are sent to edge devices by the server become a major bottleneck.

Facing the above challenges, we ask the following question: \textit{Is it possible to update only a small subset of weights while reaching a similar performance as updating all weights?}
Doing such a \textit{partial updating} can significantly reduce the server-to-edge communication overhead. 
Furthermore, fewer parameter updates also lead to less memory access on edge devices, which in turn results in smaller energy consumption than full updating \cite{bib:ISSCC14:Horowitz}. 

\fakeparagraph{Why Partial Updating Works} 
Since the model deployed on edge devices is trained with the data collected beforehand, some learned knowledge can be reused. 
In other words, we only need to distinguish and update the weights which are critical to the newly collected data.

\fakeparagraph{How to Select Weights}
Our key concept for partial updating is based on the hypothesis, that \textit{a weight shall be updated only if it has a large contribution to the loss reduction} during the retraining given newly collected data samples.
Specially, we define a binary mask $\bm{m}$ to describe which weights are subject to update and which weights are fixed (also reused). 
For any $\bm{m}$, we establish the analytical upper bound on the difference between the loss value under partial updating and that under full updating.
We determine an optimized mask $\bm{m}$ by combining two different view points: (\textit{i}) measuring each weight's ``global contribution'' to the upper bound through computing the Euclidean distance, and (\textit{ii}) measuring each weight's ``local contribution'' to the upper bound using gradient-related information. 
The weights to be updated according to $\bm{m}$ will be further sparsely fine-tuned while the remaining weights are rewound to their initial values.

\section{Related Work}
\label{ch5-sec:related}

\subsection{Partial Updating}
Although partial updating has been adopted in some prior works, it is conducted in a fairly coarse-grained manner, \eg layer-wise or neuron-wise, and targets at completely different objectives. 
Especially, under continual learning settings, \cite{bib:ICLR18:Yoon,bib:arXiv20:Jung} propose to freeze all weights related to the neurons which are more critical in performing prior tasks than new ones, to preserve existing knowledge. 
Under adversarial attack settings, \cite{bib:CCS15:Shokri} updates the weights in the first several layers only, which yield a dominating impact on the extracted features, for better attack efficacy. 
Under architecture generalization settings, \cite{bib:ICLR20:Chatterji} studies the generalization performance through the resulting loss degradation when rewinding the weights of each individual layer to their initial values.
Under meta learning settings, \cite{bib:ICLR20:Raghu,bib:AAAI21:Shen} reuse learned representations by only updating a subset of layers for efficiently learning new tasks.
Unfortunately, such techniques do not focus on reducing the number of updated weights, and thus cannot be applied in our problem settings. 

\subsection{Federated Learning}
Communication-efficient federated learning \cite{bib:ICLR18:Lin,bib:arXiv19:Kairouz,bib:arXiv20:Li} studies how to compress multiple gradients calculated on different sets of non-\textit{i.i.d.} local data, such that the aggregation of these compressed gradients could result in a similar convergence performance as centralized training on all data.
Such compressed updates are fundamentally different from our setting, where (\textit{i}) updates are not transmitted in each optimization step; (\textit{ii}) training data are incrementally collected; (\textit{iii}) centralized training is conducted.
Our typical scenarios focus on outdoor areas, which generally do not involve data privacy issues, since these collected data are not personal data. 
In comparison to federated learning, our pipeline has the following advantages: (\textit{i}) we do not conduct resource-intensive gradient backward propagation on edge devices; (\textit{ii}) the collected data are not continuously accumulated and stored on memory-constrained edge nodes; (\textit{iii}) we also avoid the difficult but necessary labeling process on each edge node in supervised learning tasks; (\textit{iv}) if few events occur on some nodes, the centralized training may avoid degraded updates in local training, \eg batch normalization.

\subsection{Compression}
The communication cost could also be reduced through some compression techniques, \eg quantizing/encoding the updated weights and the transmission signal. 
But note that these techniques are orthogonal to our approach and could be applied in addition. 
Following the compression pipeline in \cite{bib:ICLR16:Han}, the resulted sparse updating from our methods could be further quantized and Huffman-encoded.

\subsection{Unstructured Pruning}
Deep partial updating is inspired by recent unstructured pruning methods, \eg \cite{bib:ICLR16:Han,bib:ICLR19:Frankle,bib:NIPS19:Zhou,bib:ICLR20:Renda,bib:ICML21:Evci,bib:NIPS21:Peste}.
Traditional pruning methods aim at reducing the number of operations and storage consumption by setting some weights to zero. Sending a pruned DNN with only non-zero weights may also reduce the communication cost, but to a much lesser extent as shown in the experimental results, see \secref{ch5-sec:experiment_samplesratio}. 
Since our objective namely reducing the server-to-edge communication cost when updating the deployed DNN is fundamentally different from pruning, we can leverage some learned knowledge by retaining weights (partial updating) instead of zero-outing weights (pruning). 

\subsection{Domain Adaptation}
Domain adaptation targets reducing domain shift to transfer knowledge into new learning tasks \cite{bib:arXiv19:Zhuang}.
This chapter mainly considers the scenario where the inference task is not explicitly changed along the rounds, \ie the overall data distribution maintains the same along the data collection rounds. 
Thus, selecting critical weights (features) by measuring their impact on domain distribution discrepancy is invalid herein. 
Applying deep partial updating on streaming tasks where the data distribution varies along the rounds would be also worth studying, and we leave it for future works.

\section{Notations and Settings}
\label{ch5-sec:notation}

In this section, we define the notations used throughout this chapter, and provide a formalized problem setting. 
We consider a set of remote edge devices that implement on-device inference. 
They are connected to a host server that is able to perform DNN training and retraining. 
We consider the necessary amount of information that needs to be communicated to each edge device to update its inference model. 

Assume there are $R$ rounds of model updates. 
The model deployed in the $r$-th round is represented with its weight vector $\bm{w}^r$. 
The training data used to update the model for the $r$-th round is represented as $\mathcal{D}^r = \delta\mathcal{D}^{r}\cup\mathcal{D}^{r-1}$. 
Also, newly collected data samples $\delta\mathcal{D}^r$ are made available to the server in round $r-1$. 

To reduce the amount of information that needs to be sent to edge devices, only partial weights of $\bm{w}^{r-1}$ shall be updated when determining $\bm{w}^{r}$. 
The overall optimization problem for weight-wise partial updating in round $r-1$ is thus,
\begin{eqnarray}
    \min_{\delta\bm{w}^r}   & & \ell\left(\bm{w}^{r-1}+\delta\bm{w}^{r};\mathcal{D}^r\right) \label{ch5-eq:objective_r} \\
    \text{s.t.}             & & \|\delta\bm{w}^{r}\|_0 \leq k \cdot I \label{ch5-eq:constraints_r}
\end{eqnarray}
where $\ell$ denotes the loss function, $\|.\|_0$ denotes the L0-norm, $k$ denotes the updating ratio that is determined by the communication constraints in practical scenarios, and $\delta\bm{w}^{r}$ denotes the increment of $\bm{w}^{r-1}$. 
Note that both $\bm{w}^{r-1}$ and $\delta\bm{w}^{r}$ are drawn from $\mathbb{R}^I$, where $I$ is the total number of weights. 

In this case, only a fraction of $k \cdot I$ weights and the corresponding index information need to be communicated to each edge device for updating the model, namely the partial updates $\delta\bm{w}^{r}$. 
It is worth noting that the index information is relatively small in size compared to the partially updated weights (see \secref{ch5-sec:experiment}).
On each edge device, the weight vector is updated as $\bm{w}^{r} = \bm{w}^{r-1}+\delta\bm{w}^{r}$.
To simplify the notation, we will only consider a single update, \ie from weight vector $\bm{w}$ (corresponding to $\bm{w}^{r-1}$) to weight vector $\widetilde{\bm{w}}$ (corresponding to $\bm{w}^{r}$) with $\widetilde{\bm{w}} = \bm{w}+\widetilde{\delta\bm{w}}$.

\section{Deep Partial Updating}
\label{ch5-sec:method}

We develop a two-step approach for resolving the partial updating optimization problem in \equref{ch5-eq:objective_r}-\equref{ch5-eq:constraints_r}. 
The final experimental implementation in \secref{ch5-sec:experiment} contains some minor adaptations that do not change the main principles as explained next.
The overall approach is depicted in \figref{ch5-fig:approach}.

\begin{figure}[tb!]
    \centering
	\includegraphics[width=0.7\textwidth]{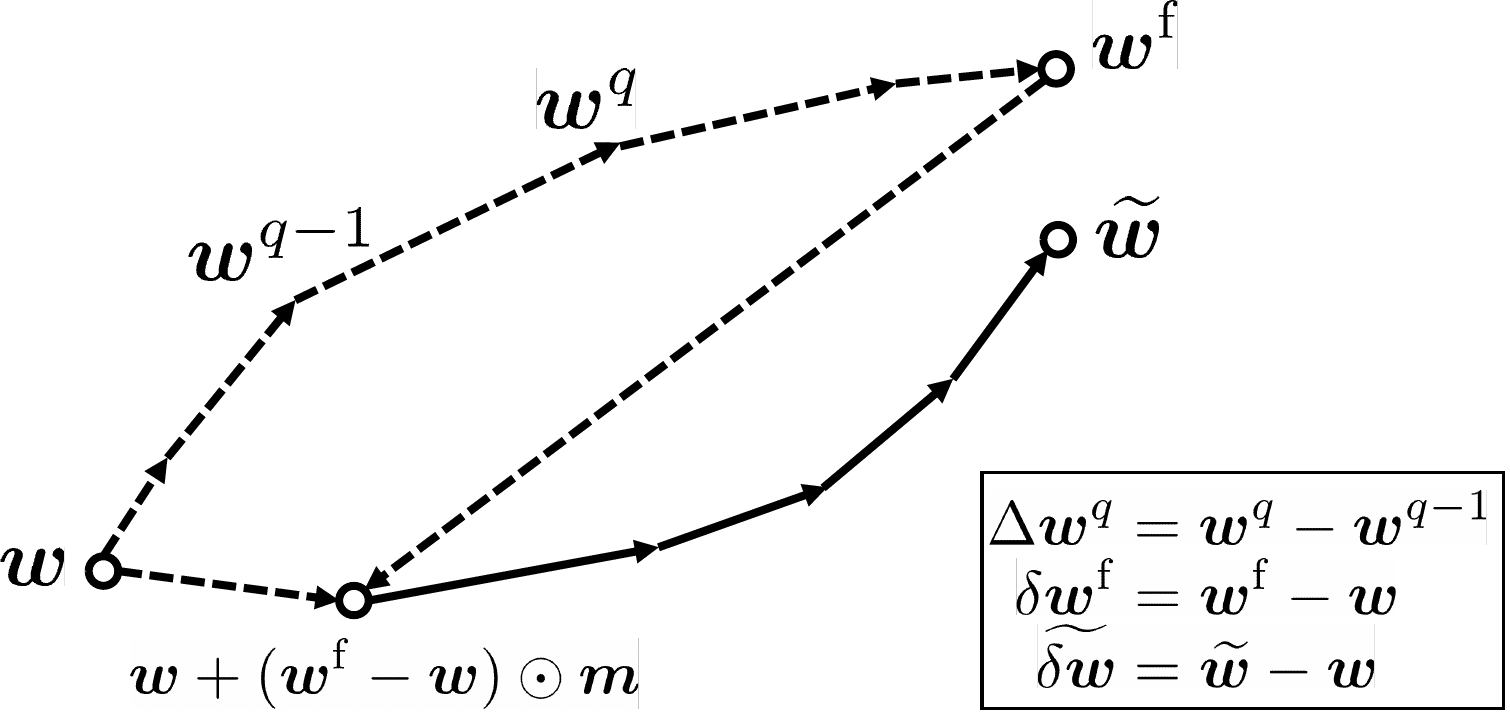}
	\caption[The overall approach of \dpu.]{The figure depicts the overall approach that consists of two steps. The first step is depicted with dotted arrows and starts from the deployed model $\bm{w}$. In $Q$ optimization steps, all weights are trained to the optimum $\bm{w}^\mathrm{f}$. Based on the collected information, a binary mask $\bm{m}$ is determined that characterizes the set of weights that are rewound to the ones of $\bm{w}$. Therefore, the second step (solid arrows) starts from $\bm{w} + \delta\bm{w}^\mathrm{f} \odot \bm{m}$. According to the mask, this initial solution is sparsely fine-tuned to the final weights $\widetilde{\bm{w}}$, \ie $\widetilde{\delta\bm{w}}$ has only non-zero values where the mask value is 1}
	\label{ch5-fig:approach}
\end{figure}

\begin{itemize}
    \item
    \underline{The First Step:} \textbf{Full Updating and Rewinding.}
    The first step not only determines the subset of weights that are allowed to change their values, but also computes the initial values for the second step. 
    In particular, we first optimize the loss function in \equref{ch5-eq:objective_r} by updating all weights from the initialization $\bm{w}$ with a standard optimizer, \eg SGD or its variants. 
    We thus obtain the minimized loss $\ell\left(\bm{w}^\mathrm{f}\right)$ with $\bm{w}^\mathrm{f} = \bm{w} + \delta\bm{w}^\mathrm{f}$, where the superscript $\mathrm{f}$ denotes ``full updating''. 
    To consider the constraint of \equref{ch5-eq:constraints_r}, the information gathered during this optimization is used to determine the subset of weights that will be changed, also that are communicated to the edge devices. 
    In the explanation of the method in \secref{ch5-sec:metric}, we use the mask $\bm{m}$ with $\bm{m} \in\{0,1\}^I$ to describe which weights are subject to change and which ones are not. 
    The weights with $m_i=1$ are trainable, whereas the weights with $m_i=0$ will be rewound from the values in $\bm{w}^\mathrm{f}$ to their initial values in $\bm{w}$, \ie unchanged. 
    Obviously, we find $\|\bm{m}\|_0 = \sum_{i} m_i = k \cdot I$. 
    
    \item 
    \underline{The Second Step:} \textbf{Sparse Fine-Tuning.} 
    In the second step we start a sparse fine-tuning from a model with $k \cdot I$ weights from the optimized model $\bm{w}^\mathrm{f}$ and $(1-k) \cdot I$ weights from the previous, still deployed model $\bm{w}$. 
    In other words, the initial weights for the second step are $\bm{w} + \delta\bm{w}^{\mathrm{f}} \odot \bm{m}$, where $\odot$ denotes an element-wise multiplication. 
    To determine the final solution $\widetilde{\bm{w}} = \bm{w}+\widetilde{\delta\bm{w}}$, we conduct a sparse fine-tuning (still with a standard optimizer), \ie we keep all weights with $m_i = 0$ constant during the optimization. 
    Therefore, $\widetilde{\delta\bm{w}}$ is zero wherever $m_i=0$, and only weights where $m_i = 1$ are updated.

\end{itemize}

\subsection{Metrics for Rewinding}
\label{ch5-sec:metric}

We will now describe a new metric that determines the weights that should be kept constant, \ie with $m_i=0$. 
Like most learning methods, we focus on minimizing a loss function.
The two-step approach relies on the following assumption: the better the loss $\ell(\bm{w} + \delta\bm{w}^\mathrm{f} \odot \bm{m})$ of the initial solution for the second step, the better the final performance.
Therefore, the first step should select a mask $\bm{m}$ such that the loss difference $\ell(\bm{w} + \delta\bm{w}^\mathrm{f} \odot \bm{m}) - \ell(\bm{w}^\mathrm{f})$ is as small as possible. 

To determine an optimized mask $\bm{m}$, we propose to upper-bound the above loss difference in two view points, and measure each weight's contribution to the bounds. 
The ``global contribution'' uses the norm information of incremental weights $\delta\bm{w}^\mathrm{f}=\bm{w}^\mathrm{f}-\bm{w}$.
The ``local contribution'' takes into account the gradient-based information that is gathered during the optimization in the first step, \ie in the path from $\bm{w}$ to $\bm{w}^\mathrm{f}$. 
Both contributions will be combined to determine the mask $\bm{m}$.

The two view points are based on the concept of smooth differentiable functions. 
A function $f(x)$ with $f: \mathbb{R}^d \rightarrow \mathbb{R}$ is called $L$-smooth if it has a Lipschitz continuous gradient $g(x)$: $\|g(x) - g(y)\|_2 \leq L \|x - y\|_2$ for all $x, y$. 
Note that Lipschitz continuity of gradients is essential to ensuring convergence of many gradient-based algorithms. 
Under such a condition, one can derive the following bounds, see also \cite{bib:Book98:Nesterov}:
\begin{equation} \label{ch5-eq:lipschitz}
    |f(y) - f(x) - g(x)^\mathrm{T} \cdot (y - x) | \leq L/2 \cdot \|y - x\|_2^2 \quad \forall x, y
\end{equation}

\fakeparagraph{Global Contribution}
One would argue that a large absolute value in $\delta\bm{w}^\mathrm{f} = \bm{w}^\mathrm{f} - \bm{w}$ indicates that this weight has moved far from its initial value in $\bm{w}$, and thus should not be rewound.
This motivates us to adopt the widely used unstructured magnitude pruning to determine the mask $\bm{m}$.
Magnitude pruning prunes the weights with the lowest magnitudes, which often achieves a good trade-off between the model accuracy and the number of zero's weights \cite{bib:ICLR20:Renda}. 

Using $a - b \leq |a - b|$, \equref{ch5-eq:lipschitz} can be reformulated as $f(y) - f(x) - g(x)^T (y-x) \leq | f(y) - f(x) - g(x)^T (y-x) | \leq L/2 \cdot \| y-x \|^2_2$.
Thus, we can bound the relevant loss difference $\ell(\bm{w} + \delta\bm{w}^\mathrm{f} \odot \bm{m}) - \ell(\bm{w}^\mathrm{f}) \geq 0$ as
\begin{equation} \label{ch5-eq:globalbound}
    \ell(\bm{w} + \delta\bm{w}^\mathrm{f} \odot \bm{m}) - \ell(\bm{w}^\mathrm{f}) \leq \bm{g}(\bm{w}^\mathrm{f})^\mathrm{T} \cdot \left( \delta\bm{w}^\mathrm{f} \odot (\bm{m} - \bm{1}) \right) + L/2 \cdot \| \delta\bm{w}^\mathrm{f} \odot (\bm{m} - \bm{1})\|_2^2
\end{equation}
where $\bm{g}(\bm{w}^\mathrm{f})$ denotes the loss gradient at $\bm{w}^\mathrm{f}$, and $\bm{1}$ is a vector whose elements are all 1. 
As the loss is optimized at $\bm{w}^\mathrm{f}$, \ie $\bm{g}(\bm{w}^\mathrm{f})\approx \bm{0}$, we can assume that the gradient term is much smaller than the norm of the weight differences in \equref{ch5-eq:globalbound}. 
Therefore, we have
\begin{equation} \label{ch5-eq:globalsum}
    \ell(\bm{w} + \delta\bm{w}^\mathrm{f} \odot \bm{m}) - \ell(\bm{w}^\mathrm{f}) \lesssim L/2 \cdot \| \delta\bm{w}^\mathrm{f} \odot (\bm{1} - \bm{m})\|_2^2
\end{equation}
The right hand side is clearly minimized if $m_i = 1$ for the largest absolute values of $\delta\bm{w}^\mathrm{f}$. 
As $\bm{1}^\mathrm{T} \cdot \left( \bm{c}^\mathrm{global} \odot (\bm{1} - \bm{m}) \right) = \| \delta\bm{w}^\mathrm{f} \odot (\bm{1} - \bm{m})\|_2^2$, this information is captured in the contribution vector 
\begin{equation} \label{ch5-eq:globalc}
    \bm{c}^\mathrm{global} = \delta\bm{w}^\mathrm{f} \odot \delta\bm{w}^\mathrm{f}
\end{equation}
The $k \cdot I$ weights with the largest values in $\bm{c}^\mathrm{global}$ are assigned to mask values $1$ and are further fine-tuned in the second step, whereas all others are rewound to their initial values in $\bm{w}$. 
\algoref{ch5-alg:gcpu} shows this first approach.

\begin{algorithm}[tbp!]
    \caption{Global Contribution Partial Updating (Prune Incremental Weights)}\label{ch5-alg:gcpu}
    \KwIn{Weights $\bm{w}$, updating ratio $k$, learning rate $\{\alpha^q\}_{q=1}^{Q}$} 
    \KwOut{Weights $\widetilde{\bm{w}}$}
    \tcc{The first step: full updating and rewinding}
    Initiate $\bm{w}^0=\bm{w}$\;
    \For {$q \leftarrow 1$ \KwTo $Q$} {
        Compute the loss gradient $\bm{g}(\bm{w}^{q-1})=\partial\ell(\bm{w}^{q-1})/\partial\bm{w}^{q-1}$\;
        Compute the optimization step with learning rate $\alpha^q$ as $\Delta\bm{w}^{q}$\;
        Update $\bm{w}^{q} = \bm{w}^{q-1} + \Delta\bm{w}^{q}$\;
    }
    Set $\bm{w}^{\mathrm{f}}=\bm{w}^{Q}$ and get $\delta\bm{w}^{\mathrm{f}} = \bm{w}^{\mathrm{f}}-\bm{w}$\;
    Compute $\bm{c}^{\mathrm{global}}=\delta\bm{w}^{\mathrm{f}}\odot\delta\bm{w}^{\mathrm{f}}$ and sort in descending order\;
    Create binary masks $\bm{m}$ with $1$ for Top-$(k \cdot I)$ indices, $0$ for others\;
    \tcc{The second step: sparse fine-tuning}
    Initiate $\widetilde{\delta\bm{w}}=\delta\bm{w}^{\mathrm{f}}\odot\bm{m}$ and $\widetilde{\bm{w}} = \bm{w}+\widetilde{\delta\bm{w}}$\;
    \For {$q \leftarrow 1$ \KwTo $Q$} {
        Compute the optimization step with learning rate $\alpha^q$ as $\Delta\widetilde{\bm{w}}^{q}$\; 
        Update $\widetilde{\delta\bm{w}}=\widetilde{\delta\bm{w}}+\Delta\widetilde{\bm{w}}^{q}\odot\bm{m}$ and $\widetilde{\bm{w}}=\bm{w}+\widetilde{\delta\bm{w}}$\;
    }
\end{algorithm}

\fakeparagraph{Local Contribution}
As experiments show, one can do better when leveraging in addition some gradient-based information gathered during the first step, \ie optimizing the initial weights $\bm{w}$ in $Q$ traditional optimization steps, $\bm{w} = \bm{w}^0 \rightarrow \cdots \; \rightarrow \bm{w}^{q-1} \rightarrow \bm{w}^{q} \rightarrow \cdots \; \rightarrow \bm{w}^{Q} = \bm{w}^\mathrm{f}$.

Using $ -a + b \leq | a - b|$, \equref{ch5-eq:lipschitz} can be reformulated as $f(x) - f(y) + g(x)^T (y-x) \leq | f(y) - f(x) - g(x)^T (y-x) | \leq L/2 \cdot \| y-x \|^2_2$. 
Thus, each optimization step is bounded as
\begin{equation} \label{ch5-eq:localbound}
    \ell(\bm{w}^{q-1}) - \ell(\bm{w}^q) \leq -\bm{g}(\bm{w}^{q-1})^\mathrm{T} \cdot \Delta\bm{w}^q + L/2 \cdot \| \Delta\bm{w}^q \|_2^2
\end{equation}
where $\Delta\bm{w}^q = \bm{w}^q - \bm{w}^{q-1}$. 
For a conventional gradient descent optimizer with a small learning rate we can use the approximation $|\bm{g}(\bm{w}^{q-1})^\mathrm{T} \cdot \Delta\bm{w}^q| \gg \|\Delta\bm{w}^q\|_2^2$ and obtain $ \ell(\bm{w}^{q-1}) - \ell(\bm{w}^q) \lesssim -\bm{g}(\bm{w}^{q-1})^\mathrm{T} \cdot \Delta\bm{w}^q$. 
Summing up over all optimization iterations yields approximately
\begin{equation} \label{ch5-eq:localsum}
    \ell(\bm{w}^\mathrm{f} - \delta\bm{w}^\mathrm{f}) - \ell(\bm{w}^\mathrm{f}) \lesssim -\sum_{q = 1}^Q \bm{g}(\bm{w}^{q-1})^\mathrm{T} \cdot \Delta\bm{w}^q
\end{equation}
Note that we have $\bm{w} = \bm{w}^\mathrm{f} - \delta\bm{w}^\mathrm{f}$ and $\delta\bm{w}^\mathrm{f} = \sum_{q = 1}^Q \Delta\bm{w}^q$. 
Therefore, with a small updating ratio $k$, \ie $\bm{m} \sim \bm{0}$, we can reformulate \equref{ch5-eq:localsum} as
$ 
    \ell\left( \bm{w} + \delta\bm{w}^\mathrm{f} \odot \bm{m} \right) - \ell(\bm{w}^\mathrm{f}) \lesssim \mathrm{U}(\bm{m}) 
$
with the upper bound 
$
    \mathrm{U}(\bm{m}) = -\sum_{q = 1}^Q \bm{g}(\bm{w}^{q-1})^\mathrm{T} \cdot (\Delta\bm{w}^q \odot (\bm{1} - \bm{m}))
$
where we suppose that the gradients are approximately constant for $\bm{m} \sim \bm{0}$ (\ie $\bm{m}$ has zero entries almost everywhere).
Therefore, an approximate incremental contribution of each weight dimension to the upper bound on the loss difference $\ell\left( \bm{w} + \delta\bm{w}^\mathrm{f} \odot \bm{m} \right) - \ell(\bm{w}^\mathrm{f})$ 
can be determined by the negative gradient vector at $\bm{m}=\bm{0}$, denoted as 
\begin{equation}
    \bm{c}^{\mathrm{local}} = - \frac{\partial \mathrm{U}(\bm{m})}{\partial \bm{m}} = - \sum_{q=1}^{Q} \bm{g}(\bm{w}^{q-1}) \odot \Delta\bm{w}^{q}
\end{equation}
which models the accumulated contribution to the overall loss reduction.
Note that the partial derivatives are computed by assuming that $\bm{m}$ is continuous in a small area around $\bm{0}$. 

\fakeparagraph{Combining Global and Local Contribution}
So far, we independently calculate the global and local contributions. 
To avoid the scale impact, we first normalize each contribution by its significance in its own set (either global or local contribution set). 
We investigated the impacts and the different combinations of both normalized contributions, see results in \secref{ch5-sec:experiment_impact}.
Interestingly, the most straightforward combination (\ie the sum of both normalized metrics) often yields a satisfied and stable performance.
Intuitively, local contribution can better identify critical weights w.r.t. the loss during training, while global contribution may be more robust for a highly non-convex loss landscape. 
Both metrics may be necessary when selecting weights to rewind.
Therefore, the combined contribution is computed as
\begin{equation}
    \bm{c} = \frac{1}{\bm{1}^\mathrm{T} \cdot \bm{c}^\mathrm{global}} \bm{c}^\mathrm{global} + \frac{1}{\bm{1}^\mathrm{T} \cdot \bm{c}^\mathrm{local}} \bm{c}^\mathrm{local}
    \label{ch5-eq:combined}
\end{equation}
and $m_i = 1$ for the $k \cdot I$ largest values of $\bm{c}$ and $m_i = 0$ otherwise. 
The pseudocode of Deep Partial Updating (\dpu), \ie rewinding according to the combined contribution to the loss reduction, is shown in \algoref{ch5-alg:dpu}.

We further analyze the complexity of \algoref{ch5-alg:dpu}. 
Recall that the dimensionality of the weights vector is denoted as $I$.
In $Q$ optimization iterations during the first step, \algoref{ch5-alg:dpu} introduces an extra time complexity of $O(QI)$, and an extra space complexity of $O(I)$ related to the original optimizer.
The rest of the first step takes a time complexity of $O(I\cdot\mathrm{log}(I))$ and a space complexity of $O(I)$, (\eg using heap sort or quick sort).
In $Q$ optimization iterations during the second step, \algoref{ch5-alg:dpu} introduces an extra time complexity of $O(QI)$, and an extra space complexity of $O(I)$ related to the original optimizer.
Thus, a total extra time complexity is $O(2QI+I\cdot\mathrm{log}(I))$ and a total extra space complexity is $O(I)$.

\subsection{(Re-)Initialization of Weights}
\label{ch5-sec:initialization}
In this section, we discuss the initialization of our method.
$\mathcal{D}^1$ denotes the initial dataset used to train the model $\bm{w}^1$ from a randomly initialized model $\bm{w}^0$. 
$\mathcal{D}^1$ corresponds to the available dataset before deployment, or collected in the $0$-th round if there are no data available before deployment. 
$\{\delta\mathcal{D}^r\}_{r=2}^R$ denotes newly collected samples in each subsequent round. 

\begin{algorithm}[tbp!]
    \caption{Deep Partial Updating}\label{ch5-alg:dpu}
    \KwIn{Weights $\bm{w}$, updating ratio $k$, learning rate $\{\alpha^q\}_{q=1}^{Q}$} 
    \KwOut{Weights $\widetilde{\bm{w}}$}
    \tcc{The first step: full updating and rewinding}
    Initiate $\bm{w}^0=\bm{w}$ and $\bm{c}^{\mathrm{local}}=\bm{0}$\;
    \For {$q \leftarrow 1$ \KwTo $Q$} {
        Compute the loss gradient $\bm{g}(\bm{w}^{q-1})=\partial\ell(\bm{w}^{q-1})/\partial\bm{w}^{q-1}$\;
        Compute the optimization step with learning rate $\alpha^q$ as $\Delta\bm{w}^{q}$\;
        Update $\bm{w}^{q} = \bm{w}^{q-1} + \Delta\bm{w}^{q}$\;
        Update $\bm{c}^{\mathrm{local}}=\bm{c}^{\mathrm{local}}-\bm{g}(\bm{w}^{q-1})\odot\Delta\bm{w}^{q}$\;
    }
    Set $\bm{w}^{\mathrm{f}}=\bm{w}^{Q}$ and get $\delta\bm{w}^{\mathrm{f}} = \bm{w}^{\mathrm{f}}-\bm{w}$\;
    Compute $\bm{c}^{\mathrm{global}}=\delta\bm{w}^{\mathrm{f}}\odot\delta\bm{w}^{\mathrm{f}}$\;
    Compute $\bm{c}$ as \equref{ch5-eq:combined} and sort in descending order\;
    Create binary masks $\bm{m}$ with $1$ for Top-$(k \cdot I)$ indices, $0$ for others\;
    \tcc{The second step: sparse fine-tuning}
    Initiate $\widetilde{\delta\bm{w}}=\delta\bm{w}^{\mathrm{f}}\odot\bm{m}$ and $\widetilde{\bm{w}} = \bm{w}+\widetilde{\delta\bm{w}}$\;
    \For {$q \leftarrow 1$ \KwTo $Q$} {
        Compute the optimization step with learning rate $\alpha^q$ as $\Delta\widetilde{\bm{w}}^{q}$\; 
        Update $\widetilde{\delta\bm{w}}=\widetilde{\delta\bm{w}}+\Delta\widetilde{\bm{w}}^{q}\odot\bm{m}$ and $\widetilde{\bm{w}}=\bm{w}+\widetilde{\delta\bm{w}}$\;
    }
\end{algorithm}

Experimental results show (see \secref{ch5-sec:experiment_fullupdating}) that training from a randomly initialized model can yield a higher accuracy \textit{after a large number of rounds}, compared to always training from the last round $\bm{w}^{r-1}$.
As a possible explanation, the optimizer could end in a hard to escape region of the search space if always trained from the last round for a long sequence of rounds.
Thus, we propose to re-initialize the weights after a certain number of rounds. 
In such a case, \algoref{ch5-alg:dpu} does not start from the weights $\bm{w}^{r-1}$ but from the randomly initialized weights.
The randomly re-initialized model (weights) can be efficiently sent to the edge devices via a single random seed. 
The device can determine the weights by means of a random generator. 
This process realizes a random shift in the search space, which is a communication-efficient way in comparison to other alternatives, such as learning to increase the loss or using the (averaged) weights in the previous rounds, as these fully changed weights still need to be sent to each node.
Each time the model is randomly re-initialized, the new partially updated model might suffer from an accuracy drop in a few rounds. 
However, we can simply avoid such an accuracy drop by not updating the model if the validation accuracy does not increase compared to the last round, see in \secref{ch5-sec:experiment_reinit}.
Note that the learned knowledge thrown away by re-initialization can be re-learned afterwards, since all collected samples are continuously stored and accumulated in the server. 
This also makes our setting different from continual learning, that aims at avoiding catastrophic forgetting without accessing old data.

To determine after how many rounds the model should be re-initialized, we conduct extensive experiments on different partial updating settings, see more discussions and results in \secref{ch5-sec:experiment_reinit}.
In conclusion, the model is randomly re-initialized as long as the number of total newly collected data samples exceeds the number of samples when the model was re-initialized last time.
For example, assume that at round $r$ the model is randomly (re-)initialized and partially updated from this random model on dataset $\mathcal{D}^r$. 
Then, the model will be re-initialized again at round $r+n$, if $|\mathcal{D}^{r+n}|>2\cdot|\mathcal{D}^r|$, where $|.|$ denotes the number of samples in the dataset.

\section{Evaluation}
\label{ch5-sec:experiment}

In this section, we experimentally show that through updating a small subset of weights, \dpu can reach a similar accuracy as full updating while requiring a significantly lower communication cost. 
We implement \dpu with Pytorch \cite{bib:NIPSWorkshop17:Paszke}, and evaluate on public vision datasets, including MNIST \cite{bib:MNIST}, CIFAR10 \cite{bib:CIFAR}, CIFAR100 \cite{bib:CIFAR}, ImageNet \cite{bib:ILSVRC15}, using multilayer perceptron (MLP), VGGNet \cite{bib:NIPS15:Courbariaux,bib:ECCV16:Rastegari}, ResNet56 \cite{bib:CVPR16:He}, MobileNetV1 \cite{bib:arXiv17:Howard}, respectively. 
Particularly, we partition the experiments into multi-round updating and single-round updating.

\fakeparagraph{Multi-Round Updating} 
We consider there are limited (or even zero) samples before the initial deployment, and data samples are continuously collected and sent from edge devices over a long period (the event rate is often low in real cases \cite{bib:IPSN19:Meyer}). 
The server retrains the model and sends the updates to each device in multiple rounds.
Regarding the highly-constrained communication resources, we choose low resolution image datasets (MNIST \cite{bib:MNIST} and CIFAR10/100 \cite{bib:CIFAR}) to evaluate multi-round updating. 
We conduct one-shot rewinding in multi-round \dpu, \ie rewinding is executed only once to achieve the desired updating ratio at each round as in \algoref{ch5-alg:dpu}, which avoids hand-tuning hyperparameters (\eg updating ratio schedule) frequently over a large number of rounds.

\fakeparagraph{Single-Round Updating} 
The deployed model is updated once via server-to-edge communication when new data from other sources become available on the server after some time, \eg releasing a new version of mobile applications based on newly retrieved internet data.
Although \dpu is elaborated and designed under multi-round updating settings, it can be applied directly in single-round updating. 
Since transmission from edge devices may be even not necessary, we evaluate single-round \dpu on the large scale ImageNet dataset. 
Iterative rewinding is adopted here due to its better performance. 
Particularly, we alternatively perform rewinding 20\% of the remaining trainable weights according to \equref{ch5-eq:combined} and sparse fine-tuning until reaching the desired updating ratio. 

\fakeparagraph{General Settings for All Experiments} 
We randomly select 30\% of the original test dataset (original validation dataset for ImageNet) as the validation dataset, and the remainder as the test dataset.
Let $\{|\mathcal{D}^1|,|\delta\mathcal{D}^r|\}$ represent the available data samples along rounds, where $|\delta\mathcal{D}^r|$ is supposed to be constant along rounds.
Both $\mathcal{D}^1$ and $\delta\mathcal{D}^r$ are randomly sampled (without replacement) from the original training dataset to simulate the data collection.
In each round, the test accuracy is reported, when the validation dataset achieves the highest Top-1 accuracy during retraining.
When the validation accuracy does not increase compared to the previous round, the models are not updated to reduce the communication overhead.
This strategy is also applied to other baselines to enable a fair comparison.  
We use the average cross-entropy as the loss function, and use Adam variant of SGD for MLP and VGGNet, Nesterov SGD for ResNet56 and MobileNetV1.
More implementation details are provided in \secref{ch5-sec:experiment_benchmark}.

\fakeparagraph{Indexing}
\dpu generates a sparse tensor.
In addition to the updated weights, the indices of these weights also need to be sent to each edge device.
A simple implementation is to send the mask $\bm{m}$, \ie a binary vector of $I$ elements.
Let $S_w$ denote the bitwidth of each single weight, and $S_x$ denote the bitwidth of each index.
Directly sending $\bm{m}$ yields an overall communication cost of $I\cdot k \cdot S_w+I \cdot S_x$ with $S_x=1$. 
To save the communication cost on indexing, we further encode $\bm{m}$. Suppose that $\bm{m}$ is a random binary vector with a probability of $k$ to contain 1. 
The optimal encoding scheme according to Shannon yields $S_x(k)=k \cdot \mathrm{log}(1/k) + (1-k) \cdot \mathrm{log}(1/(1-k))$. 
Coding schemes such as Huffman block coding can come close to this bound. 
We use $S_w\cdot k\cdot I + S_x(k)\cdot I$ to report the size of data transmitted from server to each node at each round, contributed by the partially updated weights plus the encoded indices of these weights.

\subsection{Benchmarking Details}
\label{ch5-sec:experiment_benchmark}

\subsubsection{MLP on MNIST}
\label{ch5-sec:experiment_mlp}
The MNIST dataset \cite{bib:MNIST} consists of $28\times28$ gray scale images in 10 digit classes. 
It contains a training dataset with 60000 data samples, and a test dataset with 10000 data samples. 
We use the original training dataset for training; and randomly select 3000 samples in the original test dataset for validation, and the rest 7000 samples for testing. 
We use a mini-batch with size of 128 training on 1 GeForce RTX 3090 GPU.
We use Adam variant of SGD as the optimizer, and use all default parameters provided by Pytorch. 
The number of training epochs is chosen as 60 at each round. 
The initial learning rate is $0.05$, and it decays with a factor of 0.1 every $20$ epochs.
The used MLP contains two hidden layers, and each hidden layer contains 512 hidden units.
The input is a 784-dim tensor consisting of all pixel values in each image.
All weights in MLP need around $2.67$MB.
Each data sample needs $0.784$KB.
The size of MLP equals around 3400 data samples. 
The used MLP architecture is presented as, 2$\times$512FC - 10SVM.

\subsubsection{VGGNet on CIFAR10}
\label{ch5-sec:experiment_vgg}
The CIFAR10 dataset \cite{bib:CIFAR} consists of $32\times32$ color images in 10 object classes. 
It contains a training dataset with 50000 data samples, and a test dataset with 10000 data samples. 
We use the original training dataset for training; and randomly select 3000 samples in the original test dataset for validation, and the rest 7000 samples for testing.
We use a mini-batch with size of 128 training on 1 GeForce RTX 3090 GPU.
We use Adam variant of SGD as the optimizer, and use all default parameters provided by Pytorch. 
The number of training epochs is chosen as 60 at each round. 
The initial learning rate is $0.05$, and it decays with a factor of 0.2 every $20$ epochs.
The used VGGNet is widely adopted in many previous compression works \cite{bib:NIPS15:Courbariaux,bib:ECCV16:Rastegari}, which is a modified version of the original VGG \cite{bib:ICLR15:Simonyan}.
All weights in VGGNet need around $56.09$MB.
Each data sample needs $3.072$KB.
The size of VGGNet equals around 18200 data samples.
The used VGGNet architecture is presented as, 2$\times$128C3 - MP2 - 2$\times$256C3 - MP2 - 2$\times$512C3 - MP2 - 2$\times$1024FC - 10SVM.

\subsubsection{ResNet56 on CIFAR100}
\label{ch5-sec:experiment_resnet56}
Similar as CIFAR10, the CIFAR100 dataset \cite{bib:CIFAR} consists of $32\times32$ color images in 100 object classes. 
It contains a training dataset with 50000 data samples, and a test dataset with 10000 data samples. 
We use the original training dataset for training; and randomly select 3000 samples in the original test dataset for validation, and the rest 7000 samples for testing.
We use a mini-batch with size of 128 training on 1 GeForce RTX 3090 GPU.
We use Nesterov SGD with weight decay 0.0001 as the optimizer, and use all default parameters provided by Pytorch. 
The number of training epochs is chosen as 100 at each round. 
The initial learning rate is $0.1$, and it decays with the cosine annealing schedule.
The ResNet56 used in our experiments is proposed in \cite{bib:CVPR16:He}.
All weights in ResNet56 need around $3.44$MB.
Each data sample needs $3.072$KB.
The size of ResNet56 equals around 1100 data samples.

\subsubsection{MobileNetV1 on ImageNet}
\label{ch5-sec:experiment_mobilenetv1}
The ImageNet dataset \cite{bib:ILSVRC15} consists of high-resolution color images in 1000 object classes. 
It contains a training dataset with $1.28$ million data samples, and a validation dataset with 50000 data samples. 
Following the commonly used pre-processing \cite{bib:torchResNet}, each sample (single image) is randomly resized and cropped into a $224\times224$ color image.
We use the original training dataset for training; and randomly select 15000 samples in the original validation dataset for validation, and the rest 35000 samples for testing.
We use a mini-batch with size of 1024 training on 4 GeForce RTX 3090 GPUs. 
We use Nesterov SGD with weight decay 0.0001 as the optimizer, and use all default parameters provided by Pytorch. 
The number of training epochs is chosen as 150 at each round. 
The initial learning rate is $0.5$, and it decays with the cosine annealing schedule.
The MobileNetV1 used in our experiments is proposed in \cite{bib:arXiv17:Howard}.
All weights in MobileNetV1 need around $16.93$MB.
Each data sample needs $150.528$KB.
The size of MobileNetV1 equals around 340 data samples.

\subsection{Ablation Studies on Full Updating}
\label{ch5-sec:experiment_fullupdating}

\fakeparagraph{Settings}
In this section, we compare full updating with different initialization at each round to confirm the best-performed full updating baseline. 
The compared full updating methods include, (\textit{i}) the model is trained from different random initialization at each round; (\textit{ii}) the model is trained from a same random initialization at each round, \ie with the same random seed; (\textit{iii}) the model is trained from the weights $\bm{w}^{r-1}$ of the last round at each round.
The experiments are conducted on VGGNet using CIFAR10 dataset with different amounts of training samples $\{|\mathcal{D}^1|,|\delta\mathcal{D}^r|\}$.
Each experiment runs for three times using random data samples.

\fakeparagraph{Results}
We report the mean and the standard deviation of test accuracy (over three runs) under different initialization in \figref{ch5-fig:fullupdating}.
The results show that training from a same random initialization yields a similar accuracy level while sometimes also a lower variance, as training from different random initialization at each round.
In comparison to training from scratch (\ie random initialization), training from $\bm{w}^{r-1}$ may yield a higher accuracy in the first few rounds; yet training from scratch can always outperform after a large number of rounds.
Thus, in this chapter, we adopt training from a same random initialization at each round, \ie (\textit{ii}), as the baseline of full updating.

\begin{figure}[tbp!]
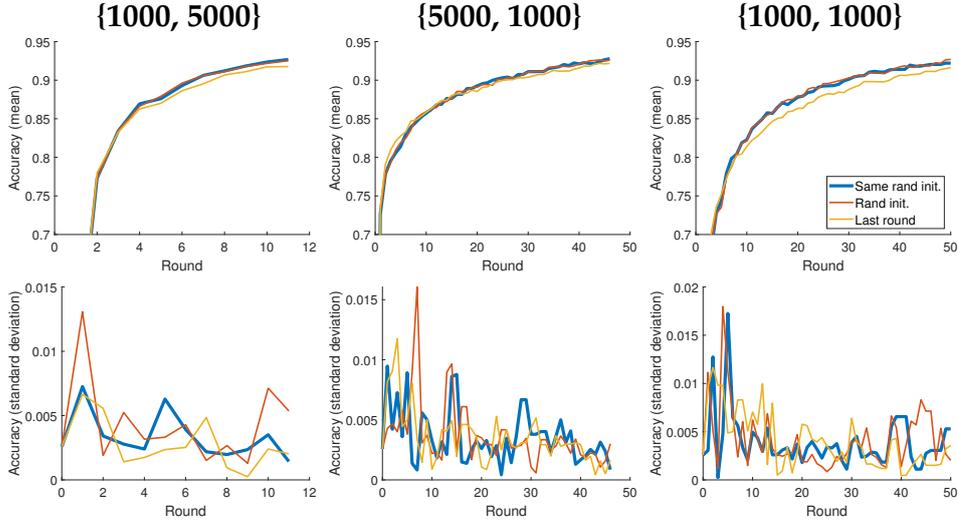

    \setlength\tabcolsep{\imgtabcolsep}
    \centering
    \begin{tabular}{ccc}
        \textbf{~~~~\{1000,~5000\}}         & \textbf{~~~~\{5000,~1000\}}           & \textbf{~~~~\{1000,~1000\}}           \\
        \tabimgb{./figs/ch5/1000-5000-f}    & \tabimgb{./figs/ch5/5000-1000-f}      & \tabimgb{./figs/ch5/1000-1000-f}      \\
        \tabimgb{./figs/ch5/1000-5000-fv}   & \tabimgb{./figs/ch5/5000-1000-fv}     & \tabimgb{./figs/ch5/1000-1000-fv}     \\
    \end{tabular}
    \caption[Comparing full updating methods with different initialization at each round.]{Comparing full updating methods with different initialization at each round.}
    \label{ch5-fig:fullupdating}
\end{figure}

\subsection{Number of Rounds for Re-Initialization}
\label{ch5-sec:experiment_reinit}

\fakeparagraph{Settings}
In these experiments, we re-initialize the model every $n$ rounds under different partial updating settings to determine a heuristic rule to set the number of rounds for re-initialization.
We conduct experiments on VGGNet using CIFAR10 dataset, with different amounts of training samples $\{|\mathcal{D}^1|,|\delta\mathcal{D}^r|\}$ and different updating ratios $k$.
Every $n$ rounds, the model is (re-)initialized again from a same random model (as mentioned in \secref{ch5-sec:experiment_fullupdating}), then partially updated in the next $n$ rounds with \algoref{ch5-alg:dpu}.
We choose $n=1,5,10,20$.
Specially, $n=1$ means that the model is partially updated from the same random model every round, \ie without reusing the learned knowledge at all.
Each experiment runs three times using random data samples.

\fakeparagraph{Results}
We plot the mean test accuracy along rounds in \figref{ch5-fig:partial_reinit}.
By comparing $n=1$ with other settings, we can conclude that within a certain number of rounds, the current deployed model $\bm{w}^{r-1}$ (\ie the model from the last round) is a better starting point for \algoref{ch5-alg:dpu} than a randomly initialized model. 
In other word, partially updating from the last round may yield a higher accuracy than partially updating from a random model with the same training effort.
This is straightforward, since such a model is already pretrained on a subset of the currently available data samples, and the previous learned knowledge could help the new training process. 
Since all newly collected samples are continuously stored in the server, complete information about the past data samples is available. 
This also makes our setting different from continual learning setting, which aims at avoiding catastrophic forgetting without accessing (at least not all) old data.

Each time the model is re-initialized, the new partially updated model might suffer from an accuracy drop in a few rounds.
Although this accuracy drop may be relieved if we carefully tune the partial updating training scheme every time, this is not feasible regarding a large number of updating rounds.
However, we can simply avoid such an accuracy drop by not updating the model if the validation accuracy does not increase compared to the last round (as discussed in \secref{ch5-sec:experiment}).
Note that in this situation, the partially updated weights (as well as the random seed for re-initialization) still need to be sent to the edge devices, since this is an on-going training process. 
After implementing the above strategy, we plot the mean accuracy in \figref{ch5-fig:partial_reinit_savecomm}. 
In addition, we also add the related results of full updating in \figref{ch5-fig:partial_reinit_savecomm}, where the model is fully updated and is re-initialized every $n$ rounds from a same random model.  
Note that full updating with re-initialization every round ($n=1$) is exactly the same as ``same rand init.'' in \figref{ch5-fig:fullupdating} in \secref{ch5-sec:experiment_fullupdating}.
From \figref{ch5-fig:partial_reinit_savecomm}, we can conclude that the model needs to be re-initialized more frequently in the first several rounds than in the following rounds to achieve a higher accuracy level.
The model also needs to be re-initialized more frequently with a large partial updating ratio $k$.
Particularly, the ratio between the number of current data samples and the number of following collected data samples has a larger impact than the updating ratio.

Thus, we propose to re-initialize the model as long as the number of total newly collected data samples exceeds the number of samples when the model is re-initialized last time.
For example, assume that at round $r$ the model is randomly (re-)initialized and partially updated from the random model on dataset $\mathcal{D}^r$. 
Then, the model will be re-initialized at round $r+n$, if $|\mathcal{D}^{r+n}|>2\cdot|\mathcal{D}^r|$. 

\begin{figure}[tbp!]
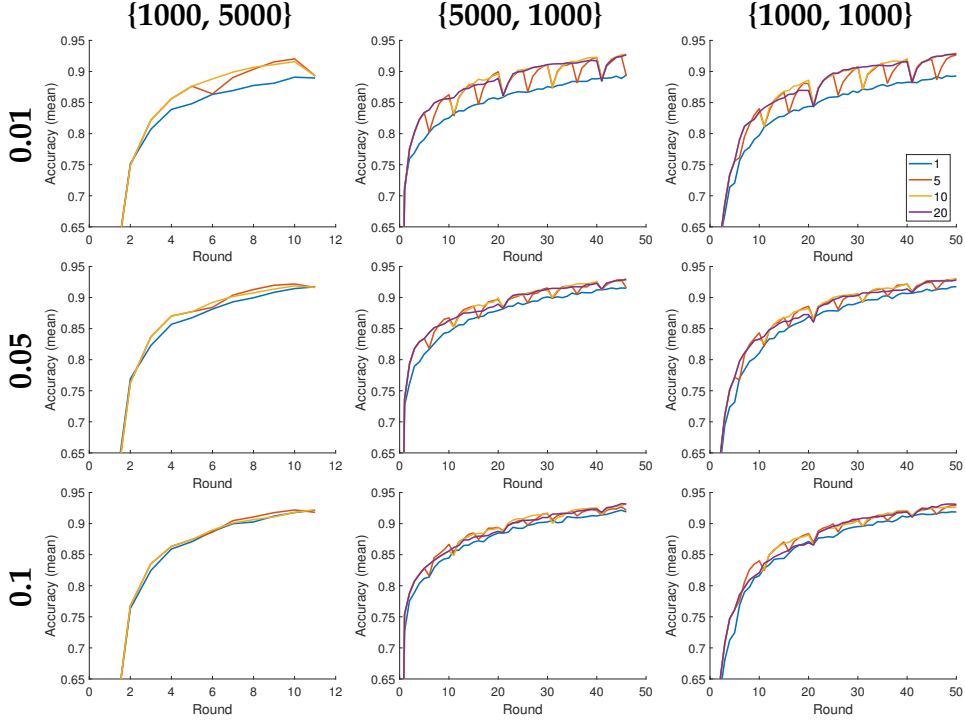

    \setlength\tabcolsep{\imgtabcolsep}
    \centering
    \begin{tabular}{m{0.3cm}ccc}
                                          & \textbf{~~~~\{1000,~5000\}}        & \textbf{~~~~\{5000,~1000\}}        & \textbf{~~~~\{1000,~1000\}}            \\ 
        \rotatebox{90}{\textbf{0.01}}     & \tabimgc{./figs/ch5/1000-5000-1-r} & \tabimgc{./figs/ch5/5000-1000-1-r} & \tabimgc{./figs/ch5/1000-1000-1-r}     \\ 
        \rotatebox{90}{\textbf{0.05}}     & \tabimgc{./figs/ch5/1000-5000-2-r} & \tabimgc{./figs/ch5/5000-1000-2-r} & \tabimgc{./figs/ch5/1000-1000-2-r}     \\
        \rotatebox{90}{\textbf{0.1}}      & \tabimgc{./figs/ch5/1000-5000-3-r} & \tabimgc{./figs/ch5/5000-1000-3-r} & \tabimgc{./figs/ch5/1000-1000-3-r}     \\ 
    \end{tabular}
    \caption[Comparison w.r.t. the mean accuracy when \dpu is re-initialized every $n$ rounds.]{Comparison w.r.t. the mean accuracy when \dpu is re-initialized every $n$ rounds ($n=1,5,10,20$) under different $\{|\mathcal{D}^1|,|\delta\mathcal{D}^r|\}$ and updating ratio ($k=0.01,0.05,0.1$) settings.}
    \label{ch5-fig:partial_reinit}
\end{figure}

\begin{figure}[tbp!]
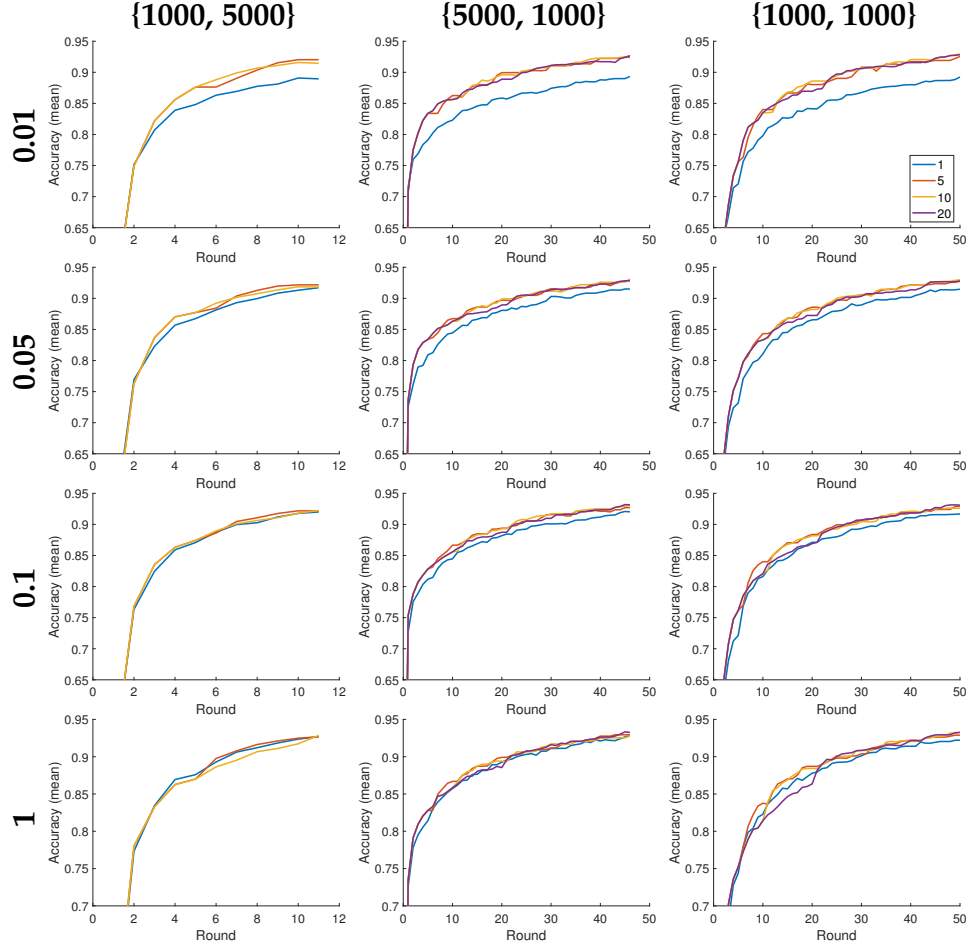

    \setlength\tabcolsep{\imgtabcolsep}
    \centering
    \begin{tabular}{m{0.3cm}ccc}
                                          & \textbf{~~~~\{1000,~5000\}}         & \textbf{~~~~\{5000,~1000\}}         & \textbf{~~~~\{1000,~1000\}}            \\ 
        \rotatebox{90}{\textbf{0.01}}     & \tabimgc{./figs/ch5/1000-5000-1-rp} & \tabimgc{./figs/ch5/5000-1000-1-rp} & \tabimgc{./figs/ch5/1000-1000-1-rp}    \\ 
        \rotatebox{90}{\textbf{0.05}}     & \tabimgc{./figs/ch5/1000-5000-2-rp} & \tabimgc{./figs/ch5/5000-1000-2-rp} & \tabimgc{./figs/ch5/1000-1000-2-rp}    \\
        \rotatebox{90}{\textbf{0.1}}      & \tabimgc{./figs/ch5/1000-5000-3-rp} & \tabimgc{./figs/ch5/5000-1000-3-rp} & \tabimgc{./figs/ch5/1000-1000-3-rp}    \\
        \rotatebox{90}{\textbf{1}}        & \tabimgc{./figs/ch5/1000-5000-fr}   & \tabimgc{./figs/ch5/5000-1000-fr}   & \tabimgc{./figs/ch5/1000-1000-fr}      \\
    \end{tabular}
    \caption[Comparison w.r.t. the mean accuracy when \dpu is re-initialized every $n$ rounds.]{Comparison w.r.t. the mean accuracy when \dpu is re-initialized every $n$ rounds ($n=1,5,10,20$) under different $\{|\mathcal{D}^1|,|\delta\mathcal{D}^r|\}$ and updating ratio ($k=0.01,0.05,0.1$ and full updating $k=1$) settings.}
    \label{ch5-fig:partial_reinit_savecomm}
\end{figure}

\subsection{Impacts from Global/Local Contributions}
\label{ch5-sec:experiment_impact}

\subsubsection{Ablation Studies of Rewinding Metrics}
\label{ch5-sec:experiment_rewind}

\fakeparagraph{Settings} 
We conduct a set of ablation experiments regarding different rewinding metrics discussed in \secref{ch5-sec:metric}. 
We compare the influence of the local and global contributions as well as their combination, in terms of the training loss after the rewinding and the final test accuracy. 
We conduct single-round updating on VGGNet.
The initial model are fully trained on a randomly selected dataset of $10^3$ samples. 
After adding $10^3$ new randomly selected samples, we conduct the first step of our approach (see \algoref{ch5-alg:dpu}) with all three rewinding metrics, \ie the global contribution, the local contribution, and the combined contribution.
Accordingly, the second step (sparse fine-tuning) is executed.
The experiment is executed over five runs with different random seeds.

\begin{table}[tbp!]
    \centering
    \caption[Comparing training loss after rewinding and the final test accuracy under different metrics.]{Comparing training loss after rewinding and the final test accuracy under different metrics.} 
    \label{ch5-tab:lossincr}
    \footnotesize
    \begin{tabular}{cccc}
        \toprule
        \multirow{2}{*}{$k$}  & \multicolumn{3}{c}{Training loss at $\bm{w}+\delta\bm{w}^\mathrm{f}\odot\bm{m}$ ~~~~(Test accuracy at $\widetilde{\bm{w}}$)}           \\ \cline{2-4} 
                              & \multicolumn{1}{c}{Global}                   & \multicolumn{1}{c}{Local}                & \multicolumn{1}{c}{Combined}                 \\ \hline
        0.01                  & $3.04\pm0.07$ $(55.0\pm0.1\%)$               & $\bm{2.59}\pm0.08$ $(55.6\pm0.1\%)$      & $2.66\pm0.09$ $(\bm{56.5}\pm0.0\%)$          \\
        0.05                  & $2.51\pm0.06$ $(57.3\pm0.2\%)$               & $1.80\pm0.10$ $(57.8\pm0.1\%)$           & $\bm{1.67}\pm0.06$ $(\bm{58.2}\pm0.1\%)$     \\
        0.1                   & $2.03\pm0.05$ $(58.3\pm0.0\%)$               & $1.34\pm0.08$ $(59.0\pm0.1\%)$           & $\bm{0.99}\pm0.03$ $(\bm{59.0}\pm0.1\%)$     \\
        0.2                   & $1.20\pm0.05$ $(59.0\pm0.1\%)$               & $0.74\pm0.03$ $(59.6\pm0.2\%)$           & $\bm{0.42}\pm0.01$ $(\bm{60.1}\pm0.2\%)$     \\ 
        \bottomrule
    \end{tabular}
\end{table}

\fakeparagraph{Results}
The training loss after rewinding (\ie $\ell(\bm{w}+\delta\bm{w}^\mathrm{f}\odot\bm{m})$) and the final test accuracy after sparse fine-tuning (\ie at $\widetilde{\bm{w}}$) are reported in \tabref{ch5-tab:lossincr}.
We use the form of mean $\pm$ standard deviation.
As seen in the table, the combined contribution always yields a lower or similar training loss after rewinding compared to the other two metrics.
The smaller deviation also indicates that adopting the combined contribution yields more robust results.
This demonstrates the effectiveness of our proposed metric, \ie the combined contribution to the analytical upper bound on loss reduction.
Rewinding with the combined contribution also acquires a higher final accuracy, which in turn verifies the hypothesis we made for partial updating, a weight shall be updated only if it has a large contribution to the loss reduction.

\subsubsection{Balancing between Global and Local Contributions}
\label{ch5-sec:experiment_balancing}

\fakeparagraph{Settings} 
In \equref{ch5-eq:combined}, the combined contribution is calculated by adding both normalized contributions together.
However, both normalized contributions may have different importance when determining the critical weights.
In order to investigate which one plays a more essential role in the combined contribution, we introduce another hyper-parameter $\lambda$ to tune the proportion of both normalized contributions as
\begin{equation}
    \bm{c}_\lambda = \lambda \cdot \frac{1}{\bm{1}^\mathrm{T} \cdot \bm{c}^\mathrm{global}} \bm{c}^\mathrm{global} + (1-\lambda) \cdot \frac{1}{\bm{1}^\mathrm{T} \cdot \bm{c}^\mathrm{local}} \bm{c}^\mathrm{local}
\end{equation}
Note that the combined contribution $\bm{c}$ used in the previous experiments is the same as $\bm{c}_\lambda$ when $\lambda=0.5$, since only the order matters when determining the critical weights.
We implement partial updating methods with the rewinding metric $\bm{c}_\lambda$ under different values of $\lambda$. 
We compare these methods under updating ratios $k=0.01,0.05,0.1$ and different $\{|\mathcal{D}^1|,|\delta\mathcal{D}^r|\}$ settings on VGGNet using CIFAR10 dataset, and with the re-initialization scheme described in \secref{ch5-sec:initialization}.
Each experiment runs three times using random data samples.

\begin{figure}[t!]
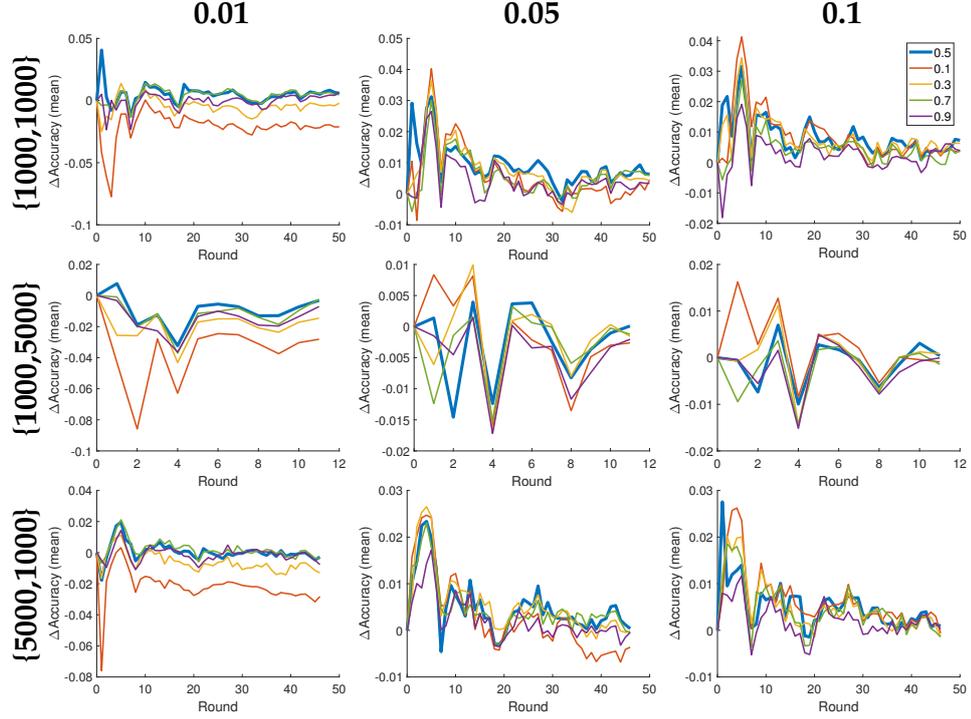

    \setlength\tabcolsep{\imgtabcolsep}
    \centering
    \begin{tabular}{m{0.3cm}ccc}
                                                  & \textbf{~~~~~~0.01}                       & \textbf{~~~~~~0.05}                       & \textbf{~~~~~~0.1}                        \\ 
        \rotatebox{90}{\textbf{\{1000,1000\}}}    & \tabimgc{./figs/ch5/1000-1000-1-ratiod}   & \tabimgc{./figs/ch5/1000-1000-2-ratiod}   & \tabimgc{./figs/ch5/1000-1000-3-ratiod}   \\ 
        \rotatebox{90}{\textbf{\{1000,5000\}}}    & \tabimgc{./figs/ch5/1000-5000-1-ratiod}   & \tabimgc{./figs/ch5/1000-5000-2-ratiod}   & \tabimgc{./figs/ch5/1000-5000-3-ratiod}   \\ 
        \rotatebox{90}{\textbf{\{5000,1000\}}}    & \tabimgc{./figs/ch5/5000-1000-1-ratiod}   & \tabimgc{./figs/ch5/5000-1000-2-ratiod}   & \tabimgc{./figs/ch5/5000-1000-3-ratiod}   \\
    \end{tabular}
    \caption[Comparison w.r.t. the mean accuracy difference (full updating as the reference) under different $\lambda$.]{Comparison w.r.t. the mean accuracy difference (full updating as the reference) under $\lambda=0.5,0.1,0.3,0.7,0.9$. The chosen settings are updating ratios $k=0.01,0.05,0.1$, $\{|\mathcal{D}^1|,|\delta\mathcal{D}^r|\}=\{1000,1000\}, \{1000,5000\}, \{5000,1000\}$.}
    \label{ch5-fig:combinedratio}
\end{figure}

\fakeparagraph{Results}
To clearly illustrate the impact of $\lambda$, we compare the difference between the accuracy under partial updating methods with various $\lambda$ and that under full updating. The mean accuracy difference (over three runs) are plotted in \figref{ch5-fig:combinedratio}.
As seen in \figref{ch5-fig:combinedratio}, $\lambda=0.5$ always obtains the best performance in general, especially when the updating ratio is small.
Thus, in the following experiments, we fix this hyper-parameter $\lambda$ as 0.5. 
In other words, the combined contribution is chosen as 
\begin{equation}
    \bm{c}_\lambda(\lambda=0.5) = 0.5 \cdot \frac{1}{\bm{1}^\mathrm{T} \cdot \bm{c}^\mathrm{global}} \bm{c}^\mathrm{global} + 0.5 \cdot \frac{1}{\bm{1}^\mathrm{T} \cdot \bm{c}^\mathrm{local}} \bm{c}^\mathrm{local}
\end{equation}
which has exactly the same functionality as \equref{ch5-eq:combined}. 
Note that it may be possible to manually find another hyper-parameter $\lambda$ that achieves better performance in certain cases. 
However, setting $\lambda$ as 0.5 already yields a satisfactory performance, and can avoid meticulous and computationally expensive hyper-parameter tuning in a large number of updating rounds.

\subsubsection{Number of Updated Weights across Layers under Different Rewinding Metrics}
\label{ch5-sec:experiment_numweights}

\fakeparagraph{Settings} 
To further study the impact of adopting different rewinding metrics, we show the distribution of updated weights across layers in this section.
We implement partial updating methods with three rewinding metrics (\ie the global contribution, the local contribution, and the combined contribution, see in \secref{ch5-sec:metric}) on VGGNet using CIFAR10 dataset. 
We compare these methods with different updating ratios $k$ under $\{|\mathcal{D}^1|,|\delta\mathcal{D}^r|\}=\{1000,1000\}$.
To study the distribution of updated weights along all rounds, we let the model partially updated in every round even if the validation accuracy may not increase compared to the previous round.
All methods start from the same randomly initialized model, and are re-initialized with this random model according to the proposed scheme in \secref{ch5-sec:initialization}.

\begin{figure}[!tbp]
    \centering
	\includegraphics[width=0.95\textwidth]{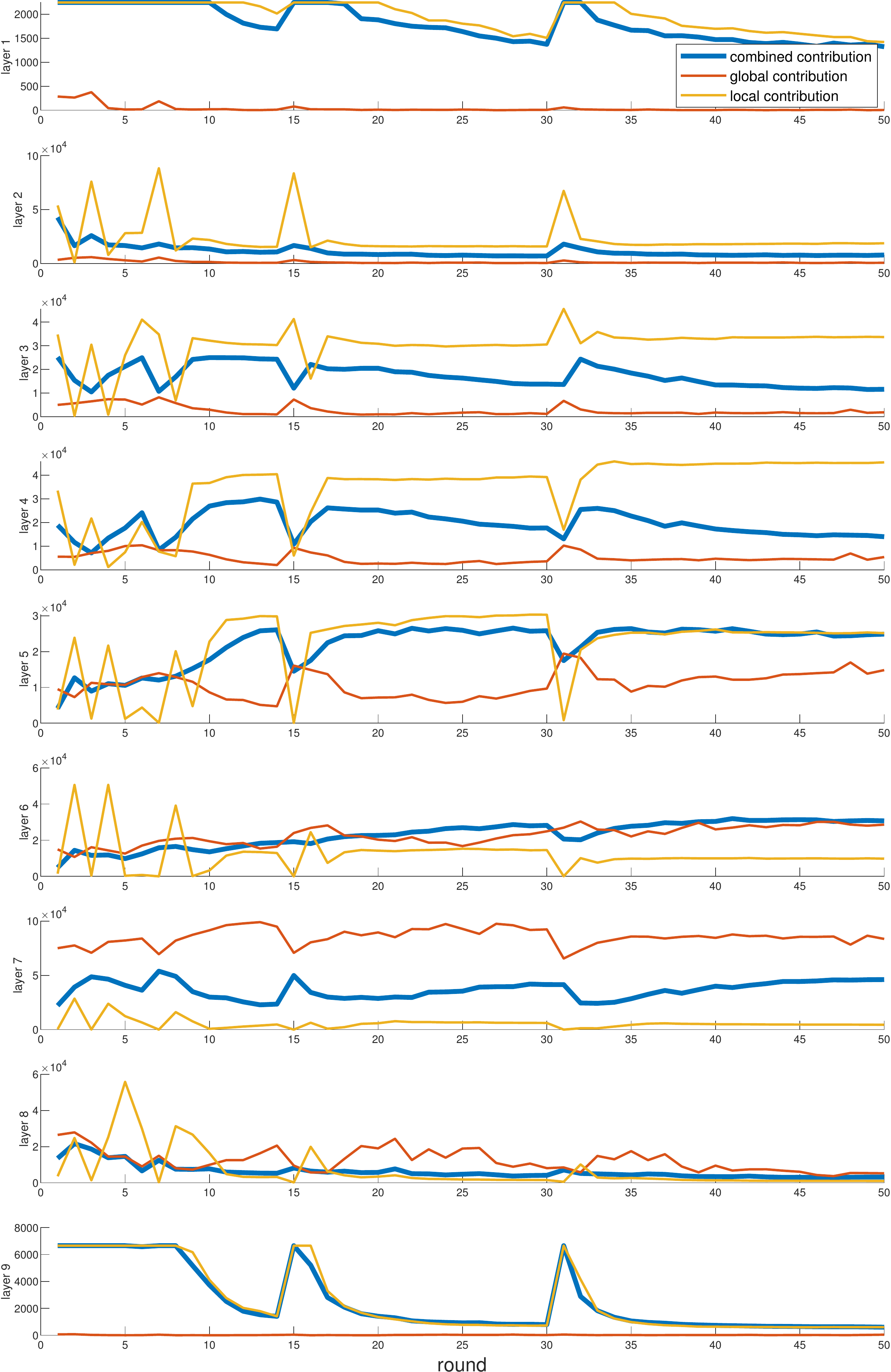}
	\caption[Number of updated weights across all layers when adopting different rewinding metrics with $k=0.01$.]{Number of updated weights across all layers (VGGNet) when adopting different rewinding metrics (updating ratio $k=0.01$).}
    \label{ch5-fig:updatednum1}
\end{figure}

\begin{figure}[!tbp]
    \centering
	\includegraphics[width=0.95\textwidth]{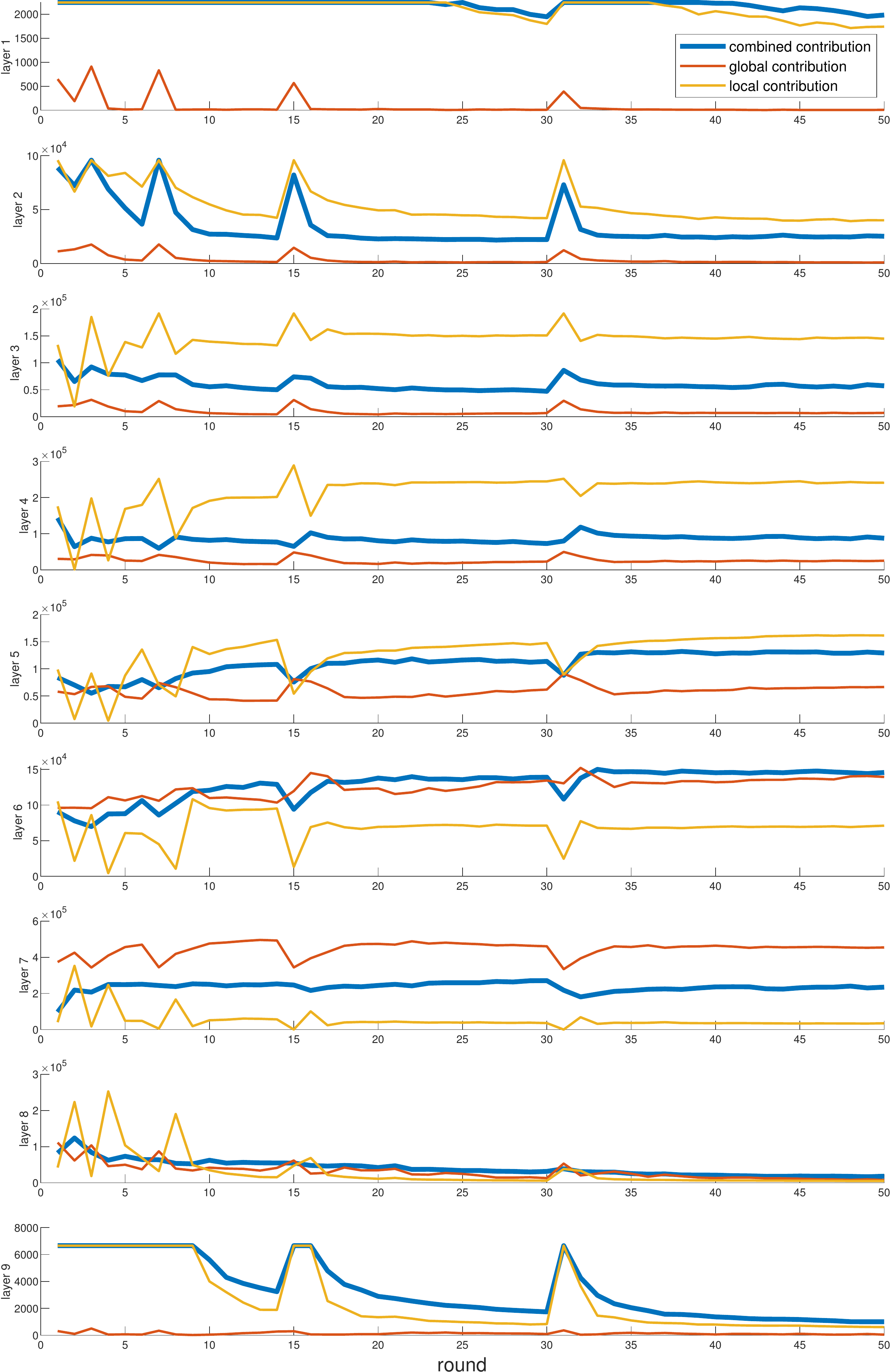}
	\caption[Number of updated weights across all layers when adopting different rewinding metrics with $k=0.05$.]{Number of updated weights across all layers (VGGNet) when adopting different rewinding metrics (updating ratio $k=0.05$).}
    \label{ch5-fig:updatednum2}
\end{figure}

\begin{figure}[!tbp]
    \centering
	\includegraphics[width=0.95\textwidth]{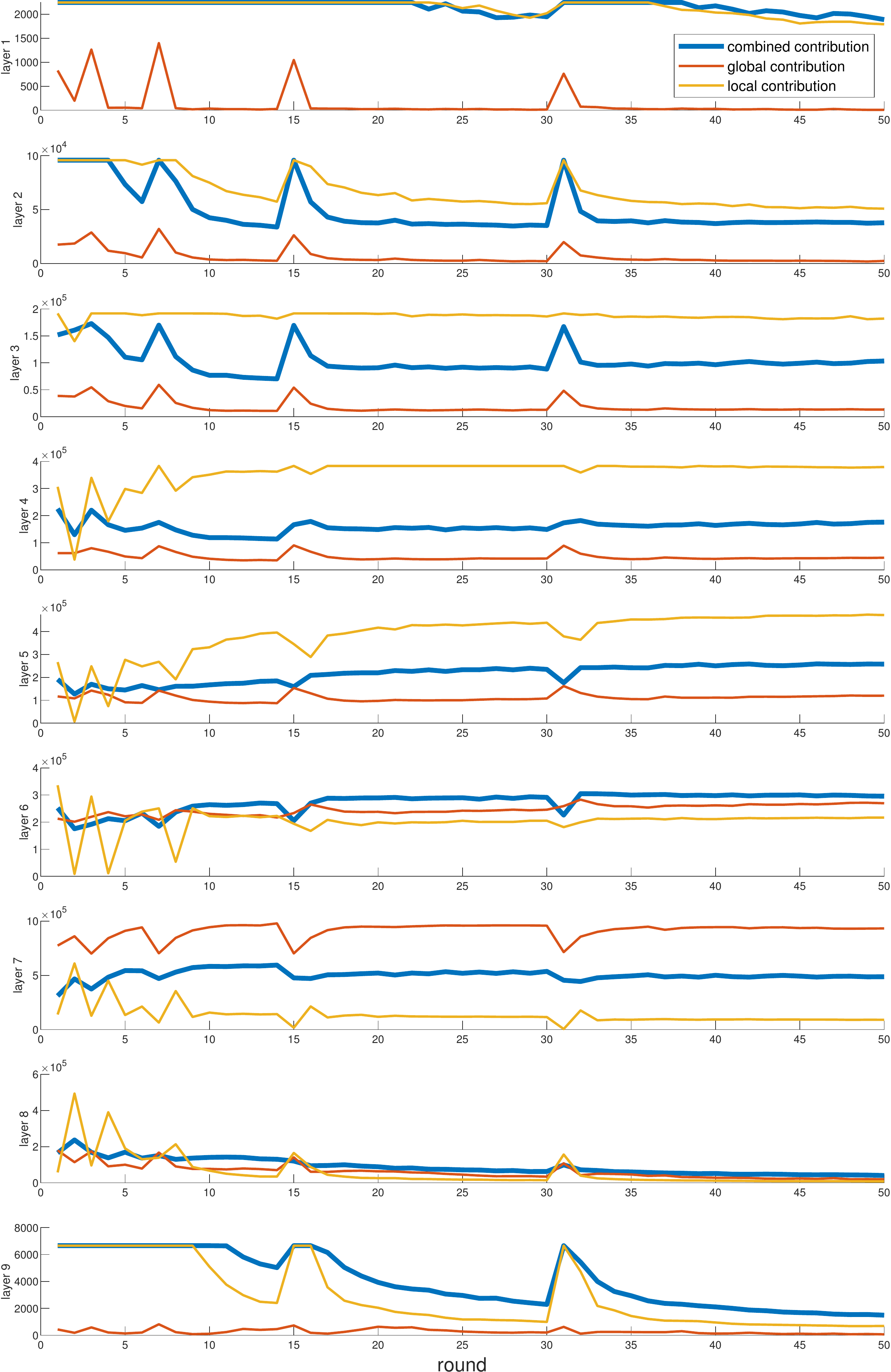}
	\caption[Number of updated weights across all layers when adopting different rewinding metrics with $k=0.1$.]{Number of updated weights across all layers (VGGNet) when adopting different rewinding metrics (updating ratio $k=0.1$).}
    \label{ch5-fig:updatednum3}
\end{figure}

\fakeparagraph{Results}
We plot the number of updated weights across all layers along rounds, under updating ratio $k=0.01,0.05,0.1$ in \figref{ch5-fig:updatednum1}, \figref{ch5-fig:updatednum2}, and \figref{ch5-fig:updatednum3}, respectively. 
We also plot the corresponding test accuracy along rounds in \figref{ch5-fig:updatedacc}.
Generally, the metric of local contribution updates more weights in the first several layers and the last layer while with a large variance along rounds.
On the contrary, global contribution selects more weights in the last several layers (until the penultimate layer) to update.
Combined contribution (the sum of normalized local/global contribution) achieves a more robust and balanced distribution of updated weights across layers than other contributions. 
It also results in the highest accuracy level especially under a small updating ratio.
Intuitively, local contribution can better identify critical weights w.r.t. the loss during training, while global contribution may be more robust for a highly non-convex loss landscape. 
Both metrics may be necessary when selecting weights to rewind.
Note that the proposed combined contribution is not the simple averaging of both local and global contribution.
For example, in ``layer 6'' of \figref{ch5-fig:updatednum3}, the number of updated weights by combined contribution already exceeds the other two metrics.

\begin{figure}[tbp]
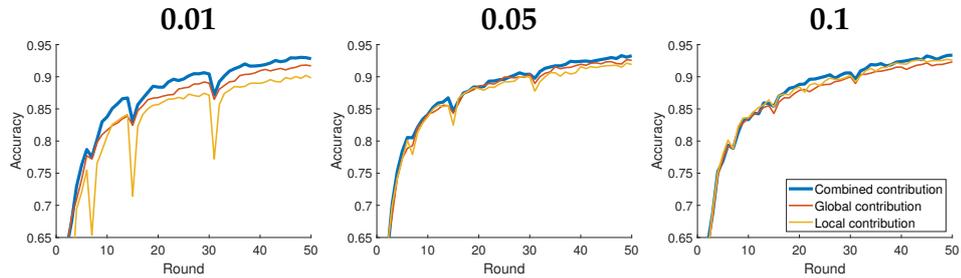

    \setlength\tabcolsep{\imgtabcolsep}
    \centering
    \begin{tabular}{ccc}
        \textbf{~~~~~~0.01}                             & \textbf{~~~~~~0.05}                               & \textbf{~~~~~~0.1}                                \\
        \tabimgb{./figs/ch5/1000-1000-1-updated-acc}    & \tabimgb{./figs/ch5/1000-1000-2-updated-acc}      & \tabimgb{./figs/ch5/1000-1000-3-updated-acc}      \\
    \end{tabular}
    \caption[The test accuracy of partial updating methods with different rewinding metrics.]{The test accuracy of partial updating methods with different rewinding metrics (updating ratio $k=0.01,0.05,0.1$).}
    \label{ch5-fig:updatedacc}
\end{figure}

\subsection{Benchmarking Multi-Round Updating}
\label{ch5-sec:experiment_multiround}

\fakeparagraph{Settings} 
To the best of our knowledge, this is the first work on studying weight-wise partial updating a model using newly collected data in iterative rounds. 
Therefore, we developed three baselines for comparison, including (\textit{i}) full updating (FU), where at each round the model is fully updated from a random initialization (\ie training from scratch, which yields a better performance see \secref{ch5-sec:initialization} and \secref{ch5-sec:experiment_fullupdating}); (\textit{ii}) random partial updating (RPU), where the model is trained from $\bm{w}^{r-1}$, while we randomly fix each layer's weights with a ratio of $(1-k)$ and sparsely fine-tune the rest; (\textit{iii}) global contribution partial updating (GCPU), where the model is trained with \algoref{ch5-alg:gcpu} without re-initialization described in \secref{ch5-sec:initialization}; (\textit{iv}) a state-of-the-art unstructured pruning method \cite{bib:ICLR20:Renda}, where the model is first trained from a random initialization at each round, then conducts one-shot magnitude pruning, and finally is sparsely fine-tuned with learning rate rewinding. 
The ratio of nonzero weights in pruning is set to the same as the updating ratio $k$ to ensure the same communication cost.
The experiments are conducted on different benchmarks as mentioned earlier.

\begin{figure}[t]
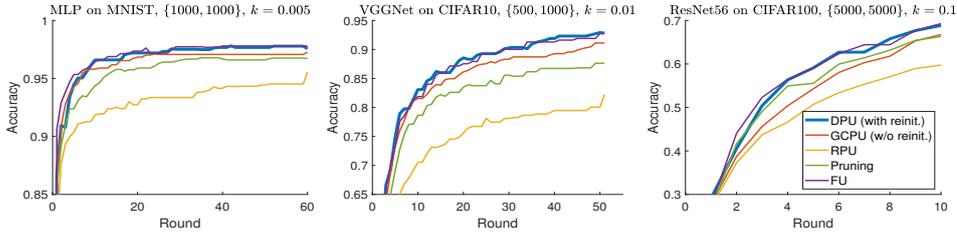

    \setlength\tabcolsep{\imgtabcolsep}
    \centering
        \begin{tabular}{ccc}
            \tabimgd{./figs/ch5/mlp-pruning}    & \tabimgd{./figs/ch5/vgg-pruning}      & \tabimgd{./figs/ch5/resnet56-pruning}      \\ 
        \end{tabular}
    \caption[\dpu is compared with other baselines on different benchmarks in terms of the test accuracy during multi-round updating.]{\dpu is compared with other baselines on different benchmarks in terms of the test accuracy during multi-round updating.}
    \label{ch5-fig:multiround}
\end{figure}

\begin{table}[t]
    \centering
    \caption[The average accuracy difference (full updating as the reference) over all rounds and the ratio of communication cost over all rounds related to full updating.]{The average accuracy difference over all rounds and the ratio of communication cost over all rounds related to full updating.} 
    \label{ch5-tab:multiround}
    \footnotesize
    \begin{tabular}{ccccccc}
        \toprule
        \multirow{2}{* }{Method}            & \multicolumn{3}{c}{Average accuracy difference}           & \multicolumn{3}{c}{Ratio of communication cost}       \\ \cmidrule(lr){2-4}  \cmidrule(lr){5-7} 
                                            & MLP               & VGGNet            & ResNet56          & MLP          & VGGNet         & ResNet56              \\ \hline
        \dpu                                & $\bm{-0.17\%}$    & $\bm{+0.33\%}$    & $\bm{-0.42\%}$    & $0.0071$     & $0.0183$       & $0.1147$              \\
        GCPU                                & $-0.72\%$         & $-1.51\%$         & $-3.87\%$         & $0.0058$     & $0.0198$       & $0.1274$              \\
        RPU                                 & $-4.04\%$         & $-11.35\%$        & $-7.78\%$         & $0.0096$     & $0.0167$       & $0.1274$              \\
        Pruning \cite{bib:ICLR20:Renda}     & $-1.45\%$         & $-4.35\%$         & $-2.35\%$         & $0.0106$     & $0.0141$       & $0.1274$              \\
        \bottomrule
    \end{tabular}
\end{table}

\fakeparagraph{Results}
We report the test accuracy of the model $\bm{w}^r$ along rounds in \figref{ch5-fig:multiround}.
All methods start from the same $\bm{w}^0$, an entirely randomly initialized model.
As seen in this figure, \dpu clearly yields the highest accuracy in comparison to the other partial updating schemes.
For example, \dpu can yield a final Top-1 accuracy of $92.85\%$ on VGGNet, even exceeds the accuracy ($92.73\%$) of full updating, while GCPU, RPU, and Pruning only acquire $91.11\%$, $82.21\%$, and $87.62\%$ respectively.
In addition, we compare three partial updating schemes in terms of the accuracy difference related to full updating averaged over all rounds, and the ratio of the communication cost related to full updating over all rounds in \tabref{ch5-tab:multiround}. 
As seen in the table, \dpu reaches a similar accuracy as full updating, while incurring significantly fewer transmitted data sent from the server to each edge node.
Specially, \dpu saves around $99.3\%$, $98.2\%$ and $88.5\%$ of transmitted data on MLP, VGGNet, and ResNet56, respectively ($95.3\%$ in average). 
The communication cost ratios shown in \tabref{ch5-tab:multiround} differ a little even for the same updating ratio. 
This is because if the validation accuracy does not increase, the model will not be updated to reduce the communication cost, as discussed earlier.
The horizontal straight line segments in \figref{ch5-fig:multiround} represent those non-updated rounds under each method.

\subsubsection{Experiments on Total Communication Cost Reduction}
\label{ch5-sec:experiment_totalcost}

\begin{figure}[tbp!]
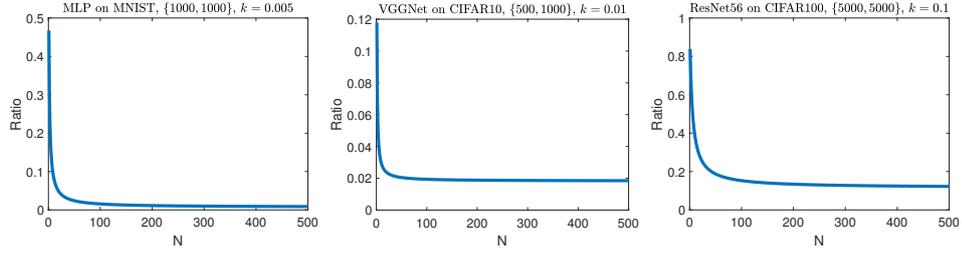

    \setlength\tabcolsep{\imgtabcolsep}
    \centering
    \begin{tabular}{ccc}
        \tabimgb{./figs/ch5/mlp-N}    & \tabimgb{./figs/ch5/vgg-N}      & \tabimgb{./figs/ch5/resnet56-N}      \\ 
    \end{tabular}
    \caption[The ratio, between the total communication cost under \dpu and that under full updating, varies with the number of nodes $N$.]{The ratio, between the total communication cost (over all rounds) under \dpu and that under full updating, varies with the number of nodes $N$.}
    \label{ch5-fig:communicationratio}
\end{figure}

\fakeparagraph{Settings}
In this section, we show the benefit due to \dpu in terms of \textit{the total communication cost reduction}, as \dpu has no impact on the edge-to-server communication which may involve sending newly collected data samples on nodes.
The total communication cost includes both edge-to-server communication and server-to-edge communication.
This experimental setup assumes that all data samples in $\delta\mathcal{D}^r$ are collected by $N$ edge nodes during all rounds and sent to the server on a per-round basis.
In other words, the first stage (see in \secref{ch5-sec:introduction}) is anyway necessary for sending new training data to the server.
For clarity, let $S_d$ denote the data size of each training sample. 
During round $r$, we define the per-node total communication cost under \dpu as $S_d\cdot|\delta\mathcal{D}^r|/N + (S_w \cdot k \cdot I + S_x(k) \cdot I)$.
Similarly, the per-node total communication cost under full updating is defined as $S_d\cdot|\delta\mathcal{D}^r|/N + S_w \cdot I$.

In order to simplify the demonstration, we consider the scenario where $N$ nodes send a certain amount of data samples to the server in $R-1$ rounds, namely $\sum_{r=2}^R |\delta \mathcal{D}^r|$ (see \secref{ch5-sec:initialization}).
Thus, the average data size transmitted from each node to the server in all rounds is $\sum_{r=2}^R S_d\cdot|\delta \mathcal{D}^r|/N$. 
A larger $N$ implies a fewer amount of transmitted data from each node to the server.

\fakeparagraph{Results}
We report the ratio of the total communication cost over all rounds required by \dpu related to full updating, when \dpu achieves a similar accuracy level as full updating (corresponding to three evaluations in \figref{ch5-fig:multiround}).
The ratio clearly depends on $\sum_{r=2}^R S_d\cdot|\delta \mathcal{D}^r|/N$, \ie the number of nodes $N$.
The relation between the ratio and $N$ is plotted in \figref{ch5-fig:communicationratio}. 

We observe that \dpu can still achieve a significant reduction on the total communication cost, \eg reducing up to $88.2\%$ even for a single node. 
Single node corresponds to the largest data size during edge-to-serve transmission per node, \ie the worst case.  
Moreover, \dpu tends to be more beneficial when the size of data transmitted by each node to the server becomes smaller. 
This is intuitive because in this case the server-to-edge communication (thus the reduction due to \dpu) dominants in the entire communication.

\subsection{Different Number of Data Samples and Updating Ratios}
\label{ch5-sec:experiment_samplesratio}

\fakeparagraph{Settings}
In this section, we show that \dpu outperforms other baselines under varying number of training samples and updating ratios in multi-round updating.
We also conduct ablations concerning the re-initialization of weights discussed in \secref{ch5-sec:initialization}.
We implement \dpu with and without re-initialization, GCPU with and without re-initialization, RPU, and Pruning \cite{bib:ICLR20:Renda} (see more details in \secref{ch5-sec:experiment_multiround}) on VGGNet using CIFAR10 dataset. 
We compare these methods with different amounts of samples $\{|\mathcal{D}^1|,|\delta\mathcal{D}^r|\}$ and different updating ratios $k$.
Without further notations, each experiment runs three times using random data samples.

\begin{figure}[t]
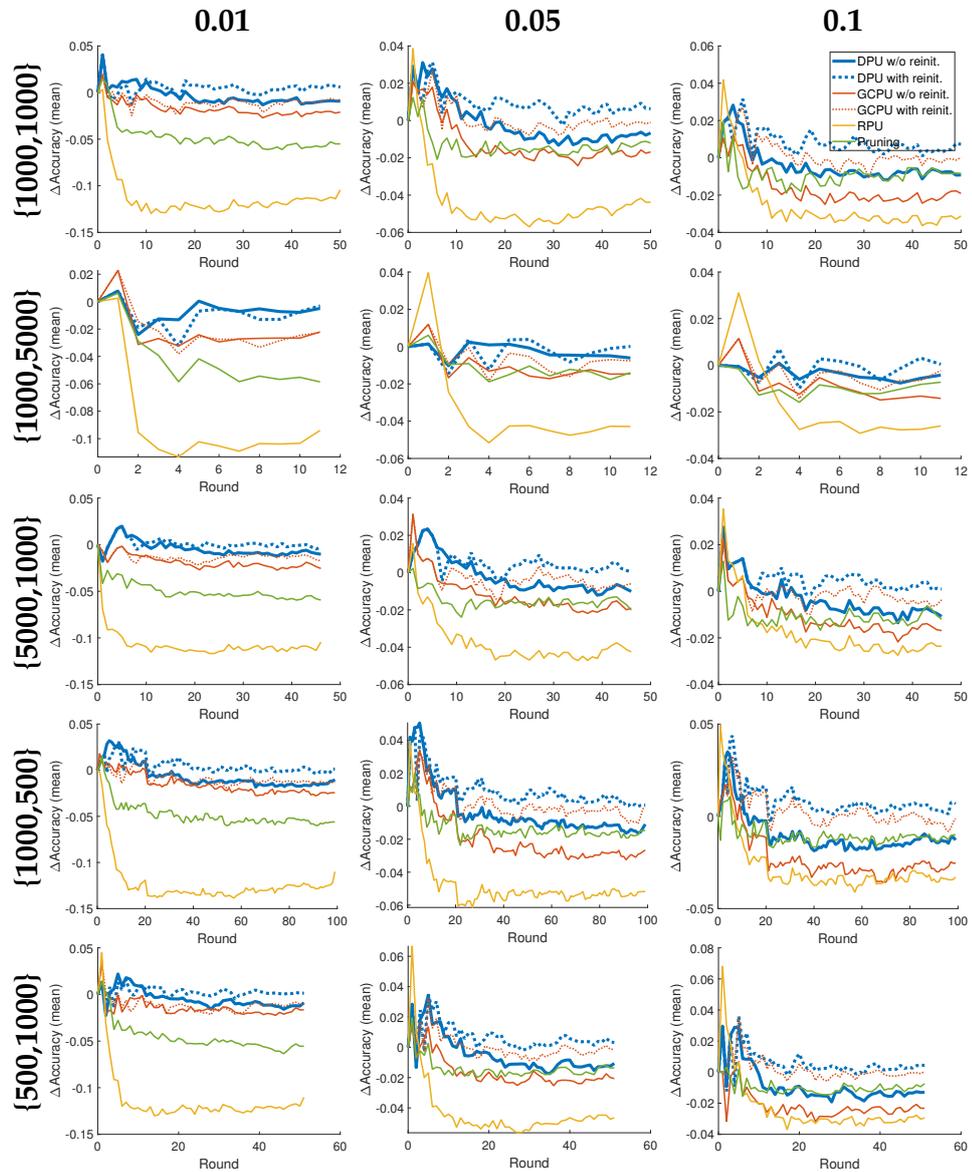

    \setlength\tabcolsep{\imgtabcolsep}
    \centering
    \begin{tabular}{m{0.3cm}ccc}
                                                    & \textbf{~~~~~~0.01}                       & \textbf{~~~~~~0.05}                           & \textbf{~~~~~~0.1}                            \\ 
        \rotatebox{90}{\textbf{\{1000,1000\}}}      & \tabimgc{./figs/ch5/1000-1000-1-pruningd} & \tabimgc{./figs/ch5/1000-1000-2-pruningd}     & \tabimgc{./figs/ch5/1000-1000-3-pruningd}     \\ 
        \rotatebox{90}{\textbf{\{1000,5000\}}}      & \tabimgc{./figs/ch5/1000-5000-1-pruningd} & \tabimgc{./figs/ch5/1000-5000-2-pruningd}     & \tabimgc{./figs/ch5/1000-5000-3-pruningd}     \\ 
        \rotatebox{90}{\textbf{\{5000,1000\}}}      & \tabimgc{./figs/ch5/5000-1000-1-pruningd} & \tabimgc{./figs/ch5/5000-1000-2-pruningd}     & \tabimgc{./figs/ch5/5000-1000-3-pruningd}     \\
        \rotatebox{90}{\textbf{\{1000,500\}}}       & \tabimgc{./figs/ch5/1000-500-1-pruningd}  & \tabimgc{./figs/ch5/1000-500-2-pruningd}      & \tabimgc{./figs/ch5/1000-500-3-pruningd}      \\
        \rotatebox{90}{\textbf{\{500,1000\}}}       & \tabimgc{./figs/ch5/500-1000-1-pruningd}  & \tabimgc{./figs/ch5/500-1000-2-pruningd}      & \tabimgc{./figs/ch5/500-1000-3-pruningd}      \\
    \end{tabular}
    \caption[Comparison w.r.t. the mean accuracy difference (full updating as the reference) under different settings.]{Comparison w.r.t. the mean accuracy difference (full updating as the reference) under different $\{|\mathcal{D}^1|,|\delta\mathcal{D}^r|\}$ (representing the available data samples along rounds) and updating ratio ($k=0.01,0.05,0.1$) settings.}
    \label{ch5-fig:number_ratio_d}
\end{figure}

\begin{figure}[t]
    \setlength\tabcolsep{\imgtabcolsep}
    \centering
    \begin{tabular}{m{0.3cm}ccc}
                                                    & \textbf{~~~~~~0.01}                       & \textbf{~~~~~~0.05}                           & \textbf{~~~~~~0.1}                            \\ 
        \rotatebox{90}{\textbf{\{1000,1000\}}}      & \tabimgc{./figs/ch5/1000-1000-1-pruning}  & \tabimgc{./figs/ch5/1000-1000-2-pruning}      & \tabimgc{./figs/ch5/1000-1000-3-pruning}      \\ 
        \rotatebox{90}{\textbf{\{1000,5000\}}}      & \tabimgc{./figs/ch5/1000-5000-1-pruning}  & \tabimgc{./figs/ch5/1000-5000-2-pruning}      & \tabimgc{./figs/ch5/1000-5000-3-pruning}      \\ 
        \rotatebox{90}{\textbf{\{5000,1000\}}}      & \tabimgc{./figs/ch5/5000-1000-1-pruning}  & \tabimgc{./figs/ch5/5000-1000-2-pruning}      & \tabimgc{./figs/ch5/5000-1000-3-pruning}      \\
        \rotatebox{90}{\textbf{\{1000,500\}}}       & \tabimgc{./figs/ch5/1000-500-1-pruning}   & \tabimgc{./figs/ch5/1000-500-2-pruning}       & \tabimgc{./figs/ch5/1000-500-3-pruning}       \\
        \rotatebox{90}{\textbf{\{500,1000\}}}       & \tabimgc{./figs/ch5/500-1000-1-pruning}   & \tabimgc{./figs/ch5/500-1000-2-pruning}       & \tabimgc{./figs/ch5/500-1000-3-pruning}       \\
    \end{tabular}
    \caption[Comparison w.r.t. the mean accuracy under different settings.]{Comparison w.r.t. the mean accuracy under different $\{|\mathcal{D}^1|,|\delta\mathcal{D}^r|\}$ (representing the available data samples along rounds) and updating ratio ($k=0.01,0.05,0.1$) settings.}
    \label{ch5-fig:number_ratio}
\end{figure}

\begin{figure}[t]
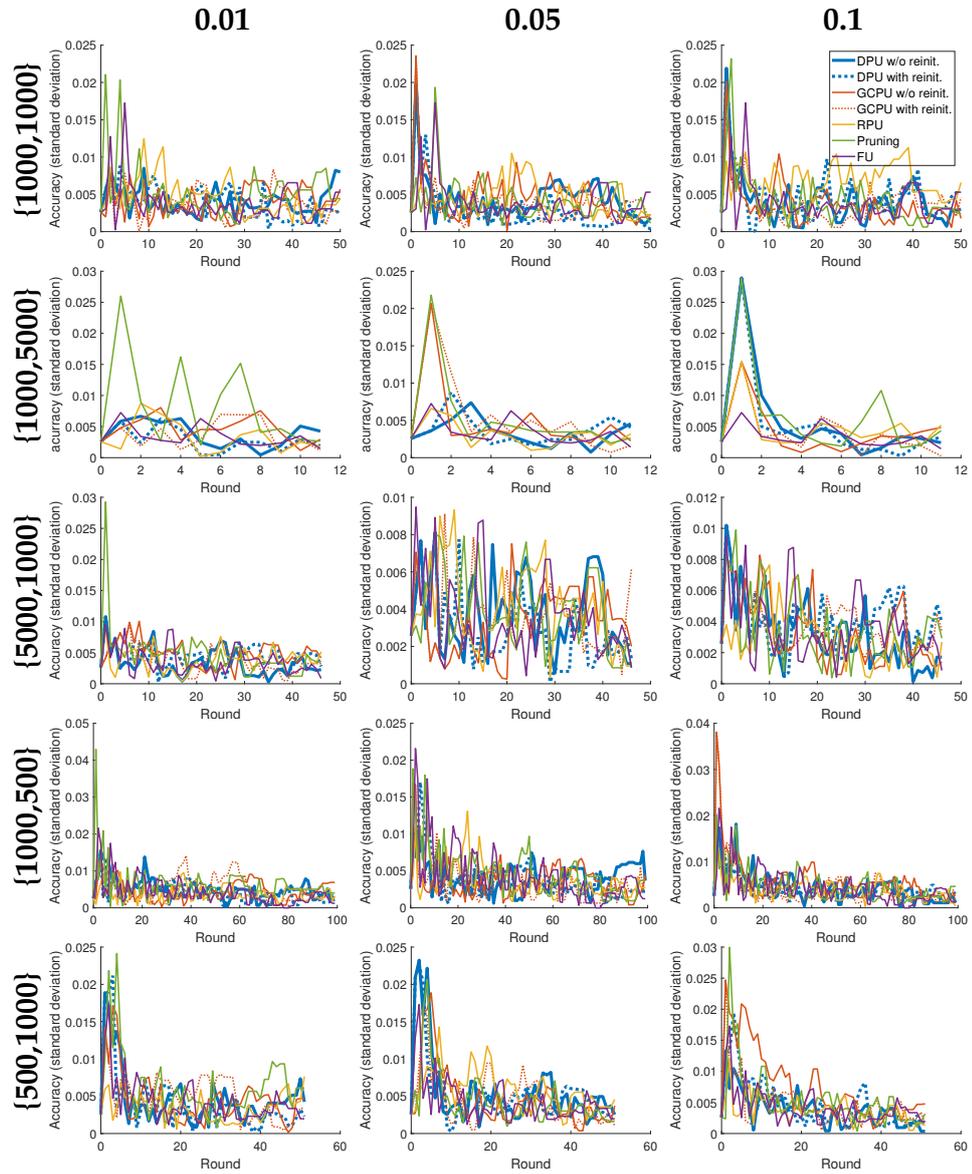

    \setlength\tabcolsep{\imgtabcolsep}
    \centering
    \begin{tabular}{m{0.3cm}ccc}
                                                    & \textbf{~~~~~~0.01}                        & \textbf{~~~~~~0.05}                          & \textbf{~~~~~~0.1}                            \\ 
        \rotatebox{90}{\textbf{\{1000,1000\}}}      & \tabimgc{./figs/ch5/1000-1000-1-pruningv}  & \tabimgc{./figs/ch5/1000-1000-2-pruningv}    & \tabimgc{./figs/ch5/1000-1000-3-pruningv}     \\ 
        \rotatebox{90}{\textbf{\{1000,5000\}}}      & \tabimgc{./figs/ch5/1000-5000-1-pruningv}  & \tabimgc{./figs/ch5/1000-5000-2-pruningv}    & \tabimgc{./figs/ch5/1000-5000-3-pruningv}     \\ 
        \rotatebox{90}{\textbf{\{5000,1000\}}}      & \tabimgc{./figs/ch5/5000-1000-1-pruningv}  & \tabimgc{./figs/ch5/5000-1000-2-pruningv}    & \tabimgc{./figs/ch5/5000-1000-3-pruningv}     \\
        \rotatebox{90}{\textbf{\{1000,500\}}}       & \tabimgc{./figs/ch5/1000-500-1-pruningv}   & \tabimgc{./figs/ch5/1000-500-2-pruningv}     & \tabimgc{./figs/ch5/1000-500-3-pruningv}      \\
        \rotatebox{90}{\textbf{\{500,1000\}}}       & \tabimgc{./figs/ch5/500-1000-1-pruningv}   & \tabimgc{./figs/ch5/500-1000-2-pruningv}     & \tabimgc{./figs/ch5/500-1000-3-pruningv}      \\
    \end{tabular}
    \caption[Comparison w.r.t. the standard deviation of accuracy under different settings.]{Comparison w.r.t. the standard deviation of accuracy under different $\{|\mathcal{D}^1|,|\delta\mathcal{D}^r|\}$ (representing the available data samples along rounds) and updating ratio ($k=0.01,0.05,0.1$) settings.}
    \label{ch5-fig:number_ratio_v}
\end{figure}

\fakeparagraph{Results}
We compare the difference between the accuracy under each method and that under full updating. 
The mean accuracy difference over three runs is plotted in \figref{ch5-fig:number_ratio_d}.
In addition, we also plot the mean and standard deviation of the absolute accuracy of these methods in \figref{ch5-fig:number_ratio} and \figref{ch5-fig:number_ratio_v}, respectively.
As seen in \figref{ch5-fig:number_ratio_d}, \dpu (with re-initialization) always achieves the highest accuracy. 
\dpu also significantly outperforms the pruning method, especially under a small updating ratio. 
Note that we preferred a smaller updating ratio in our context because it explores the limits of the approach and it indicates that we can improve the deployed model more frequently with the same accumulated server-to-edge communication cost.
The dashed curves and the solid curves with the same color can be viewed as the ablation study of our re-initialization scheme. 
Particularly given a large number of rounds, it is critical to re-initialize the start point $\bm{w}^{r-1}$ after several rounds (as discussed in \secref{ch5-sec:initialization}). 

In the first few rounds, partial updating methods almost always yield a higher test accuracy than full updating, \ie the curves are above zero. 
This is due to the fact that the amount of available samples is rather small, and partial updating may avoid some co-adaptation in full updating.
The partial updating methods perform almost randomly in the first round compared to each other, because the limited data are not sufficient to distinguish critical weights from the random initialization $\bm{w}^0$.
This also motivates us to (partially) update the deployed model when new samples are available.

\fakeparagraph{Pruning Weights vs. Pruning Incremental Weights}
One of our chosen baselines, global contribution partial updating (GCPU, \algoref{ch5-alg:gcpu}), could be viewed as a counterpart of the pruning method \cite{bib:ICLR20:Renda}, \ie pruning the incremental weights with the least magnitudes.
Specially, the elements with the smallest absolute values in $\delta\bm{w}^{\mathrm{f}}$ are set to zero (also rewinding), while the remaining weights are further sparsely fine-tuned with the same learning rate schedule as training $\bm{w}^{\mathrm{f}}$.
In comparison to traditional pruning on weights \cite{bib:ICLR20:Renda}, pruning on incremental weights has a different start point. 
Traditional pruning on weights first trains randomly initialized weights (a zero-initialized model cannot be trained due to the symmetry), and then prunes the weights with the smallest magnitudes.
However, the increment of weights $\delta\bm{w}^\mathrm{f}$ is initialized with zero in \algoref{ch5-alg:gcpu}, since the first step starts from $\bm{w}$.
By comparing GCPU (with or without re-initialization) with ``Pruning'', we conclude that retaining previous weights yields better performance than zero-outing the weights.

\subsection{Benchmarking Single-Round Updating}
\label{ch5-sec:experiment_singleround}

\fakeparagraph{Settings}
To show the versatility of our methods, we test single-round updating for MobileNetV1 \cite{bib:arXiv17:Howard} on ImageNet \cite{bib:ILSVRC15} with iterative rewinding. 
Single-round \dpu is conducted on different initially deployed models, including a floating-point (FP32) dense model and two compressed models, \ie a 50\%-sparse model and an INT8 quantized model.
The sparse model is trained with a state-of-the-art dynamic pruning method \cite{bib:NIPS21:Peste}; the quantized model is trained with straight-through-estimator with a output-channel-wise floating point scaling factors similar as \cite{bib:ECCV16:Rastegari}. 
To maintain the same on-device inference cost, partial updating is only applied on nonzero values of sparse models; for quantized models, the updated weights are still in INT8 format. 
Note that we do not impose sparsity or quantization on batch normalization and bias.

\fakeparagraph{Results}
We compare \dpu with the vanilla-updates, \ie the models are trained from a random initialization with the corresponding methods on all available samples. 
The test accuracy and the ratio of (server-to-edge) communication cost related to full updating on FP32 dense model are reported in \tabref{ch5-tab:singleround}. 
Results show that \dpu often yields a higher accuracy than vanilla updating while requiring substantially lower communication cost. 

\begin{table}[tb]
    \centering
    \caption[The test accuracy of single-round updating on different initially deployed models.]{The test accuracy of single-round updating on different initial models (MobileNetV1 on ImageNet). The updating ratio $k=0.2$. The ratio of communication cost related to full updating is reported in brackets} 
    \label{ch5-tab:singleround}
    \small
    \begin{tabular}{cccccc}
        \toprule
        \#Samples       & \multicolumn{3}{c}{$\{8\times10^5,4.8\times10^5\}$}   \\ 
                        \cmidrule(lr){2-4}                                     
                        & Initial   & Vanilla-update    & \dpu                  \\ \hline      
        FP32 Dense      & 68.5\%    & 70.7\% (1)        &  71.1\% (0.22)        \\
        50\%-Sparse     & 68.1\%    & 70.5\% (0.53)     &  70.8\% (0.22)        \\ 
        INT8            & 68.4\%    & 70.6\% (0.25)     &  70.6\% (0.07)        \\
        \bottomrule
    \end{tabular}
\end{table}

\section{Summary}
\label{ch5-sec:summary}

In this chapter, we propose a novel pipeline \dpu for edge-server system.
\dpu enables deep learning on edge-server system that has limited on-device resources and limited communication resources. 
Particularly, when newly collected data samples from edge devices or from other sources are available at the server, the server smartly selects only a subset of critical weights to update at the server-to-edge communication round.
This partial updating scheme reduces the redundant updating by reusing the pretrained weights, \ie the learned knowledge on prior data, which achieves a similar performance as full updating yet with a significantly lower communication cost.
The main contributions of \dpu are summarized as follows,
\begin{itemize}
    \item We formalize the deep partial updating paradigm, \ie how to iteratively perform weight-wise partial updating of the inference models on remote edge devices, if newly collected training samples are available at the server. This substantially reduces the computation and communication demand on edge devices. 
    \item We propose a novel approach that determines the optimized subset of weights that shall be selected for partial updating, through measuring each weight's contribution to the analytical upper bound on the loss reduction. This simple yet effective metric can be applied to any models that are trained with gradient-based optimizers.
    \item Experimental results on public vision datasets show that, under the similar accuracy level along the rounds, our approach can reduce the size of the transmitted data by $95.3\%$ on average (up to $99.3\%$), namely can update the model averagely $21$ times more frequent than full updating.
\end{itemize}

\chapter{Conclusion}
\label{ch6:conclusion}

State-of-the-art DNNs achieve excellent prediction accuracy in many perception tasks, \eg computer vision, natural language processing, reinforcement learning, etc. 
However, a large amount of resources is essential in both the inference phase and the training phase to ensure the high performance of DNNs. 
Due to the intensive resource demands, DNNs are often deployed on a cloud server with plenty of high-performed computers and shared storage infrastructures. 

On the other hand, there is a growing interest to deploy DNNs on edge devices to enable new edge intelligent applications, \eg AR/VR, mobile assistants, IoT, autonomous driving, etc.
In comparison to a cloud server, edge devices have a rather small amount of resources from memory, computation, and energy, and often also a limited scalability. 
Conventional DNNs need to be compressed in order to fit the resource constraints on edge devices. 
As DNNs are prone to be over-parameterized, this thesis focuses on reducing the redundancy of DNNs to achieve a better trade-off between resource consumption and model accuracy.

In this thesis, we studied how to enable deep learning on edge devices in four different scenarios. 
Especially, we studied (\textit{i}) efficient inference on edge devices given fixed resource constraints in \chref{ch2:inference}, (\textit{ii}) efficient adaptation on edge devices under varying resource constraints in \chref{ch3:adaptation}, (\textit{iii}) efficient learning on edge devices with a few training samples of unseen tasks in \chref{ch4:learning}, and (\textit{iv}) efficient inference and updating on edge-server systems with a constrained communication bus in \chref{ch5:edgeserver}.
Note that different scenarios may have different main resource constraints that hinder us from deploying DNNs on edge devices.
According to the main resource constraints in these scenarios, we developed different methodologies to remove the redundant components, such that the compressed DNNs require a lower resource demand while reaching a similar accuracy level as the original ones.

In the following sections of this chapter, we will first summarize our main contributions in each scenario, then discuss the potential directions for future work.

\section{Contributions}
\label{ch6-sec:contribution}

This section summarizes the main contributions of our work in each scenario. 

\subsection{Inference on Edge Devices (\chref{ch2:inference})}
\label{ch6-sec:inference}

In \chref{ch2:inference}, we enabled an efficient inference of DNNs on edge devices. 
In comparison to cloud inference, inference on edge devices does not need to upload the input data to the cloud server, which can achieve a more stable, fast, and energy-efficient inference, especially with a constrained communication bus. 
Regarding the main resource constraints from storing a large number of weights and computation during inference, we proposed \alq, an adaptive loss-aware trained quantizer for multi-bit networks. 
\alq reduces the redundancy on the quantization bitwidth. 

Unlike prior multi-bit quantization that often assigns an empirical uniform bitwidth, \alq learns an adaptive bitwidth assignment across different groups of weights according to their loss criticality. 
\alq also proposes to optimize the multi-bit quantized weights by directly minimizing the loss function rather than the reconstruction error to the full precision weights.  
The multi-bit quantized network uses cheaper operations from \texttt{xnor} and \texttt{popcount} to replace the expensive FLOPs, achieving computation efficiency;
the learned adaptive bitwidth yields a smaller average bitwidth by only allocating a high bitwidth to the loss-critical weights, achieving storage efficiency;
the direct optimization objective (\ie the loss) allows us to acquire a quantized network with higher prediction accuracy.   
In addition, \alq also enables extremely low-bit networks with an average bitwidth below 1-bit by entirely pruned groups (\ie 0-bit weights in some groups).

\subsection{Adaptation on Edge Devices (\chref{ch3:adaptation})}
\label{ch6-sec:adaptation}

The methods proposed in \chref{ch2:inference} are able to compress DNNs for efficient inference if the amount of available resources on edge devices is fixed and known beforehand. 
However, the resource constraints on the target edge devices may dynamically change during runtime, \eg the allowed execution time, the allocatable RAM, and the battery energy.
To maximize the model accuracy during on-device inference, in \chref{ch3:adaptation}, we enabled a DNN with dynamic capacity, such that the DNN can be adapted and executed under varying resource constraints. 
Particularly, we developed a new synthesis approach \dress that can sample and execute sub-networks with different resource demands from a backbone network for on-device inference. 
\dress reduces the redundancy among multiple sub-networks by weight sharing and architecture sharing.

\dress samples sub-networks in a row-based unstructured manner (a.k.a. fine-grained structure sparsity) from the backbone network, and introduces a novel compressed sparse row (CSR) format to utilize sparse tensor computation provided by recent compilation libraries. 
In \dress, the nonzero weights of the higher sparsity sub-networks are reused by the lower sparsity sub-networks, achieving memory efficiency; 
all sparse sub-networks leverage the same architecture as the backbone network, achieving re-configuration efficiency.
The sub-networks have different sparsity, and thus can be fetched and executed under various resource constraints.

\subsection{Learning on Edge Devices (\chref{ch4:learning})}
\label{ch6-sec:learning}

In \chref{ch2:inference} and \chref{ch3:adaptation}, we compressed DNNs to realize an efficient on-device inference under \textit{fixed} and \textit{varying} resource constraints, respectively.
However, when facing unseen environments or users on edge devices, it is crucial to retrain the DNN with newly collected data samples to deliver consistent performance and customized services. 
On the one hand, data samples collected by edge devices are often private and limited; on the other hand, training a DNN often consumes several orders of magnitude more peak memory than inference.
Hence, in \chref{ch4:learning}, we proposed a new meta learning method \pMeta to enable memory-efficient few-shot learning on unseen tasks.
\pMeta reduces the updating redundancy by fixing some weights during few-shot learning, which saves the memory consumption that is necessary for the updated weights.

\pMeta enables both data- and memory-efficient on-device learning given unseen tasks, which is realized by automatically identifying adaptation-critical weights during few-shot learning via a meta-trained selection mechanism. 
\pMeta adopts a hierarchical approach that combines a static selection on adaptation-critical layers and a dynamic selection on adaptation-critical channels.
To the best of our knowledge, \pMeta is the first meta learning method designed for on-device few-shot learning.
Evaluations on few-shot image classification and reinforcement learning show that \pMeta not only improves the accuracy but also reduces the peak dynamic memory by a factor of 2.5 on average over the state-of-the-art few-shot learning methods. 

\subsection{Edge-Server-System (\chref{ch5:edgeserver})}
\label{ch6-sec:edgeserver}

In \chref{ch2:inference}, \chref{ch3:adaptation} and \chref{ch4:learning}, we enabled deep learning on a single edge platform in three different scenarios. 
In \chref{ch5:edgeserver}, we designed a new pipeline \dpu to enable efficient inference and efficient updating for edge-server system.
In edge-server system, a set of resource-constrained edge devices are connected to a remote server with sufficient resources, and some information is allowed to be communicated between edge devices and the server. 
Due to the limited relevant training data beforehand, pretrained DNNs may be significantly improved after the initial deployment. 
On such an edge-server system, on-device inference is preferred over cloud inference, since it can achieve a fast and stable inference with less energy consumption. 
Yet retraining on the cloud server is preferred over on-device retraining (or federated learning) due to the limited memory and computing power on edge devices. 
Therefore, we proposed a two-stage iterative process to update the deployed inference models, (\textit{i}) at each round, edge devices collect new data samples and send them to the server, and (\textit{ii}) the server retrains the network using collected data, and then sends the updates to each edge device. 
In comparison to the edge-to-server stage, the transmissions in the server-to-edge stage are highly constrained by the limited communication resource (\eg bandwidth, energy).
Our \dpu reduces the server-to-edge communication cost by distinguishing the redundant updating given newly collected samples.

Particularly, \dpu studied how to iteratively perform weight-wise partial updating of inference models on remote edge devices, if newly collected training samples are available at the server.
In each round, \dpu smartly selects and updates a small subset of critical weights that have a large contribution to the loss reduction during the retraining.
Experimental results show that \dpu can reach a similar accuracy level as full updating yet with a significantly lower communication cost. 

\section{Potential Future Directions}
\label{ch6-sec:future}

In this section, we discuss some potential directions for the future work. 
These potential future directions are either some extensions or complementaries of the works presented in the main chapters, or some other edge intelligence scenarios that have not been studied yet due to the time limitation. 

\subsection{Hardware Accelerators of \alq} 
\label{ch6-sec:future_alq}

\alq exhibits a high compression ratio on the benchmark evaluations in \chref{ch2:inference} without introducing sparse tensor computation. 
To deploy the multi-bit networks generated by \alq, the target hardware must support bitwise \texttt{xnor} and \texttt{popcount} operations for efficient execution. 
However, the current Arm Cortex CPUs \cite{bib:arm} do not include the computation units of \texttt{popcount}. 
Although some software libraries may provide functions for \texttt{popcount}, they are less efficient in pipelined computation.  
Designing some hardware accelerators \eg with FPGA that can support bitwise \texttt{xnor}, \texttt{popcount} and accumulation operations is a promising direction to enable efficient inference with multi-bit networks.

\subsection{Quantized \dress} 
\label{ch6-sec:future_dress}

Current \dress samples sub-networks from a floating-point backbone network. 
Applying \dress on a quantized backbone network (\eg 8-bit integer network) is also worth studying.
Especially, the sampled quantized sub-networks can be further accelerated by the fast kernels of sparse quantized computation. 
For example, CMSIS-NN \cite{bib:CMSIS-NN} can achieve a $4\times$ acceleration on 8-bit integer quantized networks compared to 32-bit floating-point networks on a 32-bit Arm Cortex-M CPUs.
In addition, it would be also interesting to explore the possibility of applying \dress on multi-bit quantized networks, \ie the combination of \alq and \dress.
    
\subsection{Latency-Aware \dress}
\label{ch6-sec:future_latency}

Note also that current \dress requires predefined sparsity levels. 
However, a higher sparsity level, \ie a smaller number of nonzero weights, does not always result in a shorter inference latency \cite{bib:ICLR20:Renda}. 
In the future, we encourage the following researchers to build a direct relation between sparsity and inference latency (or energy consumption). 
This can be realized by (\textit{i}) measuring the inference latency with some hardware simulators, (\textit{ii}) leveraging some real-time models to bound the computation time theoretically.
The latency-aware \dress that does not rely on proxies may fill the gap between the realistic speedup and the theoretical reduction of FLOPs mentioned in \secref{ch3-sec:deployment}.  
    
\subsection{Low-Precision Few-Shot Learning}
\label{ch6-sec:future_fsl}

In \chref{ch4:learning}, we introduced \pMeta, a hierarchical structured partial updating on meta-trained models when only a few training samples of new unseen tasks are given.
Although \pMeta can dramatically reduce the peak dynamic memory as well as the computation burden during few-shot learning, it still needs full-precision calculation during the backward propagation. 
As noted in prior works \cite{bib:NIPS20:Raihan,bib:ICLR20:Cambier,bib:NIPS18:Wang}, adopting a low-precision backward propagation can bring a similar performance as its full-precision versions in the vanilla training. 
A straightforward future direction is to apply low-precision training on few-shot learning scenarios, where weights, activations, and gradients are all presented in low-precision formats, \eg 8-bit integer. 
The step size of 8-bit integer training could be the number of bit shifting, which may be also meta-trained in a per layer per step manner. 
Conducting 8-bit integer few-shot learning on edge devices can not only further reduce the peak memory consumption, but also speedup the training process in comparison to 32-bit floating-point training. 
    
\begin{table}[t]
    \centering
 	\caption[Static memory of the model and the training samples in example self-supervised learning.]{Static memory of the model and the training samples in example self-supervised learning.}
 	\label{ch6-tab:examples}
 	\footnotesize
 	\begin{tabular}{lccc}
 		\toprule
 		\multirow{2}{*}{Benchmark}          & WRN-28-2 \cite{bib:BMVC16:Zagoruyko}      & WRN-28-2 \cite{bib:BMVC16:Zagoruyko}      \\ 
 		                                    & CIFAR10                                   & SVHN                                      \\ \midrule
 		Static Storage of Model (MB)        & $6.02$                                    & $6.02$                                    \\
 		Static Storage of Samples (MB)      & $184.32$                                  & $305.02$                                  \\
 		\bottomrule
 	\end{tabular}
\end{table}

\subsection{Streaming Self-Supervised Learning} 
\label{ch6-sec:future_ssl}

In \chref{ch4:learning}, we studied efficient few-shot learning on edge devices, where only a few training samples are given. 
In some other cases of on-device learning, although the labeled samples are limited due to the limited labor resources, it might be easy to collect a large number of unlabeled samples. 
Learning a DNN with a small number of labeled samples and a large number of unlabeled samples is known as self-supervised learning (or semi-supervised learning). 
Current self-supervised learning methods \cite{bib:NIPS19:Berthelot,bib:NIPS20:Sohn} often need to maintain all unlabeled samples. 
Even if on small-scale datasets, the static memory for storing samples is much larger than that for storing the self-supervised model. 
We summarize the static memory consumption for training samples and the self-supervised DNNs in two sample applications in \tabref{ch6-tab:examples}. 

We consider that the unlabeled samples are collected in a round-based streaming manner, and during the collection we can query the user for labeling. 
In this scenario, the main resource constraints are (\textit{i}) the limited number of querying labels, (\textit{ii}) the memory consumption for storing data samples, particularly unlabeled samples. 
We focus on reducing the redundancy of data samples. 
We will only select a coreset of unconfident samples to label and a coreset of representative samples to store \cite{bib:NIPS21:Killamsetty}.

\noindent
The problem in round $r$ is defined as follows. 

\fakeparagraph{Inputs}
We have the current optimized model, and the stored datasets from the last round, which contain labeled set $\mathcal{D}_\mathrm{S}^{r-1}$ and unlabeled set $\mathcal{D}_\mathrm{U}^{r-1}$.
We also receive some new unlabeled samples in $\delta \mathcal{D}^r$. 

\fakeparagraph{Outputs}
We are expected to output the updated model according to newly collected data. 
We also need to update the datasets. 
Because of the limited memory and limited querying number, we update the datasets based on two selected coresets $\mathcal{C}_\mathrm{S}$ and $\mathcal{C}_\mathrm{U}$. 
Both coresets are selected from all available unlabeled samples, \ie $\mathcal{C}_\mathrm{S},\mathcal{C}_\mathrm{U} \subset \mathcal{D}_\mathrm{U}^{r-1}\cup\delta \mathcal{D}^r$. 

\fakeparagraph{Methods}
In order to select two coresets, we use a confidence score $\bm{\alpha}$ to weight each unlabeled samples, where $\bm{\alpha}\in\mathbb{R}_{+}^{|\mathcal{D}_\mathrm{U}^{r-1}|+|\delta \mathcal{D}^r|}$ 
A larger $\alpha$ means the sample can better match the learned likelihood, whereas a smaller alpha means the model has less confidence on that sample. 
Similar to \cite{bib:NIPS20:Sohn,bib:NIPS21:Killamsetty}, we also conduct a two-level minmax optimization. 
Particularly, in the inner loop, the binarized $\bm{\alpha}$ is used to weight the unsupervised loss, and the model will be then trained with semi-supervised loss.
In the outer loop, the confidence score $\bm{\alpha}$ is optimized on the current labeled dataset $\mathcal{D}_\mathrm{S}^{r-1}$ with the optimized model. 
Both loops are conducted alternatively in several iterations.
Then, the current unlabeled samples in $\mathcal{D}_\mathrm{U}^{r-1}\cup\delta \mathcal{D}^r$ are selected to build two coresets $\mathcal{C}_\mathrm{S}$ and $\mathcal{C}_\mathrm{U}$ according to the optimized score $\bm{\alpha}$.
Note that both coresets $\mathcal{C}_\mathrm{S}$ and $\mathcal{C}_\mathrm{U}$ have a constrained cardinality due to the limited querying number and the limited memory, respectively.
The samples in $\mathcal{C}_\mathrm{S}$ will be further queried for labeling.
The labeled dataset is then updated as $\mathcal{D}_\mathrm{S}^{r}=\mathcal{D}_\mathrm{S}^{r-1}\cup\mathcal{C}_\mathrm{S}$, and the unlabeled dataset is updated as $\mathcal{D}_\mathrm{U}^{r}=\mathcal{C}_\mathrm{U}$.

\clearpage
\noindent
\large{This concludes my thesis.}

\cleardoublepage
\def\bibname{Bibliography}
\bibliographystyle{alphaabbr}
\phantomsection\addcontentsline{toc}{chapter}{Bibliography}

{\small\bibliography{biblio}}
\cleardoublepage
\backmatter

\newcommand{\aut}[1]{\vspace*{0,5em}\\\noindent\newblock{#1}}
\newcommand{\tit}[1]{\newblock{\textbf{#1}}}
\newcommand{\con}[1]{\newblock{\emph{#1}}}
\newcommand{\jou}[1]{\newblock{\emph{#1}}}
\newcommand{\loc}[1]{\newblock{#1}}
\newcommand{\rem}[1]{\newblock{#1}}

\chapter{List of Publications}
\label{ch8:publication}

\small

The following list includes publications that form the basis of this thesis. The corresponding chapters are indicated in parentheses.
\\
\aut{Zhongnan Qu, Zimu Zhou, Yun Cheng, Lothar Thiele.}
\tit{Adaptive Loss-aware Quantization for Multi-bit Networks.}
\con{In Proceedings of IEEE/CVF Conference on Computer Vision and Pattern Recognition (CVPR),}
\loc{IEEE, 2020, Acceptance ratio: 22.1\%.}
\rem{(\chref{ch2:inference}) \cite{bib:CVPR20:Qu}}
\\
\aut{Zhongnan Qu, Syed Shakib Sarwar, Xin Dong, Yuecheng Li, Huseyin Sumbul, Barbara De Salvo.}
\tit{DRESS: Dynamic REal-time Sparse Subnets.}
\con{In Efficient Deep Learning for Computer Vision (ECV),}
\loc{CVPRWorkshop, 2022, Acceptance ratio: 29.9\%.}
\rem{(\chref{ch3:adaptation}) \cite{bib:CVPRWorkshop22:Qu}}
\\
\aut{Zhongnan Qu, Zimu Zhou, Yongxin Tong, Lothar Thiele.}
\tit{p-Meta: Towards On-device Deep Model Adaptation.}
\con{In Proceedings of ACM Conference on Knowledge Discovery and Data Mining (SIGKDD),}
\loc{ACM, 2022, Acceptance ratio: 15.0\%.}
\rem{(\chref{ch4:learning}) \cite{bib:KDD22:Qu}}
\\
\aut{Zhongnan Qu, Cong Liu, Lothar Thiele.}
\tit{Deep Partial Updating: Towards Communication Efficient Updating for On-device Inference.}
\con{In Proceedings of European Conference on Computer Vision (ECCV),}
\loc{Springer, 2022, Acceptance ratio: 28.4\%.}
\rem{(\chref{ch5:edgeserver}) \cite{bib:ECCV22:Qu}}
\\
\newpage
\noindent
The following list includes publications that were written during the PhD studies, yet are not part of this thesis.
\\
\aut{Fan Lu, Guang Chen, Yinlong Liu, Zhongnan Qu, Alois Knoll.}
\tit{RSKDD-Net: Random Sample-based Keypoint Detector and Descriptor.}
\con{In Proceedings of Annual Conference on Neural Information Processing Systems (NeurIPS),}
\loc{2020, Acceptance ratio: 20.1\%.}
\rem{\cite{bib:NIPS20:Lu}}
\\
\aut{Xin Dong, Barbara De Salvo, Meng Li, Chiao Liu, Zhongnan Qu, H.T. Kung, Ziyun Li.}
\tit{SplitNets: Designing Neural Architectures for Efficient Distributed Computing on Head-Mounted Systems.}
\con{In Proceedings of IEEE/CVF Conference on Computer Vision and Pattern Recognition (CVPR),}
\loc{IEEE, 2022, Acceptance ratio: 25.3\%.}
\rem{\cite{bib:CVPR22:Dong}}
\\

\cleardoublepage
\chapter{Curriculum Vit\ae}

\renewcommand{\arraystretch}{1.5}
\newcolumntype{P}[1]{>{\raggedright\arraybackslash}p{#1}}

\begin{footnotesize}

\subsection*{Personal Data}
\begin{tabular}{@{}P{2.5cm}P{9.5cm}}
Name & Zhongnan Qu \\
Date of Birth &  May 5, 1992 \\
Citizenship & China \\
\end{tabular}

\vspace{-0.5em}

\subsection*{Education}
\begin{tabular}{@{}P{2.5cm}P{9.5cm}}
2018--2022 &
ETH Zurich, \textit{Zurich Switzerland}\linebreak
Ph.D. in Computer Engineering\linebreak 
Advised by Prof. Lothar Thiele \\

2014--2018 &
TU Munich, \textit{Munich Germany}\linebreak
M.Sc. in Electrical and Computer Engineering\linebreak
Advised by Prof. Daniel Cremers and Prof. Dongheui Lee\\

2014--2017 &
TU Munich, \textit{Munich Germany}\linebreak
M.Sc. in Mechanical Engineering\linebreak
Advised by Prof. Alois Knoll\\

2013--2014 &
Munich University of Applied Sciences, \textit{Munich Germany}\linebreak
B.Eng. in Mechatronics Engineering\\

2010--2014 &
Tongji University, \textit{Shanghai China}\linebreak
B.Eng. in Mechatronics Engineering
\end{tabular}

\vspace{-0.5em}

\subsection*{Professional Experience}

\begin{tabular}{@{}P{2.5cm}P{9.5cm}}
2018--2022 &
ETH Zurich, \textit{Zurich Switzerland}\linebreak
Research and teaching assistant\\

2021 &
Meta (Facebook) Reality Labs, \textit{Seattle US (Remotely)}\linebreak
Research Intern\\

2017 &
BMW Group, \textit{Munich Germany}\linebreak
Intern\\

2014 &
Canon Group Company $\bullet$ Oc\'e Printing Systems GmbH, \textit{Munich Germany}\linebreak
Intern and Thesis Student\\

2013 &
State Grid, \textit{Henan China}\linebreak
Intern
\end{tabular}

\vspace{-0.5em}

\subsection*{Research Interests}
I focus on efficient deep learning in computer vision, natural language processing, and robotics. The vision is to deploy deep learning on edge devices for new emerging intelligent applications that face challenges of resource constraints, privacy issues, and data scarcity.

\end{footnotesize}

\end{document}